\documentclass{ucbthesis}

\setsecnumdepth{subsection}
\setsecnumdepth{subsubsection}
\usepackage{parskip}
\usepackage{tocloft}
\usepackage{pgf,tikz}
\usepackage{pgfplots}
\usepackage{amssymb} %various math symbols including \pitchfork
\usepackage{amsmath} %\DeclareMathOperator, \align
\usepackage{amsfonts}
\usepackage{tikz} %flow charts
\usetikzlibrary{shapes,arrows}
\usepackage[mathscr]{euscript}

\usepackage{authblk,setspace,subfig,caption}

\usepackage{graphicx} % more modern

 \usepackage[english]{babel}
\tikzstyle{decision} = [diamond, draw, text width=4.5em, text badly centered, node distance=4cm, inner sep=0pt]
\tikzstyle{block} = [rectangle, draw, thick,text width=6.2em, text centered, rounded corners, minimum height=4em]

\usepackage{fix-cm}
\usepackage{subfig}
\usepackage{pgfplots}
\pgfdeclarelayer{background layer}
\pgfsetlayers{background layer,main}
\definecolor{darkgray}{rgb}{0.25,0.25,0.25}
\definecolor{lightgray}{rgb}{0.75,0.75,0.75}

\usepackage{multirow}

\usepackage{float}
\usepackage{amsmath}
\usepackage{blindtext, rotating}
\usepackage{amsmath,amssymb,amsthm,amsfonts,mathtools,color,nicefrac,graphicx,multirow,bbm,rotating}
\usepackage{amsmath,amssymb}
\usepackage{amssymb, multirow, paralist}
\usepackage{url}
\usepackage{amssymb}
\usepackage{algorithm,algorithmic}

\usepackage{amsmath,amssymb,amsthm,amsfonts,mathtools,color,nicefrac,graphicx,multirow,bbm,rotating}

\usepackage{hyperref}

\usepackage{minitoc}

\usepackage{titlesec, blindtext, color}

\usepackage{color,graphicx}
\definecolor{ared}{rgb}{.647,.129,.149}

\usepackage{color,calc,graphicx,soul}
\definecolor{nicered}{rgb}{.147,.129,.149}
\makeatletter
\newlength\dlf@normtxtw
\newsavebox{\feline@chapter}
\newcommand\feline@chapter@marker[1][4cm]{%
\sbox\feline@chapter{%
\resizebox{!}{#1}{\fboxsep=1pt%
\colorbox{nicered}{\color{white}\bfseries\sffamily\thechapter}%
}}%
\rotatebox{90}{%
\resizebox{%
\heightof{\usebox{\feline@chapter}}+\depthof{\usebox{\feline@chapter}}}%
{!}{\scshape\so\@chapapp}}\quad%
\raisebox{\depthof{\usebox{\feline@chapter}}}{\usebox{\feline@chapter}}%
}
\newcommand\feline@chm[1][4cm]{%
\sbox\feline@chapter{\feline@chapter@marker[#1]}%
\makebox[0pt][l]{% aka \rlap
\makebox[1cm][r]{\usebox\feline@chapter}%
}}
\makechapterstyle{daleif1}{

\renewcommand\printchapternum{\null\hfill\feline@chm[2.5cm]\par}

}
\makeatother
\chapterstyle{daleif1}

\usepackage{tgadventor}
\usepackage{fourier}
\definecolor{cite_color}{RGB}{255,0,130}

\usepackage{xcolor}

\usepackage{verbatim}

\usepackage[english]{babel}
\newtheorem{definition}{Definition}[chapter]		% numbering starts from 1 within each chapter
\newtheorem{lemma}[definition]{Lemma}			% same numbering as definitions
\newtheorem{theorem}[definition]{Theorem}		% same numbering as definitions
\newtheorem{corollary}[definition]{Corollary}		% same numbering as definitions
\newtheorem{proposition}[definition]{Proposition}	% same numbering as definitions
\newtheorem{remark}[definition]{Remark}
\newtheorem{assumption}[definition]{Assumption}
\usepackage{mathtools}

\usepackage{graphicx}
\usepackage{ltxtable} % needed for vertical text centering hack -cotter

\usepackage[english]{babel}
\usepackage{blindtext}

\usepackage{paralist}
\usepackage{fancybox}
\newenvironment{myalg}[2] %
{
\begin{Sbox}
\begin{minipage}{0.8\textwidth}
\vspace*{0.1cm}
\begin{center}
{{#1}} (\texttt{#2})
\end{center}
\rm
\begin{tabbing}
....\=...\=...\=...\=...\=  \+ \kill
}%
{\end{tabbing}
\vspace*{-0.2cm}
\end{minipage}
\hspace*{0.1cm}
\end{Sbox}
\fbox{\TheSbox}
}

\makeatletter
\newcommand*{\compress}{\@minipagetrue}
\makeatother

\newenvironment{mylist}%
  {
    \setdefaultleftmargin{1em}{1em}{}{}{}{} 
    \compress %places itemize into minipage, removing whitespace before
    \begin{itemize}%
    \setlength{\itemsep}{0pt}%
    \setlength{\topsep}{0pt} 
    \setlength{\partopsep}{0pt}
    \setlength{\parsep}{0pt}
    \setlength{\parskip}{0pt}}%
  {\end{itemize}}

\def \R {\mathbb{R}}
\def \D {\mathcal{D}}
\def \y {\mathbf{y}}
\def \M {\mathcal{M}}
\def \E {\mathrm{E}}
\def \x {\mathbf{x}}
\def \a {\mathbf{a}}
\def \L {\mathcal{L}}
\def \H {\mathcal{H}}
\def \fh {\widehat{f}}
\def \f {\mathbf{f}}
\def \Hk {\H_{\kappa}}
\def \P {\mathcal{P}}
\def \v {\mathbf{v}}
\def \c {\mathbf{c}}
\def \X {\mathcal{X}}
\def \fhv {\widehat{\f}}
\def \S {\mathcal{S}}
\def \eh {\widehat{\epsilon}}
\def \z {\mathbf{z}}
\def \t {\mathbf{t}}
\def \ph {\widehat{\phi}}
\def \lh {\widehat{\lambda}}
\def \gt {\widetilde{g}}
\def \et {\widetilde{\varepsilon}}
\def \zt {\widetilde{\z}}
\def \Zt {\widetilde{Z}}
\def \gammat {\widetilde{\gamma}}
\def \etat {\widehat{\eta}}
\def \ab {\overline{\alpha}}
\def \Mh {\widehat{M}}
\def \veh {\widehat{\varepsilon}}
\def \gammah {\widehat{\gamma}}
\def \etah {\widehat{\eta}}
\def \xh {\widehat{\x}}
\def \Kh {\widehat{K}}
\def \rh {\widehat{r}}
\def \vet {\widetilde{\varepsilon}}
\def \span {\mbox{span}}
\def \hh {\widehat{h}}
\def \htt {\widetilde{h}}
\def \Er {\mathcal{E}}
\def \u {\mathbf{u}}
\def \Eh {\widehat{\Er}}
\def \v {\mathbf{v}}
\def \w {\mathbf{w}}
\def \Hb {\overline{\H}}
\def \DDh {\widehat{\D}}
\def \B {\mathcal{B}}
\def \Phl {\overline{\Phi}}
\def \R {\mathbb{R}}
\def \Kt {\widetilde{K}}
\def \G {\mathcal{G}}
\def \kt {\widetilde{\kappa}}
\def \b {\mathbf{b}}
\def \zt {\widetilde{\z}}
\def \h {\mathbf{h}}
\def \vt {\widetilde{\v}}
\def \Vt {\widetilde{V}}
\def \A {\mathcal{A}}
\def \B {\mathcal{B}}
\def \Hh {\widehat{H}}
\def \Ht {\widetilde{H}}
\def \alt {\widetilde{\alpha}}
\def \at {\widetilde{\a}}
\def \bt {\widetilde{\b}}
\def \hht {\widetilde{\h}}
\def \Dt {\widetilde{D}}
\def \Zt {\widetilde{Z}}
\def \lt {\widetilde{\lambda}}
\def \rb {\bar{r}}
\def \Dh {\widehat{D}}
\def \tr {\mbox{tr}}
\def \Ah {\widehat{A}}
\def \r {\mathcal{R}}
\def \Zh {\widehat{Z}}
\def \m {\mathbf{m}}
\def \Ub {\bar{U}}
\def \Vb {\bar{V}}
\def \Pb {\bar{\P}}
\def \e {\mathbf{e}}
\def \rank {\mbox{rank}}
\def \F {\mathcal{F}}
\def \Uh {\widehat{U}}
\def \uh {\widehat{\u}}
\def \Vh {\widehat{V}}
\def \vh {\widehat{\v}}
\def \sigmah {\widehat{\sigma}}
\def \mh {\widehat{\m}}
\def \m {\mathbf{m}}
\def \xt {\widetilde{\x}}
\def \Bh {\widehat{B}}
\def \Mt {\widetilde{M}}
\def \d {\mathbf{d}}
\def \vec {\mbox{vec}}
\def \k {\mathbf{k}}
\def \At {\widetilde{A}}
\def \Bt {\widetilde{B}}
\def \N {\mathcal{N}}
\def \wh {\widehat{\w}}
\def \ah {\widehat{\alpha}}
\def \wt {\widetilde{\w}}
\def \Gh {\widehat{G}}
\def \rb {\bar{r}}
\def \etah {\widehat{\eta}}
\def \s {\mathbf{s}}
\def \Xh {\widehat{X}}
\def \Ph {\widehat{P}}
\def \lb {\L}
\def \Lh {\widehat{\L}}

\def \K {\mathcal{K}}
\def \clip {\mbox{clip}}
\def \nt {\widetilde{\nabla}}
\def \Xt {\widetilde{X}}
\def \lh {\widehat{\ell}}
\def \eh {\widehat{\ell}}

\def \A {\mathcal{A}}
\newcommand{\dd}[2] { \langle {#1}, {#2} \rangle}

\newcommand{\mc}[1] {\mathcal{#1}}

\newcommand{\floor}[1]{\left\lfloor{#1}\right\rfloor}

\def \eo {\epsilon_{\text{opt}}}
\def \ep {\epsilon_{\text{prior}}}
\def \a {\mathbf{a}}
\def \b {\mathbf{b}}
\def \w {\mathbf{w}}
\def \x {\mathbf{x}}
\def \B {\mathbf{B}}
\def \R {\mathbb{R}}
\def \E {\mathbb{E}}
\def \B {\mathbf {B}}
\def \X {\mathbf {X}}
\def \Z {\mathbf {Z}}
\def \z {\mathbf{z}}

\def \gb {\hat{\mathbf{g}}}

\def \g {\hat{\mathbf{g}}}

\def \O {\mathcal{O}}

\newenvironment{myquotation}{\setlength{\leftmargini}{2em}\quotation}{\endquotation}

\def \bz {\mathbf{0}}

\def \R {\mathbb{R}}
\def \x {\mathbf{x}}
\def \L {\mathcal{G}}
\def \F {\mathcal{F}}
\def \lb {\L}
\def \Lh {\widehat{\L}}
\def \wh {\widehat{\w}}
\def \wt {\widetilde{\w}}
\def \K {\mathcal{K}}
\def \clip {\mbox{clip}}
\def \nt {\widetilde{\nabla}}
\def \Xt {\widetilde{X}}
\def \lh {\widehat{\ell}}
\def \eh {\widehat{\ell}}
\def \A {\mathcal{A}}
\def \w {\mathbf{w}}
\def \B {\mathbf{B}}
\def \E {\mathbb{E}}
\def \B {\mathbf {B}}
\def \z {\mathbf{z}}

\def \gb {\hat{\mathbf{g}}}

\def \g {\hat{\mathbf{g}}}
\def \lb {\bar{\ell}}
\def \O {\mathcal{O}}
\def \gh {\widehat{g}}
\def \I {\mathcal{I}}
\def \C {\mathcal{C}}
\def \P {\mathcal{P}}
\def \xb {\overline{\x}}
\def \zh {\widehat{\z}}

\def \wb {\bar{\w}}
\def \Fh {\widehat{\F}}
\def \g {\mathbf{g}}
\def \Ft {\widetilde{\F}}
\def \W {\mathcal{W}}
\def \G {\mathcal{G}}

\def \x {\mathbf{x}}
\def \e {\mathbf{e}}
\def \y {\mathbf{y}}
\def \z {\mathbf{z}}
\def \w {\mathbf{w}}

\def \f {\mathbf{f}}
\def \P {\mathcal{P}}
\def \Q {\mathcal{Q}}

\def \R {\mathbb{R}}
\def \u {\mathbf{u}}
\def \v {\mathbf{v}}

\def \B {\mathbb{B}}
\def \v {\mathbf{v}}
\def \E {\mathrm{E}}

\def \R {\mathbb{R}}
\def \regret {\mbox{Regret}}
\def \f {\mathbf{f}}
\def \VAR {\mbox{VAR}}

\def \eo {\epsilon_{\text{opt}}}
\def \ep {\epsilon_{\text{prior}}}

\def \x {\mathbf{x}}
\def \e {\mathbf{e}}
\def \y {\mathbf{y}}
\def \z {\mathbf{z}}
\def \w {\mathbf{w}}

\def \f {\mathbf{f}}
\def \P {\mathcal{P}}
\def \Q {\mathcal{Q}}

\def \R {\mathbb{R}}
\def \u {\mathbf{u}}
\def \v {\mathbf{v}}

\def \B {\mathbb{B}}
\def \v {\mathbf{v}}
\def \E {\mathrm{E}}

\def \R {\mathbb{R}}
\def \regret {\mbox{Regret}}
\def \f {\mathbf{f}}
\def \VAR {\mbox{VAR}}

\begin{document}

% Declarations for Front Matter

\title{Exploiting Smoothness in Statistical Learning, Sequential Prediction, and Stochastic Optimization}
\author{Mehrdad Mahdavi}
\degreesemester{June}
\degreeyear{2014}
\degree{Doctor of Philosophy}
\chair{Professor Rong Jin}
\othermembers{Associate Professor  Pang-Ning Tan\\ Associate Professor  Ambuj Tewari\\ Associate Professor  Eric Torng}
\numberofmembers{4}

\field{Computer Science}
\campus{}

\maketitle
\approvalpage
\newpage
% (This file is included by thesis.tex; you do not latex it by itself.)

\begin{abstract}
In the last several years, the intimate connection between convex optimization and learning problems, in both statistical and sequential  frameworks, has shifted the focus of algorithmic machine learning  to examine this interplay. In particular, on one hand,  this intertwinement brings forward new challenges in reassessment of  the performance of learning algorithms including generalization and regret bounds  under the assumptions imposed by convexity  such as analytical properties of loss functions (e.g., Lipschitzness, strong convexity, and smoothness). On the other hand, emergence of datasets of an unprecedented size, demands the development of novel and more efficient  optimization algorithms to tackle large-scale learning problems. 

The overarching goal of  this thesis is to reassess  the smoothness of loss functions in statistical learning, sequential prediction/online learning, and stochastic optimization and explicate its consequences. In particular we  examine  how leveraging  smoothness of loss function could be beneficial or detrimental in these settings in terms of sample complexity, statistical consistency, regret analysis,  and convergence rate. 

In the statistical learning framework, we investigate the sample complexity of learning problems when the loss function is smooth and strongly convex and the learner is provided with the target risk as a prior knowledge. We establish that under these assumptions, by exploiting the smoothness of loss function, we are able to improve the sample complexity of learning exponentially. Furthermore, the proof of our results is constructive and is rooted in a properly designed stochastic optimization algorithm which could be of significant practical importance. 

We also investigate the smoothness from the viewpoint of \textbf{statistical consistency} and show that in sharp contrast to  optimization and generalization where the smoothness is favorable because of its computational and theoretical virtues, the smoothness of surrogate loss function might deteriorate the binary excess risk. Motivated by this negative result, we provide a unified analysis of three types of errors including  optimization error, generalization bound, and the error in translating convex excess risk into a binary excess risk, and underline the conditions that smoothness might be preferred.

We then turn to elaborate the importance of smoothness in sequential prediction/online learning. We introduce a new measure to assess the performance of online learning algorithms which is referred to  as \textbf{gradual variation}. The gradual variation is measured by the sum of the distances between every two consecutive loss functions and is more suitable for gradually evolving environments such as stock prediction. Under smoothness assumption, we devise novel algorithms for online convex optimization with regret bounded by gradual variation. The proposed algorithms can take advantage of benign  sequences and at the same time protect against the adversarial sequences of loss functions. 

Finally, we investigate how  to exploit the smoothness of loss function  in convex optimization. Unlike the optimization methods based on full gradients, the smoothness assumption was not exploited by most of the existing stochastic optimization methods.  We propose a novel optimization paradigm that is referred to as \textbf{mixed optimization} which interpolates between stochastic and full gradient methods and is able to exploit the smoothness of loss functions  to obtain faster convergence rates in stochastic optimization, and condition number independent accesses of full gradients in deterministic optimization. The key underlying insight of mixed optimization is  to utilize infrequent full gradients of the objective function to progressively reduce the variance of the stochastic gradients.  These  results show an intricate interplay between stochastic and deterministic convex optimization to take advantages of their individual merits.

We also propose efficient  \textbf{projection-free} optimization algorithms to tackle the computational challenge arising from the projection steps   which are required at each iteration of most existing gradient based optimization methods to ensure the feasibility of intermediate solutions. In stochastic optimization setting, by introducing and leveraging smoothness, we develop novel methods which  only require one projection at the final iteration. In online learning setting, we consider online convex optimization with soft constraints where the constraints are allowed to be satisfied on long term. We show that by compromising on the learner's regret, one can devise efficient online learning algorithms with sub-linear   bound on both the regret and the violation of the constraints

\end{abstract}

\begin{frontmatter}

\begin{acknowledgements}
First and foremost, I feel indebted to my advisor, Professor Rong Jin, for his guidance, encouragement, and  inspiring supervision throughout the course of this research work. His  patience, extensive knowledge, and creative thinking have been the source of inspiration for me. He was available for advice or academic help whenever I needed  and gently guided me for deeper understanding, no matter how late or inconvenient the time is. When I was struggling to quit  my Ph.D. at Sharif University to join Rong's group, I was not sure about my decision, but after  four and a half years,   I am happy to say that I did not make  a wrong decision. It's hard to express how thankful I am for his unwavering support over the last  years.

I would like  to take on this opportunity to thank my thesis committee members Pan-Ning Tan, Ambuj Tewari, and Eric Torng who have accommodated my timing constraints despite their full schedules, and provided me with precious feedback for the presentation of the results, in both written and oral form.

During my Ph.D. studies, I had the pleasure of collaborating with many researchers from each and every one of which I had things to learn, and the quality  of my  research was considerably enhanced by these interactions. I would like to thank  Tianbao Yang and Lijun Zhang for all the discussions we had and the fun moments we spent on doing research and attending conferences.  The results of Chapter~5 and Chapter~8  represent part of the fruits of these collaborations. I also spent a summer as intern at Microsoft Research working with Ofer Dekel and a summer at NEC Research Labs working with Shenghuo Zhu. I learned a lot from them and would like to express my gratitude  for having me as an intern. I also would like to thank Elad Hazan, Satyen Kale, Phil Long, Shai Shalev-Shwartz, and Ohad Shamir for some helpful email correspondence. 

Living in East Lansing without my good friends would not have been easy. I want
to thank all my friends in the department and outside the department. I wish I could name you all.

Last but definitely not least, I want to express my deepest gratitude to my beloved
parents and my dearest siblings. Their  love and unwavering support have been crucial to my success, and a constant source of comfort and counsel.  Special thanks to my parents  for abiding by my absence in  last five years.

\end{acknowledgements}
\copyrightpage

\pagebreak

\begin{turn}{0}
\begin{minipage}{\linewidth}
\vspace*{6cm}
\begin{center}
\begin{minipage}{15cm}
\begin{quote}
\emph{``... theory is the first term in the Taylor series of practice"}
\\
\hspace*{\fill} --- Thomas M. Cover
\end{quote}
\end{minipage}
\end{center}
\end{minipage}
\end{turn}

\begin{dedication}
\null\vfil
\begin{center}
\hfill To my parents, Asieh and Rashid.\\
%\hfill my love, Rana.\hspace{2.72cm}$~$
\end{center}
\vfil\null
\end{dedication}

% You can delete the \clearpage lines if you don't want these to start on
% separate pages.
\setcounter{tocdepth}{2}
\tableofcontents
\clearpage
\listoffigures 
\clearpage
\listoftables 
\clearpage
\listofalgorithms
\clearpage
\clearpage

\end{frontmatter}

\pagestyle{headings}

\part{Background}
\chapter{Introduction} \label{chap:introduction}
\def \E {\mathbb{E}}
\def \D {\mathcal{D}}

In machine learning the goal is to learn from labeled examples in order to predict the labels of  unseen examples.  That is, given a training set, we aim to learn a  hypothesis, or a classifier that assigns labels to  samples that have never been  observed by the algorithm. Efficiently  finding a  hypothesis based on the training set which  minimizes some measure of performance is the main focus of machine learning. 

In order to study the learning problem in a mathematical framework,   it is necessary to define the framework in which the algorithm is to function.   Basically there are  two frameworks that have gained significant popularity within the last two decades:  the \textit{statistical learning} framework and the \textit{sequential prediction} or online learning framework. In  both settings \textit{mathematicical optimization} theory  plays an important role by  providing a unified framework to investigate the   \textit{computational} issues of learning algorithms.  Additionally,  tools from  convex optimization underline   the analysis in algorithmic machine learning that  targets most  of the practical learning algorithms  in both frameworks. 

This chapter is devoted to an overview of these three broad topics of statistical learning, sequential prediction/online learning, and convex optimization, aiming to develop a general correspondence between the first  two  and convex optimization.  In particular by characterizing sample complexity, statistical consistency, regret analysis,  and convergence
rate in terms of the properties of loss functions such as Lipschitzness, strong convexity, and smoothness, we elaborate the importance of smoothness~\footnote{The precise definition will be given later. We say that a continuously differentiable function  $f:\R^d \mapsto\R$ is $\beta$-smooth if its gradient is Lipschitz with constant $\beta$, i.e., $\|\nabla f(\w)-\nabla f(\w')\| \leq \beta \|\w - \w'\|$.} and explicate its consequences. Here we move towards the definitions in a fairly non-technical manner and the formal definitions will be given in Chapter~\ref{chap-background}.

%---------------------------------------------------------------------------------------------------------------------------
% Statistical Learning
%---------------------------------------------------------------------------------------------------------------------------

\section*{Classical Statistical Learning}
We begin by stating the  basic problem of binary classification in the standard passive supervised  learning setting (also called batch learning).  In binary classification, the learning algorithm is given a set of labeled examples  $\mathcal{S} = \left( (\x_1, y_1), \cdots, (\x_n, y_n) \right)$ drawn independent and identically distributed (i.i.d.) from a fixed but unknown  distribution $\D$ over the  space $\Xi = \mathcal{X} \times \mathcal{Y}$, where $\mathcal{X}$ is the instance space and $\mathcal{Y}$ is the label (target) space.  The goal, with the help of provided labeled examples, is to output a hypothesis or classifier $h$ from a predefined hypothesis class $\mathcal{H} = \{h: \mathcal{X} \mapsto \mathcal{Y}\}$  that does well on unseen examples coming from the same distribution.  In other words, we would like to find a hypothesis to generalize well from the training set to the entire domain of examples.

 To measure the performance of a classifier $h \in \H$ on unseen samples, we utilize a loss function $\ell: \mathcal{H} \times \Xi \mapsto \R_{+}$. The mostly used loss function in  binary classification problem with instance space defined as  $\Xi = \mathbb{R}^d \times \{-1,+1\}$ is the  0-1 loss  $\ell(h, (\x,y)) = \mathbb{I}{[h(\x) \neq y]}$, where $\mathbb{I}{[\cdot]}$ is the indicator function. The risk of a particular classifier $h\in\H$ is the probability that the classifier does not predict the correct label on a random data point generated
by the  underlying distribution $\D$, i.e., $\mc{L}_{\D}(h) = \mathbb{P}_{(\x,y) \sim \D} [h(\x) \neq y] = \E_{(\x,y)\sim \D}[\ell(h, (\x,y))]$. This performance measure is also called \textit{generalization error} or \textit{true risk} in statistical learning community.   An equivalent way to express the generalization bound is the \textit{sample complexity } analysis. Roughly speaking, the sample complexity of an algorithm is  the number of examples which is sufficient to ensure that, with probability at least $1-\delta$ (w.r.t. the random choice of $\mc{S}$), the algorithm picks a hypothesis  with an error that is at most $\epsilon$ from the optimal one.   The difference between the risk of a particular classifier $h$ and of the optimal classifier $h_* = \arg \min_{h \in \H} \mc{L}_{\D}(h)$ is called the excess risk of $h$, i.e., $\mathscr{E}(h) = \mc{L}_{\D}(h) - \mc{L}_{\D}(h_*)$.

Since the underlying distribution $\D$ is unknown to the learner, it is impossible for learner to directly minimize the generalization error or true risk.  Therefore, one has to resort to using the training data in $\mathcal{S}$ to estimate the probabilities of error for the classifiers in $\mathcal{H}$. This alternative approach  is known as the Empirical Risk Minimization (ERM) method and aims  to pick a hypothesis which has small  error on the training set, i.e., empirical risk.  The performance of the empirical risk minimization has been throughly investigated and well understood using tools from empirical process theory. It is a well stablished fact that  a problem is learnable with ERM method  if and if only if the empirical error for all hypothesis in $\mc{H}$ converges uniformly to the true risk. Furthermore,  the uniform convergence holds if the complexity of  hypothesis class $\mc{H}$ satisfies some combinatorial characteristics. This is  one of the main achievements of statical learning theory to characterize, and establish necessary and sufficient conditions for the learnability of learning problems using the ERM  rule.

 While the ERM method is theoretically appealing,  from a practical point of view one would like to consider problems that are efficiently learnable which is referred to as \textit{computational complexity} of learning algorithm.  This issue becomes more important by noting the fact that  in many cases ERM approach suffers from substantial problems such as computational requirements in minimizing 0-1 loss over training set.  Indeed, solving the ERM problem for 0-1 loss function is known to be an NP-hard problem. Consequently, it is natural to consider loss functions that act as surrogates for the non-convex  0-1 loss, and lead to practical algorithms.  Of course, such a surrogate loss must be reasonably related to the original binary loss function since otherwise this approach fails. For classification problem, good surrogate loss functions have been recently
identified, and the relationship between the excess classification risk and the excess
risk of these surrogate loss functions has been exactly described.

An important family of learning problems that can be learnt  efficiently  are called  \texttt{Convex Learning Problems}. In general, a convex learning problem is a setting in which the surrogate loss   function and  the hypothesis space $\mc{H}$ are both convex. This setting encompasses an enormous variety of well-know practical learning algorithms such as regression, support vector machines (SVMs), boosting, and logistic regression, where these algorithms differ in the type of the convex loss function being used as the surrogate of the 0-1 loss. 

Interestingly, for convex learning problems the ERM rule, of minimizing the empirical convex loss over a convex domain $\mc{H}$, becomes a convex optimization problem, making an intimate connection between machine learning and mathematical  optimization. Learnability in this setting  departs from the learnability via ERM method and  strongly depends on the characteristics of convex domain such as boundedness  and the analytical properties (curvature) of loss function such as Lipschitzness, smoothness (i.e, differentiable with Lipschitz gradients), and strong convexity (i.e., at any point  one can find a convex quadratic lower bound for the function). Beyond learnability, the sample complexity of learning algorithms can also be characterized in terms of  the analytical properties of loss function. Therefore, smoothness and strong convexity of convex surrogate loss function play a crucial role in characterizing learnability and analysis of sample complexity of convex learning problems.

\subsection*{Smoothness and sample complexity}
While the main focus of statistical learning theory was on understanding learnability and sample complexity by investigating the complexity of hypothesis class in terms of known combinatorial measures under uniform convergence property, recent advances in online learning and optimization theory opened a new trend in understanding  the generalization ability of learning  algorithms in terms of the characteristics of loss functions being used in convex learning problems. In particular, a staggering number of results have focused on strong convexity of loss function and obtained better generalization bounds which are referred to  as \textit{fast} rates.  In terms of smoothness of loss function,  recently it has shown that under this assumption, it is possible to obtain \textit{optimistic} rates (in the sense that smooth losses yield better generalization bounds when the problem is easier) which are more appealing than the case  the convex surrogate loss is Lipschitz continuous.  This motivates us to take a step forward in this direction and investigate the smoothness of loss functions in more depth. 

\subsection*{Smoothness and binary excess risk}

As noted above, the convex surrogates of the  0-1 loss are highly preferred because of the computational  and theoretical virtues that convexity brings in.  Since the choice of convex surrogates could significantly affect the binary excess risk, the relation between risk bounds in terms of binary 0-1 loss and it is corresponding convex surrogate has been the focus of learning community over the last decade. It was delivered that  the binary excessive risk can be upper bound by the convex excessive risk through a transform function that only depends on the surrogate convex loss. 

Although a great deal of work has been devoted to understanding the relation between binary excess risk and convex excess risk, there remain a variety of open problems. In particular,  this transformation is well understood under mild conditions such as convexity and it is unclear how other properties of convex surrogates such as  smoothness  may affect this relation. This becomes more  critical  if we consider smooth   surrogates  as  witnessed by the fact that the smoothness  is  further beneficial both computationally- by attaining an optimal  convergence rate for optimization error, and in a statistical sense- by providing an improved  optimistic rate for generalization bound.  Given these positive news of using smooth convex surrogates, an open research question is how the smoothness of a convex surrogate will affect the binary excess risk. So we are thrived  to investigate the impact of the smoothness of a convex loss function on transforming the excess risk in terms of the convex surrogate loss into  the binary excess risk.

%---------------------------------------------------------------------------------------------------------------------------
% Online Learning
%---------------------------------------------------------------------------------------------------------------------------
\section*{Sequential Prediction/Game Theoretic Learning}

An alternative paradigm to analyze the learning problems is the sequential or online learning that can be phrased as a repeated two-player game between the learner and the adversary, making an intimate connection between learning and adversarial game theory. 

In sequential prediction/online learning framework, the learner is faced with a sequence of samples appearing at discrete
time intervals and is required to make predictions sequentially. In contrast to statistical setting, in which, the data source is typically assumed to be  i.i.d. with an unknown distribution,  in the online framework we relax or eliminate any stochastic assumptions imposed on the samples and they might be chosen adversarially. As a result,  online learning framework  is better suited for adversarial and interactive learning tasks such as spam email detection and stock market prediction where decisions of the learner could negatively affect future instances the learner receives.

By dropping the  statistical assumptions on the  observed sequence, it is not immediately clear how the prediction problem can be made meaningful and which goals are reasonable. One popular possibility is to measure the performance of the learner by the loss he/she has accumulated during the learning process and compare it to the loss of best fixed solution. The cumulative loss suffered on a sequence of rounds is the sum of instantaneous losses suffered on each one of the rounds in the sequence.   In particular, the  goal becomes  to minimize the  gap between the cumulative loss of the online learner  and the loss of a strategy that selects the best action fixed in hindsight. This performance gap is called the \textit{regret}.  The analysis of regret mainly focuses on investigating how the regret depends on the length of the time horizon the game proceeds.  We note that the best fixed action is chosen form a comparator class of predictors against which the learner will be compared and   can only be computed in full knowledge of the sequence of loss functions.  

 Regret analysis stands in stark contrast to the statistical framework in which the learner  is evaluated based on his/her accuracy after seeing all training examples, making the online learning setting inherently harder.  The theoretical utility of online learning has long been appreciated. More recently, it has become the mainstay of optimization, where it serves as computational platform from which a variety of large-scale learning problems can be solved. The analogous of statistical learnability in online setting is referred to as Hannan consistency.  A hypothesis class is learnable in online setting, i.e., Hannan consistent,  if for any sequence of samples, there exists an algorithm which attains sub-linear regret in terms of number of rounds the interaction proceeds. Interestingly,  unlike statistical learning theory, the analysis of online learning is mostly algorithmic where efficient algorithms are proposed to solve the learning problem and its performance is analyzed to guarantee the Hannan consistency.  

%By dropping the  statistical assumptions on the  sequence, explicitly minimizing the total cumulative loss in an unreachable goal and thus we have to settle for a less ambitious goal. 

Recently tools from convex optimization made it possible to capture many online learning problems under a  generic problem template, and in many circumstances  obtain  improved regret bounds. This unified framework, which is referred to as \texttt{Online Convex Optimization}, assumes that the learner is forced to make decisions from a convex set  and the adversary is supposed to play convex functions.     Additionally, it has been demonstrated that  the curvature of convex loss functions  played by the adversary  such as strong convexity   gives a  great advantage to the player to attain better regret bounds.  Surprisingly, tools from online optimization also provided insights to get better convergence rates or more efficient algorithms for some stochastic and deterministic convex optimization problems.

\subsection*{Smoothness and regret bounds}
Unlike strong convexity, the smoothness of loss functions is not a desirable property in online setting as it yields the same regret bounds as the loss functions being Lipschitz continuous.  However, there are scenarios that the  smoothness of the sequence of loss functions played by the adversary becomes important. One such scenario is the online learning from  loss functions that might have  some patterns and not being  fully  adversarial. For example, the weather condition or the stock price at one moment may have some correlation with the next and their difference is usually small, while abrupt changes only occur sporadically. Therefore devising online convex optimization algorithms which can take into account the gradual behavior of the environment  and at the same time protect against the worst case  sequences would be more desireable.  In terms of regret analysis, this translates to having algorithms with regret bounded  in terms of \textit{variation} of loss functions  instead of time horizon that is  main measure in the  standard  setting of sequential prediction.   In these evolving settings, the  smoothness of loss function becomes critical. More importantly, no gradual variation bound is achievable if the loss functions are no longer smooth. This necessitates the need to develop online methods  that exploit the smoothness  assumption in the learning process or in the analysis to obtain improved regret bounds in terms of  variation of the sequence of loss functions and underlines our motivation in this thesis.

%---------------------------------------------------------------------------------------------------------------------------
% Convex Optimization
%---------------------------------------------------------------------------------------------------------------------------
\section*{Convex Optimization and Learning}

In the problem of convex optimization, we are interested in minimizing a given convex function $f: \mathbb{R}^d \mapsto \mathbb{R}$ form a predefined family of convex functions $\mc{F}$ over a convex set $\W \subseteq \mathbb{R}^d$.   The goal is to find an \textit{approximate} solution with an accuracy $\epsilon$, i.e., finding a $\widehat{\w} \in \W$ where $f(\widehat{\w}) -\min_{\w \in \W} f(\w) \leq \epsilon$. A typical optimization algorithm initially chooses a point from the feasible convex set $\w_0 \in \W$ and iteratively updates  these points based on some information about the function at hand  until it achieves the desired accuracy.

To capture the efficiency of an optimization procedure,  we follow the black-box model of optimization~\footnote{As indicated by Yurii Nesterov in his seminal book~\cite{nesterov2004introductory}, in general, optimization problems are unsolvable and we need to relax the goal to make it  reachable.}. In this model we assume that there exists an oracle  which provides information about the query points such as function value, gradient, and second gradient (i.e, Hessian). The number  of queries issued to an oracle  to find a solution with a predefined level of accuracy is called \textit{oracle complexity} when it is stated in terms of desired accuracy $\epsilon$ or equivalently \textit{convergence rate} when it is stated in terms of the number of queries.  

As  already mentioned, learning problems under both statistical and online learning frameworks can be directly formulated  as optimization problems. In statistical setting and especially convex learning problems,   the learning algorithm   corresponds to the optimization algorithm that solves the minimization problem of picking a hypothesis from the set of hypotheses that minimizes empirical loss  over training sample. Similarly, in the online convex optimization, the online learner iteratively chooses decisions from a closed, bounded and non-empty convex set and encounters convex cost functions.

Formulating and investigating both  statistical and online learning problems in the context of  convex optimization   makes an intimate connection between  learning and mathematical  optimization.   Therefore,  the study of fast iterative methods for approximately solving  convex programming problems is a central focus of research in convex optimization, with important applications in machine learning,  and many other areas of computer science. The usefulness of convex optimization  in the development of various
learning algorithms  is well established in the past several years. Additionally,  challenges exist in machine learning applications demand the development of new optimization algorithms.

In optimization for supervised machine learning and in particular the empirical risk minimization paradigm with convex surrogates and gradient information, there exist two regimes in which popular algorithms tend to operate:  \textit{deterministic} (also known as  batch optimization or full gradient method) regime in which whole training data are used to compute the gradient at each iteration and \textit{stochastic} regime which  samples a fixed number of  training samples  per iteration, typically a single training sample, to compute the gradient at each iteration.  Although stochastic optimization methods suffer from the low convergence rate in comparison to batch methods,  the lightweight computation per iteration makes them attractive for many large-scale learning problems. Hence,  with the increasing amount of data that is available for training,  stochastic convex optimization   has emerged as the most scalable approach for large-scale machine learning which is known to yield moderately accurate solutions in a relatively short time.  

We emphasize that the role of convex optimization goes beyond computational issues  and it also provides tools to characterize the learnability in convex learning problems via efficient stochastic optimization algorithms for learning these problems. 

Analogous to both statistical and online learning frameworks, the curvature of function to be optimized, significantly affects  the convergence rate of optimization methods. Perhaps the most extensively studied are strong convexity and smoothness of function.   

\subsection*{Smoothness and convergence rate}
Exploiting smoothness of loss function, in particular in stochastic optimization,  to obtain better convergence rate has been one the main research challenges  in recent years. Despite enormous advances in exploiting smoothness in deterministic optimization, it has not been utilized in stochastic optimization. In particular, stochastic optimization of smooth loss functions exhibits the same convergence rate as stochastic optimization under Lipschitzness assumption of function.   Therefore, this thesis is motivated by the need of developing stochastic optimization algorithms with better convergence rates under smoothness assumption. The key question is  whether or not smoothness property of loss functions could be leveraged to develop much faster stochastic optimization methods.

\subsection*{Smoothness and projection-free optimization}
At the core of many iterative constrained optimization algorithms in both online and stochastic convex optimization  is a projection step to ensure the feasibility of solutions for intermediate iterations. This is a serious deficiency, since in many applications the projection  onto the constrained domain might be computationally expensive and sometimes as hard as solving the original optimization problem. It is therefore of considerable interest to devise optimization methods which do not require projection steps or need a bounded number of   projection operations. At it will became clear later in this thesis, by smoothing a strongly convex objective function, we are able to reduce the number of projections into a single projection at the end of the optimization process. In contrast to the other parts of the thesis where we assume and  exploit the smoothness, this is the only result  that  injects  and leverages  the smoothness  to gain from the merits of smoothness to be able to devise more efficient algorithms.

%--------------------------------------------------------------------------------------------
% Contributions
%--------------------------------------------------------------------------------------------
\section*{Main Contributions}

In this section we shall elaborate on the main problems considered in this thesis and our
key contributions to address these problems. A common theme in all of the algorithms is that they exploit the smoothness of loss function for more efficient methods.

\subsection*{Part II: Statistical Learning}
\vspace{0.25cm}
\begin{itemize}
\item{\textbf{Logarithmic  sample complexity for learning from smooth and strongly convex losses with target risk.}}~The first problem we consider in this thesis has a statistical nature. In particular, we consider learning in passive setting but with a slight modification. We assume that the target expected loss, also referred to as \textit{target risk}, is provided in advance for  learner as  prior knowledge. Unlike most studies in the learning theory  that only incorporate the prior knowledge into the generalization bounds, we are able to explicitly utilize the target risk in the learning process. 

Our analysis reveals a surprising result on the sample complexity of learning: by exploiting the target risk in the learning algorithm,  we show that when the loss function is both smooth and strongly convex, the sample complexity reduces to $O\left(\log \left(\frac{1}{\epsilon}\right)\right)$, an exponential improvement compared to the sample complexity $O\left(\frac{1}{\epsilon}\right)$ for learning with strongly convex loss functions. 

 Unlike the previous works on sample complexity, the  proof of our result is constructive and is based on a computationally efficient stochastic optimization algorithm which makes it practically interesting. The proposed ClippedSGD algorithm uses knowledge of the target risk  to appropriately clip gradients obtained from a stochastic oracle. The clipping  is beneficial because it reduces the variance in stochastic gradients and makes it possible to reduce the sample complexity. This happens under the assumption that the loss function is smooth and strongly convex.

\item{\textbf{Statistical consistency of smoothed hinge loss.}}~The second problem we address in   statistical learning setting is to investigate the relation between the excess risk that can be achieved by minimizing the empirical binary risk  and  the excess risk of smooth convex surrogates. 

As mentioned earlier,  convex surrogates of the  0-1 loss are highly preferred because of the computational  and theoretical virtues that convexity brings in.  This is  of more importance if we consider smooth   surrogates  as  witnessed by the fact that the smoothness  is  further beneficial both computationally- by attaining an \textit{optimal}  convergence rate for optimization, and in a statistical sense- by providing an improved \textit{optimistic} rate for generalization bound. However, we investigate the smoothness property from the viewpoint of statistical consistency and show  how it affects the binary excess risk for smoothed hinge loss. In particular, we intend to answer the following fundamental questions:
 \begin{quote}
 "\textit{How does the smoothness of surrogate convex loss affect the binary excess risk? Considering the advantages of smooth losses in terms of optimization and generalization, is it beneficial or detrimental  in terms of statistical consistency? Under what conditions on these three types of errors it is better to use smooth losses?}"
 \end{quote}
We show that in contrast to optimization  and generalization errors  that favor the choice of smooth surrogate loss, the smoothness of loss function may deteriorate the binary excess risk. Motivated by this negative result, we provide a unified analysis that integrates optimization error, generalization bound, and the error in translating convex excess risk into a binary excess risk when examining the impact of smoothness on the binary excess risk.  We  show that under favorable conditions appropriate choice of smooth convex surrogate  loss will result in a binary excess risk that is better than $O(1/\sqrt{n})$ which is unimprovable for general  non-smooth Lipschitz losses.

\end{itemize}

\subsection*{Part III: Sequential Prediction/Online Learning}
\vspace{0.25cm}
\begin{itemize}

\item{\textbf{Regret bounded by gradual variation for smooth online convex optimization.}}~As our third problem, we study the online convex optimization problem under the assumption that even the loss functions are arbitrary, but there is a hidden pattern that can be exploited in learning process.  Therefore, an interesting question that inspires our work in the analysis of online learning algorithms is the following:
 \begin{quote}
"\textit{Can we have online algorithms that can take advantage of benign  sequences and at the same time protect against the adversarial sequences?}"
 \end{quote}
To answer this question, we introduce the \textit{gradual variation}, measured by the sum of the distances between every two consecutive loss functions, to asses the performance of online learning algorithms in gradually evolving environments such as stock prediction.  We propose two novel algorithms, an Improved Follow the Regularized Leader (IFTRL) algorithm and an Online Mirror Prox (OMP) method, that achieve a regret bound which only scales as the square root of the gradual variation for the linear and general smooth convex loss functions.

To establish the main results, we discuss a lower bound for online gradient descent, and  a necessary condition on the smoothness of the cost functions for obtaining a gradual variation bound.   For the closely related problem of prediction with expert advice, we show that an online algorithm modified from the multiplicative update algorithm can also achieve a similar regret bound for a different measure of deviation. Finally, for loss functions which are strongly convex in applications such as portfolio management problem,  we show a regret which is only logarithmic in terms of the gradual variation.

The gradual variation- in addition to its intrinsic interest as an extension of regret analysis- has several specific consequences.  First, since gradual variation lower bounds the regret bound, devising algorithm with small gradual variation is also guarantees  to achieve small regret.  Second,  algorithms with small  gradual variation are specifically designed to attain small variation, therefore they can capture the correlation between loss functions if it exist  and boost the performance.

\item{\textbf{Gradual variation for composite online convex optimization.}}~As an impossibility result for obtaining gradual variation for  convex losses, we show that  for non-smooth functions when the only information presented to the learner is the first order information about the cost functions, it is impossible to obtain a regret bounded by gradual variation.  However,  we show that  a gradual  variation  bound is achievable  for a special class of non-smooth functions that is composed of a  smooth component and a  non-smooth component. 

We consider two categories for the non-smooth component.  In the first category, we assume that the non-smooth component is a fixed function  and is relatively easy such that the composite gradient mapping can be solved without too much computational overhead compared to gradient mapping. In  the second category, we assume that the non-smooth component can be written as an explicit maximization structure. In general, we consider a time-varying non-smooth component,  present a primal-dual  prox method, and prove a min-max regret bound by gradual variation. When the non-smooth components are equal across all trials, the usual regret is bounded by the min-max bound plus a variation in the non-smooth component.

\end{itemize}

\subsection*{Part IV: Stochastic Optimization}
\vspace{0.25cm}
\begin{itemize}
\item{\textbf{Improved convergence rate for stochastic optimization of smooth losses.}}~We then turn to exploiting smoothness in stochastic optimization. Recently stochastic optimization methods have experienced a renaissance in the design of fast algorithms for large-scale learning  problems.  Unlike the optimization methods based on full gradients, the smoothness assumption was not exploited by most of stochastic optimization methods.  More importantly, for general Lipschitz continuous convex functions, simple stochastic optimization methods such as stochastic gradient descent  exhibit the same convergence rate as that for the smooth functions, implying that smoothness of the loss function is essentially not very useful and can not be exploited in stochastic optimization.   Therefore, by noting this  significant gap between the convergence rate of optimization of smooth functions in stochastic and deterministic optimization,  the natural question  arises is:
 \begin{quote}
 "\textit{Can  smoothness property of function be exploited  to speed up the convergence rate of stochastic optimization of smooth functions?}". 
 \end{quote}
We will provide an affirmative answer to this question. In particular,  we propose a novel optimization paradigm which interpolates between stochastic and full gradient methods and is able to exploit the smoothness of loss functions in optimization process to obtain faster rates. The results show an intricate interplay between stochastic and deterministic convex optimization. The {{MixedGrad}} algorithm we propose fits in the mixed optimization paradigm  and is  an alternation of deterministic  and stochastic gradient steps, with different of frequencies for each type of steps.   We show that  it attains an $O(1/T)$ convergence rate for smooth losses. 

\item{\textbf{Condition number independent accesses of full gradient oracle for smooth and strongly convex optimization.}} The optimal iteration complexity of the gradient based algorithm for smooth and strongly convex  objectives  is $O(\sqrt{\kappa}\log 1/\epsilon)$, where $\kappa$ is the conditional number (the ratio of strong convexity to smoothness parameters). Despite its linear convergence rate in terms of target accuracy $\epsilon$, in the case that the optimization problem is ill-conditioned, we need to evaluate a larger number of full gradients, which could be computationally expensive. Therefore, a natural question is:
 \begin{quote}
 "\textit{Can we manage the dependency on the condition number and devise optimization methods independent of condition number in accessing the full gradient oracle?}" 
 \end{quote}
We show that in the mixed optimization regime  introduced in this thesis, we may also leverage the smoothness assumption of loss functions to devise algorithms with iteration complexities  that are independent of condition number in accessing the  full gradient oracle.   We utilize the idea of mixed optimization in progressively reducing the variance of stochastic gradients to optimize smooth and strongly convex functions, and propose the Epoch Mixed Gradient Descent (EMGD) algorithm that is independent of condition number in accessing the full gradients.  Similar to the  MixedGrad  algorithm, a distinctive step in EMGD is the mixed gradient descent, where we use a combination of the full gradient and the stochastic gradient to update the intermediate solutions. By performing a fixed number of mixed gradient descents, we are able to improve the sub-optimality of the solution by a constant factor, and thus achieve a linear convergence rate. Theoretical analysis shows that EMGD is able to find an $\epsilon$-optimal solution by computing $O(\log 1/\epsilon)$ full gradients and $O(\kappa^2\log 1/\epsilon)$ stochastic gradients. 

We also provide experimental evidence complementing our theoretical results for classification problem on few medium-sized data sets.

\item{\textbf{Efficient projection-free online and stochastic convex optimization.}}~Another problem we address in this thesis is efficient projection-free optimization methods for stochastic and online convex optimization.  Our motivation stems from the observation that most  of the gradient-based optimization algorithms  require a projection onto the convex set $\mathcal{W}$ from which the decisions are made. While  the projection is straightforward for simple shapes (e.g., Euclidean ball), for arbitrary complex sets this is  the main computational bottleneck and may be inefficient in practice. For instance, for many applications in machine learning such as metric learning, the convex domain is the positive semidefinite cone for which is the projection step requires a full eigendecomposition. For many other problems, the projection step is itself an offline optimization problem and might be as hard as solving the original optimization problem. This observation immediately leads to the following question that inspires our work:
 \begin{quote}
 "\textit{To what extent it is possible to reduce the number of expensive projection steps in online and stochastic optimization? Can we  trade expensive projection steps off for other types of light computational operations?}". 
 \end{quote}
 We consider this problem in two settings: stochastic optimization and online convex optimization. In stochastic setting,  we develop  novel stochastic optimization algorithms that do not need intermediate projections. Instead, only one projection at the last iteration is needed to obtain a feasible solution in the given domain. Our theoretical analysis shows that with a high probability, the proposed algorithms achieve an $O(1/\sqrt{T})$ convergence rate for general convex optimization, and an $O(\ln T/T)$  rate for  strongly convex optimization under mild conditions about the domain and the objective function. The key insight which underlines the proposed projection-free algorithm for strongly convex functions is smoothing the objective function. This is in contrast to other problems in this thesis where we try to leverage the smoothness of objective, while  here we introduce smoothness to gain from its computational virtues in alleviating the projection steps.

In online setting, we consider an alternative online convex optimization problem. Instead of requiring that decisions belong to a constrained convex domain  for all rounds, we only require that the constraints, which define the convex set, be satisfied in the long run. By turning the problem into an online convex-concave optimization problem, we propose an efficient algorithm which achieves an $O(\sqrt{T})$  regret bound and an $O(T^{3/4})$ bound for the violation of constraints. Then we modify the algorithm in order to guarantee that the constraints are satisfied in the long run. This gain is achieved at the price of getting $O(T^{3/4})$  regret bound. We also prove an impossibility result which shows that simple ideas such as augmenting the objective function with penalized constraints   fail to solve the problem and results in a linear bound $O(T)$ for either the regret or the violation of the constraints.

%\newpage
\end{itemize}
\section*{Thesis Overview}
The remainder of this thesis is organized as follows. Chapter~\ref{chap-background} lays out the foundation for the rest of the thesis. In particular, we provide a survey of some of the background material from statistical learning,  sequential prediction/online learning theory, and as well as convex
optimization. It will become clear in this chapter that there exist deep connections between these three areas.  

Part~\ref{part-statistical} of the thesis focuses on the statistical learning, investigating the sample complexity of learning when the target risk is known to the learner and the consistency of smoothed hinge loss.  In Chapter~\ref{chap:passive_target} we focus on statistical learning with target risk under the assumption that the loss function is smooth and strongly convex. Chapter~\ref{chap:smooth_consistency} investigates the consistency of smoothed hinge loss and provides negative and positive results on transforming its excess risk to a binary excess risk.

Part~\ref{part-online} of the thesis is on sequential prediction/online learning and introduces the gradual variation measure to asses the performance of online convex optimization algorithms in gradually evolving environments. Chapter~\ref{chap:regret}   discusses the necessity of smoothness to obtain regret bounds in terms of  gradual variation, followed by two efficient algorithms to obtain gradual variation  bounds for smooth online convex optimization problems. The adaption to other settings such as expert advice problem and strongly convex loss functions are also discussed.  The extension of results to special composite loss functions with a smooth component is discussed in Chapter~\ref{chap:non-smooth}.

Part~\ref{part-stochastic} of the thesis is devoted  to devising efficient stochastic optimization algorithms by leveraging the smoothness of loss functions.  We propose the mixed optimization paradigm for stochastic optimization in Chapter~\ref{chap:mixed} and extend to smooth and strongly convex losses in Chapter~\ref{chap:mixed-strong}. The stochastic optimization methods with bounded projections and online optimization with soft constraints are elaborated in Chapter~\ref{chap:projection}. 

Finally, Part~\ref{part-appendix} summarizes rather standard things on convex analysis and concentration inequalities that are used in the proof of results in the thesis and  is mainly for reference. In order  to facilitate independent reading of various chapters, some of the definitions from convex analysis  are even repeated several times.

\section*{Bibliographic Notes}
Some of the results in this dissertation have appeared in prior publications. The material in Chapter~\ref{chap:passive_target} is based on a work published in Conference on Learning Theory (COLT)\cite{mahdavi-colt-2013} and the content of Chapter~\ref{chap:smooth_consistency} is new~\cite{mahdavi2014binary}. The material in Chapter~\ref{chap:regret} and Chapter~\ref{chap:non-smooth} come from~\cite{chiang2012online} which is published at COLT and its extended version has been recently published in Machine Learning journal~\cite{yang-2013-ml}.  The results in Chapter~\ref{chap:mixed} and Chapter~\ref{chap:mixed-strong} follow~\cite{mahdavi-nips-201e} and~\cite{mahdavi-nips-condition}, respectively, which are appeared in Advances in Neural Information Processing Systems (NIPS).  The content of Chapter~\ref{chap:projection} is mostly compiled from~\cite{sgd-one-2012},~\cite{long-term-2012}, and~\cite{mahdavi-nips-milti} which are published at NIPS and Journal of Machine Learning Research (JMLR).

%\section*{Notation}
%This section serves as a glossary for the mathematical symbols used in the thesis. Vectors are shown by lower case bold letters, such as $\x \in \R^d$. Matrices are indicated by upper case letters such as $A$ and their pseudoinverse is represented by $A^{\dagger}$. We use $[m]$ as a shorthand for the set of integers $\{1, 2, \ldots, m \}$. Throughout the paper we denote by $\|\cdot\|$ and $\|\cdot\|_1$ the $\ell_2$ (Euclidean) norm and $\ell_1$-norm, respectively. We use $\E$ and $\E_t$ to denote the expectation and conditional expectation with respect to all randomness in early $t-1$ trials, respectively.

\label{chap-background}

\chapter{Preliminaries}
\label{chap-background}

\def \E {\mathbb{E}}
\def \D {\mathcal{D}}
\def \Sb {\mathbb{S}} 
\def \Lh {\mathcal{L}_{\mathcal{S}}}
\def \Ld {\mathcal{L}_{\mathcal{D}}}
\def \Ls {\mathcal{L}_{\mathcal{S}}}

The goal of this chapter is to give a gentle and formal overview of the  material  related to the work has been done in this thesis.  In particular, we will discuss key concepts and questions in statistical learning, online learning, and convex optimization and will highlight the role of analytical properties of loss functions such as Lipschitzness, strong convexity, and smoothness in all these settings.   The exposition given here is necessarily very brief and the detailed discussion will be provided in the relevant chapters.

\section{Statistical Learning }

\subsection{Statistical Learning Model}
In a typical supervised  statistical learning problem (also known as passive learning or batch learning), we are given an instance space $\mathcal{X}$, and a space $\mathcal{Y}$ of labels or target values.  Each element in the data domain $\mathcal{X}$ represents an object to be classified, e.g., the content of an email in spam detection application or the features of an image in vision applications. The target space $\mc{Y}$ can be either discrete $\mc{Y} = \{-1,+1\}$, as in the case of classification, or continuous $\mc{Y} = \R$, as in the case of regression. To model  learning problems in statistical or probabilistic setting, we assume that the product space $\Xi \equiv \mathcal{X} \times \mathcal{Y}$ is endowed with a probability measure $\D$ which  is unknown to the learner.  However,  it is possible to sample  an arbitrary number of pairs $\mathcal{S} = \left((\x_1, y_1), (\x_2,y_2), \cdots, (\x_n,y_n) \right) \in \Xi^n$ according to the underlying distribution $\D$.  We term this set of examples the training
set, or the training sample. The existence of distribution $\D$  is necessary to ensure that the already collected samples $\mathcal{S}$  have something in common with the new and unseen data.

A hypothesis or classifier $h: \mathcal{X}\mapsto \mathcal{Y}$ is a function that assigns labels $h(\x) \in \mathcal{Y}$ to any  instance $\x \in \mathcal{X}$ such that the assigned label is a good approximation of the possible response $y$ to an arbitrary instance  $\x$ generated according to the distribution $\mathcal{D}$. In other words, the  hypothesis $h$ captures the functional relationship between the input instances  and the output, which in turn makes it possible to predict the output value for future input instances. 

In light of no free lunch theorem~\cite{wolpert1996lack}, learning is impossible unless we make assumptions regarding the nature of the problem  at hand. Therefore, when approaching a particular learning problem,  it is desirable to take into account some prior knowledge we might have about our problem.   In this regard, we assume that the learning algorithm is confined to a predetermined set of candidate hypotheses $\mathcal{H} = \{\mathcal{X} \mapsto \mathcal{Y}\}$  to which we wish to compare the result of our learning algorithm.   In a specific learning context, the hypothesis class can represent our
beliefs on the true nature of the classification rule for the problem.  For example the hypothesis class for binary classification might be a subset of a vector space with bounded norm  which represents the linear classifiers, i.e., $\mc{H} = \{\x \mapsto \text{sign}(\dd{\w}{\x}+b),~\w \in \R^d,~b \in \R,~\|\w\| \leq R \}$.

In order to measure the performance of a learning algorithm, we usually use a loss function $\ell: \H \times \Xi \mapsto \R_{+}$.  The instantaneous loss incurred by a learning algorithm on instance $\z = (\x,y) \in \Xi$ for picking hypothesis $h \in \H$ is given by  $\ell(h,\z)$.   For example, in \textit{binary classification} problems,  $\Xi=\mathcal{X} \times\{-1,1\}$,  $\H$ is a set of functions $h:\mathcal{X} \mapsto \{-1,1\}$, and the loss function is the binary or 0-1 loss defined  as $\ell_{0-1}(h(\x), y) = \mathbb{I}{[h(\x) \neq y]}$, where $\mathbb{I}[\cdot]$ is the indicator function that takes value 1 if its argument is true and 0 otherwise.   In \textit{regression}  or real classification problems, where the goal is to predict real valued labels $\mathcal{Y} = \R$, the common loss function used to evaluate the performance of a regressor $h$ on sample $\z = (\x, y)$  is the squared loss $\ell(h, \z) = (h(\x) - y)^2$.

A classifier is constructed on the basis of  the $n$ independent and identically distributed (i.i.d) samples $\mathcal{S} = \left((\x_1, y_1), (\x_2,y_2), \cdots, (\x_n,y_n) \right)$  from $\Xi$.  The ultimate goal of a typical learning algorithm  is to use as few samples as possible, and as little computation as possible, to pick a classifier $h \in \H$ that is competitive with respect to the best hypothesis from  $\H$ with respect to the \textit{expected risk} or \textit{generalization error} defined as: 
$$
\Ld(h) = \mathbb{E}_{\z \sim \D}[{\ell(h,\z)}],
$$
where  $\mathbb{E}[\cdot]$ denotes the expectation with respect to the (unknown) probability distribution $\D$ underlying our samples and $\ell(h, (\x,y))$ is the binary loss $\mathbb{I}{[h(\x) \neq y]}$ for classification and squared loss $(h(\x) - y)^2$ for regression problem.  For binary classification the generalization error of a hypotheses $h: \mathcal{X} \mapsto \{-1, +1\}$ is simply the probability that it predicts the wrong label on a randomly drawn instance from $\Xi$, i.e., $\mathcal{L}_{\D}(h) = \mathbb{P}_{(\x,y) \sim \D} {[h(\x) \neq y]}$. 

The difference between the risk of a particular classifier $h$ and of the optimal classifier $h_* = \arg \min_{h \in \H} \mc{L}_{\D}(h)$ is called the \textit{excess risk} of $h$, i.e., 
$$\mathscr{E}(h) = \mc{L}_{\D}(h) - \mc{L}_{\D}(h_*).$$

In designing any typical  solution to a supervised machine learning problem, there are few key questions that must be considered. The first of these concerns approximation that characterizes how rich the solution space $\mathcal{H}$ is to approximate the true underlying model.  The second fundamental issue  concerns estimation  that characterizes how well the obtained solution performs in making future
predictions on unseen data and how much training samples suffices to find the solution.  The third key question concerns the computational efficiency  that characterizes how efficiently can we make use of the training data to choose an accurate hypothesis.

The basic model to analyze learning algorithms in computational learning theory is the Probably Approximately Correct (PAC) model proposed by the pioneering work of Valiant~\cite{valiant1984theory}. It applies to learning binary valued functions and uses the 0-1 loss under realizability assumption, i.e., the algorithm gets  samples that are consistent with a hypothesis in a fixed  class $\H$,  $\exists\; h_* \in \mathcal{H}; \mathbb{P}_{(\x,y) \sim \D} [h_*(\x) = y] = 1$. In PAC model we bound the loss of the algorithm with a high probability over the random draw of samples. A decision-theoretic extension of the PAC framework which is known as \textit{agnostic  learning} is introduced by~\cite{kearns1994toward} that generalizes the PAC model to general loss functions and without assuming realizability assumption as defined below:

\begin{definition} [{Agnostic PAC Learnability}] A hypothesis class $\mc{H}$ is agnostic PAC learnable with respect to $\Xi$ and a loss function $\ell: \mc{H} \times \Xi \rightarrow \mathbb{R}_{+}$, if there exists a function $m_{\mc{H}}: (0,1)^2 \rightarrow \mathbb{N}$ and a learning algorithm $\mc{A}$ with the following property: for every $\epsilon, \delta \in (0, 1)$ and for any distribution $\D$ over the domain $\Xi$, when running the algorithm $\mc{A}$ on $m \geq m_{\mc{H}}({\epsilon}, \delta)$ i.i.d. examples generated by $\D$, the algorithm returns $h \in \mc{H}$ such that, with probability of at least $1 - \delta$, 
\[ \Ld(h) \leq \min_{h' \in \mc{H}} \Ld (h') + \epsilon.\]
If further $\mc{A}$ runs in $\text{poly}(1/\epsilon, 1/\delta, n)$, then $\mathcal{H}$ is said to be efficiently agnostic PAC-learnable.
\end{definition}

The goal of the PAC framework is to understand how large a data set needs to be in
order to give good generalization. It also provides bounds for the computational cost of
learning. In agnostic PAC learnability there are two fundamental questions that need to be addressed carefully: these are  \textit{computational efficiency} and \textit{sample complexity}.   

The  computational aspect of  learning  measures the amount of computation required to implement a learning algorithm. The sample complexity of an algorithm is  the number of examples which is sufficient to ensure that, with probability at least $1-\delta$ (w.r.t. the random choice of $\mc{S}$), the algorithm picks a hypothesis  with an error that is at most $\epsilon$ from the optimal one. We  note that while computational complexity concerns the efficiency  of learning, the sample complexity is  a statistical measure and concerns  the difficulty of learning from the hypothesis $\mathcal{H}$  with respect to the underlying distribution $\D$.   An equivalent way to present the sample complexity is to give a generalization  bound. It states that with probability at least $1-\delta$, to attain a  risk $\mathcal{L}_{\D}(h)$ which departs from the  optimal risk  $\min_{h' \in \mc{H}} \mathcal{L}_{\D}(h')$ by at most $\epsilon$,  is upper bounded by some quantity that depends on the sample size $n$ and $\delta$.

Sample complexity of passive learning is well established  and goes back to early works in the learning theory where the lower bounds  $\Omega\left(\frac{1}{\epsilon} ( \log \frac{1}{\epsilon}+ \log \frac{1}{\delta})\right)$  and $\Omega\left(\frac{1}{\epsilon^2} ( \log \frac{1}{\epsilon}+ \log \frac{1}{\delta})\right)$ were obtained   in classic PAC and general agnostic PAC settings, respectively~\cite{ehrenfeucht1989general,learnabilityvcdim89,anthony1999neural}. It worth emphasizing that the PAC framework is a distribution-free model  and we are interested in sample-complexity guarantees that hold regardless of the distribution $\D$  from which examples are drawn.

\subsection{Empirical Risk Minimization}

Since in the probabilistic setting, we assume  that there is an underlying probability distribution $\D$  over the sample space $\mathcal{X} \times \mathcal{Y}$ which captures the relationship between the samples given to the algorithm during training and the new instances it will receive in the future;  the training examples must be `representative' in some way of the examples to be seen in the future.  Clearly, learning is hopeless if there is no correlation between past and present
rounds. 

Note however that the distribution $\D$ is not known to the learner; the learner sees the distribution only through the training examples $\mathcal{S} = (\z_1, \z_2, \cdots, \z_n) \in \Xi^n$, and based on these examples, must learn to predict well on new instances from the same distribution. Therefore, we cannot compute 
 the generalization error directly and machine learning aims to find estimators  
 based on the observed data samples $\mathcal{S}$ .

%%%%%%%%%%%%%%%%%%%%%%%%%%%%%%%%%%%%%%%%
A simple and well-known learning approach is the empirical risk minimization (ERM) method. Basically, the idea of ERM is to replace the unknown true risk 
$$\Ld = \mathbb{P}_{(\x,y) \sim \D} {[h(\x) \neq y]},$$ 
by its empirical counterpart rooted in the training set $\mathcal{S}$ and minimize this empirical risk as defined below:
$$
\Ls(h) = \frac{1}{n} {\Big{|} \{i:  i \in [n] \;\;\text{and}\;\; h(\x_i) \neq y_i \}\Big{|}}.
$$
The empirical error $\mathcal{L}_{\mathcal{S}}(h)$ of any hypothesis $h\in \mathcal{H}$ is its average error over the training samples in $\mathcal{S}$, while the generalization error $\mathcal{L}_{\mathcal{D}}(h)$ is its expected error based on a random sample realized by the distribution $\D$.  We note that the empirical error is a useful quantity, since it can easily be determined from the training data and it provides a simple estimate of the true error.  The empirical loss over the training data provides an estimate whose loss is close to the optimal loss  if the class $\mathcal{H}$ is  sufficiently large so that the loss of the best function in $\mathcal{H}$ is close to the optimal loss and is  small enough so that finding the best
candidate in $\mathcal{H}$ based on the data is computationally feasible. In this regard, generalization error bounds give an upper bound on the difference between the true and empirical error of functions in a given class, which holds with high probability with respect to the sampling of the training set.

Then, our task becomes to evaluate the expected risk relying on  the empirical error.  Having this quantity bounded,  a learning algorithm may choose the hypothesis that is the most accurate on the sample, and is guaranteed that its loss on the distribution will also be low.    The statistical learning theory  is concerned with characterizing learnability and providing bounds on the deviations of this estimate from the expected error.  One of its main achievements   is a complete characterization of the necessary and sufficient conditions
for generalization of ERM, and for its consistency.

A fundamental  answer, formally proven for supervised classification and regression, is that learnability is equivalent to uniform convergence, and that if a problem is learnable, it is learnable via empirical risk minimization. 

\begin{definition} [Uniform Convergence] A hypothesis class $\mathcal{H}$ has the uniform convergence property with respect to $\Xi$ and the loss function $\ell: \mathcal{H} \times \Xi \mapsto \R_{+}$, if for any probability distribution  $\mathcal{D}$ over $\Xi$, there exists a function $m_{\mathcal{H}}: (0,1)^2 \mapsto \mathbb{N}$ such that for any sample $\mathcal{S}$ of size $m_{\mathcal{H}}(\epsilon, \delta)$ drawn i.i.d based on $\mathcal{D}$, with probability at least $1 - \delta$, $\forall h \in \mathcal{H}$ it holds:
$$ \big{|} \mathcal{L}_{\mathcal{S}}(h) - \mathcal{L}_{\mathcal{D}}(h) \big{|} \leq \epsilon.$$

\end{definition}

Hence, the crucial step towards
proving learnability is to obtain a result on the uniform convergence of
sample errors to true errors. Uniform convergence of empirical quantities to their mean provides ways to bound the gap between the expected risk and the empirical risk by the complexity of the hypothesis set.  Hence, the complexity of the hypothesis class $\mathcal{H}$ is the critical factor in determining the distribution-free sample-complexity of a supervised learning problem. Several complexity measures for hypothesis classes have been proposed, each providing a different type of guarantee including the Vapnik-Chervonenkis (VC) dimension~\cite{vapnik1971uniform} and the Rademacher complexity~\cite{bartlett2003rademacher,koltchinskii2001rademacher}. The main virtue of the Vapnik-Chervonenkis theorem and  Rademacher complexity is that they convert the problem of uniform deviations of empirical averages into a combinatorial  and data dependent problems, respectively.

We note that uniform convergence arguments is not the only  possible way to characterize learnability. Since the first results of Vapnik and Chervonenkis on uniform laws of large numbers for classes of binary valued functions, there has been a considerable amount of work aiming at obtaining generalizations and refinements of these bounds. These techniques include sample compression~\cite{floyd1995sample}, algorithmic stability~\cite{bousquet2002stability}, and PAC-Bayesian analysis~\cite{mcallester1998some} which  also have been shown for characterizing learnability and proving generalization bounds. We will also discuss the stochastic optimization machinery~\cite{shalev-shwartz:2010:learnability}  to characterize learnability in general settings later in this chapter.  

%%%%%%%%%%%%%%%%%%%%%%%%%%%%%%%%%%%%%%%%
\subsection{Surrogate Loss Functions and Statistical Consistency}
Although ERM approach has a lot of theoretical merits, since we should seek to minimize the training error based on 0-1 loss,  it typically is  a combinatorial problem, leading to NP-hard optimization problem which  is not computationally realizable.   

A common practice  to circumvent this difficulty is to revert to minimize a  surrogate loss function, i.e., to replace the indicator function by a  surrogate function  and find the minimizer with respect to this surrogate function.  Obviously, the surrogate loss needs to be computationally easy to optimize, while close in some sense to the 0-1 loss. In particular, if the surrogate function is assumed to be convex, it allows the optimization to be performed efficiently with only modest computational resources.  Examples of such surrogate loss functions for 0-1 loss  include  logistic loss $\ell_{\log}(h, (\x,y)) = \log (1+\exp(-y h(\x)))$ in logistic regression~\cite{friedman2000additive},  hinge loss $\ell_{\text{hinge}}(h, (\x,y)) = \max(0, 1- y h(\x))$ in support vector machines (SVMs)~\cite{cortes1995support} and exponential loss $\ell_{\text{exp}}(h, (\x,y)) = \exp(-y h(\x))$ in boosting (e.g., AdaBoost~\cite{freund1995desicion}).  When the hypothesis class $\mc{H}$  consists of functions that are linear in a parameter vector $\w$, i.e., linear classifiers, these  loss functions are depicted in Figure~\ref{fig:surrogate-losses}.

\begin{figure}
\begin{center}
\begin{tikzpicture}[domain=0:5,scale=0.8]

    \draw[black,dashed,line width=2,-] (2,5) -- (2.5,5) node[right] {0-1 loss};;
    \draw[red,line width=2,-] (2,4) -- (2.5,4)  node[right] {Hinge loss};
    \draw[black,line width=2,-] (2,3) -- (2.5,3)  node[right] {Logistic loss};
    \draw[blue,line width=2,-] (2,2) -- (2.5,2)  node[right] {Exponential loss};

%  \draw[step=1cm,color=gray] (-4,-2) grid (4,4);
  \draw[->,line width=0.5mm] (-6,0) -- (6,0) node[right] {$y \dd{\w}{\x}$};
  \draw[->, line width=0.5mm] (0,-2) -- (0,5.5) node[above] {$\ell(\w, (\x,y))$};
  
  \draw[domain=-1.6:5,smooth,variable=\x,blue,line width=0.5mm] plot ({\x},{exp(-\x)});     
  \draw[domain=-4:5,smooth,variable=\x,red,line width=0.5mm] plot ({\x},{max(0,1-\x)});
  \draw[domain=-3.2:5,smooth,variable=\x,black,line width=0.5mm] plot ({\x},{log2(1+exp(-\x))});
  
  \draw[color=blue,line width=0.6mm,black,style=dashed] (0,0) -- (4,0) node[right] {};
  \draw[color=blue,line width=0.6mm,black,style=dashed] (-4,1) -- (0,1) node[below] {};
\end{tikzpicture}
\end{center}
\caption[Illustrations of the 0-1 loss function and  its convex surrogates]{Illustrations of the 0-1 loss function, and three surrogate convex loss functions: hinge
loss, logistic loss, and exponential loss as scalar functions of $y \dd{\w}{\x}$.}
\label{fig:surrogate-losses}
\end{figure}
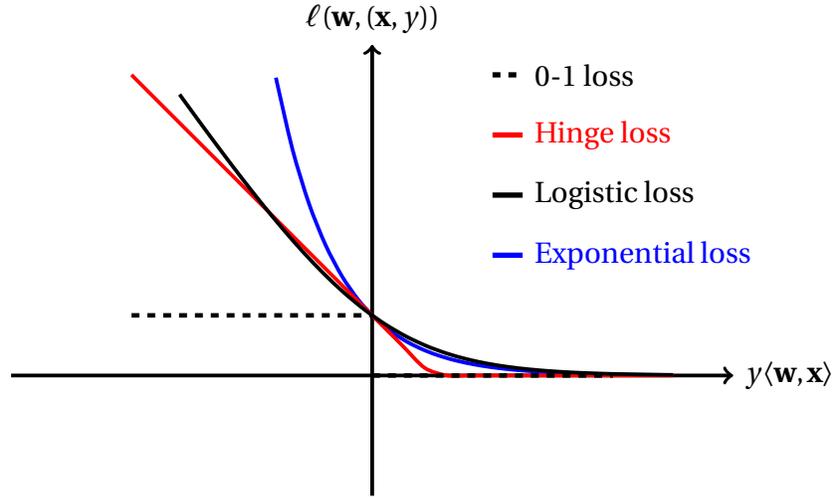

Having defined the surrogate loss functions, then the task is to minimize the relaxed empirical loss in terms of the surrogate losses. However, in practice, the ubiquitous approach to find the solution is  the regularized empirical risk minimization  which adds a regularization function  $\mathcal{R}(h)$ to the  objective and solves
\begin{eqnarray}\label{eqn:2:erm}
  h_{\mathcal{S}} \in \arg \min_{h \in \mathcal{H}} \left\{\Lh(h) + \mathcal{R}(h) \equiv  \frac{1}{n}{\sum_{i = 1}^{n} \ell(h,\z_i)}+\mathcal{R}(h)\right\}.
\end{eqnarray}
The goal of introducing regularizer is to prevent over-fitting. Of course, given some training data, it is always possible to build a function that fits exactly the data. But, in the presence of noise, this may  lead to a poor performance on unseen instances. An immediate consequence of adding the  regularization term is to favor simpler classifiers to increase its  generalization capability. Some of the commonly used regularizers  in the literature are $\mathcal{R}(h) = \|h\|_2^2$, and  $\mathcal{R}(h) = \|h\|_1$.  We note that  solving the optimization problem in~(\ref{eqn:2:erm}) is a convex optimization problem for which efficient algorithms exist to find a near optimal solution in a reasonable amount of time.  

Although the idea of replacing the non-convex 0-1 loss function with convex surrogate  loss functions seems appealing and resolves the efficiency issue of the ERM method, but it  has statistical consequences that must be balanced against the computational virtues of convexity.  The question then is how well does minimizing such a convex surrogate perform relative to minimizing the actual classification error. Statistical consistency concerns this issue.  Consistency requires convergence of the empirical risk to the expected risk for the minimizer of the empirical risk together with convergence of the expected risk to the minimum risk achievable by functions in $\mathcal{H}$~\cite{bartlett2006convexity}. An important line of research in statistical learning theory focused on relating the convex excess risk  to the binary excess risk. It is known that under  mild conditions, the classifier learned by minimizing the empirical loss of convex surrogate is consistent to the Bayes classifier~\cite{zhang2004statistical,lugosi2004bayes,jiang2004process,lin2004note,steinwart2005consistency,bartlett2006convexity}.  For instance, it was shown in~\cite{bartlett2006convexity} that the necessary and sufficient condition for a convex loss $\ell(\cdot)$ to be consistent with the binary loss is that $\ell(\cdot)$ is  differentiable at origin and $\ell'(0) < 0$. It was further established in the same work that the binary excessive risk can be upper bound by the convex excess risk through a $\psi$-transform that depends on the surrogate convex loss $\ell(\cdot)$. A detailed elaboration of this issue will be given in Chapter~\ref{chap:smooth_consistency} where we examine the statistical consistentency of smooth convex surrogates.

%%%%%%%%%%%%%%%%%%%%%%%%%%%%%%%%%%%%%%%%%%%%%%%%%
\subsection{Convex Learning Problems}
We  now turn our attention to  \textit{convex learning}  problems where  $\Xi$ be an arbitrary measurable set,  $\mathcal{H}$ be a closed, convex subset of a vector space,  and the loss  function $\ell(h, \z)$ be convex w.r.t. its first argument. This family of learning encompasses  a rich body of existing learning methods for which efficient algorithms exists such as support vector machines, boosting  and logistic regression.   Convex learning problems makes an important family of learning problems, mainly because most of what we can learn efficiently falls into this family.    

Before diving into formal definition, we need to familiarize ourselves with the following definitions about convex analysis~\cite{borwein2010convex,nesterov2004introductory} which will come in handy throughout this dissertation (for standard definitions about convex analysis see Appendix~\ref{chap:appendix-convex}).
\begin{definition}[Convexity] A set $\W$ in a vector space is convex if for any two vectors $\w, \w' \in \W$, the line segment connecting two points is contained in $\W$ as well. In other fords, for any $\lambda \in [0,1]$, we have that $\lambda \w + (1-\lambda)\w' \in \W$. A function $f: \W \mapsto \R$ is said to be convex if $\W$ is convex and  for every $\w, \w' \in \W$ and $\alpha \in [0,1]$,
\[ f(\lambda \w + (1-\lambda)\w') \leq \lambda f(\w) + (1-\lambda)f(\w'). \]
A continuously differentiable function is convex if $f(\w) \geq f(\w') + \dd{\nabla f(\w')}{\w-\w'}$ for all $\w,\w' \in \W$. If $f$ is non-smooth then this inequality holds for any sub-gradient $\g \in \partial f(\w')$.
\end{definition}
The formal definition of convex learning problems is given below.

\begin{definition}[Convex Learning Problem] A learning problem with hypothesis space $\mathcal{H}$, instance space $\Xi = \mathcal{X} \times \mathcal{Y},$ and the loss function $\ell: \Xi \times \mathcal{H} \mapsto \R_{+}$ is said to be convex if  the hypothesis class $\mathcal{H}$ is a parametrized convex set $\H = \{h_{\w}: \x \mapsto \langle \w, \x \rangle: \w \in \R^d,   \|\w\| \leq R\}$ and for all $\z = (\x, y) \in \Xi$, the loss function $\ell(\cdot, \z)$ is a non-negative convex function.  
\end{definition}
In the remainder  of thesis, when it is clear from the context, we will represent the hypothesis class with $\W$ and simply use vector $\w$ to represent $h_{\w}$, rather than working with  hypothesis $h_{\w}$.

 We note that for convex learning problems, the ERM rule becomes a convex optimization problem which can be efficiently solved. This stands in sharp contrast to non-convex loss functions such as 0-1 loss  for which solving  the ERM rule is computationally cumbersome and known to be NP-hard. Obviously, this efficiency comes at a price and not every convex learning problem is  guaranteed to be learnable and convexity by itself is not sufficient for learnability. This requires to impose more assumptions on the setting to ensure   the learnability of the problem. In particular,  it can been shown that if the hypothesis space $\W$ is bounded and the loss function is  Lipschtiz or smooth as formally defined below, then the convex learning problem is learnable~\cite{shalev-shwartz:2010:learnability,shalev2014understanding}.\\

\begin{definition}[Lipschitzness] A function $f: \W \mapsto \R$ is $\rho$-Lipschtiz over the set $\W$ if for every $\w, \w' \in \W$ we have that $|f(\w) - f(\w')| \leq \rho ||\w - \w'||$. \\
\end{definition}

\begin{definition}[Smoothness] A differentiable  loss function $f: \W \mapsto \R$ is said to be  $\beta$-smooth  with respect to a norm $\|\cdot\|$, if it holds that
\begin{equation}
\label{eqn:smoth}
f(\w) \leq f(\w') +  \dd{\nabla f(\w')}{\w-\w'} + \frac{\beta}{2}\|\w-\w'\|^2, \;\forall \; \w,\w'\in\W.
\end{equation}
\end{definition}
We note that smoothness also follows if the gradient of the loss function is $\beta$-Lipschtiz, i.e., $\| \nabla f(\w) - \nabla f(\w')\| \leq \beta \|\w - \w' \|$. Smooth functions arise, for instance, in logistic and least-squares regression,  and in general for learning linear predictors where the loss function has a Lipschitz continuous gradient.

There has been an upsurge of interest  over the last decade in finding tight upper bounds on the sample complexity of convex learning problems   by utilizing prior knowledge on the  curvature of the loss function,   that led to stronger generalization bounds  in agnostic PAC setting. In~\cite{DBLP:journals/tit/LeeBW98}  \textit{fast} rates obtained for squared loss, exploiting the strong convexity of this loss function, which only holds under pseudo-dimensionality assumption.  With the recent development in online strongly convex optimization~\cite{hazan-log-newton}, fast rates approaching $O(\frac{1}{\epsilon} \log \frac{1}{\delta}) $ for convex Lipschitz strongly convex loss functions has been obtained in~\cite{kakade-2008-strong,fastrates2008,compl-linear-nips-2008} and for exponentially concave loss functions in~\cite{mahdavi2014excess}. For smooth non-negative loss functions,~\cite{srebro-2010-smoothness} improved the sample complexity  to \textit{optimistic} rates for non-parametric learning using the notion of local Rademacher complexity~\cite{bartlett2005local}.

%%%%%%%%%%%%%%%%%%%%%%%%%%%%%%%%%%%%%%%%%%%%%%%
\section{Sequential Prediction/Online Learning}
The statistical model discussed above, first assumes the existence of a stochastic
model for   generating  instances  according to the underlying distribution $\D$, and then samples a training set $\mc{S} = \left( (\x_1,y_1), (\x_2,y_2), \cdots, (\x_n,y_n)\right)$   and investigates the ERM strategy to find a hypothesis $h \in \mc{H}$ which generalizes well on unseen instances.  Although, this model is valid for  cases for which  a tractable statistical model reasonably describes the underlying
process, but it may be unrealistic in practical problems where the process is hard to model from a statistical viewpoint and may even react to the learners's decisions, e.g.,  applications such as portfolio management, computational finance, and whether prediction. 

Sequential prediction/online learning (also known as universal prediction of individual sequences) is a strand of learning theory avoiding making any stochastic assumptions about the way the observations are generated and the goal is to develop prediction methods that are robust in the sense that they work well even in the worst case.  For instance, in the problem of online portfolio management~\cite{bell1988game}, an online investor wants to distribute her wealth on a set of available financial instruments without any assumption on the market outcome in advance. Obviously in applications of this kind, the main challenge  is that the learner can not make any statistical assumption about the process generating the instances and the data are continuously evolving or adversarially changing. Online learning or sequential prediction is an elegant paradigm  to capture  these problems  that alleviates  the statistical assumption usually made in statistical setting and was introduced in the seminal works of Hannan~\cite{hannan1957approximation}  and Blackwell~\cite{blackwell-approach} in the 1950's in repeated game playing  and the work of Littlestone~\cite{littlestone1988learning} in the 1990's in  learning theory (see e.g.,~\cite{bianchi-2006-prediction} and~\cite{Shalev-Shwartz12-book} for through discussion).

\subsection{Mistake Bound Model and Regret Analysis}

 The problem of sequential prediction may be cast as a repeated game between a decision maker- also called the forecaster- and an environment- also called adversary. In this model, learning proceeds in $T$ consecutive rounds, as we see examples one by one.  At the beginning of  round $t$, the learning algorithm $\mc{A}$ has the hypothesis $h_t \in \mc{H}$ and the adversary picks an instance $\z_t = (\x_t, y_t)$.  The adversary at round $t$ can select the instance $\z_t \in \Xi $ in an adversarial worst case fashion based on previous instances $\z_1,\ldots,\z_{t-1}$ and based on previous hypotheses $h_1,\ldots, h_{t-1}$ selected by the learner. Then, the learner receives the instance $\x_t$ and predicts $h_t(\x_t)$.  At the end of the round, the true label $y_t$ is revealed to the learner and $\mc{A}$ makes a mistake if $h_t(\x_t) \neq y_t$. Unlike statistical setting, here the prediction task is sequential: the outcomes are only revealed one after another; at time $t$, the learner guesses the next outcome $y_t$  before it is revealed.  The algorithm then updates its hypothesis, if necessary, to $h_{t+1}$ and this continues till time $T$. 
 
The sequential prediction model discussed above,  which is reminiscent  of the framework of competitive analysis~\cite{borodin1998online},  is known as \textit{mistake bound model}  and was introduced to learning community in~\cite{littlestone1988learning}. In this model  the goal of the learner is to  sequentially deduce information from previous rounds so as to improve its predictions on future rounds.   When the  algorithm is conservative (or lazy), meaning that the algorithm only changes its hypothesis  when
it makes a mistake is called \textit{mistake driven}.  We note that  many seemingly unrelated problems fit in the framework of the abstract sequential decision problem  including online prediction problems in the experts model~\cite{cesa1997use}, Perceptron like classification algorithms~\cite{block1962perceptron},  Winnow algorithm~\cite{littlestone1988learning}, and learning in repeated game playing~\cite{freund1999adaptive}- to name a few. Here we list few  sequential prediction problems to better illustrate the setting. \\
 
 \noindent \textbf{Example 1: Online classification and regression.} As an illustrative example let us consider the  \textit{online binary classification}  problem where at each round $t$ the learner receives an instance $\x_t \in \mc{X}$ as input and is required to predict a label $\hat{y}_t \in \mc{Y} = \{-1,+1\}$ for the input instance. Then, the learner receives the correct label $y_t \in \mc{Y}$ and  suffers the 0-1 loss: $\mathbb{I}[y_t \neq \hat{y}_t]$.  As another example, let us consider \textit{online regression} problems where  at each round $t$ a feature vector $\x_t \in \R^d$ is given to the online learner, and a value $y_t \in \R$ has to be estimated using linear predictors with bounded norm, i.e, $\mathcal{H} = \{\x \mapsto {\w, \x}: \w \in \R^d: \|\w\| \leq R\}$, that the learner predicts $\dd{\w}{\x_t}$. The loss function at round $t$ for a predictor $\w$ is $\ell_t(\w) = (y_t - \dd{\w_t}{\x_t})^2$.  \\
 
 \noindent \textbf{Example 2: Prediction with expert advice.}  Every day the manager of a company should decide to produce one of the $K$ different products without knowing the market demand in advance. At the end of the day, he/she will be informed the gain achieved by  selling the product but nothing about the potential income from other products. The goal of the manager is to maximize the income of company over a sequence of many  periods. This is a problem of repeated decision-making which is called learning from expert advice where the objective functions  to be optimized are unknown and revealed (perhaps only partially) in an online manner. In the general prediction with expert advice game a learner competes against $K \in \mathbb{N}$ experts in a game consisting of $T$ rounds. Each round $t$,  each expert reveals a prediction from $\mc{Y} = \{0,1\}$. The learner  form its own prediction by sampling an expert from $h_t = \w_t \in \mc{H} \equiv \Delta_K$,    where $\Delta_K$ is the set of probabilities over $K$ experts (i.e., simplex). The true outcome $y_t$ is then revealed and the learner and all of the experts receive a penalty depending on how well their prediction fits with the revealed outcome. The aim of the learner in this game is to incur a cumulative loss over all rounds that is not much worse than the best expert.

One natural measure of the quality of learning in mistake bound model of sequential setting is the number of worst case mistakes the learner makes. In particular, under the realizability assumption  (i.e. where there exists a hypothesis in $\mc{H}$  which performs perfectly on the sequence),  the learner's goal becomes to have a bounded number of mistakes   which is known as \textit{mistake bound}. The optimal mistake bound for a hypothesis class 
$\mc{H}$ is the minimum mistake bound  over all learning algorithms $\mc{A}$ on the worst case sequence of examples:
\begin{eqnarray}\label{eq:regret}\index{regret}
\text{Mistake}(\mc{A}, \mc{H}, \Xi,T) = \min_{h_1, h_2, \cdots, h_T \in \mc{H}} \max_{\x_1, \x_2, \cdots, \x_T \in \mc{X}} \sum_{t=1}^T \mathbb{I}{[h_t(\x_t) \neq h(\x_t)]},
\end{eqnarray}
where $h_1, h_2, \cdots, h_T \in \mathcal{H}$ is the sequence of hypothesis generated by $\mc{A}$. 

We say that a hypothesis class $\mc{H}$ is learnable in the online learning model if there
exists an online learning algorithm $\mc{A}$  with a finite worst case mistake bound, no matter how long the sequence of examples $T$ is. We note that the mistake bound model, in comparison to PAC model, is strong in the sense that it does not depend on any assumption about the instances.   It is also remarkable  that despite the inherent differences between PAC and mistake bound frameworks, mistake bounds have corresponding risk bounds that are not worse, and sometimes better, than those obtainable with a direct statistical approach. In particular, by a simple reduction,  it is straightforward  to show that if an algorithm $\mathcal{A}$ learns a hypothesis class $\H$ in the mistake bound model, then $\mathcal{A}$ also learns $\H$ in the probably approximately correct model.

We note that due to impossibility theorem by Cover~\cite{cover-imposiibility-1965}, any online predictor that makes deterministic predictions is doomed to achieve a sub-linear  regret universally for all sequences. To  circumvent this obstacle,  two typical solutions have been examined which are  randomization and convexification. In the former we allow the learner to make randomized predictions, making the algorithm unpredictable against the adversary,  and in the latter we replace the non-convex 0-1 loss with a cover surrogate loss function (see e.g.,~\cite{blackwell1995minimax} and~\cite{Shalev-Shwartz12-book}).

Similar to statistical learning, we can also generalize online setting to the agnostic (a.k.a. non-realizable) setting where there is no classifier in $\mc{H}$ which performs perfectly on the sequence. In this case an adversary can make the cumulative loss of our online learning algorithm arbitrarily large.  To overcome this deficiency, the performance of the forecaster is compared to  some notion of "how well it could have performed". In particular, the performance of the online learner is compared to that of the best single decision for the
sequence, in hindsight, chosen from the hypothesis in $\mathcal{H}$.  This brings us to the objective which is  commonly known as \textit{regret} which is formally defined by:
\begin{eqnarray}
\label{eqn:regeret}
\text{Regret}(\mc{A}, \mc{H},\Xi,T) = \sum_{t=1}^{T}{\ell(h_t,  (\x_t, y_t))} - \min_{h \in \mc{H}} \sum_{t=1}^{T}{\ell(h, (\x_t,y_t))}
\end{eqnarray}
where $\ell(h, (\x,y))$ is the loss function to measure the discrepancy between the prediction $h(\x)$ and the corresponding observed element, e.g., 0-1 loss function $\mathbb{I}{[h(\x) \neq y]}$ for binary classification.  

Regret measures the difference between the cumulative loss of the learner's strategy and the minimum possible loss had the sequence of loss functions been known in advance and the learner could choose the best fixed action in hindsight. In particular, we are interested in rates of increase of $\text{Regret}(\mc{A}, \mc{H}, \Xi, T)$ in terms of $T$. When this is sub-linear in the number of rounds, that is,  $o(T)$, we call the solution Hannan consistent \cite{bianchi-2006-prediction}, implying that the learner's average per-round loss approaches the average per-round loss of the best fixed action in hindsight. It is noticeable that  the performance bound must hold for any sequence of loss functions, and in particular  if the sequence is chosen adversarially.  \\

\noindent \textbf{Online learnability.} In the online setting, the analogous of PAC learnability  was addressed by Littlestone~\cite{littlestone1988learning} who described a combinatorial characterization of hypothesis classes that are learnable in mistake bound model under realizability assumption.  The extension of these results to agnostic online setting was addressed in~\cite{agnostic-online-2009}.  Recall that, in the PAC model, VC-dimension of hypothesis class $\mc{H}$ characterizes learnability of $\mc{H}$ if we ignore computational considerations. Moroever, VC-dimension characterizes learnability in the agnostic PAC model as well.  In online setting what is known as Littlestone's dimension plays the same rule.  Recently, notion of Sequential Rademacher complexity  has been introduced  to   characterizing online learnability which plays the similar role as the Rademacher complexity in statistical learning theory~\cite{DBLP:conf/nips/RakhlinST10}. 
%%%%%%%%%%%%%%%%%%%%%%%%%%%%%%%%%%%%%%%%%%%%%%%%
\subsection{Online Convex Optimization and Regret Bounds}
The online convex optimization (OCO) framework generalizes many known online learning problems in the realm of sequential  prediction and repeated game playing. Among these are online classification and sequential portfolio optimization. The unified  setting of OCO was introduced in \cite{gordon1999regret} and the exact term was used before in~\cite{DBLP:conf/icml/Zinkevich03}. Since the introduction of OCO, there have been a dizzying number of extensions and variants that is the focus of this section. 

Assume we are given a fixed convex set $\W$ and some set of convex functions $\mc{F}$ on $\W$. In OCO, a decision maker is iteratively required to choose a decision $\w_t\in \W$. After making the decision $\w_t$ at round $t$, a convex loss function $f_t \in \mc{F}$ is chosen by adversary and the decision maker incurs loss $f_t(\w_t)$. The loss function  is chosen completely arbitrarily and even in an adversarial manner given the current and past decisions of the decision maker.  Online linear optimization is a special case of OCO in which the set $\mc{F}$ is the set of linear functions, i.e., $\mc{F}=\{\w \mapsto \dd{\mathbf{f}}{\w}: \mathbf{f}\in \R^d\}$. The goal of online convex optimization is to come up with a sequence of solutions $\w_1, \ldots, \w_T$ that minimizes the \textit{regret}, which is defined as the difference in the cost of the sequence of decisions accumulated up to the trial $T$ made by the learner and the cost of the best fixed decision in hindsight, i.e.
$$
\text{Regret}(\mc{A}, \W, \mathcal{F}, T) = \sum_{t=1}^{T}{f_t(\w_t)} - \min_{\w \in \W} \sum_{t=1}^{T}{f_t(\w)}.
$$
Based on the  type of feedback revealed to learner by adversary at the end of each iteration, we distinguish
two types of OCO problem.  In the \textit{full information} OCO,  after suffering the loss, the decision maker gets full knowledge of the function $f_t(\cdot)$. In the \textit{partial information} setting (also bandit  OCO), the decision maker only learns the value $f_t(\w_t)$ and does not gain any other knowledge about $f_t(\cdot)$. 

We also distinguish between  \textit{oblivious} and \textit{adapive} adversaries.  In the oblivious  or non-adaptive model, adversary is assumed to know our algorithm, and can pick the worst possible sequence of cost functions for it. However this sequence must be fixed in advance before game starts; during the game, adversary receives no feedback about our chosen decisions.  In the more powerful adaptive  model,  the adversary is assumed to know not only our algorithm, but  also the history of the game up to the current round. In other words, at the end 
of each round $t$, our decision $\w_t$ is revealed to adversary, and the next cost function 
$f_t(\cdot)$ may depend arbitrarily on $\w_1, \cdots, \w_t$.

The design of algorithms for regret minimization in OCO setting recently has been influenced by tools from the convex optimization. It has long been known that special kinds of loss functions permit tighter regret bounds than other loss functions.  Two most important family of loss functions that has been considered are convex Lipschtiz and strongly convex functions. 

Before presenting the known results on regret bounds for different families of loss functions, we will first need  the definition of strong convexity as:
\begin{definition}[Strong Convexity] A loss function $f:\W \mapsto \R$ is said to be $\alpha$-strongly convex w.r.t a norm $\|\cdot\|$, if  there exists a constant $\alpha > 0$ (often called the modulus of strong convexity) such that, for any $\lambda\in[0, 1]$ and for all $\w, \w'\in \W$, it holds  that
\begin{equation*}
f(\lambda \w+ (1-\lambda)\w')\leq \alpha f(\w) + (1-\lambda) f(\w') - \frac{1}{2}\lambda(1-\lambda)\alpha\|\w-\w'\|^2.
\end{equation*}
\end{definition}
If $f(\w)$ is twice differentiable, then an equivalent  definition of strong convexity is $ \nabla^2 f(\w) \succeq \alpha \mathbf{I}$ which indicates that  the smallest eigenvalue of the Hessian of $f(\w)$ is uniformly
lower bounded by $\alpha$ everywhere. When $f(\w)$ is differentiable, the strong convexity is equivalent to
\[f(\w) \geq f(\w') + \langle \nabla f(\w'), \w-\w'\rangle + \frac{\alpha}{2}\|\w-\w'\|^2,\; \forall \;\w, \w'\in\W.\]

Henceforth, we shall review few algorithms for  OCO  and state their regret bounds  under different assumptions on the curvature of the sequence of the adversarial loss functions. 

\subsubsection{Online Gradient Descent}
We start with the first, perhaps the simplest  online convex optimization algorithm, which nevertheless captures the spirit of the idea behind the most of existing methods. This algorithm  applies to the most general setting of online convex optimization and is referred to as Online Gradient Descent (OGD). The OGD method is rooted in the standard
gradient descent algorithm and was introduced to the online setting by Zinkevich~\cite{DBLP:conf/icml/Zinkevich03}. 

The OGD method  starts with an arbitrary decision $\w_1 \in \W$, and iteratively modifies it according to the cost functions that are encountered so far as follows:

%%%%%%%%%%%%%%%%%%%%%%%%%%%%%%%%%%%%
\begin{figure}[H]
\begin{center}
\begin{myalg}{Online Gradient Descent}{OGD}
{\bf Input:} \+ convex set $\W$, step size $\eta > 0$ \-\\ 
{\bf Initialize:} \+ $\w_1 \in \W$  \- \\ \\
{\bf for} $t=1,2,\ldots, T$  \+ \\
Play $ \w_{t+1} =  \Pi_{\W} \big{(} \w_t- \eta \nabla f_t (\w_t) \big{)} $  \\
Receive loss function $f_{t+1}(\cdot)$ and incur $f_{t+1}(\w_{t+1})$ \- \\
{\bf end for} 
\end{myalg}
\end{center}
\end{figure}
%%%%%%%%%%%%%%%%%%%%%%%%%%%%%%%%%%%%

Here $\Pi_{\W}(\cdot)$ denotes the orthogonal projection onto the convex set $\W$.

The OGD algorithm is straightforward to implement, and updates take time $O(d)$ given the gradient. However, the projection step might be computationally cumbersome for complex domains $\W$.  The following theorem states  that the regret bound for OGD method when applied to convex Lipschitz functions  is on the order of $O(\sqrt{T})$ which has been proven to be tight up to constant factors~\cite{abernethy2008optimal}. 

\begin {theorem} \label{thm-chapter-2-ogd}
Let $f_1, f_2, \ldots, f_T$ be an arbitrary sequence of convex, differentiable functions defined over the convex set $\W \subseteq \mathbb{R}^d$. Let $G = \max_{t \in [T]} || \nabla f_t(\w)||$ and $R = \max_{\w, \w' \in \W} ||\w-\w'|| $. Then \text{OGD} with step size $\eta_t = \frac{R}{G\sqrt{t}}$ achieves the following guarantee for any $\w \in \W$, for all $T \geq 1$.  
\begin{eqnarray}
 \sum_{t=1}^{T}{f_t(\w_t)} - \sum_{t = 1}^{T}{f_t(\w)} \leq \frac{\|\w - \w_1\|^2}{2\eta} + \frac{\eta}{2} \sum_{t=1}^{T}{\|\nabla f_t(\w_t)\|^2}= O(R G\sqrt{T})
\end{eqnarray}  
\end {theorem}

In fact it is possible to show that OGD can attain a  logarithmic regret  $O(\log T)$ for strongly convex functions  by appropriately tuning the step sizes as stated below. 

\begin {theorem} \label{thm-chapter-2-ogd-sc}
Let $f_1, f_2, \ldots, f_T: \W \rightarrow \mathbb{R}$ be an arbitrary sequence of $\alpha$-strongly convex  functions. Under the same conditions as Theorem~\ref{thm-chapter-2-ogd} with step size $\eta_t = \frac{1}{\alpha t}$, the OGD algorithm achieves the following regret:  
\begin{eqnarray}
 \sum_{t=1}^{T}{f_t(\w_t)} - \sum_{t = 1}^{T}{f_t(\w)} \leq G^2 \sum_{t=1}^{T}{\frac{1}{\alpha t}} \leq  \frac{G^2}{2 \alpha} (1+ \log T) = O(\log T)
\end{eqnarray}  
\end {theorem}

\begin{remark}
In the above algorithm, the updating rule for OGD   uses the gradients of the loss functions  $\nabla f_{t} (\w)$ at each iteration to update the solution. In fact it is not required to assume that the loss functions are differentiable and it suffices to assume that the loss functions only have a sub-gradient everywhere in the domain $\W$, making the algorithm suitable  for non-smooth settings. In particular for any $\g_t \in \partial f_t(\w_t)$ where $\partial f_t(\w_t)$ is the set of sub-gradient at point $\w_t$, the algorithm is able to achieve the same regret bounds in both cases. 
\end{remark}

\subsubsection{Follow The Perturbed Leader}
The first efficient algorithm for the general online \textit{linear} optimization problems is due to Hannan~\cite{hannan1957approximation}, and was subsequently rediscovered and clarified in~\cite{Kalai:2005:EAO:1113185.1113189}. The Follow The Perturbed Leader (FTPL) algorithm assumes that there is an oracle that can efficiently solve the offline optimization problem. Having access to such an oracle, the FTPL selects the decision that appears to be the best so far, but for a version of actual cost vectors which have been perturbed by the addition of some noise.  The addition of random noise to the observed cost functions has the effect of slowing down the algorithm so that instead of tracking small fluctuations in cost functions, it has the tendency to stick to the same decision unless there is a compelling reason to switch to another decision.

Although  the FTPL algorithm works in a similar setting as OGD, but there is a  crucial difference between them which makes FTPL more suitable for online combinatorial learning.  In particular, the decision set $\W$ does not need to be convex as long as the offline optimization problem can be solved efficiently. This is a significant advantage of the FTPL approach, which can be utilized to tackle more general problems with discrete decision spaces.   The FTPL method for online decision making relies on a linear optimization procedure $\mathcal{M}$ over the set $\W$ that computes $\mathcal{M}(\f) = \arg \min_{\w \in \W} \dd{\f}{\w}$ for all $\f \in \mathbb{R}^{d}$. Then, FTPL chooses $\w_t$ by first drawing a perturbation $\boldsymbol{\mu}_t \in [0,\frac{1}{\eta}]^d$ uniformly at random, and computing $\w_{t+1} ~=~ \mathcal{M}\left(\sum_{\tau=1}^t \f_\tau + \boldsymbol{\mu}_t\right)$. 
%%%%%%%%%%%%%%%%%%%%%%%%%%%%%%%%%%%%%%%%%%%%%%%
\begin{figure}[H]
\begin{center}
\begin{myalg}{Follow The Perturbed Leader}{FTPL}
{\bf Input:} \+ a general domain $\W$, step size $\eta > 0$, offline linear oracle $\mathcal{M}$ over $\W$ \-\\ 
{\bf Initialize:} \+ $\w_1 \in \arg \min_{\w \in \W} \mathcal{M}(\mathbf{0})$  \- \\ \\
{\bf for} $t=1,2,\ldots, T$  \+ \\
Draw $\boldsymbol{\mu}_t \in [0,\frac{1}{\eta}]^d$ uniformly at random \\
Play $\w_{t+1} ~=~ \mathcal{M}\left(\sum_{\tau=1}^t \f_\tau + \boldsymbol{\mu}_t\right)$ \\
Receive loss vector $\f_{t+1} \in \R^d$ and incur $\dd{\f_{t+1}}{\w_{t+1}}$ \- \\
{\bf end for} 
\end{myalg}
\end{center}
\end{figure}
%%%%%%%%%%%%%%%%%%%%%%%%%%%%%%%%%%%%%%%%%%%%%%%
The regret of FTPL algorithm is stated in the following theorem. 
\begin{theorem}[Regret Bound for FTPL] \label{thm-chapter-2-ftpl}
Let  $\f_{1},\ldots,\f_{T}$ be an arbitrary sequence of linear loss functions from unit ball and  let $\w_{1},\ldots,\w_{T}$ be the sequence of decisions generated by the FTPL constrained to a general set $\W$. Then for any $\w\in \W$, the FTPL algorithm with parameter $\eta = \sqrt{R/GT}$ satisfies:
$$
\E\left[ \sum_{t=1}^T \dd{\f_t}{\w_t} \right] -  {\sum_{t=1}^{T} \dd{\f_t}{\w}}~\le~ 2\sqrt{RGT} ~,
$$
where $ \max_{\w} |\dd{\f_t}{\w}| \leq G, t \in [T]$ is an upper bound on the magnitude of the rewards, and $R$ is an upper bound on the $\ell_1$ diameter of $\W$, i.e., $  R \geq \max_{\w,\w' \in \W} \|\w-\w'\|_1$. 
\end{theorem}

\begin{figure}[t]
 \begin{center}
  \begin{tikzpicture}[node distance=4.5cm, auto, >=stealth]
   % nodes
   \node[block] (a)               {Adversary};
   \node[block] (b)  [right of=a] {Linearization};
   \node[block] (d)  [right of=b] {FTRL};
   
   % edges
   \draw[thick,->] (a.north east) -- node[above]{$f_t(\w)$} (b.north west);
   \draw[thick,->] (b.north east) -- node[above]{$\mathbf{f}_t = \nabla f_t(\w_t)$} (d.north west);
   \draw[thick,->] (d.south west) -- node[below]{$\w_{t+1} \in \mc{W}$} (b.south east);
   \draw[thick,->] (b.south west) -- node[below]{$\w_{t+1} \in \mc{W}$} (a.south east);
  \end{tikzpicture}
  \label{flowchart}
 \end{center}
\caption[The linearazation method for online convex optimization.]{The reduction of general online convex optimization problem to online optimization with linear
functions.}
\label{fig:linearalization}
\end{figure}
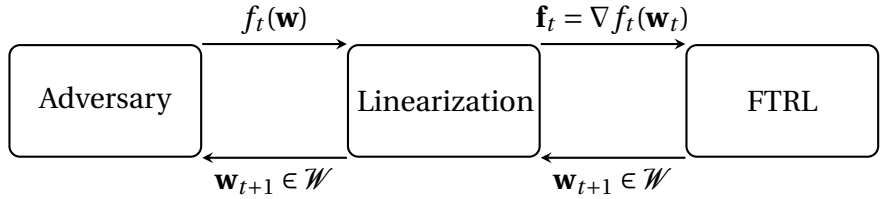

\subsubsection{Follow The Regularized Leader}

 A natural modification of the basic FTPL algorithm or  fictitious  play in game theory is the Follow The Regularized Leader (FTRL) algorithm,  in which we minimize the loss on all past rounds plus a regularization term. Regularization is an alternative to perturbation to stabilize decisions during the prediction.  Naturally, different regularization functions will yield different algorithms
for different applications.  

In FTRL we assume that the loss functions  are linear  $f_t(\w) = \dd{\mathbf{f}_t}{\w}, \f_t \in \R^d$ and the generalization to convex loss functions can be accomplished by \textit{linearization} strategy.  This reductions is  shown in Figure~\ref{fig:linearalization}. The main idea behind linearization trick is that if an algorithm $\mc{A}$  is able to achieve good regret against linear loss functions, then $\mc{A}$  can be used to achieve good regret against sequences of convex loss functions as well. To see this,  we note that from the definition of convexity, i.e,  $f_t(\w) \geq f_t(\w_t) + \dd{\nabla f_t(\w_t)}{\w - \w_t}$, we can feed the learner $\mc{A}$ with linear  loss functions   $f_t'(\w) = \dd{\nabla f_t(\w_t)}{\w}$ and then guarantee that $f_t(\w_t) - f_t(\w) \leq \dd{\f_t}{\w_t - \w}$. Therefore, any bound on the regret of linear loss functions directly translate to a regret bound for convex losses.

The FTRL algorithm at each round $t$ solves an offline optimization problem based based on the sum
of loss functions up to time $t-1$ and a regularization function.  Let  $\mc{R}(\w)$ be  a strongly convex differentiable function. The detailed steps of FTRL are as follows~\cite{interior-ieee-2012}:

\begin{figure}[H]
\begin{center}
\begin{myalg}{Follow the Regularized Leader}{FTRL}
{\bf Input:} \+ convex set $\W$, step size $\eta > 0$, regularization $\mc{R}(\w)$ \-\\ 
{\bf Initialize:} \+ $\w_1 \in \arg \min_{\w \in \W} \mc{R}(\w)$  \- \\ \\
{\bf for} $t=1,2,\ldots, T$  \+ \\
Play $  \w_{t+1} = \arg \min_{\w \in \W} \Bigg{[}\sum_{s = 1}^{t}{\dd{\mathbf{f}_s}{\w}} + \mc{R}(\w) \Bigg{]}\ $ \\
Receive loss function $f_{t+1}(\cdot)$ and incur $f_{t+1}(\w_{t+1})$ \\
Set $\f_{t+1} = \nabla f_{t+1}(\w_{t+1})$ \- \\
{\bf end for} 
\end{myalg}
\end{center}
\end{figure}
It is noticeable that what  FTRL implements at  each iteration is  the  regularized empirical risk minimization over the previous trials as we saw in the statistical learning. In online setting, the  regularizer has the role of forcing the consecutive solutions is to stay closer to each other, similar role role the perturbation plays in FTPL algorithm. Furthermore, different choices of the regularizer lead to different  algorithms. For instance, in the simplest case if we let $\mc{R}(\w) = \frac{1}{2} ||\w||^2$, then the FTRL  behaves same as OGD algorithm.

It is not hard to prove a simple bound on the regret of the FTRL algorithm for a given strongly convex regularization function $\mc{R}(\w)$ and the learning rate $\eta$. 

\begin{theorem}[Regret Bound for FTRL]\label{thm-chapter-2-ftrl}
Let $\f_1, \f_2, \ldots, \f_T \in \mathbb{R}^d$ be an arbitrary sequence of linear loss functions over convex set $\W \subseteq \mathbb{R}^d$. Let  $\mc{R}: \mc{K} \mapsto \mathbb{R}$ be a 1-strongly convex function with respect to a norm $\|\cdot\|$. Let  $R = \max_{\w \in \W} \mc{R}(\w) -\mc{R}(\w_1)$ and $\max_{t \in [T]}\|\f_t\|_* \leq G$. Then for any ${\w \in \W}$ by setting $\eta = \sqrt{\frac{R}{G^2T}}$ the regret of FTRL is bounded by:
\begin{eqnarray}
 \sum_{t=1}^{T}{\dd{\f_t}{\w_t}} -  \sum_{t = 1}^{T}{\dd{\f_t}{\w}} \leq \frac{\mc{R}(\w) - \mc{R}(\w_1)}{\eta} + \eta \sum_{t=1}^{T}{\|\f_t\|_*^2} = O( G \sqrt{ R T})
\end{eqnarray}  

\end {theorem}

\subsubsection{Online Mirror Descent}

Another algorithm for OCO problem is the online version of celebrated proximal point algorithm in offline convex optimization. As mentioned before, the implicit goal of regularization used in the FTRL algorithm is to control by how much the consecutive solutions differ from each other.  The proximal point algorithm is designed with the explicit goal of keeping  $\w_{t+1}$ as close as possible to $\w_t$. The closeness of two solutions are measured by the Bregman divergence induced by a strongly convex Legendre function $\Phi(\cdot)$ defined over convex domain $\mc{K}$.  A Legendre function is a strictly convex functions with continuous partial derivatives and gradient blowing up at the boundary of its domain (see Appendix~\ref{chap:appendix-convex} for detailed discussion). We assume that $\mathcal{W} \cap \mathcal{K} \neq \emptyset$. Online proximal point method solves an optimization problem which expresses the tradeoff between the distance from the old solution and the loss suffered by by the current convex function as:
\begin{eqnarray}\label{eqn:proximal}
\w_{t+1} = \arg \min_{\w \in \W \cap \mathcal{K}} \Bigg{[} \eta_t f_t(\w) + \mathsf{B}_{\Phi} (\w, \w_{t}) \Bigg{]}, 
\end{eqnarray}
where $\mathsf{B}_{\Phi} (\w, \w_{t})$ is the Bregman divergence induced by $\Phi(\cdot)$ (Definition~\ref{def-app-a-bregman}). When the loss functions are linear, i.e., $f_t(\w) = \dd{\mathbf{f}_t}{\w}$  for some $\mathbf{f}_t \in \mathbb{R}^d$, or one replaces the  objective $f_t(\w)$ in~(\ref{eqn:proximal}) with its linearized term,  i.e., $f_t(\w) \approx f_t(\w_t) + \dd{\w - \w_t}{\nabla f_t(\w_t)}$, the proximal point method becomes   the Online Mirror Descent (OMD) algorithm as detailed below:

\begin{figure}[H]
\begin{center}
\begin{myalg}{Online Mirror Descent}{OMD}
{\bf Input:} \+ convex sets $\W$, step size $\eta > 0$, Legendre function $\Phi: \mathcal{K}\mapsto\R$ \-\\ 
{\bf Initialize:} \+ $\w_1 \in \arg \min_{\w \in \mc{K}} \Phi(\w)$  \- \\ \\
{\bf for} $t=1,2,\ldots, T$  \+ \\
Play $\w_{t+1} = \arg \min_{\w \in \W \cap \mathcal{K}} \Bigg{[} \eta_t \dd{\w-\w_t}{\nabla f_t(\w_t)}  + \mathsf{B}_{\Phi} (\w, \w_{t}) \Bigg{]}$ \\
Receive loss function $f_{t+1}(\cdot)$ and incur $f_{t+1}(\w_{t+1})$ \- \\
{\bf end for} 
\end{myalg}
\end{center}
\end{figure}

\begin{theorem}[Regret Bound for OMD]\label{thm-chapter-2-omd}
Let $f_1, f_2, \ldots, f_T$ be an arbitrary sequence of convex, differentiable functions defined over the convex set $\W \subseteq \mathbb{R}^d$. Let $G = \max_t || \nabla f_t(\w_t)||_{*}$ and $R = \max_{\w, \w' \in \W} ||\w-\w'|| $. Let $\Phi: \mc{K} \mapsto \mathbb{R}$ be a Legendre function which is 1-strongly convex w.r.t the norm $\|\cdot\|$ and  $\mc{W}\cap \mc{K} \neq \emptyset$.  Then the regret of the OMD  can be bounded by
\begin{eqnarray}
 \sum_{t=1}^{T}{f_t(\w_t)} -  \sum_{t = 1}^{T}{f_t(\w)} \leq \frac{ \mathsf{B}_{\Phi} (\w, \w_{1})}{\eta} + \frac{\eta}{2 } \sum_{t=1}^{T}{{\|\nabla f_t(\w_t)\|^2_*}{}} = O(RG\sqrt{T}),
\end{eqnarray}  
 where $\|\cdot \|_{*} $ is the dual norm to $\|\cdot\|$.
\end{theorem}
We note that many classical online learning algorithms can be viewed as   variants of OMD, generally either with the Euclidean geometry such as Perceptron algorithm and OGD, or in the simplex geometry, using an entropic distance generating function such as Winnow~\cite{littlestone1988learning} and Online Exponentiated Gradient algorithm~\cite{kivinen1997exponentiated}.

\subsubsection{Online Newton Step}

As mentioned in the analysis of OGD algorithm, it  attains  a logarithmic regret $O(\log{T})$ if the sequence of loss functions which have bounded gradient and are  strongly convex. Another case in which we can obtain logarithmic regret is the case of exp-concave loss functions (i.e., the function $\exp(-\alpha f(\w))$ is concave for some $\alpha$). Exp-concavity is weaker condition than the bounded gradient and strong convexity.  Online Newton Step (ONS)~\cite{hazan-log-newton} is the adaption of   Newton method for convex optimization to online setting and  is able to achieve logarithmic regret when learned on exp-concave functions which makes is more general than OGD.

\begin{figure}[H]
\begin{center}
\begin{myalg}{Online Newton Step}{ONS}
{\bf Input:} \+ convex sets $\W$, step size $\eta > 0$ \-\\ 
{\bf Initialize:} \+ $\w_1 \in \arg \min_{\w \in \mc{K}} \Phi(\w)$  \- \\ \\
{\bf for} $t=1,2,\ldots, T$  \+ \\
Play  $\w_{t+1} = {\Pi}_{\W}^{\mathbf{A}_{t}} \Bigg{(} \w_{t} - \eta_t \mathbf{A}_{t}^{-1} [\nabla f_t(\w_{t})] \Bigg{)} $ where $\mathbf {A}_{t-1} = \sum_{s=1}^{t-1}{\nabla f_s(\w_s) \nabla f_s(\w_s)^{\top} }$ \\
Receive loss function $f_{t+1}(\cdot)$ and incur $f_{t+1}(\w_{t+1})$ \- \\
{\bf end for} 
\end{myalg}
\end{center}
\end{figure}
Here $\Pi_{\W}^{\mathbf {A} } (\cdot)$ is the projection induced by  the a matrix $\mathbf {A}$, i.e., $\Pi_{\W}^{\mathbf {A}}(\w) = \min_{\mathbf{z} \in \W} (\mathbf{z}-\w)^{\top} \mathbf{A} (\mathbf{z}-\w)$. We note that compared to the provisos algorithms which only exploit first-order information about the loss functions, the \textit{analysis} of ONS method is based on second order information, i.e. the second derivatives of the loss functions, whereas the implementation of ONS relies only on first-order information.  

The following result shows that the ONS method achieves a logarithmic regret for exp-concave functions which is stronger result than the performance of OGD for strongly convex losses. 
\begin{theorem}[Regret Bound for ONS]\label{thm-chapter-2-ons}
Let $f_1, f_2, \ldots, f_T$ be an arbitrary sequence of $\alpha$-exp-concave functions defined over the convex set $\W \subseteq \mathbb{R}^d$ with $G = \max_t || \nabla f_t(\w_t)||$. Let $R = \max_{\w, \w' \in \W} ||\w-\w'||$.  Then the ONS algorithm achieves the following regret bound for any $\w \in \W$:
\begin{eqnarray}
 \sum_{t=1}^{T}{f_t(\w_t)} - \sum_{t = 1}^{T}{f_t(\w)} \leq \Big{(} \frac{1}{\alpha} + GR\Big{)} d \log T = O(d \log T)
\end{eqnarray}  
\end{theorem}

We note that also this result  seems interesting, but for some of the functions of interest such as logistic loss the exp-cancavity parameter $\alpha$ could be exponentially large in $d$~\cite{open-problem-2012}. The exponential dependence on the diameter of the feasible set can make this bound worse than the $O(\sqrt{T})$ bound obtained by OGD.
%%%%%%%%%%%%%%%%%%%%%%%%%%%%%%%%%%%%%%%%%%%%%%%%%%%%
\subsection{Variational Regret Bounds}

Most previous works, including those discussed above, considered the most general setting in which the loss functions could be arbitrary and possibly chosen in an adversarial way.  However, the environments around us may not always be adversarial, and the loss functions may have some patterns which can be considered to  achieve a smaller regret.  Consequently, it is objected that requiring an algorithm to have a small regret for all sequences leads to results that are too loose to be practically interesting.  As a result,  the bounds obtained for worst case scenarios become pessimistic for these regular sequences. Therefore,  it would be  desirable to develop algorithms that yield tighter bounds for more regular sequences, while still providing protection against worst case sequences. To this end, we need to replace the number of rounds  appeared in the regret bound with some other notion of performance. In particular, this new measure should depend on variation in the sequence of costs functions emitted to the learner. Having such an algorithm guarantees that  if the cost sequence has low variation, the algorithm would be able to perform better.

One work along this direction is that of \cite{Hazan-2008-extract}. For the online linear optimization problem, in which each loss function is linear $f_t(\w) = \dd{\mathbf{f}_t}{\w}$ and can be seen as a vector, they considered the case in which the loss functions have a small variation, defined as 

$$\text{Variation}(\mathcal{A}, \W, \mathcal{F}, T)=\sum_{t=1}^T \|\mathbf{f}_t -\boldsymbol \mu\|_2^2,$$ 

where $\boldsymbol \mu=\sum_{t=1}^T \mathbf{f}_t{} /T$ is the average of the loss functions.  For this, they showed that a regret of $O(\sqrt{\text{Variation}})$ can be achieved, and they also have an analogous result for the prediction with expert advice problem.  According to this definition, a small $\text{Variation}(\mathcal{A}, \W, \mathcal{F}, T)$ means that most of the loss functions center around some fixed loss function $\boldsymbol \mu$. This seems to model a stationary environment, in which all of  the loss functions are produced according to some fixed distribution.

The variation bound is defined in terms of total difference between individual linear cost vectors to their mean. In Chapter~\ref{chap:regret} of this thesis, we introduce another measure which is called \textit{gradual variation}. Gradual variation is more general and applies to environments which may be evolving but is a somewhat gradual way.  For example, the weather condition or the stock price at one moment may have some correlation with the next and their difference is usually small, while abrupt changes only occur sporadically.  Formally, the gradual 
variation of a sequence of loss functions is defined as:
%%%%%%%%%%%%%%%%%%%%%%%%%%%%%%%%%%%%%%%%%%
\begin{eqnarray*}
    \text{GradualVarition} (\mathcal{A}, \W, \mathcal{F}, T) = \sum_{t=1}^{T-1} \max\limits_{\w \in \W} \|\nabla f_{t+1}(\w) - \nabla f_{t}(\w)\|_2^2. 
\end{eqnarray*}
It is easy to verify that the gradual variation lower bounds the variation bound and hence algorithms with regret bounded by gradual variation are more adaptive to regular patterns  than algorithm with bounded variational regret bounds.

%%%%%%%%%%%%%%%%%%%%%%%%%%%%%%%%%%%%%%%%%%%%%%%%%%%
\subsection{Bandit Online Convex Optimization}

In bandit OCO, once the online learner commits to  the decision $\w_t$ at round $t$, he does not have access to the function $f_t(\cdot)$ chosen by adversary and  instead  receives the scalar loss $f_t(\w_t)$ he suffers at point $\w_t$.  In the optimization community, this problem  usually known as zeroth-order or derivative-free convex optimization as  we only have access to function  values   to solve the optimization problem~\cite{jamieson2012query,shamir2012complexity}.  A simple approach for bandit OCO which was the main dilemma in most existing works is to utilize a reduction to the full information OCO setting. To do so,  one needs to approximate the gradient of the loss functions at each iteration based on the observed scalar loss and feed it to the full information algorithm. This reduction has been illustrated in Figure~\ref{fig:bandit-reduction}.
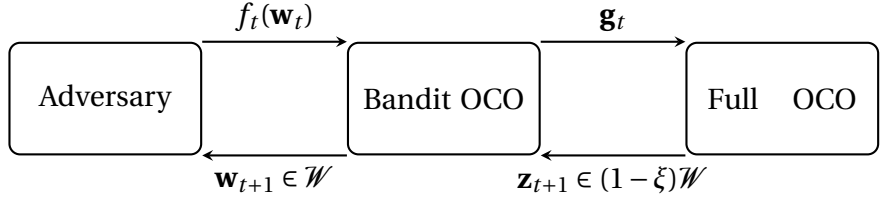
\begin{figure}[t]
 \begin{center}
  \begin{tikzpicture}[node distance=4.5cm, auto, >=stealth]
   % nodes
   \node[block] (a)               {Adversary};
   \node[block] (b)  [right of=a] {Bandit OCO};
   \node[block] (d)  [right of=b] {Full $\;\;\;\;$OCO};
   
   % edges
   \draw[thick,->] (a.north east) -- node[above]{$f_t(\w_t)$} (b.north west);
   \draw[thick,->] (b.north east) -- node[above]{$\g_t$} (d.north west);
   \draw[thick,->] (d.south west) -- node[below]{$\z_{t+1} \in (1-\xi) \mc{W}$} (b.south east);
   \draw[thick,->] (b.south west) -- node[below]{$\w_{t+1} \in \mc{W}$} (a.south east);
  \end{tikzpicture}
  \label{flowchart}
 \end{center}
\caption[Reduction of bandit OCO  to full information OCO.]{The reduction of bandit online convex optimization to online convex optimization with full information. The full OCO needs to play from a shrinked domain $(1-\xi)\W$ to ensure that the sampled points belong to the domain.}
\label{fig:bandit-reduction}
\end{figure}

A simple idea to estimate the gradients has been  utilized in~\cite{flaxman-2005-online} and  a modified gradient descent approach for bandit OCO has been presented that attains $O(T^{3/4})$ regret bound. The key idea of their algorithm is to compute the stochastic approximation of the gradient of cost functions by single point evaluation of the cost functions. The main observation was that  one can estimate the gradient of a function $f(\w)$ by taking a random  vector $\u$ from unit sphere $\mathbb{S} = \{\x \in \mathbb{R}^d ; \|\x\| = 1\}$ and scaling it by $f(\w+\delta \u)$, i.e.~$\hat{\g} = f(\w+\delta \u)\u$. Then $\E[\hat{\g}]$ is proportional to the gradient of a {\em smoothed} version of $f(\cdot)$ defined as $\hat{f}(\w) = \E_{\v \in \B}[f(\w+\delta \v)]$ for $\v$ random from the unit ball $ \mathbb{B} = \{\x \in \mathbb{R}^d ; \|\x\|\leq 1\}$. To ensure that the sampled query  points belong the domain $\W$, the OGD is run over a shrinked domain $(1-\xi)\W$ where we further assume that  $r \mathbb{B} \subseteq \W \subseteq R \mathbb{B}$.

\begin{figure}[H]
\begin{center}
\begin{myalg}{Expected Online Gradient Descent}{EOGD}
{\bf Input:} \+ convex set $\W$, step size $\eta > 0$,  $\delta$, $r$, and $\xi$ \-\\ 
{\bf Initialize:} \+ $\z_0 = \mathbf{0}$  \- \\ \\
{\bf for} $t=1,2,\ldots, T$  \+ \\
Pick a random unit vector $\u_t$ uniformly at random \\
Play $\w_t = \z_t + \delta \u_t$  and observe  $f_t(\w_t)$ \\
Update  $\z_{t+1} = \Pi_{(1-\xi)\W} \Big{(}  \z_t - \eta f_t(\w_t) \u_t\Big{)}$ \- \\
{\bf end for} 
\end{myalg}
\end{center}
\end{figure}

\begin{theorem}[Regret Bound for EOGD]\label{thm-chapter-2-eogd}
Let $f_1,f_2,\ldots,f_T$ be a sequence of convex, differentiable functions defined over the convex domain $\W \subseteq \mathbb{R}^d$ where $r \mathbb{B} \subseteq \W \subseteq R\mathbb{B}$.   Let $\g_1,\dots,\g_n$ are vector-valued random variables with $\E[\g_t | \w_t] = \E[ \nabla f_t(\w_t)]$ and $\|\g_t\|\leq G$, for some $G>0$. Then, for $\eta = \frac{R}{G\sqrt{n}}$, $\xi = \delta/r$, and $\delta = T^{-1/4}\sqrt{\frac{RdGr}{3(rG+C)}}$
     $$
     \mathbf{E}\bigg[\sum_{t=1}^T f_t(\w_t)\bigg]- \min_{\w\in \W} \sum_{t=1}^T f_t(\w)
     \leq O(T^{3/4})
     $$
\end{theorem}
This regret bound is later improved to $O(T^{2/3})$~\cite{Awerbuch:2004:ARE:1007352.1007367} for online bandit linear optimization. More recently,~\cite{DaniHK07} proposed an inefficient algorithm for online bandit linear optimization with the optimal regret bound $O(poly(d)\sqrt{T})$ based on multi-armed bandit algorithm. The key disadvantage of~\cite{DaniHK07} is that it is not computationally efficient. Abernethy, Hazan, and Rakhlin~\cite{AbernethyHR08} presented an efficient randomized algorithm with an optimal regret bound $O(poly(d)\sqrt{T})$ that exploits the properties of self-concordant barrier regularization.

For general online bandit convex optimization,~\cite{agarwal-2010-optimal} proposed optimal algorithms in a multi-point bandit setting, in which multiple points can be queried for the cost values.  With multiple queries, they showed that the EOGD algorithm can give an $O(\sqrt{T})$ expected regret bound.  The key idea of multiple point bandit online convex optimization, proposed in~\cite{agarwal-2010-optimal}, is to approximate the gradient using two point function evaluations. More specifically, at each iteration $t$ we randomly choose a unit direction $\u_t$ and measure the function values at points $\w_t + \delta \u_t$ and $\w -\delta \u_t$, i.e. $\ell_t(\w_t + \delta \u_t)$ and $\ell_t(\w_t - \delta \u_t)$, where $\delta > 0$ is a small perturbation that is $O(1/T)$. Given two point function evaluation, we approximate the gradient $\nabla f_t(\w_t)$ by   $ \g_t = \frac{d}{2\delta}\left(f_t(\w_t + \delta \u_{t}) - f_t(\w_t - \delta \u_{t}) \right)\u_t$. The nice property of  this sampling strategy is that the norm of sampled gradient is no longer dependent on $\delta$, i.e., $\|\g_t\| \leq dG$, and yield the same regret bound as OGD.

%%%%%%%%%%%%%%%%%%%%%%%%%%%%%%%%%%%%%%%%%%%%%%%%%%%
\subsection{From Regret  to Risk Bounds} 

So far we have dealt with two different models for learning: statistical and sequential or online.  The online setting  is in stark contrast to the  statistical  setting  in few aspects.  First, in  statistical learning or  the batch model  there is a strict division between the training phase and the testing phase. In contrast,  in the online model, training and testing occur together all at the same time since every example acts both as a test of what was learned in the past and as a training example for improving our
predictions in the future.  This requires that the online learner to be adaptive to the environment.  

A second key difference relates to the generation of examples.  Statistical learning scenario  follows the key assumption that the distribution over data points is fixed over time, both for training and test
points,  and samples are assumed to be drawn i.i.d. from an underlying distribution $\D$. Furthermore, the goal is to learn a hypothesis with a small expected loss or generalization
error. In contrast,  in online setting there is no notion of generalization and  algorithms are measured using a mistake bound model and the notion of regret, which are based on worst-case or
adversarial assumption where the adversary deliberately trying to ruin the learnerÕs performance.

Finally, we  distinguish between the processing model of statistical and online learning settings. 
Online algorithms  process one sample at a time and can thus be significantly
more efficient both in time and space and more practical than batch algorithms,
when processing modern data sets of several million or billion points. hence, these algorithms are more suitable for large scale learning. This stands in contrast to statistical learning algorithms such as ERM and it would be tempting to switch to online learning algorithm.

Given the close relationship between these two settings and clear advantage of online learning from computational viewpoint,   a paramount question is "\textit{whether or not algorithms developed in the sequential setting can be used for statistical learning with guaranteed generalization bound?}".  More precisely, can we devise algorithms that exhibits the desirable characteristics of online learning  but also has good generalization properties. Since a regret bound holds for all sequences of training samples, it also holds for an i.i.d. sequence. What remains to be done is to  extract a single hypothesis  out of the sequence produced by the sequential method, and to convert the regret guarantee into a guarantee about generalization. Such a process has been dubbed an \texttt{Online-to-Batch Conversion} and foreshadows a key achievement, which is any online  learning algorithm with sub linear regret can be converted into a batch algorithm.  Here  we introduce two methods  to convert  an online algorithm that attains low regret  into a batch learning algorithm that attains low risk. Such online-to-batch conversions are interesting both from the practical and the theoretical perspectives~\cite{gene-Cesa-BianchiCG04,DekelS05,kakade-2008-strong}.

Formally, let $f_1, f_2, \cdots, f_T$ be an i.i.d sequence of loss functions, $f_t: \mc{W} \mapsto \R$. In statistical setting one can think of each loss function $f_t(\w)$ as $f_t(\w) = \ell(\w, (\x_t,y_t))$ for a fixed loss function $\ell: \W \times \Xi \mapsto \R_{+}$ and a random instance $(\x_t,y_t) \in \Xi$   sampled following  the underlying distribution $\D$. We feed these loss functions to an online learning algorithm $\mc{A}$ and assume that the online learner produces a sequence $\w_1, \w_2, \cdots, \w_T \in \W$ of hypothesis. The goal is to construct a single hypothesis $\hat{\w} \in \W$ with small generalization error.  Here we consider two solutions for this problem: \textit{randomized conversion} and \textit{averaging}. 

The simplest conversion scheme is to  simply choose a random  hypothesis uniformly at random from the sequence of hypothesis $\w_1, \w_2, \cdots, \w_T$. At the first glance this idea seems naive, but it has few desirable properties. First, the average loss of the algorithm is an unbiased estimate of the expected risk of $\hat{\w}$, i.e.,  $\mathbf{E} [\ell(\hat{\w})] = (1/T) \sum_{t=1}^{T}{f_t(\w_t)}$. Second, the conversion is applicable regardless of any convexity assumption. Finally, in expectation, the excess loss of $\hat{\w}$ is upper bounded by the average per-round regret of online learner, i.e., 
\[ \mc{L}_{\D}(\hat{\w}) - \min_{\w \in \W} \mc{L}_{\D}(\w) \leq \frac{\mathbf{E}\big{[}\text{Regret}(\mc{A}, \W, \mc{F},T)\big{]}}{T}.\]
%where $\text{Regret}(\mc{A}, \W, T) = \sum_{t=1}^{T}{f_t(\w_t)} - \min_{\w \in \W} \sum_{t=1}^{T}{f_t(\w)}$.

An alternative solution which is only applicable to learning from convex loss functions over convex hypothesis spaces, is to output the average solution $\hat{\w} = (1/T) \sum_{t=1}^{T}{\w_t}$.  This conversion is also enjoys the same properties as the randomized conversion with an additional important feature. That is, we can able to show high-probability bounds on the excess risk provided that loss functions are bounded.

%%%%%%%%%%%%%%%%%%%%%%%%%%%%%%%%%%%%%%%%%%%%%%%%%%
\section{Convex Optimization}

A generic convex optimization problem may be written as
$$ \min f(\w) \quad \text{subject to} \quad \w \in \W,$$ 
where  $f: \R^d \mapsto \R$ chosen from a specific family of functions $\mc{F}$  is a proper convex function, and $\W \subseteq \mathbb{R}^d$ is  nonempty, compact, and convex set  which is  also called the constraint or feasible set.  We denote by $\w_*$ the optimal solution to above problem and assume that it exists, i.e., $\w_* = \arg \min_{\w \in \W}f(\w)$. Ideally, the goal of an optimization algorithm is to
compute the optimal solution, but almost always it is impossible to compute an exact $\w_*$ in finite time. hence, we turn to find an $\epsilon$-approximate solution.  A solution $\w \in \W$ is an  $\epsilon$ sun-optimal if  $f(\w) - \min_{\w' \in \W} f(\w') \leq \epsilon$.  

For a given  family  $\mc{F}$ of convex functions over the feasible set $\W$, our primary focus is to determine  the efficiency of an optimization procedure to produce  sub-optimal solutions.  To analyze the efficiency of convex optimization algorithm one typically follows the   oracle model of optimization which lies in the heart of the complexity theory of convex optimization~\cite{nemircomp1983,nesterov2004introductory}.

\subsection{Oracle Complexity of Optimization}

 A typical convex optimization procedure initially picks some point in the feasible convex set $\W$ and iteratively updates these points based on some local information about the function it calculates around these successive points.  The method can decide which points to query at based on the
results of earlier queries, and tries to use as few queries as possible to achieve its task.  The crucial question 
we are interested to answer about a specific optimization problem is the number of queries the algorithm makes to find an $\epsilon$-accurate solution. The oracle complexity is a general model to analyze the computational complexity of optimization algorithms.

In the oracle  model, there is an oracle $\mc{O}$ and an information set $\mc{I}$. The oracle $\mc{O}$ is simply a function  $\psi: \W \mapsto \mc{I}$ that for any query point $\w \in \W$ returns an output from $\mc{I}$.  The information set provided to the algorithms varies depending on the type of the oracle. In particular,  a zero-order oracle returns $f(\w)$ for a  given query ${\w \in \mathcal{W}}$,  first-order oracle returns gradient $\mc{I} = \{\nabla f(\w)\}$ (respectively a sub-gradient  $\mc{I} = \{\mathbf{g} \in \partial f(\w)\}$ if the function is not differentiable), and a second-order oracle return the Hessian at the queried point. We also distinguish between noisy (or stochastic) and exact (or deterministic) oracle models. In the noisy oracle model, the information returned by the oracle are   corrupted with zero-mean noise with bounded variance. 

The algorithm iteratively updates the solution based on the information accumulated in previous iterations. In particular, in optimization with zero and first-order exact oracle model which is the main focus of large scale optimization methods, an optimization method updates the solution using $\w_t = \phi_t(\w_0, \hdots, \w_{t-1}, \nabla f(\w_0), \hdots, \nabla f(\w_{t-1}), f(\w_0), \hdots, f(\w_{t-1}))$ where $\phi_t: \times \cup_{s=1}^{t}{\mc{I}_s} \mapsto \W$ is updating mechanism utilized by the optimization algorithm at iteration $t$ to determine the next query point $\w_t$. Roughly speaking, we measure complexity of an algorithm by the number of queries that it makes to a prescribed oracle for computing the final solution.

Given a positive integer $T$ corresponding to the number of iterations,  the minimax oracle optimization error after $T$ steps, over a set of functions ${\mathcal{F}}$, is defined as follows:

\[\displaystyle \mathrm{OracleComplexity}(\mathcal{F}, \W, \mc{O}, T) = \inf_{\psi} \sup_{f \in \mathcal{F}} \left( f(\w_T) - \inf_{\w \in \mathcal{W}} f(\w) \right).\]

In other words, the minimax oracle complexity   is the best possible rate of convergence (as a function of the number of queries) for the optimization error when one restricts to black-box procedures in order to guarantee delivering an $\epsilon$-accurate solution to any function $f \in \mc{F}$.
%-------------------------------------------------------------------------------------
\begin{table}[t]
\begin{center}
\begin{tabular}{lcccc} 
Oracle & Lipschitz & Lipschitz \& Strongly Convex  & Smooth & Smooth \& Strongly Convex \\ \hline \\ Deterministic & $\frac{\rho}{\epsilon^2}$ & $\frac{\rho^2}{\alpha \epsilon}$ &  $\frac{\beta}{\sqrt{\epsilon}}$& $\sqrt{\kappa} \log \frac{\alpha}{\epsilon}$ \\ 
Stochastic & $\frac{\rho}{\epsilon^2}$ &  $\frac{\rho^2}{\alpha^2 \epsilon}$ & $\frac{\beta}{\epsilon}+\frac{\rho}{\epsilon^2}$ & $ \sqrt{\kappa} \log \left( \frac{\beta}{\epsilon}\right) + \frac{1}{\alpha \epsilon}$  \\ \\ \hline
\end{tabular}
\caption[Lower bound on the oracle complexity of first-order optimization methods.]{Lower bound on the oracle complexity for stochastic/deterministic first-order optimization methods. Here $\rho$, $\alpha$, and $\beta$ are the Lipschitzness,  strong convexity, and smoothness parameters, respectively. The parameter $\kappa$ is the condition number of function and is defined as $\kappa = \beta/\alpha$.}
\label{table:lower}
\end{center}
\end{table}

%-------------------------------------------------------------------------------------
A large body of literature is devoted to obtaining rates of convergence of specific procedures for
various   set of convex functions ${\mathcal{F}}$ of interest (essentially smooth/non-smooth, and strongly convex/non-strongly convex) and different types of oracles (essentially noisy or stochastic/deterministic or exact, zero order or derivative free, first order, and second order).  The oracle complexity of first-order deterministic and stochastic  oracle  models   are summarized in Table~\ref{table:lower} for different family of loss functions elicited from~\cite{nesterov2004introductory} for deterministic and from~\cite{sgd-lower-bounds,nemircomp1983} for stochastic optimization. The algorithms which attain these lower bounds will be discussed later. 
%%%%%%%%%%%%%%%%%%%%%%%%%%%%%%%%%%%%%%%%%%%%%%%%%%
%%%%%%%%%%%%%%%%%%%%%%%%%%%%%%%%%%%%%%%%%%%%%%%%%%
\subsection{Deterministic Convex Optimization}
Here we briefly review  the optimization algorithms in the \textit{first-order}  oracle model which are called gradient based methods for simplicity.  More precisely, we assume  that the only information the optimization methods can learn about the particular problem instance is the values and derivatives of these components $(f(\w), \nabla f(\w))$ at query points $\w \in \W$.  Recently, first-order methods have experienced a renaissance in the design of fast algorithms for large-scale optimization problems. This is due the fact that although higher order methods such as  interior point methods~\cite{nemirovski2008interior} have linear convergence rate, but this  fast rate  comes at the cost of more expensive
iterations, typically requiring the solution of a system of linear equations in the input variables. Consequently, the cost of each iteration typically grows at least quadratically with the problem
dimension, making interior point methods impractical for very-large-scale convex programs.
 
The convergence rate of gradient based methods usually depends on the properties of the objective function to be optimized. When the objective function is strongly convex and smooth, it is well known that  gradient descent methods can achieve a geometric convergence rate~\cite{boyd-convex-opt}. When the objective function is  smooth but not strongly convex, the optimal convergence rate of a gradient descent method is  $O(1/T^2)$, and is achieved by the Nesterov's methods~\cite{RePEc:cor:louvco:2007076}. For  the objective function which is strongly convex but not smooth, the convergence rate becomes $O(1/T)$~\cite{Shalev-Shwartz:2007:PPE:1273496.1273598}. For general non-smooth objective functions, the optimal rate of any first order method is $O(1/\sqrt{T})$. Although it is not improvable in general, recent studies are able to improve this rate to $O(1/T)$ by exploring the special structure of the objective function~\cite{nesterov2005smooth,nesterov2005excessive}. In addition, several methods are developed for composite optimization, where the objective function is written as a sum of a smooth and a non-smooth function~\cite{lan2012optimal,RePEc:cor:louvco:2007076,lin2010smoothing}. The proof of coming results can be found in~\cite{nesterov2004introductory} and in the reference papers.

\subsubsection{Gradient Descent Method}
Perhaps the simplest and most intuitive algorithm for deterministic optimization is gradient decent (GD) method which which was proposed by  Cauchy in 1846~\cite{cauchy1847methode}~\footnote{The original Cauchy's algorithm  uses the direction that descends most and the best step-size which convergences slowly. Afterwards,  a lot of researches have been done on how to choose the step-size for more efficient algorithms~\cite{greenstadt1967relative,barzilai1988two}}.  To find a solution within the domain $\W$ that optimizes the given objective function $f(\w)$, GD computes  the gradient of $f(\w)$ by querying  a first-order deterministic oracle, and updates the solution   by moving it in the opposite direction of the gradient. To ensure that the solution stays within the domain $\W$, GD has to project the updated solution back into the $\W$ at  every iteration.

\begin{figure}[H]
\begin{center}
\begin{myalg}{Projected Gradient Descent}{GD}
{\bf Input:} \+ convex set $\W$, $\eta > 0$, function $f \in \mc{F}$, first-order oracle $\mc{O}$ \-\\ 
{\bf Initialize:} \+ $\w_1 \in \W$  \- \\ \\
{\bf for} $t=1,2,\ldots, T$  \+ \\
Query the oracle $\mc{O}$ at point $\w_t$ to get $\nabla f(\w_t)$\\
Update $\w_{t+1} = \Pi_{\W} \left(\w_t - \eta \nabla f(\w_t)\right)$ \- \\
{\bf end for} 
\end{myalg}
\end{center}
\end{figure}

\begin{theorem}[Convergence Rate of GD]\label{thm-chapter-2-gd}
 Assume that $f\in\mathcal{F}$ be a convex function defined over the convex domain $\W \subseteq \mathbb{R}^d$. Let $\w_* = \arg \min_{\w \in \W}f(\w)$ be the optimal solution. Then, for the convergence rate of  GD algorithm

\begin{itemize}
\item if $f$ be   $\rho$-Lipschitz, by setting $\eta = \frac{R}{\rho \sqrt{t}}$   we have
  \[f \left(\frac{1}{T}\sum_{t=1}^{T}{\w_t} \right) - f(\w) \leq \frac{\rho \|\w_* - \w_1\|}{\sqrt{T}} .\]
%%%%%%%%%%%%%%%%%%%%%%%%%%%%%%%%%%%
\item if $f$ be $\beta$-smooth by setting $\eta = \frac{1}{\beta}$ we have
$$f(\w_T) - f(\w_*) \leq \frac{2 \beta \| \w_* - \w_1 \|^2}{T} .$$
\item if $f$ be $\beta$-smooth and $\alpha$-strongly convex, and  $\kappa = \frac{\beta}{\alpha}$ be the condition number of $f$, by setting $\eta = \frac{2}{\alpha + \beta}$ we have
 $$f(\w_T) - f(\w_*) \leq \frac{\beta}{2} \|\w_1 - \w_*\|^2 \left(\frac{\kappa - 1}{\kappa+1}\right)^{T}  .$$
\end{itemize}
\end{theorem}
By comparing the rates obtained in Theorem~\ref{thm-chapter-2-gd} to the lower bounds in Table~\ref{table:lower}, one can realize that the GD obtains the optimal bound only  for Lipschitz functions. We also note that by examining the bounds it turn out that the GD method is independent of the dimension of the convex domain $\W$ as long as the Euclidean norm of solutions and gradients are independent of the ambient dimension of convex domain $\W$ which makes it attractive for optimization in high dimension. The dependency on the condition number for smooth and strongly convex functions makes the GD method inappropriate for  learning problems as the condition number usually depends on the regularization parameter, leading to huge number of accesses to full gradient oracle despite its linear convergences in terms of target accuracy $\epsilon$. We will resolve this issue in Chapter~\ref{chap:mixed-strong}. The computational bottleneck of the projected GD is often the projection step which is a convex optimization problem by itself and might be expensive for many domains. In Chapter~\ref{chap:projection} we propose efficient optimization methods which do not require intermediate projection steps.

\subsubsection{Accelerated Gradient Descent Method}\label{chp-2-sub-acc-gd}
The convergence rate of GD method for optimization smooth loss functions is $O(1/T)$ which is far away from  the lower bound $O(1/T^2)$ discussed before.  Nesterov showed in 1983  that we can improve the convergence rate of GD without using anything more than gradient information at various points of the domain. Accelerated GD~\cite{nesterov1983method,nesterov2004introductory,nesterov2005smooth} bridges the gap between the lower bound for smooth optimization and lower bound provided by oracle complexity with a simple twist of GD method and is able to obtain the optimal $O(1/T^2)$ convergence rate  for minimizing smooth functions.% and achieves an $O(1/T^2)$ convergence rate.

\begin{figure}[H]
\begin{center}
\begin{myalg}{Accelerated Gradient Descent}{AGD}
{\bf Input:} \+ $\eta > 0$, function $f \in \mc{F}$, first-order oracle $\mc{O}$ \-\\ 
{\bf Initialize:} \+ $\w_0 = \z_0 = \mathbf{0}$, $\lambda_0 = 0$  \- \\ \\
{\bf for} $t=1,2,\ldots, T$  \+ \\
Query the oracle $\mc{O}$ at point $\w_t$ to get $\nabla f(\w_t)$\\
Set  $\eta_{s} = \frac{1}{2} \left({1 + \sqrt{1+ 4 \eta_{t-1}^2}}\right), \ \text{and} \ \gamma_t = \frac{1 - \eta_t}{\eta_{t+1}}.$ \\
Update $\z_{t+1}  = \z_t - \frac{1}{\beta} \nabla f(\w_t)$ \\
Update $\w_{t+1}  = (1-\gamma_t)\z_{t+1} +  \gamma_t  \z_t$
\- \\
{\bf end for} 
\end{myalg}
\end{center}
\end{figure}

The following theorem shows that AGD achieves an $O(1/T^2)$ convergence rate which is tight.

\begin{theorem}[Convergence Rate of AGD]\label{thm-chapter-2-agd} 
Let $f \in \mathcal{F}$ be a convex and $\beta$-smooth function and $\w_*$ be the optimal solution. Then the  accelerated gradient descent   outputs a solution which satisfies:
$$f(\z_T) - f(\w_*) \leq \frac{2 \beta \|\w_1 - \w_*\|^2}{T^2} .$$
\end{theorem}

%%%%%%%%%%%%%%%%%%%%%%%%%%%%%%%%%%%%%%%%%%%%%%%%%%%%
\subsubsection{Mirror  Descent Method}
Mirror Descent  (MD) is a first-order optimization procedure which generalizes the classic GD method  to non-Euclidean geometries by relying on a distance generating function specific to the geometry. The original MD algorithm was developed to perform the gradient descent in spaces where the gradient only makes sense in the dual space. In this cases, the MD first maps the point $\w_t$  into a dual space by mapping $\Phi$, then performs the gradient update in the dual space, and finally maps  the resulting point back to the primal space. When the mapping $\Phi(\w) = \frac{1}{2}\|\w\|^2$ then the primal and dual spaces are same and the MD performs a simple gradient descent.

\begin{figure}[H]
\begin{center}
\begin{myalg}{Mirror Descent}{MD}
{\bf Input:} \+ $\eta > 0$, function $f \in \mc{F}$, first-order oracle $\mc{O}$ \-\\ 
{\bf Initialize:} \+ $\w_0 =  \mathbf{0}$  \- \\ \\
{\bf for} $t=1,2,\ldots, T$  \+ \\
Query the oracle $\mc{O}$ at point $\w_t$ to get $\nabla f(\w_t)$\\
Update $\nabla \Phi(\z_{t+1}) = \nabla \Phi(\w_{t}) - \eta \nabla f(\w_t)$ \\
Update $ \w_{t+1}  = \mathrm{argmin}_{\w \in \mathcal{W} \cap \mathcal{K}} \mathsf{B}_{\Phi}(\w, \z_{t+1})$
\- \\
{\bf end for} 
\end{myalg}
\end{center}
\end{figure}

\begin{theorem}[Convergence Rate of MD]\label{thm-chapter-2-md}
Let $\Phi$ be a mirror map. Assume also that $\Phi$ is $\alpha$-strongly convex on $\mathcal{W} \cap \K$ with respect to $\|\cdot\|$. Let $R = \sup_{\w \in \mathcal{K} \cap \mathcal{W}} \Phi(\w) - \Phi(\w_1)$ and $f$ be convex and $\rho$-Lipschitz w.r.t. $\|\cdot\|$, then MD algorithm with $\eta = \frac{\rho}{R} \sqrt{\frac{2 \alpha}{T}}$ satisfies

    $$f\bigg(\frac{1}{T} \sum_{t=1}^T \w_t \bigg) - \min_{\w \in \mathcal{W}} f(\w) \leq \rho R \sqrt{\frac{2}{\alpha T}} .$$
\end{theorem}

The MD algorithm can alternatively be expressed as nonlinear projected sub-gradient type
method, derived from  a general distance generating function (Bregmen divergence in Definition~\ref{def-2-bregmen}) instead of the usual Euclidean squared
distance  as~\cite{mirror-beck-2003}:
\begin{equation} 
\w_{t+1} =\min_{\w \in \W} \left\{ \langle \w, \nabla f(\w_t) \rangle + \frac{1}{\eta} \mathsf{B}_{\Phi}(\w, \w_t) \right\}.
\end{equation}
\begin{remark}
In terms of convergence rate, the MD obtains the same rate as GD method but MD has advantage by exploiting the geometry of convex domain. More specifically,  since MD method adapts to the structure  of domain $\W$ via mapping $\Phi$,  it has less dependency on the dimensionality of the domain which could be appealing for large scale optimization problems. As an example, it is easy to verify  that  for optimization over simplex, i.e., $\Delta = \{\w  \in \R_{++}^d: \sum_{i}^{}{w_i} = 1\}$ by using the negative entropy  $\Phi(\w) = \sum_{i=1}^{d}{\log w_i}$ as the mapping function, the dependency of MD to $d$ is in order of $\log d$, while  regular GD algorithm has a linear $O(d)$ dependency.
\end{remark}

%%%%%%%%%%%%%%%%%%%%%%%%%%%%%%%%%%%%%%%%%%%%%%%%%%%%
\subsubsection{Mirror Prox Method}

In the black-box oracle model  the algorithm has access to  the values
and gradients of function, without knowing the structure of the objective function.  But in many circumstances  we never meet a pure black box model and  have some information  about the structure of the underlying function. Intestinally,  the proper use of the structure of the problem can  help to obtain better convergence rate for specific family of loss functions~\cite{nesterov2005excessive,nesterov2005smooth,Nemirovski2005}. In particular, in~\cite{Nemirovski2005} it been shown that for non-smooth Lipschitz continuous functions which admit a smooth saddle-point representation one can obtain a rate of convergence of order $O(1/T)$ with a properly designed gradient descent method, despite the fact that the original function is non-smooth and can not be optimized with a convergence rate better then $O(1/\sqrt{T})$ in black-box model.   As an example consider   the function $f$ to be optimized is of the form $f(\w) = \max_{1 \leq i \leq n} f_{i} (\w)$ where each individual  functions $f_i(\w), i \in [n]$ is convex, $\beta$-smooth and $\rho$-Lipschitz in some norm $\|\cdot\|$. In this case the function $f(\w)$ is not smooth and the best convergence rate one can hope in the black-box model is $O(1/\sqrt{T})$. 

Let $\Phi : \mathcal{K} \rightarrow \mathbb{R}$ be a mirror map on $\mathcal{W}$  and let $\w_1 \in \mathrm{argmin}_{\w \in \mathcal{W} \cap \mathcal{K}} \Phi(\w)$. The mirror prox (extragradient in a specialized case) method  is detailed below. 
\begin{figure}[H]
\begin{center}
\begin{myalg}{Extra Gradient  Descent  Method}{EGD}
{\bf Input:} \+ $\eta > 0$, function $f \in \mc{F}$, first-order oracle $\mc{O}$ \-\\ 
{\bf Initialize:} \+ $\w_1 = \z_1 =  \mathbf{0}$  \- \\ \\
{\bf for} $t=1,2,\ldots, T$  \+ \\
Query the oracle $\mc{O}$ at point $\w_t$ to get $\partial f(\w_t)$\\
Update $ \nabla \Phi(\z_{t+1}') = \nabla \Phi(\w_{t}) - \eta \partial f(\w_t)$\\

Update $\z_{t+1} \in \mathrm{argmin}_{\z \in \mathcal{W} \cap \mathcal{K}} \mathsf{B}_{\Phi}(\z,\z_{t+1}')$ and query the oracle to get $\partial f(\z_{t+1})$\\
Update $\nabla \Phi(\w_{t+1}') = \nabla \Phi(\w_{t}) - \eta \partial f(\z_{t+1})$ \\
Update  $\w_{t+1} \in \mathrm{argmin}_{\w \in \mathcal{W} \cap \mathcal{K}} \mathsf{B}_{\Phi}(\w,\w_{t+1}')$
\- \\
{\bf end for} 
\end{myalg}
\end{center}
\end{figure}

The EGD method first makes a step of MD to go from $\w_t$ to $\z_{t+1}$, and then it makes a similar step to obtain $\w_{t+1}$, starting again from $\w_t$ but this time using the gradient of $f$ evaluated at $\z_{t+1}$.  The following theorem exhibits the rate of convergence for EGD algorithm.

\begin{theorem}[Convergence Rate of EGD]\label{thm-chapter-2-egd}
Let $\Phi$ be a  $\alpha$-strongly convex on $\mathcal{K} \cap \mathcal{W}$ with respect to $\|\cdot\|$. Let $R = \sup_{\w \in \mathcal{K} \cap \mathcal{W}} \Phi(\w) - \Phi(\w_1)$ and $f$ be convex and $\beta$-smooth w.r.t. $\|\cdot\|$. Then EGD  with $\eta = \frac{\alpha}{\beta}$ has a convergence rate as: 

    $$f\bigg(\frac{1}{T} \sum_{t=1}^T \z_t \bigg) - \min_{\w \in \W}f(\w) \leq \frac{\beta R^2}{\alpha T} .$$
\end{theorem}

\subsubsection{Conditional Gradient Descent Method}
The main computational bottleneck of the gradient descent methods in solving constrained optimization problems is the projection step which might be as hard as solving the original optimization problem (see Appendix~\ref{chap:appendix-convex} for few expensive projections). Surprisingly the projection step can be avoided by  replacing the expensive projection operation with other kinds of light computational operations. One such an example is the   Conditional Gradient Descent (CGD) method which is also known as Frank-Wolf algorithm. The Frank-Wolfe method  was originally introduced  by Frank and
Wolfe in 1950~\cite{frank56}, where they aimed to present an algorithm for minimizing a quadratic function over a polytope using only linear optimization steps over the feasible set. 

The CGD algorithm  proceeds by iteratively solving a linear optimization problem to find a direction $\mathbf{p}_t$ inside the domain $\W$ that has the maximum correlation with the negative gradient at the current solution, i.e., $\arg \max_{\mathbf{p}\in\W} \dd{\mathbf{p}}{-\nabla f(\w_t)}$, and updating the solution as a linear combination of the obtained direction and previous solution. This procedure guarantees that the updated solutions remain  inside the feasible domain $\W$ and does not require the projection of updated solutions. More specifically, the CGD method replaces the projection step with a linear optimization problem over the constrained domain which is more efficient as long as the linear problem is easy to be solved.
\begin{figure}[H]
\begin{center}
\begin{myalg}{Conditional Gradient Descent}{CGD}
{\bf Input:} \+ convex set $\W$, $\eta > 0$, a smooth convex function $f \in \mc{F}$\-\\ 
{\bf Initialize:} \+ $\w_1 \in \W$  \- \\ \\
{\bf for} $t=1,2,\ldots, T$  \+ \\
Find $\mathbf{p}_t = \arg \min_{\mathbf{p} \in \W} \dd{\nabla f(\w_t)}{\mathbf{p}}$ \\
Update $\w_{t+1} = (1 - \eta_t)\w_t + \eta_t \mathbf{p}_t$ \- \\
{\bf end for} 
\end{myalg}
\end{center}
\end{figure}
The following result shows the convergence rate of CGD for smooth functions.
\begin{theorem}[Convergence Rate of CGD]\label{thm-chp-2-FW}
Assume that $f \in \mathcal{F}$ be a $\beta$-smooth convex function with respect to some norm $\|\cdot\|$ defined over the convex domain $\W$. Let $R = \sup_{\w,\w'} \|\w - \w'\|$. Then by setting $\eta_t = \frac{2}{t+1}$ in CGD method, we have:
$$f(\w_T) - f(\w_*) \leq \frac{2 \beta R^2}{t+1}$$ 
\end{theorem}
In Chapter~\ref{chap:projection} we will show that by replacing the projection step with gradient computation of constrain function, it is possible to devise efficient stochastic optimization methods which only require a single projection at the final iteration.

%%%%%%%%%%%%%%%%%%%%%%%%%%%%%%%%%%%%%%%%%%%%%%%%
\subsection{Stochastic Convex Optimization}

So far we assumed the the optimization algorithm has access to a noiseless oracle. It is more realistic to consider noisy oracles, where one does not have access to exact objective function or gradient values, but rather to their noisy estimates (usually zero mean and bounded variance). In particular, for a fixed closed convex subset $\W \subset \mathbb{R}^d$ of $\mathbb{R}^d$ we consider the following optimization problem:
\begin{eqnarray}\label{eqn:2:sgd}
\begin{aligned}
\min_{\w\in \W} f(\w) \quad \text{for} \quad f(\w) = \E[{F}(\w, \xi)] = \int_{\Xi}^{}{{F}(\w, \xi)dP(\xi)},
\end{aligned}
\end{eqnarray}
where we assume that the expected value function $f(\w)$ is continuous and convex on $\W$. We note if the function $F(\w, \xi)$ be convex on $\W$, then it follows that $f(\w)$ is also convex and the problem becomes a convex programming problem. The main difficulty in solving  the stochastic optimization problem in~(\ref{eqn:2:sgd}) is that the multidimensional
integral (expectation) cannot be computed with a high accuracy~\cite{nemirovski2009robust}, and in statistical learning problems we usually do not know what the distribution $P$ is. Therefore, there are two solutions to address this issue: these are  stochastic approximation (SA) and the sample average approximation (SAA) methods.  The main idea of SAA approach to solving stochastic programs is as follow. A sample $\xi_1, \xi_2, \cdots, \xi_n$ of $n$ realizations of the random vector in objective is generated and the stochastic objective is approximated estimated by the sample average function. Then, a deterministic optimization
algorithm is applied to solve the approximate function. We note that we can not perform a full gradient descent on $f(\w)$ as we would need to know the underlying distribution   to compute a gradient of $f(\w)$. 

In SA  we assume that there is an stochastic oracle $\mc{O}$, which, for a given point $(\w, \xi) \in \W \times \Xi$ returns an unbiased estimates of  subgradient of $f(\w)$.  In other words, it returns $\g$ such that $\E[\g] \in \partial f(\w)$. Stochastic optimization methods   allow the optimization method to take a step which is only in expectation along the negative of the gradient.   Based on this oracle, a simple algorithm to optimize the objective is Stochastic Gradient Descent (SGD).  SGD is in the same spirit of GD but it replaces the true gradients with stochastic gradients in updating the solutions:

\begin{figure}[H]
\begin{center}
\begin{myalg}{Stochastic Gradient Descent}{SGD}
{\bf Input:} \+ convex set $\W$, $\eta > 0$, function $f \in \mc{F}$, stochastic first-order oracle $\mc{O}$ \-\\ 
{\bf Initialize:} \+ $\w_0 =  \mathbf{0}$  \- \\ \\
{\bf for} $t=1,2,\ldots, T$  \+ \\
Query the stochastic  oracle $\mc{O}$ at point $\w_t$ to get $\g_t$ where $\E[\g_t] \in \partial f(\w_t)$ \\
Update $\w_{t+1} = \Pi_{\W} (\w_t - \eta \g_t)$\- \\
{\bf end for} \\
{\bf Return:} $\hat{\w} = \frac{1}{T}\sum_{t=1}^{T}{\w_t}$
\end{myalg}
\end{center}
\end{figure}
Under mild conditions as outlined below, one con show the SGD algorithm convergence to the optimal solution with convergence rate $O(1/\sqrt{T})$ with a high probability:
\begin{eqnarray*}
\begin{aligned}
&\E_{\xi_t}[\g_t] = \nabla f(\w)\label{eqn:ge}\\
&\E_{\xi_t}[\exp(\|\g_t- \nabla f(\w)\|_*^2/\sigma^2)]\leq \exp(1).\label{eqn:gb}
\end{aligned}
\end{eqnarray*}
It is also straightforward to generalize the the mirror descent method to stochastic setting by replacing the Euclidean distance in the update of SGD with another Bregman divergence adopted to a specific domain.

By comparing  SGD method for stochastic optimization and OGD method for regret minimization, we note that both methods are closely related. Although both methods looks similar algorithmically, but there are main conceptual differences between SGD and OGD. We note that in stochastic optimization the goal is to generate a sequence of solutions which quickly convergences to the minimum of a function defined as  $f(\w) = \E[{F}(\w, \xi)]$, while is online learning the goal is to generate a sequence of solutions that accumulates a small loss during the learning measured in terms of regret. In other words SGD provides an incremental solution to a stochastic optimization problem and OGD provides a solution to adopt a sequence of adversarially generated loss functions. We note that regret minimization algorithms equipped with online to batch conversion schemas discussed before settle an  efficient paradigm to solve general optimization problems, but sometimes it seems essential to go beyond this barrier to obtain optimal convergence rates in stochastic setting~\cite{hazan-2011-beyond,ICML2012Rakhlin}.  

\begin{remark} It is remarkable that in stark contrast  to deterministic optimization where the smoothness of objective function makes a significant improvement in terms of convergence rate (i.e., Theorem~\ref{thm-chapter-2-gd}), in stochastic optimization the smoothness is not a desirable property as it yields the same convergence rate as the Lipschitz functions. In particular, as it has been shown in Appendix~\ref{chap:appendix-technical}, a tight analysis of stochastic mirror descent algorithm has an $O(\frac{\beta R^2}{T}+ \frac{\sigma R}{\sqrt{T}}) $ convergence rate for smooth objective functions, which is dominated by the slow $O(1/\sqrt{T})$ rate unless the variance of stochastic gradients becomes zero $\sigma = 0$. As it will be discussed in Chapter~\ref{chap:mixed}, the mixed optimization paradigm we introduce in thesis is able to leverage the smoothness of objection function to attain an $O(1/T)$ rate by accessing the full gradient oracle $\log T$ times on top of the $O(T)$ accesses of the stochastic gradient oracle. 

\end{remark}
%%%%%%%%%%%%%%%%%%%%%%%%%%%%%%%%%%%%%%%%%%%%%%%%%
\subsection{Convex Optimization for Learning Problems}

Formulating   statistical  learning tasks and in particular convex learning problems as a convex optimization problem  makes an intimate connection between  learning and mathematical  optimization.  Therefore,  optimization methods play a central role in solving machine learning problems and  challenges exist in machine learning applications demand the development of new optimization algorithms. 

To see this, consider the typical problem of the supervised learning consisting of a input space $\Xi = \mc{X} \times \mc{Y}$  and a suitable set of  hypotheses $\W$ for prediction such as the set of linear predictors, i.e., $\W = \{ \x \mapsto \dd{\w}{\x}: \w \in \R^d\}$. Then, the learner is provided with a  training sample $\mc{S} = \left( (\x_1,y_1), (\x_2,y_2), \cdots, (\x_n,y_n)\right) \in \Xi^n$ and  is supposed  to pick a hypothesis $\w \in \W$ which minimizes appropriate empirical cost over the training sample based on a predefined surrogate loss function $\ell: \W \times \Xi \mapsto \R_{+}$. The last step of this learning process  corresponds to an optimization algorithm that solves the minimization problem of picking that hypothesis from the set of hypotheses.  As a result, convex optimization forms the backbone of many algorithms for statistical learning.  This formulation includes support vector machine (SVM), support vector regression (SVR), Lasso,  logistic regression,  and ridge regression among many others as detailed below:

\begin{itemize}
\item Hinge loss (Support Vector Machine)): $\displaystyle  \sum_{i=1}^n \max(0, 1-y_i \dd{\w}{\x_i})$.
\item Logistic loss (Logistic Regression): $\displaystyle \min_{\w \in \W}  \sum_{i=1}^n \log(1 + \exp(- y_i \dd{\w}{\x_i}))$.
\item Least-squares loss (Regression): $\displaystyle \min_{\w \in \W} \sum_{i=1}^n (y_i - \dd{\w}{\x_i})^2$.
\item Exponential loss (Boosting):  $\displaystyle \min_{\w \in \W} \sum_{i=1}^n  \exp(-y_i \dd{\w}{\x_i})$.
\end{itemize}

The domain $\W$ in above formulations, captures the constrains on the classifier $\w$. Commonly considered examples are the  bounded Euclidean ball $\W = \{\w \in \R^d: \|\w\|_2 \leq R\}$, bounded  $\ell_1$ ball $\W = \{\w \in \R^d: \|\w\|_1 \leq B\}$ or the  box $\W = \{\w \in \R^d: \|\w\|_{\infty} \leq B\}$. We note that instead of moving the constraint into the $\W$, by leveraging on the theory of Lagrangian method in constrained optimization, one can simply move the constraint into the objective and solve the unconstrained optimization problem, i.e., $\W = \R^d$.
 
To fully understand the application of convex optimization methods to solving machine learning problems, let us consider the following optimization problem:
\begin{eqnarray}\label{eqn:emp}
\min_{\w \in \W} \Lh(\w) =  \frac{1}{n} \sum_{i=1}^{n}{\ell(\w, (\x_i, y_i))}
\end{eqnarray}
A preliminary approach for solving the optimization problem in~(\ref{eqn:emp}) is the batch gradient descent (GD) algorithm. It starts with some initial point, and iteratively updates the solution using the equation $\w_{t+1} = \Pi_{\W}(\w_t - \eta \nabla \mc{L}_{\S}(\w_t))$ where
$$ \nabla \Lh(\w) = \frac{1}{n} \sum_{i=1}^{n}{\partial \ell(\w, (\x_i,y_i))\x_i}. $$
The main shortcoming of GD method is its high cost in computing the full gradient $\nabla \mc{L}_{\S}(\w_t)$, i.e., O(n) gradient computations, when the number of training examples is large.  Stochastic gradient descent (SGD) alleviates this limitation of GD by sampling one (or a small set of) examples and computing a stochastic (sub)gradient at each iteration based on the sampled examples. Since the computational cost of SGD per iteration is independent of the size of the data (i.e., $n$), it is usually appealing for large-scale learning and optimization~\cite{NIPS2007_726,nemirovski2009robust,Shalev-Shwartz:2007:PPE:1273496.1273598}. Despite of their slow rate of convergence 
compared with the batch methods, stochastic optimization methods have shown to be very effective for large scale and  online learning problems,  both theoretically~\cite{nemirovski2009robust,lan2012optimal} and empirically~\cite{Shalev-Shwartz:2007:PPE:1273496.1273598}. We note although a large number of iterations is usually needed to obtain a solution of desirable accuracy, the lightweight computation per iteration makes SGD attractive for many large-scale learning problems.

%%%%%%%%%%%%%%%%%%%%%%%%%%%%%%%%%%%%%%%%%%%%%%%%%%
\subsection{From Stochastic  Optimization to Convex Learning Theory}
As mentioned earlier, most of existing learning algorithms follow the framework of empirical risk minimizer  or regularized ERM,  which was developed to great extent by Vapnik and Chervonenkis~\cite{vapnik1971uniform}.  Essentially, ERM methods use the empirical loss  over $\mc{S}$, i.e., 
$$ \Ls(\w) = \frac{1}{n} \sum_{i = 1}^{n}{\ell(\w, (\x_i, y_i))},$$
as a criterion to pick a hypothesis.  From optimization viewpoint, the  ERM  methods resembles the widely used Sample Average Approximation (SAA) method in the optimization community  when the hypothesis space and the loss function are convex. If uniform convergence holds, then the empirical risk minimizer is consistent, i.e., the population risk of the ERM converges to the optimal population risk, and the problem is learnable using  ERM.

A rather different paradigm for risk minimization is stochastic optimization. Recall that the goal of learning is to  approximately minimize the risk  
\begin{eqnarray}\label{eqn:risk}
\mathcal{L}_{\D}(\w) = \mathbb{E}_{(\x, y) \sim \D}[\ell(\w, (\x,y))].
\end{eqnarray}
However,  since  the distribution $\D$ is unknown to the learner, we can not utilize standard gradient methods to directly minimize the expected loss in~(\ref{eqn:risk}). This is because we are not able to compute the gradient $\nabla \mathcal{L}_{\D}(\w)$ at a particular query point $\w$. We note that this is different from the application of SGD for solving the optimization problem in~(\ref{eqn:emp}) because in~(\ref{eqn:emp}) the randomness is over the uniform sampling from the objective function which is known (essentially we have a randomized optimization method), while in~(\ref{eqn:risk}) the randomness is imposed on the instance space $\Xi$ through a distribution $\D$ which is unknown to the learner in advance. 

In stochastic optimization all we need is not the exact gradient of objective function, but an unbiased estimate of the true gradient $\nabla \mathcal{L}_{\D}(\w)$. Surprisingly, it turns out that  the construction of this unbiased estimate is extremely simple  for risk minimization as follows. First, we sample an instance $\z = (\x_i, y_i) \in \Xi $ according to $\D$ and set the stochastic gradient to be 
$$\g = \partial \ell (\w, (\x_i,y_i)) \x_i,$$ 
which will be an unbiased estimate of true gradient, i.e., $\E[\g] = \nabla\mathcal{L}_{\D}(\w)$.   

The beauty of SGD for direct risk minimization is that it is efficient and it delivers the same sample complexity as the ERM  method.     To motivate stochastic optimization as an alternative to the ERM method,~\cite{DBLP:shalev2009stochastic,shalev2009learnability} challenged the ERM method and showed that there  is a real gap between learnability and uniform convergence by investigating non-trivial problems where no uniform convergence holds, but they are still learnable using SGD algorithm~\cite{nemirovski2009robust}. These results  uncovered an important relationship between learnability and stability, and showed that stability  together with approximate empirical risk minimization, assures learnability~\cite{shalev-shwartz:2010:learnability}.  Unlike ERM method in which the learnability is characterized by attendant complexity of hypothesis space, in SGD based learning, stability is a general notion to characterize learnability. In particular, in learning setting under i.i.d. samples where uniform convergence is not necessary for learnability, but where stability is both sufficient and necessary for
learnability.

\part{~Statistical Learning}\label{part-statistical}
\chapter[Passive Learning with Target Risk]{Passive Learning with Target Risk} \label{chap:passive_target}
\def \lb {\mathcal{L}_{\mc{D}}}
\def \Lh {\mathcal{L}_{\mc{S}}}
\def \L {\mathcal{L}_{\mc{D}}}
\def \E {\mathbb{E}}

 The setup of this chapter will be in the  classical statistical learning setting discussed  in Chapter~\ref{chap-background}, but with a slight modification. In particular, we assume that the target expected loss, also referred to as target risk, is provided in advance for  learner as  prior knowledge. Unlike most studies in the learning theory  that only incorporate the prior knowledge into the generalization bounds, we are able to explicitly utilize the target risk in the learning process. By leveraging on the smoothness of loss function, our analysis reveals a surprising result on the sample complexity of learning: by exploiting the target risk in the learning algorithm,  we show that when the loss function is both smooth and strongly convex, the sample complexity reduces to $O\left(\log \left(\frac{1}{\epsilon}\right)\right)$, an exponential improvement compared to the sample complexity $O(\frac{1}{\epsilon})$ for learning with strongly convex loss functions.  Furthermore,  our proof is constructive and is based on a computationally efficient stochastic optimization algorithm, dubbed  ClippedSGD, for such settings which demonstrate that the  proposed algorithm is practically useful.

The remainder of the chapter is organized as follows:  Section~\ref{sec-3-setup} motivates the problem and setups the notation. Section~\ref{sec:3:stochastic} motivates the  main intuition behind the proposed  algorithm. The proposed ClippedSGD algorithm and main result on its  sample complexity are discussed in Section \ref{sec:3:algorithm}. The proof of logarithmic sample complexity  is given in Section~\ref{sec:3:analysis} and the omitted proofs are deferred to Section~\ref{sec:clipped-proofs}.   Section~\ref{sec:3:conclusion} summarizes the chapter  and Section~\ref{sec:3:related} surveys the related works.

\section{Setup and Motivation}\label{sec-3-setup}

Recall that in  the standard statistical or passive supervised  learning setting, we consider an input space   $\Xi \equiv \mc{X} \times \mc{Y}$  where $\mc{X} \subseteq \R^d$ is the space for instances and $\mc{Y}$ is the set of labels, and   a hypothesis class $\mc{H}$ from which we choose a classifier.  We assume that the domain space $\Xi$ is endowed with an unknown   probability measure  $\mc{D}$ and measure the performance of a specific hypothesis $h$ by defining a nonnegative loss function $\ell: \mc{H} \times \Xi \rightarrow \mathbb{R}_{+}$.  The  risk of a hypothesis $h$ with respect to the underlying distribution $\D$ is defined as:  
$$\lb(h) = \mathbb{E}_{\z \sim \mc{D}}  [\ell(h, \z)].$$ 
Given a sample $\mc{S} = (\z_1, \cdots, \z_n) = ( (\x_1, y_1), \cdots, (\x_n, y_n) ) \sim \Xi^n$, the goal of a learning algorithm is to pick a hypothesis $h: \mc{X} \rightarrow \mc{Y}$ from  $\mc{H}$ in such a  way that its risk $\lb(h)$ is close to the minimum possible risk of a hypothesis in $\mc{H}$.

In the new setting we consider for learning here, we assume that before the start of the learning process, the learner has in mind a \textit{target expected loss}, also referred to as \textit{target risk},  denoted by $\ep$\footnote{We use $\ep$ instead of $\epsilon$ to emphasize the fact that this parameter is known to the learner in advance.}, and tries to learn a classifier with the expected risk of $O(\ep)$ by labeling a small number of training examples.  We further assume the target risk $\ep$ is feasible, i.e.,  $\ep \geq \eo$ where $\eo = \min_{h \in \mathcal{H}} \mathcal{L}_{\mathcal{D}}(h)$. To address this problem, we develop an efficient algorithm, based on stochastic optimization, for passive learning with target risk. The most surprising property of the proposed algorithm is that when the loss function is both smooth and strongly convex, it only needs $O(d\log ({1}/{\ep}))$ labeled examples to find a classifier with the expected risk of $O(\ep)$, where $d$ is the dimension of data. This is a significant improvement compared to the sample complexity for empirical risk minimization. We note that the target risk assumption is fully exploited by the learning algorithm and  stands in  contrast to all those assumptions such as  the nature of unknown distribution $\D$, sparsity, and margin
 that usually enter into the generalization bounds and are often perceived as a rather  crude way to incorporate such assumptions.

The key intuition behind the  ClippedSGD algorithm is that by knowing target risk as  prior knowledge, the learner has  better control over the variance in stochastic gradients, which contributes mostly to the slow convergence in stochastic optimization and consequentially large sample complexity in passive learning. The trick is to run  the stochastic optimization in multiple stages with a \textit{fixed} size  and decrease the variance of stochastically perturbed gradients at each iteration by a properly designed mechanism.  Another crucial feature of the proposed algorithm is  to utilize the target risk $\ep$ to gradually refine the hypothesis space as the algorithm proceeds. Our algorithm differs significantly  from standard stochastic optimization algorithms and is able to achieve a geometric convergence rate with the knowledge of target risk $\ep$.

To analyze the sample complexity of ClippedSGD algorithm, we pursue the stochastic optimization viewpoint for risk minimization detailed in Chapter~\ref{chap-background}. Precisely,  we focus on the  convex learning problems for which we assume that the hypothesis class $\H$ is a parametrized convex set $\H = \{h_{\w}: \x \mapsto \langle \w, \x \rangle: \w \in \R^d,   \|\w\| \leq R\}$ and for all $\z = (\x, y) \in \Xi$, the loss function $\ell(\cdot, \z)$ is a non-negative convex function.  Thus, in the remainder we simply use vector $\w$ to represent $h_{\w}$, rather than working with  hypothesis $h_{\w}$. We will assume throughout that $\mc{X} \subseteq \R^d$ is the unit ball so that $\|\x\| \leq 1$.  Finally, the conditions under which we can get the desired result on sample complexity depend on analytic properties of the loss function. In particular, we assume that the loss function is strongly convex and smooth as defined in Chapter~\ref{chap-background} and can be found in Appendix~\ref{chap:appendix-convex}. We would like to emphasize that  in our setting, we only need  that the expected  loss function $\L(\w)$ be strongly convex, without having to assume strong convexity for individual loss functions.  \\

%%%%%%%%%%%%%%%%%%%%%%%%%%%%%%%%%%%%%%%%%%%%%%%%%%%
\section{The Curse of Stochastic Oracle}
\label{sec:3:stochastic}

We begin by discussing stochastic optimization  for risk minimization, convex learnability,  and  then the main intuition that motivates the proposed algorithm. 

As mentioned earlier in Chapter~\ref{chap-background}, most existing learning algorithms follow the framework of empirical risk minimizer (ERM) or regularized ERM methods that use the empirical loss  over $\mc{S}$, i.e., $\Lh(\w) = \frac{1}{n} \sum_{i = 1}^{n}{\ell(\w, \z_i)}$, as a criterion to pick a hypothesis. In regularized ERM methods, the learner  picks a hypothesis that jointly minimizes $\Lh(\w) $ and a regularization function over $\w$.  

A rather different paradigm for risk minimization is stochastic optimization. Recall that the goal of learning is to  approximately minimize the risk  $\lb(\w) = \mathbb{E}_{\z \sim \D}[\ell(\w, \z)]$.  However,  since  the distribution $\D$ is unknown to the learner, we can not utilize standard gradient methods to minimize the expected loss.  Stochastic optimization methods  circumvent this problem by allowing the optimization method to take a step which is only in expectation along the negative of the gradient.   To directly solve $\min_{\w \in \mc{H}} \big{[}\lb({\w}) = \E_{\z \sim \D}[\ell(\w, \z)] \big{]}$, a typical stochastic  optimization algorithm initially picks some point in the feasible set $\mc{H}$   and  iteratively updates these points based on first order perturbed gradient information about the function at those points.  For instance, the widely used SGD algorithm  starts with $\w_0 = \mathbf{0}$; at each iteration $t$, it queries the stochastic oracle $\mathcal{O}_s$ at $\w_t$ to obtain a perturbed but unbiased gradient $\g_t$ and updates the current solution by
$$ \w_{t+1} = {\Pi}_{\mc{H}} \left(\w_t - \eta_t \g_t\right),$$
where $\Pi_{\H}(\cdot)$ projects the solution $\w$ into the domain $\H$. 

To capture the efficiency of optimization procedures in a general sense, one can use  oracle complexity of the algorithm which, roughly speaking,  is the minimum number of calls to any  oracle needed by any method to achieve desired accuracy~\cite{nesterov2004introductory}. We note that the oracle complexity corresponds to the sample complexity of learning from the stochastic optimization viewpoint  previously discussed. 
This viewpoint for learning theory has been taken by few very  recent works~\cite{DBLP:shalev2009stochastic,shalev2009learnability} where the ERM method has been challenged  and it has been shown  that there  is a real gap between learnability and uniform convergence. This has been done  by investigating non-trivial problems where no uniform convergence holds, but they are still learnable using SGD algorithm. These results  uncovered an important relationship between learnability and stability, and showed that stability  together with approximate empirical risk minimization, assures learnability~\cite{shalev-shwartz:2010:learnability}.  Unlike ERM method in which the learnability is characterized by attendant complexity of hypothesis space, in SGD based learning, stability is a general notion to characterize learnability. In particular, in learning setting under i.i.d. samples where uniform convergence is not necessary for learnability, but where stability is both sufficient and necessary for
learnability.

To motivate the main intuition behind the proposed method, we begin by stating the following theorem which provides a lower bound on the sample complexity of stochastic optimization algorithms that is taken from~\cite{nemircomp1983}.

\begin{theorem}[Lower Bound on Oracle Complexity] Suppose $\lb({\w}) = \E_{\z \sim \D}[\ell(\w, \z)]$ is  $\alpha$-strongly and $\beta$-smooth convex  function defined over convex domain $\mc{H}$. Let $\mathcal{O}_s$ be a stochastic oracle that for any point $\w \in \mc{H}$ returns an unbiased estimate $\g$, i.e.,  $\E[\g] = \nabla \lb(\w)$, such that $\E\left[\|\g-\nabla \lb(\w)\|^2\right] \leq \sigma^2$ holds. Then for any stochastic optimization algorithm $\A$ to find a solution $\wh$ with $\epsilon$ accuracy respect to the optimal solution $\w_*$, i.e.,  $\E \left[ \lb(\wh) - \lb(\w_*) \right] \leq \epsilon$,  the number of calls  to $\mc{O}_s$ is lower bounded by
\begin{eqnarray}
{O}(1) \left( \sqrt{\frac{\beta}{\alpha}} \log \left( \frac{\beta \| \w_0 - \w_*\|^2}{\epsilon}\right) + \frac{\sigma^2}{\alpha \epsilon}\right).
\label{eqn:lower}
\end{eqnarray}
\label{thm:lower}
\end{theorem}
The first term in~(\ref{eqn:lower}) comes from deterministic oracle complexity and the second term is due to noisy gradient information provided by stochastic oracle $\mc{O}_s$. As indicated in~(\ref{eqn:lower}), the slow convergence rate for stochastic optimization is due to the variance in stochastic gradients, leading to at least ${O}\left({\sigma^2}/{\epsilon}\right)$ queries to be issued. We note that the idea of mini-batch~\cite{mini-batch-2011,duchirandomizedsmooth}, although it reduces the variance in stochastic gradients, does not reduce the oracle complexity.

We close this section by  informally presenting why logarithmic sample complexity is, in principle, possible,   under the assumption that target risk is known to the learner $\A$. To this end, consider the setting of  Theorem~\ref{thm:lower} and assume that the learner $\A$ is given the prior accuracy $\ep$ and is asked to find an  $\ep$-accurate solution. If it happens that the variance of stochastic oracle $\mc{O}_s$ has the same magnitude as $\ep$, i.e., 
$\E\left[\|\g-\nabla \lb(\w)\|^2\right] \leq \ep$,  then from~(\ref{eqn:lower}) it follows that  the second term  vanishes and the learner $\A$  needs to issue only ${O} \left(\log {1}/{\ep}\right)$ queries to find the solution. But, since there is no control on  the stochastic oracle $\mc{O}_s$, except that  the variance of stochastic gradients are bounded,  $\A$ needs a mechanism to manage the variance of perturbed gradients 
at each iteration in order to alleviate the influence of noisy gradients. One strategy is to replace the unbiased estimate of gradient with a biased one, 
which unfortunately may yield loose bounds.  To overcome this problem, we introduce a strategy that shrinks the solution space with respect to the target risk $\ep$ to
control the damage caused by biased estimates.

As an illustrative example to see how the knowledge of target risk is helpful, we  consider a simple one dimensional regression problem with loss function $\ell(w,x) =  (wx - b)^2$ where $b$  is a random variable that can either be $\delta$ or $1$  with $\Pr[b = \delta] = 1 - \delta^2$. Here we choose $\delta$ to be a very small value $\delta \ll 1$. The loss function  is non-negative, smooth, and strongly convex and is appropriate for our setting. For this setting we have, $\eo \leq \E_b[\ell(0)] = \delta^2 \times 1 + (1 - \delta^2) \times \delta^2 \leq 2\delta^2$ which  can be arbitrarily small.  For this example, the solution obtained by ERM with a small number of training examples will be on order of $\delta$ and therefore its expected risk will be on the order of $\delta^2$. However, from the viewpoint of the learner, this expected risk is \textit{unknown} unless the learner could figure out  $\Pr(b = 1) = \delta^2$, which unfortunately requires an order of $1/\delta^2$ samples. On the other hand, by having the target feasible risk as prior knowledge the learner is able to find out  $\Pr(b = 1)$ with a  small number of samples. 
%%%%%%%%%%%%%%%%%%%%%%%%%%%%%%%%%%%%%%%%%%%%%%
%%%%%%%%%%%%%%%%%%%%%%%%%%%%%%%%%%%%%%%%%%%%%%
\section{The ClippedSGD Algorithm}
\label{sec:3:algorithm}
 In this section we proceed to describe the proposed algorithm and state the main result on its sample complexity.

\subsection{The Algorithm Description}
We now turn to describing our algorithm.  Interestingly, our algorithm is quite dissimilar to the classic stochastic optimization methods. It proceeds by running the algorithm
online on fixed chunks of examples, and using the intermediate hypotheses and target risk $\ep$ to gradually refine the hypothesis  space. As  mentioned above, we assume in our setting that the target expected risk $\ep$ is  provided to the learner a priori. We further assume the target risk $\ep$ is feasible for the solution within the domain $\H$, i.e., $\ep \geq \eo$.  The proposed algorithm explicitly takes advantage of the  knowledge of expected risk $\ep$ to attain an $O\left(\log(1/\ep)\right)$ sample complexity.

%For notational brevity we denote  by $\g_k^t$ the stochastic gradient at iteration $t$ of stage $k$, i.e., $\g_k^t\ell'\left(\dd{\w_k^t}{\x_k^t}, y_t\right) \x_k^t$.
%For convenience, let $\ell_t(\w)$ be a non-negative loss function defined as  $\ell_t(\w) = \ell(\dd{\w}{\x_t}, y_t)$ where $\ell(t, y)$ is a convex loss function with bounded first order and second order derivatives, i.e.
%\[
%    \|\ell'(t, y)\| \leq G \quad \text{and} \quad \|\ell''(t, y)\| \leq \beta.
%\]
Throughout we shall consider linear predictors of form $\dd{\w}{\x}$ and assume that the loss function of interest   $\ell (\dd{\w}{\x}, y)$  is $\beta$-smooth.  It is straightforward to see that $\lb(\w) = \E_{(\x, y)\sim\D}\left[ \ell(\dd{ \w}{\x},y) \right]$ is also $\beta$-smooth. In addition to the smoothness of the loss function, we also assume that $\lb(\w)$ to be $\alpha$-strongly convex. We denote by $\w_*$ the optimal solution that minimizes $\lb(\w)$, i.e., $\w_* = \mathop{\arg\min}_{\w \in \H} \lb(\w)$, and denote its optimal value by $\eo$.

Let $(\x_t, y_t), t = 1, \ldots, T$ be a sequence of i.i.d. training examples.  The proposed algorithm  divides the $T$ iterations into the $m$ stages, where each stage consists of $T_1$ training examples, i.e., $T = m T_1$. Let $(\x_k^t, y_k^t)$ be the $t$th training example received at stage $k$, and let $\eta$ be the step size used by all the stages. At the beginning of each stage $k$, we initialize the solution $\w$ by the average solution $\wh_k$ obtained from the last stage, i.e.,
\begin{eqnarray}
    \wh_k = \frac{1}{T_1} \sum_{t=1}^{T_1} \w_k^t,  \label{eqn:average}
\end{eqnarray}
where $\w_k^t$ denotes the $t$th solution at stage $k$. Another feature of the proposed algorithm is a domain shrinking strategy that adjusts the domain as the algorithm proceeds using intermediate hypotheses and target risk.  We  define the domain $\H_k$ used at stage $k$ as
\begin{eqnarray}
    \H_k = \left\{\w \in \H: \|\w - \wh_k\| \leq \Delta_k \right\}, \label{eqn:omega}
\end{eqnarray}
where $\Delta_k$ is the domain size,  whose value will be discussed later. Similar to the SGD method, at each iteration of stage $k$, we receive a training example $(\x_k^t, y_k^t)$, and compute the gradient $\g_k^t = \ell'\left(\dd{\w_k^t}{\x_k^t}, y_t\right) \x_k^t$. Instead of using the gradient directly,  a clipped version of the gradient, denoted by $\v_k^t = \clip\left(\gamma_k, \g_k^t\right)$, will be used for updating the solution. More specifically, the clipped vector $\v_k^t \in \R^d$ is defined as
\begin{eqnarray}
    [\v_k^t]_i = \clip\left(\gamma_k, \left[  \g_k^t \right]_i\right) = \mbox{sign}\left(\left[ \g_k^t \right]_i\right)\min\left(\gamma_k, \left| \left[ \g_k^t \right]_i\right|\right), i=1, \ldots, d \label{eqn:clip}
\end{eqnarray}
where $\gamma_k = 2\xi\beta\Delta_k$ with $\xi \geq 1$. Given the clipped gradient $\v_k^t$, we follow the standard framework of stochastic gradient descent, and update the solution by
\begin{eqnarray}
    \w_k^{t+1} = \Pi_{\H_k}\left(\w_k^t - \eta \v_k^t \right). \label{eqn:update-clipped}
\end{eqnarray}
%%%%%%%%%%%%%%%%%%%%%%%%%%%%%%%%%%%%%%%%%%%%%%%%%
\begin{algorithm}[t]
\label{alg:clippedSGD}
\caption{{ClippedSGD} Algorithm}
\begin{algorithmic}[1]
\STATE {\textbf{Input:}} 
\begin{mylist}
\item step size $\eta$
\item stage size $T_1$
\item number of stages $m$
\item target expected risk $\ep$
\item parameters $\varepsilon \in (0, 1)$ and $\tau \in (0, 1)$ used for updating domain size $\Delta_k$
\item parameter $\xi \geq 1$ used to clip  the gradients \\ 
\end{mylist}
\STATE {\textbf{Initialize:}} $\wh_1 = 0$, $\Delta_1 = R$, and $\H_1 = \H$

\FOR{$k = 1, \ldots, m$}
    \STATE Set $\w_k^t = \wh_k$ and $\gamma_k = 2\xi\beta\Delta_k$
    \FOR{$t=1, \ldots, T_1$}
        \STATE Receive training example $(\x_t, y_t)$
        \STATE Compute the gradient $\g_k^t$ and 
        \STATE Clip the gradient $\g_k^t$ to $\v_k^t$ using 
\[       [\v_k^t]_i =  \mbox{sign}\left(\left[ \g_k^t \right]_i\right)\min\left(\gamma_k, \left| \left[ \g_k^t \right]_i\right|\right), i=1, \ldots, d\]
                \STATE Update the solution by $\w_k^{t+1} = \Pi_{\H_k}\left(\w_k^t - \eta \v_k^t \right)$

    \ENDFOR
    \STATE Update $\Delta_k$ using~(\ref{eqn:delta-clipped}).
    \STATE Compute the average solution $\wh_{k+1}$ according to~(\ref{eqn:average})
    \STATE Shrink the domain $\H_{k+1}$ using the expression in~(\ref{eqn:omega}).
\ENDFOR
\end{algorithmic} \label{alg:clippedSGD}
\end{algorithm}
%%%%%%%%%%%%%%%%%%%%%%%%%%%%%%%%%%%%%%%%%%%%%%%%
The purpose of introducing the clipped version of the gradient is to effectively control the variance in stochastic gradients, an important step toward achieving the geometric convergence rate. At the end of each stage, we will update the domain size by explicitly exploiting the target expected risk $\ep$ as
\begin{eqnarray}
\Delta_{k+1}= \sqrt{\varepsilon \Delta_k^2 + \tau \ep} \label{eqn:delta-clipped}\;,
\end{eqnarray}
where $\varepsilon \in (0, 1)$ and $\tau \in (0, 1)$ are two parameters, both of which  will be discussed later. 

Algorithm~\ref{alg:clippedSGD} gives the detailed steps for the proposed method. The three  important aspects of Algorithm~\ref{alg:clippedSGD}, all  crucial  to  achieve  a geometric convergence rate, are highlighted as follows:
\begin{itemize}
\item Each stage of the proposed algorithm is comprised of the same number of training examples. This is in contrast to the epoch gradient algorithm~\cite{hazan-2011-beyond} which divides $m$ iterations into exponentially increasing epochs, and runs SGD with averaging on each epoch.  Also, in our case the learning rate is fixed for all iterations.
\item The proposed algorithm uses a clipped gradient for updating the solution in order to better control the variance in stochastic gradients; this stands in contrast to  the SGD method, which uses original gradients  to update the solution.
\item The proposed algorithm takes into account the targeted expected risk and intermediate  hypotheses when updating the domain size at each stage. The purpose of domain shrinking  is to reduce the damage caused by biased gradients that resulted from clipping operation.
\end{itemize}
%%%%%%%%%%%%%%%%%%%%%%%%%%%%%%%%%%%%%%%%%%%%%%%%%%%%
\subsection{Main Result on Sample Complexity}
The main theoretical result on the performance of the ClippedSGD algorithm  is given in the following theorem.
\begin{theorem}[Convergence Rate]
\label{thm:main}
Assume that the hypothesis space $\H$ is compact and the loss function $\ell$ is $\alpha$-strongly convex and $\beta$-smooth.  Let $T = m T_1$ be the size of  the sample and $\epsilon_{\rm{prior}}$ be the target expected loss given to the learner in advance such that  $\epsilon_{\rm{opt}} \leq \epsilon_{\rm{prior}}$ holds. Given $\varepsilon \in (0, 1)$ and $\tau \in (0, 1)$, set $\xi$, $\eta$, and $T_1$ as
\begin{eqnarray*}
\xi = \frac{4\beta}{\alpha \tau}, \; T_1 = 4\max\left\{\frac{\xi^3\beta d + 2\xi \beta\sqrt{d}}{\varepsilon \alpha}\ln\frac{ms}{\delta}, \frac{16\xi^2\beta^2}{\alpha^2\varepsilon^2} \right\}, \; \eta = \frac{1}{2\xi\beta\sqrt{T_1}}, \label{eqn:BT}
\end{eqnarray*}
where
\begin{eqnarray}
    s = \left\lceil \log_2 \frac{\xi \beta R^2}{\epsilon_{\rm{prior}}} \right\rceil. \label{eqn:s}
\end{eqnarray}
After running Algorithm~\ref{alg:clippedSGD} over $m$ stages, we have, with a probability $1 - \delta$,
\[
    \L(\wh_{m+1}) \leq \frac{\beta R^2}{2}\varepsilon^m + \left(1 + \frac{\tau}{1 - \varepsilon}\right){\epsilon_{\rm{prior}}},
\]
implying that only $O(d\log [1/\epsilon_{\rm{prior}}])$ training examples are needed in order to achieve a risk of $O(\epsilon_{\rm{prior}})$.
\end{theorem}

We note that comparing to the bound in Theorem~\ref{thm:lower}, for  Algorithm~\ref{alg:clippedSGD}  the level of error to which the linear convergence holds is not determined by the noise level in stochastic gradients, but by the target risk. In other words, the algorithm is able to tolerate the noise by knowing the target risk as prior knowledge and achieves a linear convergence to the level of the target risk even when the variance of stochastic gradients is much larger than the target risk.  In addition, although the result given in Theorem~\ref{thm:main} assumes a bounded domain with $\|\w\| \leq R$, however, this assumption can be lifted by effectively exploring the strong convexity of the loss function and further assuming that the loss function is Lipschitz continuous with constant $G$, i.e.,  $|\L(\w_1)- \L(\w_2)|\leq G\|\w_1-\w_2\|,\; \forall\; \w_1,\w_2\in\H$. More specifically, the fact that the $\L(\w)$ is $\alpha$-strongly convex with first order optimality condition, from Lemma~\ref{lemma-app-A-sc} for the optimal solution $\w_* = \arg \min_{\w \in \mc{H}} \L(\w)$, we have
\[\L(\w) - \L(\w_*) \geq \frac{\alpha}{2}\|\w - \w_* \|^2, \;\; \forall \w \in \mc{H}.\]
This inequality combined with Lipschitz continuous assumption implies that for any $\w \in \H$ the inequality $\|\w-\w_*\| \leq R_* := 2G/\alpha$ holds, and therefore we can simply set $R = R_*$. We also note that  this dependency can  be resolved  with a weaker assumption than Lipschitz continuity,  which only depends on the gradient of loss function at origin.  To this end,  we define $|\ell'(0, y)| = G$. Using the fact that $\L(\w)$ is $\alpha$-strongly, it is easy to verify that $\frac{\alpha}{2} \|\w_*\|^2 - G \|\w_*\| \leq 0$, leading to $\|\w_*\| \leq R_* := \frac{2}{\alpha}G$ and, therefore, we can simply set $R = R_*$. 

We  now use our analysis of Algorithm~\ref{alg:clippedSGD} to obtain a sample complexity analysis
for learning  smooth strongly convex problems with a bounded hypothesis class. To make it easier to parse, we only keep the dependency on the main parameters  $d$, $\alpha$, $\beta$, $T$, and $\ep$ and hide the dependency on other constants in $\O(\cdot)$ notation. Let $\wh$ denote the output of Algorithm~\ref{alg:clippedSGD}. By setting $\varepsilon = 0.5$ and  letting $c = O(\tau)$ to be an arbitrary small number, Theorem~\ref{thm:main} yields the following:

\begin{corollary}[Sample Complexity]  Under the same conditions as Theorem~\ref{thm:main}, by running  Algorithm~\ref{alg:clippedSGD} for minimizing $\L(\w)$ with a number of iterations (i.e., number of training examples)  $T$, if it holds that,

\[ T \geq O\left(d \kappa^4 \left(\log \frac{1}{\epsilon_{\rm{prior}}}\log \log \frac{1}{\epsilon_{\rm{prior}}} + \log\frac{1}{\delta} \right) \right)\]  
where $\kappa = \beta/\alpha$ denotes the condition number of the loss function and $d$ is the dimension of data,  then with a probability $1 - \delta$,   $\wh$ attains a risk of $O(\epsilon_{\rm{prior}})$, i.e., $\L(\wh) \leq (1+c) \epsilon_{\rm{prior}}$.
\label{corollary:sample}
\end{corollary}

As an example of a concrete problem that may be put into the setting of the present work is the regression problem with squared loss. It is easy to show that average square loss function is Lipschitz continuous with a Lipschitz constant  $\beta = \lambda_{\max} (\mathbf{X}^{\top}\mathbf{X})$ which denotes the largest eigenvalue of matrix $\mathbf{X}^{\top}\mathbf{X}$ where $\mathbf{X} \in \mathbb{R}^{n \times d}$ is the data matrix. The strong convexity is guaranteed as long as  the population data covariance matrix is not rank-deficient and its  minimum eigenvalue  is lower bounded by a constant $\alpha>0$. For this problem, the optimal minimax sample complexity is known to be $O(\frac{1}{\epsilon})$, but as it implies from Corollary~\ref{corollary:sample}, by the knowledge of target risk $\ep$, it is possible to reduce the sample complexity  to $O\left(\log (1/{\epsilon_{\rm{prior}}})\right)$.

\begin{remark} It is indeed remarkable that the sample complexity of Theorem~\ref{thm:main} has $\kappa^4 = \left(\beta/\alpha\right)^4$ dependency on the condition number of the loss function, which is worse than the $\sqrt{{\beta}/{\alpha}}$ dependency in the lower bound in (\ref{eqn:lower}). Also, the explicit dependency of sample complexity on  dimension $d$ makes the proposed algorithm inappropriate for non-parametric settings.
\end{remark}
%%%%%%%%%%%%%%%%%%%%%%%%%%%%%%%%%%%%%%%%%%%%
\section{Analysis of Sample Complexity}
\label{sec:3:analysis}
Now we turn to proving  the main theorem. The proof will be given in a series of lemmas and theorems where the proof of few are given in the Section~\ref{sec:clipped-proofs}. The proof makes use of the Bernstein inequality for martingales, idea of peeling process, self-bounding property of smooth loss functions, standard analysis of stochastic optimization, and novel ideas to derive the claimed sample complexity for the proposed algorithm.

The proof of  Theorem~\ref{thm:main} is by induction and we start with the key step  given in the following theorem.
\begin{theorem}
\label{thm:induction}
Assume $\epsilon_{\rm{prior}} \geq \epsilon_{\rm{opt}}$. For a fixed stage $k$, if $\|\wh_k - \w_*\| \leq \Delta_k$, then, with a probability $1 - \delta$, we have
\[
    \|\wh_{k+1} - \w_*\|^2 \leq a \Delta_k^2 + b \epsilon_{\rm{prior}}
\]
where
\begin{eqnarray}
    a = \frac{2}{\alpha T_1}\left(2\xi\beta\sqrt{T_1} + \left[\xi^3\beta d + 2\xi\beta\sqrt{d}\right]\ln\frac{s}{\delta} \right), \quad b = \frac{8}{\alpha \xi} \label{eqn:ab}
\end{eqnarray}
and $s$ is given in (\ref{eqn:s}), provided that $\xi \geq 16\beta/\alpha$ and $\eta = 1/(2\xi\beta\sqrt{T_1})$ hold.
\end{theorem}

%Below is the proof of Theorem~\ref{thm:main} based on the induction result in Theorem~\ref{thm:induction}.
Taking this statement as given for the moment, we proceed with the proof of Theorem~\ref{thm:main}, returning later to establish the claim stated in  Theorem~\ref{thm:induction}. \\
\begin{proof}[Proof of Theorem~\ref{thm:main}]
By setting $a$ and $b$ in (\ref{eqn:ab}) in Theorem~\ref{thm:induction} as $a \leq \varepsilon$ and $ b \leq {2\tau}/{\beta}$, we have $\xi \geq 4\beta/(\alpha\tau)$ and
\[
T_1 \leq \frac{2}{\alpha\varepsilon}\left(2\xi\beta\sqrt{T_1} + \left[\xi^3\beta d + 2\xi\beta\sqrt{d}\right]\ln\frac{s}{\delta} \right)
\]
implying that
\[
T_1 \geq 4\max\left\{\frac{\xi^3\beta d + 2\xi\beta\sqrt{d}}{\varepsilon \alpha}\ln\frac{s}{\delta}, \frac{16\xi^2\beta^2}{\alpha^2\varepsilon^2} \right\}.
\]
Thus, using Theorem~\ref{thm:induction} and the definition of $\xi$ and $T_1$, we have, with a probability $1 - \delta$,
\[
    \Delta^2_{k+1}  \leq \varepsilon \Delta_k^2 + \frac{2 \tau}{\beta} \ep.
\]
After $m$ stages, with a probability $1 - m\delta$, we have
\[
    \Delta^2_{m+1} \leq \varepsilon^{m}\Delta^2_1 + \frac{2\tau}{\beta}\ep\sum_{i=0}^{m-1} \varepsilon^i \leq \varepsilon^m\Delta^2_1 + \frac{2\tau}{\beta(1 - \varepsilon)}\ep.
\]
%and therefore
By the $\beta$-smoothness of $\lb(\w)$,  it implies that
\begin{eqnarray*}
\lb (\wh_{m+1}) - \lb(\w_*) \leq \frac{\beta}{2}\| \wh_{m+1} - \w_*\|^2 
 &\leq&  \frac{\beta}{2}\varepsilon^m \Delta^2_1 + \frac{\tau}{1 - \varepsilon} \ep,\\
 &\leq& \frac{\beta R^2}{2}\varepsilon^m +  \frac{\tau}{1 - \varepsilon} \ep,
\end{eqnarray*}
where  the last inequality follows from  $\Delta_1 \leq  {R}$.  The bound stated in the theorem follows the assumption that $\lb(\w_*) = \eo \leq \ep$.
%\[
%    \lb(\wh_{m+1}) \leq \eo + \frac{\beta}{2}\varepsilon^m \Delta^2_1 + \frac{\tau}{1 - \varepsilon} \ep \leq \frac{\beta R^2}{2}\varepsilon^m + \left(1 + \frac{\tau}{1 - \varepsilon}\right) \ep.
%\]
\end{proof}
%\begin{remark}We remark that the proof of Theorem~\ref{thm:induction} relies only on the smoothness and strong convexity assumption of the loss function and boundedness of the hypotheses $\H$. But if we further assume that the loss function is Lipschitz continuous with constant $G$, i.e.,  $|\ell(\w_1)- \ell(\w_2)|\leq G\|\w_1-\w_2\|,\; \forall\; \w_1,\w_2\in\H$, then we can relax the boundedness assumption as follows. First  we note that the definition of strong convexity with first order optimality condition implies that for a $\alpha$-strongly convex function $\ell (\w)$ if $\w_* = \arg \min_{\w \in \mc{H}} \ell(\w)$, it holds
%\[ \ell(\w) - \ell(\w_*) \geq \frac{\alpha}{2}\|\w - \w_* \|^2, \;\; \forall \w \in \mc{H}\]
%This inequality combined with Lipschitz continuous assumption implies that for any $\w \in \H$ the inequality $\|\w-\w_*\| \leq \frac{2G}{\alpha}$ holds, and therefore we can easily set $R$ to be $2G/\alpha$.
%\end{remark}

%\subsection{Proof of Theorem~\ref{thm:induction}}

We now turn to proving Theorem~\ref{thm:induction}. To bound $\|\wh_{k+1} - \w_*\|$ in terms of $\Delta_k$, we start with the standard analysis of online learning. In particular, from the strong convexity assumption of $\lb (\w)$ and updating rule (\ref{eqn:update-clipped}) we have,
\begin{eqnarray}
 \label{eqn:1}
  \lb(\w_k^t) - \lb(\w_*) &\leq& \langle \nabla \lb(\w_k^t), \w_k^t - \w_* \rangle - \frac{\alpha}{2}\|\w_k^t - \w_*\|^2 \nonumber \\
&=&  \langle \v_k^t, \w_k^t - \w_* \rangle + \langle \nabla \lb(\w_k^t) - \v_k^t , \w_k^t - \w_* \rangle - \frac{\alpha}{2}\|\w_t - \w_*\|^2 \nonumber \\
& \leq&  \frac{\|\w_k^{t} - \w_*\|^2 - \|\w_k^{t+1} - \w_*\|^2}{2\eta} + \frac{\eta d}{2}\gamma_k^2  \nonumber \\  &  & + \underbrace{\langle \nabla \lb(\w_k^t) - \v_k^t , \w_k^t - \w_* \rangle}\limits_{\triangleq v_k^t} - \frac{\alpha}{2}\|\w_t - \w_*\|^2,
\end{eqnarray}
where the last step follows from $\|\v_k^t\| \leq \gamma_k \sqrt{d}$.
By adding all the inequalities of (\ref{eqn:1}) at stage $k$, we have
\begin{eqnarray}
\sum_{t=1}^{T_1} \lb(\w_k^t) - \lb(\w_*) & \leq & \frac{\|\wh_k - \w_*\|^2}{2\eta} + \frac{d\eta}{2}\gamma_k^2 T_1 + \sum_{t=1}^{T_1} v_k^t - \frac{\alpha}{2}\sum_{t=1}^{T_1} \|\w_t - \w_*\|^2 \nonumber \\
& \leq & \frac{\Delta_k^2}{2\eta} + \frac{d\eta}{2}\gamma_k^2 T_1 + V_k - \frac{\alpha}{2}W_k, \label{eqn:2}
\end{eqnarray}
where $V_k$ and $W_k$ are defined as $V_k = \sum_{t=1}^{T_1} v_k^t$ and $W_k = \sum_{t=1}^{T_1} \|\w_k^t - \w_*\|^2$, respectively.
%\begin{eqnarray}
%A_k = \sum_{t=1}^{T_1} A_k^t, \quad H_k = \sum_{t=1}^{T_1} \|\w_k^t - \w_*\|^2. \label{eqn:H}
%\end{eqnarray}
In order to bound $V_k$, using the fact that $\nabla \lb(\w_k^t) = \E_t[\g_k^t]$, we rewrite $V_k$ as
\begin{eqnarray*}
V_k & = & \sum_{t=1}^{T_1} \underbrace{\langle - \v_k^t + \E_t[\v_k^t], \w_k^t - \w_* \rangle}_{\triangleq d_k^t} + \sum_{t=1}^{T_1} \underbrace{\langle \E_t\left[\g_k^t\right] - \E_t[\v_k^t], \w_k^t - \w_* \rangle}_{\triangleq e_k^t}   \\
& = & D_k + E_k,
\end{eqnarray*}
where $D_k = \sum_{t=1}^{T_1} d_k^t$ and $E_k = \sum_{t=1}^{T_1} e_k^t$ which represent the variance and bias of the clipped gradient $\v_k^t$, respectively. We now turn to separately upper bound each term.

The following lemma bounds the variance term $D_k$ using the Bernstein inequality for martingale. Its proof can be found in Section~\ref{sec:clipped-proofs}.
\begin{lemma} \label{lemma:d}
For any $L > 0$ and $\mu > 0$, we have
\begin{eqnarray*}
    \Pr\left(W_k \leq \frac{\epsilon_{\rm{prior}} T_1}{2\mu\beta}\right) + \Pr\left(D_k \leq \frac{1}{L}W_k + \left(L\gamma_k^2 d + \gamma_k\Delta_k\sqrt{d}\right)\ln\frac{s}{\delta} \right) \geq 1 - \delta
\end{eqnarray*}
where $s$ is given by
\[
    s = \left\lceil \log_2 \frac{8\beta\mu R^2}{\epsilon_{\rm{prior}}} \right\rceil.
\]
\end{lemma}
The following lemma bounds $E_k$ using the self-bounding property of smooth functions and the proof is deferred to  Section~\ref{app:lemma2}.
\begin{lemma} \label{lemma:e}
\[
E_k \leq \frac{4T_1}{\xi}\epsilon_{\rm{opt}} + \frac{4\beta}{\xi} W_k \leq \frac{4T_1}{\xi}\epsilon_{\rm{prior}} + \frac{4\beta}{\xi} W_k.
\]
\end{lemma}
Note that  without the knowledge of $\ep$, we have to bound $\eo$ by $\Omega(1)$, resulting in a very loose bound for the bias term $E_k$. It is knowledge of the target expected risk $\ep$ that allows us to come up with a significantly more accurate bound for the bias term $E_k$, which consequentially leads to a geometric convergence rate.

We now proceed to bound $\sum_{t=1}^{T_1} \lb(\w_k^t) - \lb(\w_*)$ using the two bounds in Lemma~\ref{lemma:d} and \ref{lemma:e}. To this end, based on the result obtained in Lemma~\ref{lemma:d}, we consider two scenarios. In the first scenario, we assume
\begin{eqnarray}
    W_k \leq \frac{\ep T_1}{2\mu\beta} \label{eqn:condition-1}
\end{eqnarray}
In this case, we have
\begin{eqnarray}
    \sum_{t=1}^{T_1} \lb(\w_k^t) - \lb(\w_*) \leq \frac{\beta}{2}W_k \leq \frac{\ep}{2\mu} T_1. \label{eqn:bound-3-1}
\end{eqnarray}
In the second scenario, we assume
\begin{eqnarray}
    D_k \leq \frac{1}{L}W_T + \left(L\gamma_k^2 d + \gamma_k\Delta_k\sqrt{d}\right)\ln\frac{s}{\delta}. \label{eqn:condition-2}
\end{eqnarray}
In this case, by combining the bounds for $D_k$ and $E_k$ and setting $L = \frac{\xi}{4\beta}$, we have
\begin{eqnarray*}
    V_k & \leq & \frac{8\beta}{\xi} W_k + \left(\frac{\xi d}{4\beta}\gamma_k^2 + \gamma_k\Delta_k\sqrt{d}\right)\ln\frac{s}{\delta} + \frac{4T_1}{\xi}\ep \\
    & = & \frac{8\beta}{\xi} W_k + \left(\xi^3\beta d + 2\xi \beta\sqrt{d}\right)\Delta_k^2\ln\frac{s}{\delta} + \frac{4T_1}{\xi}\ep,
\end{eqnarray*}
where the last equality follows from the fact $\gamma_k = 2\xi \beta \Delta_k$. If we choose $\xi$ such that $\frac{8\beta}{\xi} \leq \frac{\alpha}{2}$ or $\xi \geq\frac{16\beta}{\alpha} > 1$ holds, we get
%\[
%    \frac{8\beta}{B} \leq \frac{\alpha}{2}, \mbox{ or }\; B \geq\frac{16\beta}{\alpha} > 1
%\]
%we have
\begin{eqnarray*}
    V_k \leq \frac{\alpha}{2}W_k + \left(\xi^3\beta d+ 2\xi\beta\sqrt{d}\right)\Delta_k^2\ln\frac{s}{\delta} + \frac{4T_1}{\xi}\ep
\end{eqnarray*}
Substituting the above bound for $V_k$ into the inequality of (\ref{eqn:2}), we have
\[
\sum_{t=1}^{T_1} \lb(\w_k^t) - \lb(\w_*) \leq \frac{\Delta_k^2}{2\eta} + \frac{\eta}{2}\gamma_k^2 T_1 + \left(\xi^3\beta d+ 2\xi \beta\sqrt{d}\right)\Delta_k^2\ln\frac{s}{\delta} + \frac{4T_1}{\xi}\ep
\]
By choosing $\eta$ as $\eta = \frac{\Delta_k}{\gamma_k\sqrt{T_1}} = \frac{1}{2\xi\beta\sqrt{T_1}}$, we have
\begin{eqnarray}
\lb(\wh_{k+1}) - \lb(\w_*) \leq \frac{1}{T_1}\left(2\xi\beta\sqrt{T_1}+\left[\xi^3\beta d+ 2\xi\beta\sqrt{d}\right]\ln\frac{s}{\delta}\right)\Delta^2_k + \frac{4}{\xi}\ep. \label{eqn:bound-3-2}
\end{eqnarray}
By combining the bounds in (\ref{eqn:bound-3-1}) and (\ref{eqn:bound-3-2}), under the assumption that at least one of the two conditions in (\ref{eqn:condition-1}) and (\ref{eqn:condition-2}) is true, by setting $\mu = B/8$, we have
\begin{eqnarray*}
\lb(\wh_{k+1}) - \lb(\w_*) \leq \frac{1}{T_1}\left(2\xi\beta\sqrt{T_1}+\left[\xi^3\beta d+ 2\xi\beta\sqrt{d}\right]\ln\frac{s}{\delta}\right)\Delta^2_k + \frac{4}{\xi}\ep,
\end{eqnarray*}
implying
\begin{eqnarray*}
 \|\wh_{k+1} - \w_*\| \leq \frac{2}{\alpha T_1}\left(2\xi\beta\sqrt{T_1}+\left[\xi^3\beta d+ 2\xi\beta\sqrt{d}\right]\ln\frac{s}{\delta}\right)\Delta^2_k + \frac{8}{\alpha \xi}\ep.
\end{eqnarray*}
We complete the proof by using Lemma~\ref{lemma:d}, which states that the probability for either of the two conditions hold is no less than $1 - \delta$.

%---------------------------------------------------------------------------------------------------------------------------
% SGD with Optimal Rates
%---------------------------------------------------------------------------------------------------------------------------

%%%%%%%%%%%%%%%%%%%%%%%%%%%%%%%%%%%%%%%%%%%%%%%%%%%
\section{Proofs of Sample Complexity}\label{sec:clipped-proofs}
%\addcontentsline{toc}{section}{\protect\numberline{}Appendix}

\subsection{Proof of Lemma~\ref{lemma:d}}
\label{app:lemma1}
The proof is based on the Bernstein's inequality for martingales which can be found in Lemma~\ref{theorem:bernsteinB}. Define martingale difference $d_k^t = \left\langle \w_k^t - \w_*, \E_t[\v_k^t] - \v_k^t\right\rangle$  and martingale $D_k = \sum_{t=1}^{T_1} d_k^t$. Let $\Sigma_T^2$ denote the  conditional variance  as
\begin{eqnarray*}
    \Sigma_T^2 = \sum_{t=1}^{T_1} \E_{t}\left[(d_k^t)^2 \right]
    &\leq& \sum_{t=1}^{T_1} \E_t\left[\left\|\E_t[\v_k^t] - \v_k^t\right\|^2 \right]\|\w_k^t - \w_*\|^2 \\
    &\leq& \sum_{t=1}^T d \gamma_k^2  \|\w_k^t - \w\|^2 = d \gamma_k^2  W_k,
\end{eqnarray*}
which follows from the Cauchy's Inequality and the  definition of clipping. 

Define $M = \max\limits_{t} |d_k^t| \leq 2\sqrt{d}\gamma_k\Delta_k$.
%\[
%K = \max\limits_{t} |d_k^t| \leq 2\sqrt{d}\gamma_k\Delta_k
%\]
To prove the inequality in Lemma~\ref{lemma:d}, we follow the idea of peeling process~\cite{koltchinskii-2011-oracle}. Since $W_k \leq 4R^2 T_1$, we have
\begin{eqnarray*}
\lefteqn{\Pr\left(D_k \geq 2\gamma_k\sqrt{W_k d \rho} + \sqrt{2}M\rho/3\right)} \\
& = & \Pr\left(D_k \geq 2\gamma_k\sqrt{W_k d \rho} + \sqrt{2}M\rho/3, W_k \leq 4R^2T_1\right) \\
& =  & \Pr\left(D_k \geq 2\gamma_k\sqrt{W_k d \rho} + \sqrt{2}M\rho/3, \Sigma_T^2 \leq \gamma_k^2 d W_k, W_k \leq 4R^2T_1 \right) \\
& \leq  & \Pr\left(D_k \geq 2\gamma_k\sqrt{W_k d \rho} + \sqrt{2}M\rho/3, \Sigma_T^2 \leq \gamma_k^2 d W_k, W_k \leq \ep T_1/(2\beta\mu) \right) \\
&  & + \sum_{i=1}^s \Pr\left(D_k \geq 2\gamma_k\sqrt{W_k d\rho} + \sqrt{2}M\rho/3, \Sigma_T^2 \leq \gamma_k^2 d W_k, \frac{\ep 2^{i-1} T_1}{2\beta\mu} < W_k  \leq \frac{\ep 2^i T_1}{2\beta\mu} \right) \\
& \leq  & \Pr\left(W_k \leq \frac{\ep  T_1}{2\beta\mu}\right) + \sum_{i=1}^s \Pr\left(D_k \geq \sqrt{\frac{\ep 2^{i+1}   T_1\gamma_k^2 d}{2\beta\mu}\rho} + \frac{\sqrt{2}}{3}M\rho, \Sigma_T^2 \leq \frac{ \ep 2^i T_1\gamma_k^2 d}{2\beta\mu}\right) \\
& \leq  & \Pr\left(W_k \leq \frac{\ep  T_1}{2\beta\mu}\right) + se^{-\rho},
\end{eqnarray*}
where $s$ is given by
\[
    s = \left\lceil \log_2 \frac{8\beta\mu R^2}{\ep} \right\rceil.
\]
The last step follows the Bernstein inequality for martingales. We complete the proof by setting $\rho= \ln(s/\delta)$ and using the fact that
\[
    2\gamma_k\sqrt{W_k \rho d} \leq \frac{1}{L}W_k + \gamma_k^2 \rho d L .
\]
%-------------------------------------------------------------------------------------------------------------------------------
%
%-------------------------------------------------------------------------------------------------------------------------------
\subsection{Proof of Lemma~\ref{lemma:e}}
\label{app:lemma2}
To bound $E_k$, we need the following two lemmas.  The first lemma  bounds  the deviation of the expected value of a clipped random variable from  the original variable, in terms of its variance (Lemma A.2 from~\cite{DBLP:journals/corr/abs-1108-4559}).
\begin{lemma}
\label{lem:clip}
Let $X$ be a random variable, let $\Xt = \rm{clip}(X, C)$ and assume that $|\E[X]| \leq C/2$ for some $C > 0$. Then
\[
    |\E[\Xt] - \E[X]| \leq \frac{2}{C}\left|\rm{Var}[X]\right|
\]
\end{lemma}

Another key  observation used for bounding  $E_k$  is the  fact that for any non-negative $\beta$-smooth convex function,  we have the following self-bounding property. We note that this self-bounding property has been used in~\cite{srebro-2010-smoothness}  to get better (optimistic) rates of convergence for non-negative smooth losses.

\begin{lemma}
 \label{lem:smooth-chap-3}
For any $\beta$-smooth non-negative function $f: \R\rightarrow \R$, we have $|f'(w)| \leq \sqrt{4\beta f(w)}$
\end{lemma}
\begin{proof} See Appendix~\ref{chap:appendix-convex}
\end{proof}
%As a simple proof,  first from the smoothness assumption,  by  setting $w_1 = w_2 - \frac{1}{\beta}f'(w_2)$ in (\ref{eqn:smoth})  and rearranging the terms we obtain $f(w_2) - f(w_1) \geq \frac{1}{2 \beta} | f'(w_2)|^2$.  On the other hand, from the convexity of loss function  we have $f(w_1) \geq f'(w_2) + \dd{f'(w_1)}{w_1 - w_2}$. Combining these inequalities and considering the fact that the function is non-negative gives the desired inequality. \\

\begin{proof}[Proof of Lemma~\ref{lemma:e}] To apply the above lemmas, we write $e_k^t$ as
\begin{eqnarray*}
e_k^t & = & \sum_{i=1}^d \E_t\left[\ell'( \langle\w_k^t, \x_k^t \rangle, y_t)[\x_k^t]_i - \clip\left(\gamma_k, \ell'(\langle\w_k^t, \x_k^t \rangle, y_t)[\x_k^t]_i \right) \right] [\w_k^t - \w_*]_i
\end{eqnarray*}
In order to apply Lemma~\ref{lem:clip}, we check if the following condition holds
\begin{eqnarray}
    \gamma_k \geq 2\left|\E_t\left[\ell'\left(\langle \w_k^t, \x_k^t \rangle, y_t\right)[\x_k^t]_i\right]\right| \label{eqn:cond-1}
\end{eqnarray}
Since
\begin{eqnarray*}
& & \left|\E_t\left[\ell'\left(\langle \w_k^t, \x_k^t \rangle, y_t\right)[\x_k^t]_i\right]\right| \\
& \leq & \left|\E_t\left[\left\{\ell'\left(\langle \w_k^t, \x_k^t \rangle, y_t\right) - \ell'\left(\langle \w_*, \x_k^t \rangle, y_t\right)\right\}[\x_k^t]_i\right]\right| + \left|\E_t\left[\ell'\left(\langle \w_*, \x_k^t \rangle, y_t \right)[\x_k^t]_i\right]\right| \\
& \leq & \beta \|\w_k^t - \w_*\| \leq \beta \Delta_k
\end{eqnarray*}
where the last inequality follows from $\E_t\left[\ell'\left(\langle \w_*, \x_k^t \rangle, y_t\right)[\x_k^t]_i\right] = 0$ since $\w_*$ is the minimizer of $\lb(\w)$, we thus have
\[
    \gamma_k = 2\xi\beta\Delta_k \geq 2\beta\Delta_k \geq 2\left|\E_t\left[\ell'\left(\langle \w_k^t, \x_k^t \rangle, y_t\right)[\x_k^t]_i\right]\right|
\]
where $\xi \geq 1$, implying that the condition in (\ref{eqn:cond-1}) holds. Thus, using Lemma~\ref{lem:clip}, we have
\begin{eqnarray*}
e_k^t &\leq& \sum_{i=1}^d \left|[\w_k^t - \w_*]_i\right|\frac{1}{\gamma_k}\E_t\left[\left(\ell'(\langle\w_k^t, \x_k^t \rangle, y_t)[\x_k^t]_i\right)^2\right] \\
 &\leq& \frac{2\|\w_k^t - \w_*\|_{\infty}}{\gamma_k} \E_t\left[\left(\ell'(\langle\w_k^t, \x_k^t \rangle, y_t)\right)^2\right]
\end{eqnarray*}
Using Lemma~\ref{lem:smooth-chap-3} to upper bound the right hand side, we further simplify the above bound for $e_k^t$ as
\begin{eqnarray*}
e_k^t  &\leq& \frac{8\beta\|\w_k^t - \w_*\|_{\infty}}{\gamma_k}\E_t\left[\ell\left(\langle \w_k^t, \x_k^t \rangle, y_t \right) \right] \\
&=& \frac{8\beta\|\w_k^t - \w_*\|_{\infty}}{\gamma_k}\lb(\w_k^t)\\
&\leq&  \frac{8\beta \Delta_k}{\gamma_k}\lb(\w_k^t)\\
 &= & \frac{4}{\xi}\lb(\w_k^t)
\end{eqnarray*}
where the second inequality follows from $\|\w_k^t - \w_*\|_{\infty} \leq \|\w_k^t - \w_*\| \leq \Delta_k$. % and Lemma~\ref{lem:smooth-chap-3}, and 
Therefore we obtain
\begin{eqnarray*}
E_k  =  \sum_{t=1}^{T_1} e_k^t \leq  \frac{4}{\xi}\sum_{t=1}^{T_1} \lb(\w_k^t) &=& \frac{4}{\xi}\sum_{t=1}^{T_1} \lb(\w_*) + \frac{4}{\xi}\sum_{t=1}^{T_1} \lb(\w_k^t) - \lb(\w_*) \\
& \leq & \frac{4T_1}{\xi} \lb(\w_*) + \frac{4\beta}{\xi}\sum_{t=1}^{T_1} \|\w_k^t - \w_*\|^2\\ &=& \frac{4T_1}{\xi} \lb(\w_*) + \frac{4\beta}{\xi}W_k,
\end{eqnarray*}
where the second inequality follows from the smoothness assumption of $\lb(\w)$.
\end{proof}

%%%%%%%%%%%%%%%%%%%%%%%%%%%%%%%%%%%%%%%%%%%%%%%%%%
%%%%%%%%%%%%%%%%%%%%%%%%%%%%%%%%%%%%%%%%%%%%%%%%%%%
\section{Summary} \label{sec:3:conclusion}
In this chapter, we have studied the sample complexity of passive learning when the target expected  risk is given to the learner as  prior knowledge. The crucial fact about target risk assumption is that, it can be fully exploited by the learning algorithm and  stands in  contrast to most  common types of prior knowledges that usually enter into the generalization bounds and are often perceived as a rather  crude way to incorporate such assumptions. We showed that by explicitly employing the target risk $\ep$ in a properly designed stochastic optimization algorithm, it is possible to attain the given target risk $\ep$ with a logarithmic sample complexity $\log \left(\frac{1}{\ep}\right)$,  under the assumption that the loss function is both strongly convex and smooth.

There are various directions for future research. The current study is restricted to the parametric setting where the hypothesis space is of finite dimension. It would be interesting to see how to achieve a logarithmic sample complexity in a non-parametric setting where hypotheses lie in a functional space of infinite dimension. Evidently, it is impossible to extend the current algorithm for the non-parametric setting;  therefore additional analysis tools are needed to address the challenge of infinite dimension arising from the non-parametric setting.  It is also an interesting problem to relate target risk assumption we made here to the low noise margin condition which is often made in active learning for binary classification since both settings appear to share the same sample complexity. However it is currently unclear how to derive a connection between these two settings. We believe this issue is worthy of further exploration and leave it as an open problem. 
\label{sec:conclusion}

%---------------------------------------------------------------------------------------------------------------------------
% Related work
%---------------------------------------------------------------------------------------------------------------------------
\section{Bibliographic Notes} \label{sec:3:related}
Sample complexity of passive learning is well established  and goes back to early works in the learning theory where the lower bounds  $\Omega\left(\frac{1}{\epsilon} ( \log \frac{1}{\epsilon}+ \log \frac{1}{\delta})\right)$ 
and $\Omega\left(\frac{1}{\epsilon^2} \left( \log \frac{1}{\epsilon}+ \log \frac{1}{\delta}\right)\right)$ were obtained   in classic PAC and general agnostic PAC settings, respectively~\cite{ehrenfeucht1989general,learnabilityvcdim89,anthony1999neural}.
There has been an upsurge of interest  over the last decade in finding tight upper bounds on the sample complexity  by utilizing prior knowledge on the  analytical properties of the loss function,   that led to stronger generalization bounds  in agnostic PAC setting. In~\cite{DBLP:journals/tit/LeeBW98}  \textit{fast} rates obtained for squared loss, exploiting the strong convexity of this loss function, which only holds under pseudo-dimensionality assumption.  With the recent development in online strongly convex optimization~\cite{hazan-log-newton}, fast rates approaching $O(\frac{1}{\epsilon} \log \frac{1}{\delta}) $ for convex Lipschitz strongly convex loss functions has been obtained in~\cite{fastrates2008,compl-linear-nips-2008}. For smooth non-negative loss functions,~\cite{srebro-2010-smoothness} improved the sample complexity  to \textit{optimistic} rates
\[O\left (\frac{1}{\epsilon}\left(\frac{\eo+\epsilon}{\epsilon} \right) \left( \log^3 \frac{1}{\epsilon}+ \log \frac{1}{\delta} \right)\right)\] 
for non-parametric learning using the notion of local Rademacher complexity~\cite{bartlett2005local}, where $\eo$ is the optimal risk.

The proposed ClippedSGD algorithm is related to the recent studies that examined the learnability from the viewpoint of stochastic convex optimization. In~\cite{sridharan-2012-learning,shalev-shwartz:2010:learnability}, the authors presented learning problems that are learnable by stochastic convex optimization but not by empirical risk minimization (ERM). Our work follows this line of research. The proposed algorithm achieves the sample complexity of $O(d\log(1/\ep))$ by explicitly incorporating the target expected risk $\ep$ into the stochastic convex optimization algorithm. It is however difficult to incorporate such knowledge into the framework of ERM. Furthermore, it is worth noting that in~\cite{ramdas-2013-optimal,sridharan-2012-learning,DBLP:conf/nips/RakhlinST10,agnostic-online-2009}, the authors explored the connection between online optimization and statistical learning in the opposite direction. This was done by exploring the complexity measures developed in statistical learning for the learnability of online learning. We note that our work does not contradict the lower bound in~\cite{srebro-2010-smoothness} because a \textit{feasible} target risk $\ep$ is given in our learning setup and is fully exploited by the proposed algorithm. Knowing that the target risk $\ep$ is feasible makes it possible to improve the sample complexity from $\O({1}/{\ep})$ to $\O(\log({1}/{\ep}))$. We also note that although the logarithmic sample complexity is known for active learning~\cite{hanneke-thesis,mariatruesample2010}, we are unaware of any existing passive learning algorithm that is able to achieve a logarithmic sample complexity by incorporating any kind of prior knowledge.

  {The proposed algorithm is also closely related to the recent works that stated  $O(1/n)$ is the optimal convergence rate for stochastic optimization when the objective function is strongly convex~\cite{primal-dual-nemirovsky,hazan-2011-beyond,ICML2012Rakhlin}. In contrast, the proposed algorithm is able to achieve a geometric convergence rate for a target optimization error. Similar to the previous argument, our result does not contradict the lower bound given in~\cite{hazan-2011-beyond} because of the knowledge of a feasible optimization error. Moreover, in contrast to the multistage algorithm  in~\cite{hazan-2011-beyond} where the size of stages increases exponentially,  in our algorithm, the size of each stage  is fixed to be a constant.}

\chapter{Statistical Consistency of Smoothed Hinge Loss} \label{chap:smooth_consistency}

\def \R {\mathbb{R}}
\def \x {\mathbf{x}}
\def \E {\mathrm{E}}
\def \Mt {\widetilde{M}}
\def \Rt {\mathcal{R}}
\def \L {\mathcal{L}}
\def \S {\mathcal{S}}
\def \e {\mathbf{e}}
\def \ab {\mathbf{\alpha}}
\def \B {\mathcal{B}}
\def \N {\mathcal{N}}
\def \u {\mathbf{u}}
\def \z {\mathbf{z}}
\def \xh {\widehat{\x}}
\def \abh {\widehat{\ab}}
\def \uh {\widehat{\u}}
\def \y {\mathbf{y}}
\def \diag {\mbox{diag}}
\def \Gh {\widehat{G}}
\def \Sb {\overline{\S}}
\def \w {\mathbf{w}}
\def \wh {\widehat{\w}}
\def \th {\widehat{\tau}}
\def \conv {\mbox{conv}}
\def \v {\mathbf{v}}
\def \sgn {\mbox{sign}}
\def \a {\mathbf{a}}
\def \Er {\mathcal{E}}
\def \K {\mathcal{K}}
\def \b {\mathbf{b}}
\def \Uh {\widehat{U}}
\def \D {\mathcal{D}}

\def \span {\mbox{span}}
\def \Kh {\widehat{K}}
\def \ah {\widehat{\alpha}}
\def \Dh {\widehat{\D}}
\def \Xh {\widehat{X}}
\def \Xt {\widetilde{X}}
\def \sgn {\mbox{sgn}}
\def \P {\mathcal{P}}
\def \diam {\mbox{diam}}
\def \dis {\mbox{DIS}}
\def \C {\mathcal{C}}
\def \A {\mathcal{A}}
\def \et {\widetilde{\ell}}
\def \ep {\ell_{\phi}}
\def \ebin {\ell_b}
\def \Rh {\widehat{R}}
\def \fh {\widehat{h}}
\def \ind {\mathbb{I}}
\def \X {\mathcal{X}}
\def \Y {\mathcal{Y}}
\def \H {\mathscr{H}_{\kappa}}
\def \E {\mathbb{E}}
\def \P {\mathbb{P}}
\def \Ex {\mathscr{E}}
\def \F {\mathcal{H}}
\def \Pxy {\mathcal{D}}
\def \Rp {R_{\phi}}
\def \Ld {\mathcal{L}_{\mathcal{D}}}
\def \Lds {\mathcal{L}^{*}_{\mathcal{D}}}
\def \Ldphi {\mathcal{L}^{\phi}_{\mathcal{D}}}
\def \Ldphis {\mathcal{L}^{\phi,*}_{\mathcal{D}}}
\def \Ls {\mathcal{L}_{\mathcal{S}}}
\def \Lss {\mathcal{L}_^{*}{\mathcal{S}}}
\def \Lsphi {\mathcal{L}^{\phi}_{\mathcal{S}}}
\def \Lsphis {\mathcal{L}^{\phi,*}_{\mathcal{S}}}
\def \Ldhinges {\mathcal{L}^{\rm{Hinge},*}_{\mathcal{D}}}
\def \Ldhinge {\mathcal{L}^{\rm{Hinge}}_{\mathcal{D}}}

\sloppy

In Chapter~\ref{chap-background} we discussed that convex surrogates of the  0-1 loss are highly preferred because of the computational  and theoretical virtues that convexity brings in and  most prominent practical methods studied in machine learning make significant use of convexity.  This is  of more importance if we consider smooth   surrogates  as  witnessed by the fact that the smoothness  is  further beneficial both computationally- by attaining an \textit{optimal}  convergence rate for optimization, and in a statistical sense- by providing an improved \textit{optimistic} rate for generalization bound.

This chapter concerns itself with the statistical consistency of smooth convex surrogates. The statistical consistency  finds general quantitative relationships between the excess risk errors associated with convex  and those associated with 0-1 loss.  Consistency results provide reassurance that optimizing a surrogate does not ultimately hinder the search for a function that achieves the binary excess risk, and thus allow such a search to proceed within the scope of computationally efficient algorithms. Statistical consistency of surrogates under conditions such as convexity is a well studied problem in learning community  and quantitative relationships between binary risk and convex excess risk has been established.   In this chapter we investigate the smoothness property from the viewpoint of statistical consistency and show  how it affects the binary excess risk. We show that in contrast to optimization   and generalization errors  that favor the choice of smooth surrogate loss, the smoothness of loss function may degrade the binary excess risk. Motivated by this negative result, we provide a unified analysis that integrates optimization error, generalization bound, and the error in translating convex excess risk into a binary excess risk when examining the impact of smoothness on the binary excess risk.  We  show that under favorable conditions appropriate choice of smooth convex loss will result in a binary excess risk that is better than $O(1/\sqrt{n})$. 

The reminder of this paper is organized as follows. In Section~\ref{sec:setting} we set up notation and describe the setting. Section~\ref{sec:calibration} briefly discusses the classification-calibrated convex surrogate losses  on which our analysis relies. We derive the $\psi$-transform for smoothed hinge loss and elaborate  its binary excess risk in Section~\ref{sec:smoothed-hinge}. Section~\ref{sec:analysis-hinge} provides a unified analysis of three types of errors and derives conditions  in terms of smoothness to obtain better rates for the binary excess risk. The omitted proofs are included in Section~\ref{sec-proofs-hinge}. Section~\ref{sec:conclusion-hinge} concludes the paper. 

%----------------------------------------------------------------------------------------------------------------------------
%
%----------------------------------------------------------------------------------------------------------------------------
\section{Motivation} \label{sec:setting}
Let $\S = \left((\x_1, y_1), (\x_2, y_2), \cdots, (\x_n, y_n) \right)$ be a set of i.i.d. samples drawn from an unknown distribution $\Pxy$ over $\Xi = \X\times \{-1, +1\}$, where $\x_i \in \X \subseteq \R^d$ is an instance and $y_i \in \{-1, +1\}$ is the binary class assignment for $\x_i$. Let $\kappa(\cdot, \cdot)$ be an universal kernel  and let $\H$ be the Reproducing Kernel Hilbert Space (RKHS) endowed with kernel $\kappa(\cdot, \cdot)$. According to~\cite{zhou2003capacity}, $\H$ is a rich function space whose closure includes all the smooth functions.  We consider  predictors from $\H$ with bounded norm to form the measurable function class $\F = \{h \in \H: \|h\|_{\H} \leq B\}$.

For a  function  $h:\X \mapsto \R$,  the risk of $h$  is defined as:
 $$\Ld(h) = \E_{(\x,y)\sim\Pxy} \left[ \mathbb{I}{\left[yh(\x) \leq 0\right]} \right]= \mathbb{P}\left[yh(\x) \leq 0 \right].$$ 
 
 Let $h_*$ be the optimal classifier that attains the minimum risk, i.e. $h_* = \arg\min_{h} \mathbb{P}\left[yh(\x) \leq 0 \right]$. We assume $h_* \in \H$ with $\|h_*\|_{\H} \leq B$. This boundedness condition  is satisfied for any RKHS with a bounded kernel (i.e. $\sup_{\x \in \X} \kappa(\x, \x) \leq B$). Henceforth, let $\Lds$ stand for the minimum achievable risk by the optimal classifier $h_*$, i.e., $\Lds = \Ld(h_*)$. Define the \textit{binary excess risk} for a prediction function $h \in \F$ as
$$\Ex(h) = \Ld(h) - \Lds.$$

Our goal is to efficiently learn a prediction function $h \in \F$ from the training examples  in $\S$ that minimizes the binary excess risk $\Ex(h)$. Many studies of binary excess risk assume that the optimal classifier $h \in \F$ is learned by minimizing the empirical binary risk, $\min_{h \in \F} \frac{1}{n}\sum_{i=1}^n \mathbb{I}{[y_i h(\x_i) \leq 0]}$, an approach that is usually referred to as Empirical Risk Minimization (ERM)~\cite{vapnik1998statistical}. To understand the generalization performance of the classifier learned by ERM, it is important to have upper bounds on the excess risk of the empirical minimizer that hold with a high probability and that take into account complexity measures of classification functions. It is well known that, under certain conditions, direct empirical classification error minimization is consistent~\cite{vapnik1998statistical}  and achieves a fast convergence rate under low noise situations~\cite{mammen1999smooth}.

One shortcoming of the ERM based approaches is that they need to minimize  0-1 loss, leading to non-convex optimization problems that are potentially NP-hard~\footnote{We note that several works~\cite{kalai2008agnostically,kalai2009isotron} provide efficient algorithms for direct 0-1 empirical error minimization but under strong (unrealistic) assumptions on data distribution or label generation.}~\cite{arora1993hardness,hoffgen1995robust}.   A common practice  to circumvent this difficulty is to replace the indicator function $\mathbb{I} [\cdot \leq 0]$ with some convex loss  $\phi(\cdot)$ and find the optimal solution by minimizing the convex surrogate loss. Examples of such surrogate loss functions for 0-1 loss  include  logit loss $\phi_{\log}(h; (\x,y)) = \log (1+\exp(-y h(\x)))$ in logistic regression~\cite{friedman2000additive},  hinge loss $\phi_{\text{Hinge}}(h; (\x,y)) = \max(0, 1- y h(\x))$ in support vector machine (SVM)~\cite{cortes1995support} and exponential loss $\phi_{\text{exp}}(h; (\x,y)) = \exp(-y h(\x))$ in AdaBoost~\cite{freund1995desicion}. Given a convex surrogate loss function $\phi: \R \mapsto \R_{+}$ (e.g., hinge loss, exponential loss, or logistic loss) we define the risk with respect to the convex loss $\phi$ (convex risk or $\phi$-risk) as 
$$\Ldphi(h) = \E_{(\x,y)\sim\Pxy}[\phi(yh(\x))].$$ 

Similarly we define the \textit{optimal} $\phi$-risk  as $\Ldphis = \inf_{h \in \F} \E_{(\x,y)\sim \Pxy}[\phi(yh(\x))]$. The \textit{excess $\phi$-risk} or \textit{convex excess  risk} of a classifier $h \in \F$ with respect to the convex surrogate loss $\phi(\cdot)$ is defined as
$$\Ex_{\phi}(h) = \Ldphi(h) - \Ldphis.$$

An important line of research in statistical learning theory focused on relating the convex excess risk $\Ex_{\phi}(h)$ to the binary excess risk $\Ex(h)$ that will be elaborated in next section.

It is known that under  mild conditions, the classifier learned by minimizing the empirical loss of convex surrogate is consistent to the Bayes classifier ~\cite{zhang2004statistical,lugosi2004bayes,jiang2004process,lin2004note,steinwart2005consistency,bartlett2006convexity}.  For instance, it was shown in~\cite{bartlett2006convexity} that the necessary and sufficient condition for a convex loss $\phi(\cdot)$ to be consistent with the binary loss is that $\phi(\cdot)$ is  differentiable at origin and $\phi'(0) < 0$. It was further established in the same work that the binary excessive risk can be upper bound by the convex excess risk through a $\psi$-transform that depends on the surrogate convex loss $\phi(\cdot)$.

Since the choice of convex surrogates could significantly affect the binary excess risk, in this chapter, we will investigate the impact of the smoothness of a convex loss function on the binary excess risk. This is motivated by the recent results that show the advantages of using smooth convex surrogates in reducing the optimization complexity and the generalization error bound. More specifically, ~\cite{nesterov2004introductory,tseng:2009:accelerated} show that a faster convergence rate (i.e., $O(1/T^2)$) can be achieved by first order methods when the objective function to be optimized is convex and smooth such as  accelerated gradient descent method introduced in Chapter~\ref{chap-background}; in~\cite{srebro-2010-smoothness}, the authors show that a smooth convex loss will lead to a better optimistic generalization error bound rooted in the self-bounding property of smooth losses (Lemma~\ref{lem:smooth}). Given the positive news of using smooth convex surrogates, an open research question is how the smoothness of a convex surrogate will affect the binary excess risk. The answer to this question, as will be revealed later, is negative: the smoother the convex loss, the poorer approximation will be for the binary excess risk. Thus, the second contribution of this work is to integrate  these results for smooth convex losses, and examine the overall effect of replacing 0-1 loss with a smooth convex loss when taking into account three sources of errors, i.e. the optimization error, the generalization error, and the error in translating the convex excess risk into the binary risk. As we will show, under favorable conditions, appropriate choice of smooth convex loss will result a binary excess risk better than $O(1/\sqrt{n})$.

%----------------------------------------------------------------------------------------------------------------------------
%
%----------------------------------------------------------------------------------------------------------------------------
\section{Classification Calibration and Surrogate Risk Bounds}\label{sec:calibration}
Although it is computationally convenient to minimize the empirical risk based on a convex surrogate, the ultimate goal of any classification method is to find a function $h \in \H$ that minimizes the binary loss. Therefore, it is crucial to investigate the conditions which guarantee that if the $\phi$-risk of $h$ gets close to the optimal $\Ldphis$, the binary risk of $h$ will also approach the optimal binary risk $\Lds$.  This question has been an active trend in statistical learning theory over the last decade where the necessary and sufficient conditions have been established for relating the binary excess risk to a convex excess risk ~\cite{zhang2004statistical,lugosi2004bayes,jiang2004process,lin2004note,steinwart2005consistency,bartlett2006convexity}.

In this chapter we follow the strategy introduced in~\cite{bartlett2006convexity} in order
to relate the binary excess risk to the excess $\phi$-risk. Their methodology, through  the notion of    \textit{classification calibration}, allows us to find quantitative
relationship between the excess risk associated with $\phi$ and the excess risk associated with 0-1 loss.  It is  established in~\cite{bartlett2006convexity} that the binary excessive risk can be bounded by the convex excess risk, based on the convex loss function $\phi$, through a $\psi$-transform.

\begin{definition} Given a loss function $\phi: \R \mapsto [0, \infty)$, define the function  $\psi: [0,1] \mapsto [0, \infty)$ by
$$\tilde{\psi} (z) = H^-\left(\frac{1+z}{2}\right) - H\left(\frac{1+z}{2}\right)$$
where
\[
H^-(\eta)=\inf_{\alpha:\alpha(2\eta-1)\leq 0}\left( \eta \phi(\alpha)+(1-\eta) \phi(-\alpha)\right)\textrm{ and } H(\eta)= \inf_{\alpha\in\R} \left( \eta \phi(\alpha)+(1-\eta) \phi(-\alpha)\right).
\]
The transform function $\psi: [0,1]\mapsto [0,\infty)$ is defined to be the convex closure of $\tilde{\psi}$.
\end{definition}
The following theorem from~\cite[Theorem 1]{bartlett2006convexity} shows that the binary excess risk can be bounded by the convex excess risk using transform function $\psi: [0,1] \mapsto [0, \infty)$ that depends on the surrogate convex loss function.
%let $R(h)$ and $\Ldphis(h)$ be the binary and convex risk for prediction function $f$, respectively, and let $R^*$ and $R^*_{\phi}$ be the optimal binary and $\phi$-risk, respectively.  T
\begin{theorem} \label{thm:classification-caliberated} For any non-negative loss function $\phi(\cdot)$, any measurable function $h \in \F$, and any probability distribution $\Pxy$ on $\X\times\Y$, there is  a nondecreasing function $\psi: [0,1] \mapsto [0, \infty)$ that
\begin{eqnarray}\label{eqn:bound-phi}
\psi(\Ld(h) - \Lds) \leq \Ldphi(h) - \Ldphis
\end{eqnarray}
holds. Here the minimization is taken over all measurable functions.
\end{theorem}
%where $R(h) = \E_{(\x,y)\sim \Pxy}\left[\ind (y f(\x) \leq 0) \right]$, $R^* = \min_{f} R(h)$,$\Ldphis(h) = \E_{(\x, y) \sim \Pxy}\left[ \phi(yf(\x))\right]$, and $\Ldphis = \min_f \Ldphis(h)$.
\begin{definition} A convex loss $\phi$ is classification-calibrated if, for any $\eta \neq 1/2$,
\[
H^-(\eta) > H(\eta).
\]
\end{definition}
This condition is essentially an extension of~\cite[Theorem 2.1]{zhang2004statistical} and can be viewed as a form of Fisher consistency that is appropriate for classification.

It has been shown  in~\cite{bartlett2006convexity} that the necessary and sufficient condition for a convex loss $\phi(z)$ to be classification-calibrated  is if it is differentiable at the origin  and $\phi'(0) < 0$. In particular, for a certain convex function $\phi(\cdot)$, the $\psi$-transform can be computed by
\begin{eqnarray*}
\psi(z) = \inf\limits_{\alpha z \leq 0}\left(\frac{1 + z}{2}\phi(\alpha) + \frac{1 - z}{2}\phi(-\alpha)\right) - \inf\limits_{\alpha \in \R}\left(\frac{1 + z}{2}\phi(\alpha) + \frac{1 - z}{2}\phi(-\alpha)\right),
\end{eqnarray*}
that can be further simplified as $\psi(z) = \phi(0)-H\left(\frac{1+z}{2}\right)$ when $\phi$ is classification-calibrated.  Examples of $\psi$-transform for the convex surrogate functions of known practical algorithms mentioned before are as follows: (i) for hinge loss $\phi(\alpha) = \max(0, 1-\alpha)$ , $\psi(z) = |z|$, (ii) for exponential loss $\phi(\alpha)=e^{-\alpha}$, $\psi(z) = 1 - \sqrt{1 - z^2} \geq z^2/2$, and (iii) for truncated quadratic loss $\phi(\alpha) = [\max(0, 1 - \alpha)]^2$, $\phi(z) = z^2$.

 \begin{remark}
 We note that the inequality in (\ref{eqn:bound-phi}) provides insufficient guidance on choosing appropriate loss function. A few brief comments are appropriate.  First, it does not measure explicitly how the choice of the convex surrogate $\phi(\cdot)$ affects the excess risk $\Ldphi(h) - \Ldphis$. Second, it does not take into account the impact of loss function on optimization efficiency, an important issue for practitioners when dealing with big data. It is thus unclear, from Theorem~\ref{thm:classification-caliberated}, how to choose an appropriate loss function that could result in a small generalization error for the binary loss when the computational time is limited. In this chapter, we address these limitations by examining a family of convex losses that are constructed by smoothing the hinge loss function using different smoothing parameters. We study the binary excessive risk of the learned classification function by taking into account errors in optimization, generalization, and translation of convex excess risk into binary excess risk.
\end{remark}
%----------------------------------------------------------------------------------------------------------------------------
%
%----------------------------------------------------------------------------------------------------------------------------
\section{Binary Excess Risk for Smoothed Hinge Loss}\label{sec:smoothed-hinge}

As stated before, to efficiently learn a prediction function  $h \in \F$, we will replace the binary loss with a smooth convex loss. Since hinge loss is one of the most popular loss functions used in machine learning and is the loss of choice for classification problems in terms of the margin error~\cite{ben2012minimizing}, in this work, we will focus on the smoothed version of the hinge loss. Another advantage of using the hinge loss is that its $\psi$-transform is a linear function. Compared with the $\psi$-transforms of other popular convex loss functions (e.g. exponential loss and truncated square loss) that are mostly quadratic, using the hinge loss as convex surrogate will lead to a tighter bound for the binary excess risk.

The smoothed hinge loss considered in this chapter is defined as
\begin{eqnarray}\label{eqn:smooth-hing}
\phi(z; \gamma) = \max\limits_{\alpha  \in [0, 1]} \alpha (1 - z) + \frac{1}{\gamma}\mathscr{R}(\alpha),
\end{eqnarray}
where $\mathscr{R}(\alpha) = -\alpha\log\alpha - (1 - \alpha)\log(1 - \alpha)$ and $\gamma > 0$ is the smoothing parameter.  It is straightforward  to verify that the loss function in~(\ref{eqn:smooth-hing}) can be simplified as
\[
\phi(z; \gamma) = \frac{1}{\gamma}\log(1 + \exp(\gamma(1 - z))).
\]
It is not immediately clear from Theorem~\ref{thm:classification-caliberated} how the relationship between smooth convex excess risk $\Ex_{\phi}(\cdot)$ and binary excess risk is affected by the smoothness parameter $\gamma$.
In addition, as discussed  in~\cite{bartlett2006convexity}, whereas conditions such as convexity and smoothness have natural relationship to optimization and generalization, it is not 
immediately obvious how properties such as convexity and smoothness of convex surrogate  relates to  statistical consequences.  In what follows,  we show that, indeed  smoothness of loss function has a \textit{negative} statistical consequence and can degrade the binary excess risk.

%----------------------------------------------------------------------------------------------------------------------------
%
%----------------------------------------------------------------------------------------------------------------------------
\subsection{$\psi$-Transform for Smoothed Hinge Loss}
The first step in our analysis is to derive the $\psi$-transform for the loss function defined in~(\ref{eqn:smooth-hing}) as stated in the  following theorem.

\begin{theorem} \label{thm:psi}  The  $\psi$-transform of smoothed hinge loss with smoothing parameter $\gamma$ is given by
\[
\psi(\eta; \gamma) = -\frac{1 + \eta}{2\gamma}\log\left(\frac{1}{1 + e^{\gamma}}\left[1 + e^{\gamma} \frac{C_1}{1 + \eta}\right]\right) -\frac{1 - \eta}{2\gamma}\log\left(\frac{1}{1 + e^{\gamma}}\left[1 + e^{\gamma} \frac{C_2}{1 - \eta}\right]\right)
\]
where $C_1$ and $C_2$ are defined as
$C_1 = -\eta e^{\gamma} + \sqrt{\eta^2e^{2\gamma} + 1 - \eta^2}$ and $C_2 = \eta e^{\gamma} + \sqrt{\eta^2e^{2\gamma} + 1 - \eta^2}$.
\end{theorem}
The $\psi$-transform given in Theorem~\ref{thm:psi} is too complicated to be useful. The theorem below provides a simpler bound for the $\psi$-transform in terms of the smoothness parameter $\gamma$.
\begin{theorem} \label{thm:psi-simple}
For $\eta \in (-1, 1)$, we have $$\psi(\eta; \gamma) \geq |\eta| - \frac{1}{\gamma}\log\frac{1}{|\eta|}.$$
\end{theorem}
\begin{remark}The bound obtained in Theorem~\ref{thm:psi-simple} demonstrates that when $\gamma$ approaches to infinity, the $\psi$-transform for smoothed hinge loss $\phi(\eta; \gamma)$ becomes $|\eta|$.  According to~\cite{bartlett2006convexity}, the $\psi$-transform for the hinge loss is $\psi(\eta) = |\eta|$. Therefore, this result is consistent with the $\psi$-transform for smoothed  hinge loss, which is the limit of $\phi(z;\gamma)$ as $\gamma$ approaches infinity.
\end{remark}

%\subsection{A bound on binary excess risk $\Ex(h)$ based on smooth convex excess risk $\Ex_{\phi}(h)$}

\subsection{Bounding  $\Ex(h)$ based on  $\Ex_{\phi}(h)$}

Based on the  transform function $\psi(\cdot; \gamma)$  that is  computed for smoothed hinge loss with smoothing parameter $\gamma$, we are now in the position to  bound its corresponding  binary excess risk $\Ex(h)$.  Our main result in this section is the following theorem  that shows  how binary excess risk can be bounded by the excess $\phi$-risk for smoothed hinge loss.
\begin{theorem} \label{thm:bound-binary-excess-risk}
Consider any measurable function $h \in \F$ and the smoothed hinge loss $\phi(\cdot)$ with  parameter $\gamma$ defined in~(\ref{eqn:smooth-hing}). Then, binary excess risk $\Ex(h)$ can be bounded by the smooth convex excess risk $\Ex_{\phi}(h)$ as
\[
\Ex(h) \leq \Ex_{\phi}(h) + \frac{\Ex_{\phi}(h)}{1 + \gamma \Ex_{\phi}(h)} \log\frac{1}{\Ex_{\phi}(h)}.
\]
\end{theorem}
\begin{proof}
Using the result from Theorem~\ref{thm:classification-caliberated}, we have $\Ex_{\phi}(h) \geq \psi(\Ex(h); \gamma)$ and therefore an immediate result from the $\psi$-transform for smoothed hinge loss that is obtained in Theorem~\ref{thm:psi-simple} indicates
\[
\Ex(h) + \frac{1}{\gamma}\log\Ex(h) \leq \Ex_{\phi}(h).
\]
Define $\Delta = \Ex(h) - \Ex_{\phi}(h)$. We have
\[
\Delta + \frac{1}{\gamma}\log(\Delta + \Ex_{\phi}(h)) = \Delta + \frac{1}{\gamma}\log\Ex_{\phi}(h) + \frac{1}{\gamma}\log\left(1 + \frac{\Delta}{\Ex_{\phi}(h)}\right) \leq 0.
\]
Based on the  $\log (1 + x) \leq x$ inequality, the sufficient condition for the above inequality to hold is to have
\[
\Delta + \frac{\Delta}{\gamma\Ex_{\phi}(h)} \leq \frac{1}{\gamma}\log\frac{1}{\Ex_{\phi}(h)}
\]
and therefore
\[
\Delta \leq \frac{\gamma^{-1}}{1 + (\gamma\Ex_{\phi}(h))^{-1}}\log\frac{1}{\Ex_{\phi}(h)} = \frac{\Ex_{\phi}(h)}{1 + \gamma\Ex_{\phi}(h)}\log\frac{1}{\Ex_{\phi}(h)}.
\]
The final bound is obtained by substituting $\Ex(h) - \Ex_{\phi}(h)$ for $\Delta$ in the left hand side of above inequality.
\end{proof}
As indicated by Theorem~\ref{thm:bound-binary-excess-risk}, the smaller the smoothing parameter $\gamma$, the poorer the approximation is in bounding the binary excess $\Ex(h)$ with smooth convex excess risk $\Ex_{\phi}(h)$. On the other hand, the smoothness of loss function has been proven to be beneficial in terms of optimization error and generalization bound. The mixture of negative and positive results  for using smooth convex surrogates motivates us to develop an integrated bound for binary excess risk that takes into account all types of errors. One of the main contributions of this work is to show that under favorable conditions, with appropriate choice of smoothing parameter, the smoothed hinge loss will result in a bound for the binary excess risk better than $O(1/\sqrt{n})$.
%----------------------------------------------------------------------------------------------------------------------------
%
%----------------------------------------------------------------------------------------------------------------------------
\section{A Unified Analysis}\label{sec:analysis-hinge}
Using the smoothed hinge loss, we define the convex loss for a prediction function $h \in \F$ as $\Ldphi(h) = \E[\phi(yh(\x); \gamma)]$. Let $h_{\gamma}^*$ be the optimal classifier that minimizes $\Ldphi(h)$. Similar to the case of binary loss, we assume $h_{\gamma}^* \in \H$ with $\|h_{\gamma}^*\| \leq B$. The \textit{smooth convex excess risk} for a given prediction function $h \in \F$ is then given by $\Ex_{\phi} (h) = \Ldphi(h)  - \Ldphi (h_{\gamma}^*)$.
Given the smooth convex loss $\phi(z;\gamma)$ in~(\ref{eqn:smooth-hing}), we find the optimal classifier by minimizing the empirical convex loss, i.e. $\min_{h \in \H, \|h\|_{\H} \leq B} \Lsphi(h)$, where the empirical convex loss $\Lsphi(h)$ is given by
\begin{eqnarray}\label{eqn:tseng-alg-1}
\Lsphi(h) = \frac{1}{n}\sum_{i=1}^n \phi(y_i h(\x_i); \gamma).
\end{eqnarray}
Let $\fh$ be the solution learned from solving the empirical convex loss over training examples. There are three sources of errors that  affect bounding the binary excess risk $\Ex(\fh)$. First, since $\fh$ is obtained by numerically solving an optimization problem, the error in estimating the optimal solution, which we refer to as optimization error~\footnote{We note that in literature the error in estimating the optimal solution for empirical minimization is usually referred to as {\textit estimation error}. We emphasize it as optimization error because different convex surrogates could lead to very different iteration complexities and consequentially different optimization efficiency.}, will affect $\Ex(\fh)$. Additionally, since the binary excess risk can be bounded by a nonlinear transform of the convex excess risk, both the bound for $\Ex_{\phi}(\fh)$ and the error in approximating $\Ex(\fh)$ with $\Ex_{\phi}(\fh)$ will affect the final estimation of $\Ex(\fh)$.  We aim at investigating how the smoothing parameter $\gamma$ affect all these three types of errors.  As it is investigated in~Theorem~\ref{thm:bound-binary-excess-risk},  a smaller smoothing parameter $\gamma$ will result  in a poorer approximation of $\Ex(\fh)$. On the other hand,  a smaller smoothing parameter $\gamma$ will result in a smaller estimation error and a smaller bound for $\Ex_{\phi}(\fh)$. Based on the understanding of how  smoothing parameter $\gamma$ affects the three errors, we identify the choice of $\gamma$ that results in the best tradeoff between all three error and consequentially a binary excess risk $\Ex(\fh)$ better than $O(1/\sqrt{n})$.

To investigate how the smoothing parameter $\gamma$ affects the binary excess risk $\Ex(\fh)$,  we intend to unify three types of errors.  The analysis is comprised of two components, i.e. bounding the binary excess risk $\Ex(h)$ by a smooth convex excess risk $\Ex_{\phi}(h)$ that has been established in Theorem~\ref{thm:bound-binary-excess-risk} and bounding $\Ex_{\phi}(h)$ for a solution $h$ that is suboptimal in minimizing the empirical convex loss $\Lsphi(h)$ that is the focus of this section.
%----------------------------------------------------------------------------------------------------------------------------
%
%----------------------------------------------------------------------------------------------------------------------------
\subsection{Bounding Smooth  Excess Convex Risk $\Ex_{\phi}(h)$}
We now turn to bounding the excess $\phi$-risk $\Ex_{\phi}(h)$ for the smoothed hinge loss.  To bound $\Ex_{\phi}(h)$ we need to consider two types of errors: optimization error due to the approximate optimization of the empirical $\phi$-risk, and the generalization error bound for the empirical risk minimizer.     After obtaining these two errors for smooth convex surrogates, we provide a unified bound on the excess $\phi$-risk $\Ex_{\phi}(h)$ of empirical convex risk minimizer in terms of $n$.

We begin  by bounding  the error arising from solving the optimization problem numerically. One nice property of smoothed hinge loss function is that both its first order and second order derivatives are bounded, i.e.
\[
|\phi'(z;\gamma)| = \left|\frac{\exp(\gamma(1 - z)}{1 + \exp(\gamma(1 - z))}\right| \leq 1, \quad \phi''(z;\gamma) = \gamma \frac{\exp(\gamma(1 - z))}{(1 + \exp(\gamma(1 - z)))^2} \leq \frac{\gamma}{4}.
\]
Due to  the smoothness of $\phi(z; \gamma)$, we can apply the accelerated optimization algorithm~\cite{nesterov2004introductory,tseng:2009:accelerated} to achieve an $O(1/k^2)$ convergence rate for the optimization, where $k$ is the number of iterations the optimization algorithm proceeds (see e.g., accelerated gradient descent algorithm in Subsection~\ref{chp-2-sub-acc-gd}). More specifically, we will apply Algorithm 1 from~\cite{tseng:2009:accelerated} to solve the  numerical optimization problem in~(\ref{eqn:tseng-alg-1}) over the convex domain $\F = \{h \in \H: \|h\|_{\H} \leq B\}$ which results in the following  updating rules at  $s$th iteration:
\begin{eqnarray}\label{eqn:acc-gd}
\begin{aligned}
g_s &= (1-\theta_s)h_s + \theta_sf_s \\
f_{s+1} &= \arg \min_{f \in \F} \left ( \langle \nabla \Lsphi(g_s), f - g_s \rangle + \frac{\theta_s}{2}\|f - f_s\|_{\H} \right) \\
h_{s+1} &= (1-\theta_s)h_s + \theta_s f_{s+1}.
\end{aligned}
\end{eqnarray}
The following theorem that  follows immediately from~\cite[Corollary 1]{tseng:2009:accelerated} and the fact $\phi''(z;\gamma) \leq \gamma/4$, bounds the optimization error for the optimization problem after  $k$ iterations.
\begin{lemma} \label{lemma:opt}
Let $\fh = h_{k+1}$ be the solution obtained by running accelerated gradient descent  method (i.e., updating rules in~(\ref{eqn:acc-gd})) to solve the optimization problem in~(\ref{eqn:tseng-alg-1}) after $k$ iterations with $\theta_0 = 1$ and $\theta_k = 2/(k+2)$ for $k \geq 1$. We have
\[
\Lsphi(\fh) \leq \min\limits_{\|h\|_{\H} \leq B} \Lsphi(h) + \frac{\gamma B^2}{(k+2)^2}.
\]
\end{lemma}
We now turn to understanding the generalization error for the smooth convex loss.   There are many theoretical results giving upper bounds of the generalization error. However, a recent result ~\cite{srebro-2010-smoothness} has showed  that it is possible to obtain optimistic rates  for generalization bound of smooth convex loss (in the sense that smooth losses yield better generalization bounds when the problem is easier), which are more appealing than the generalization of simple Lipschitz continuous losses. The following  theorem from~\cite[Theorem 1]{srebro-2010-smoothness} bounds the generalization error for any solution $h \in \F$  when the learning has been performed by a smooth convex surrogate $\phi(\cdot)$. 
\begin{lemma} \label{lemma:generalization} With a probability $1 - \delta$, for any $\|h\|_{\H} \leq B$, we have
\begin{eqnarray*}
\Ldphi(h) - \Lsphi(h) & \leq & K_1\left(\frac{(B + \gamma B^2)t}{n} + \sqrt{\Lsphi(h)\frac{(B + \gamma B^2)t}{n}}\right)\\
\Ldphi(h) - \Lsphi(h) & \leq & K_2\left(\frac{(B + \gamma B^2)t}{n} + \sqrt{\Ldphi(h)\frac{(B + \gamma B^2)t}{n}}\right).
\end{eqnarray*}
where $t = \log(1/\delta) + \log^3n$ and $K_1$ and $K_2$ are  universal constants.
\end{lemma}
The bound stated in this lemma is \textit{optimistic} in the sense that it reduces to $\tilde{O}(1/\sqrt{n})$  when the problem is difficult and be better when the problem is easier, approaching $\tilde{O}(1/n)$ for  linearly separable data, i.e., $\Ldphis = 0$ in the second inequality. These two  lemmas essentially enable us to transform a bound on the optimization error and generalization bound  into a bound on the convex excess risk. In particular, by combining Lemma~\ref{lemma:opt} with Lemma~\ref{lemma:generalization}, we have the following theorem that bounds the smooth convex excess risk $\Ex_{\phi}(\fh) = \Ldphi(\fh) - \Ldphi(h_{\lambda}^*)$ for the empirical convex risk minimizer.
\begin{theorem}\label{thm:combined-1}
Let $\fh$ be the solution output from updating rules in~(\ref{eqn:acc-gd})  after $k$ iterations. Then, with a probability at least $1 - \delta$, we have
\begin{eqnarray*}
\Ex_{\phi}(\fh)  \leq \frac{\gamma B^2}{(k+2)^2} + K\left(\frac{(B + \gamma B^2)t}{n} + \sqrt{\Ldphis\frac{(B + \gamma B^2)t}{n}} + \sqrt{\frac{\gamma B^2(B + \gamma B^2)t}{(k+2)^2 n}}\right)
\end{eqnarray*}
where $K$ is a universal constant, $t = \log(1/\delta) + \log^3n$,  and $\Ldphis = \min_{\|h\|_{\H} \leq B} \Ldphi(h)$.
\end{theorem}
Since our overall interest is to understand how the smoothing parameter $\gamma$ affects the convergence rate of excess risk in terms of $n$, the number of training examples, it is better to parametrize both the number of iterations $k$ and smoothing parameter $\gamma$ in $n$, and bound the $\Ex_{\phi}(\fh)$ only in terms of $n$. This is given in the following corollary.
\begin{corollary} \label{cor:1}
Assume $\gamma \geq 1$ and $B \geq 1$. Paramertize $k$ and $\gamma$ in terms of $n$ as $k+2 = n^{\alpha/2}$ and $\gamma = n^{\beta}$. Then, with a probability at least $1 - \delta$,
\begin{eqnarray}
\Ex_{\phi}(\fh) \leq C(B, t)\left(n^{\beta - \alpha} + n^{\beta - 1} + n^{\beta - (\alpha + 1)/2} + [\Ldphis]^{1/2}n^{(\beta - 1)/2} \right) \label{eqn:bound-in-n}
\end{eqnarray}
where $C(B, t)$ is a constant depending on both $B$ and $t$ with $t = \log(1/\delta) + \log^3n$.
\end{corollary}
The bound given in (\ref{eqn:bound-in-n}) depends on $\Ldphis$. We would like to further characterize $\Ldphis$ in terms of $\gamma$. First, we have
\begin{eqnarray*}
\phi(z;\gamma) & = & \max\limits_{\alpha \in [0, 1]} \max(0, 1 - z) + \frac{1}{\gamma}\mathscr{R}(\alpha) \\
& \leq & \max\limits_{\alpha \in [0, 1]} \max(0, 1 - z) + \frac{1}{\gamma}\log 2 =  \phi_{\text{Hinge}}(z) + \frac{\log 2}{\gamma},
\end{eqnarray*}
where $\phi_{\text{Hinge}}(z) = \max(0, 1 - z)$ is the hinge loss. As a result, we have
\[
\Ldphis \leq \Ldhinges + \frac{\log 2}{\gamma}
\]
where  $\Ldhinges = \min\limits_{\|h\|_{\H} \leq B} \E_{(\x,y)\sim\Pxy}\left[\phi_{\text{Hinge}}(yh(\x)) \right]$
is the optimal risk with respect to the hinge loss.  In general, we will assume
\begin{eqnarray}
\Ldphis \leq \Ldhinges + \frac{a}{\gamma^{1 + \xi}} \label{eqn:delta}
\end{eqnarray}
where $a > 0$ is a constant and $\xi \geq 0$ characterizes how  fast $\Ldphis$ will converge to $\Ldhinges$ with increasing $\gamma$. To see why the assumption in (\ref{eqn:delta}) is sensible, consider the case when the optimal classifier $h_{\text{Hinge}}^*  = \arg\min_{\|h\|_{\H} \leq B} \Ldhinge(h)$ can perfectly classify all the data points with margin $\epsilon$, in which we have
\[
\Ldphis \leq \Ldhinges + O\left(\frac{e^{-\epsilon\gamma}}{\gamma}\right)
\]
which satisfy the condition in (\ref{eqn:delta}) with arbitrarily large $\xi$. It is easy to verify that the condition (\ref{eqn:delta}) holds with $\xi > 0$ if $h_{\text{Hinge}}^*$ can perfectly classify $O(1 - \gamma^{-1-\xi})$ percentage of data with margin $\epsilon$.

Using the assumption in (\ref{eqn:delta}), we have the following result that characterizes the smooth convex excess risk bound $\Ex_{\phi}(\fh)$ stated in terms of the  parameters $\alpha$, $\delta$ and $\Ldhinges$.
\begin{theorem} \label{thm:bound-convex-excess-risk}
Assume $\alpha \geq 1/2$. Set $\beta$ as
\[
\beta = \frac{\min(1/2, \alpha - 1/2)}{1 + \xi}.
\]
With a probability $1 - \delta$, we have
\[
\Ex_{\phi}(\fh) \leq O(n^{-\tau_1} + [\Ldhinges]^{1/2}n^{-\tau_2})
\]
where
\[
\tau_1 = \frac{1 + 2\xi\min(1, \alpha)}{2(1 + \xi)}, \quad \tau_2 = \frac{1/2 + \xi}{2( 1 + \xi)}
\]
\end{theorem}
\begin{proof}
Replacing $\Ldphis$ in Corollary~\ref{cor:1} with the expression in (\ref{eqn:delta}), we have, with a probability $1 - \delta$,
\[
\Ex_{\phi}(\fh) \leq C(R, t, a)\left(n^{\beta - \alpha} + n^{\beta - 1} + n^{\beta - (\alpha + 1)/2} + [\Ldhinges]^{1/2}n^{(\beta - 1)/2} + n^{-1/2-\xi\beta}\right)
\]
We first consider the case when $\alpha > 1$. In this case, we have
\[
\Ex_{\phi}(\fh) \leq O\left(n^{\beta - 1} + n^{-1/2 - \xi\beta} + [\Ldhinges]^{1/2}n^{(\beta - 1)/2}\right)
\]
By choosing $\beta - 1 = -1/2 - \xi \beta$, we have $\beta = \frac{1/2}{1 + \xi}$ and
\[
\Ex_{\phi}(\fh) \leq O(n^{-(1/2+\xi)/(1+\xi)} + [\Ldhinges]^{1/2}n^{-(1/2+\xi)/[2(1+\xi)]}
\]
In the second case, we have $\alpha \in [1/2, 1]$. Hence we  have
\[
\Ex_{\phi}(\fh) \leq O\left(n^{\beta - \alpha} + [\Ldhinges]^{1/2}n^{(\beta - 1)/2} + n^{-1/2-\xi\beta}\right)
\]
By setting $\beta - \alpha = -1/2 - \xi\beta$, we have $\beta = \frac{\alpha - 1/2}{1 + \xi}$ and
\[
\Ex_{\phi}(\fh) \leq O\left(n^{-\frac{\xi\alpha + 1/2}{1 + \xi}} + [\Ldhinges]^{1/2}n^{-(1/2 + \xi)/[2(1+\xi)]}\right).
\]
We complete the proof by combining the results for the two cases.
\end{proof}
%----------------------------------------------------------------------------------------------------------------------------
\subsection{Bounding Binary Excess Risk $\Ex(h)$}
We now combine the results from Theorem~\ref{thm:bound-binary-excess-risk} and Corollary~\ref{cor:1} to bound $\Ex(h)$.
\begin{theorem} \label{thm:combined}
Assume $\alpha \geq 1/2$. For a failure probability $\delta \in (0,1)$, define $n_0$ as
\[
n_0 \leq K_3(B, \delta)\left(\frac{1}{\Ldhinges}\right)^{1/(2\tau_1 - 2\tau_2)}
\]
where $K_3(B, \delta)$ is a constant depending on $B$ and $\delta$, and $\tau_1$ and $\tau_2$ are defined in Theorem~\ref{thm:bound-convex-excess-risk}. Set $\beta$ as that in Theorem~\ref{thm:bound-convex-excess-risk} if $n \leq n_0$ and $0$, otherwise. Then, with a probability $1 - \delta$, we have
\[
\Ex_{\phi} \leq \left\{
\begin{array}{lc}
K_4(B, \delta) n^{-\tau_1}\log n & n \leq n_0 \\
K_5(B, \delta) n^{-1/2}\log n     & n > n_0
\end{array}
\right.
\]
where $K_4(B, \delta)$ and $K_5(B, \delta)$ are constants depending on $B$ and $\delta$. 
\end{theorem}
Theorem~\ref{thm:combined} follows from Theorem~\ref{thm:bound-binary-excess-risk} and similar analysis for Theorem~\ref{thm:bound-convex-excess-risk}, from which we have
\begin{eqnarray*}
\Ex(\fh) = \Ld(\fh) - \Lds \leq  O\left(\min\left(\gamma^{-1}, \Ldphi(\fh) - \Ldphis\right)\log n\right)
\end{eqnarray*}

\begin{remark}
According to Theorem~\ref{thm:combined}, when the number of training examples $n$ is not too large, for the binary excess risk of empirical minimizer we have, with a high probability,
\[
\Ex(\fh) \leq O(n^{-\tau_1}\log n).
\]
In the case when $\xi > 0$ and $\alpha > 1/2$ (i.e. when the number of optimization iterations is larger than $\sqrt{n}$ and $\Ldphis$ converges to $\Ldhinges$ faster than $1/\gamma$), we have $\tau_1 > 1/2$, implying that using a smooth convex loss will lead to a generalization error bound better than $O(n^{-1/2})$ when the number of training examples is limited.  This implies that  for smooth loss function to  achieve a binary excess error to the extent which is achievable by corresponding non-smooth loss we can run the first order optimization method for a less number of  iterations. This is because our result examines the binary excess risk by taking into account the optimization complexity. 

We also note $1/(2\tau_1 - 2\tau_2)$ is given by
\[
\frac{1}{2\tau_1 - 2\tau_2} = \frac{1 + \xi}{1/2 + \xi\min(1, 2\alpha - 1)}
\]
When $\alpha \leq 3/4$, we have $n_0 \geq K_3(B, \delta) [\Ldhinges]^{-2}$, which could be a large number when $\Ldhinges$ is very small.
\end{remark}

\section{Proofs of Statistical Consistency}\label{sec-proofs-hinge}

\subsection{Proof of Theorem~\ref{thm:psi}}\label{app:a}
We first compute
\[
z = \mathop{\arg\min}\limits_{z'} \frac{1 + \eta}{2}\phi(z'; \gamma) + \frac{1 - \eta}{2}\phi(-z'; \gamma)
\]
By setting the derivative to be zero, we have
\begin{eqnarray*}
\frac{1+\eta}{1 + \exp(-\gamma (1 - z))} = \frac{1 - \eta}{1 + \exp(-\gamma(1 + z))}
\end{eqnarray*}
and therefore
\[
(1 + \eta)\exp(-\gamma z) - (1 - \eta)\exp(\gamma z) + 2\eta\exp(\gamma) = 0.
\]
Solving the equation, we obtain
\[
\exp(-\gamma z) = \frac{-\eta \exp(\gamma) + \sqrt{\eta^2\exp(2\gamma) + (1- \eta^2)}}{1 + \eta}
\]
and
\[
\exp(\gamma z) = \frac{\eta \exp(\gamma) + \sqrt{\eta^2\exp(2\gamma) + (1- \eta^2)}}{1 - \eta}.
\]
It is easy to verify that $\sgn(z) = \sgn(\eta)$. This is because if $\eta > 0$, we have
\[
\exp(-\gamma z) \leq \frac{\sqrt{1 - \eta^2}}{1 + \eta} = \sqrt{\frac{1 - \eta}{1 + \eta}} < 1
\]
and therefore $z > 0$. On the other hand, when $\eta < 0$, we have
\[
\exp(\gamma z) = \frac{1 + \eta}{-\eta \exp(\gamma) + \sqrt{\eta^2\exp(2\gamma) + (1- \eta^2)}} \leq \sqrt{\frac{1 +\eta}{1 - \eta}} < 1,
\]
and therefore $z < 0$. Using the solution for $z$, we compute $\phi(\eta)$ as
\begin{eqnarray*}
\psi(\eta;\gamma) & = & \frac{1 + \eta}{2}\phi(z;\gamma) + \frac{1 - \eta}{2}\phi(z;\gamma) - \min\limits_z \frac{1 + \eta}{2}\phi(z;\gamma) + \frac{1 - \eta}{2}\phi(z;\gamma) \\
& = & -\frac{1 +\eta}{2\gamma} \log \frac{1 + \exp(\gamma(1 - z))}{1 + \exp(\gamma)} - \frac{1 - \eta}{2\gamma} \log\frac{1 + \exp(\gamma(1 + z))}{1 + \exp(\gamma)}.
\end{eqnarray*}
By defining  constants $C_1= -\eta e^{\gamma} + \sqrt{\eta^2e^{2\gamma} + 1 - \eta^2}$ and  $C_2 = \eta e^{\gamma} + \sqrt{\eta^2e^{2\gamma} + 1 - \eta^2}$, we can rewrite the transform function $\psi(\eta;\gamma)$ as
\[
\psi(\eta; \gamma) = -\frac{1 + \eta}{2\gamma}\log\left(\frac{1}{1 + e^{\gamma}}\left[1 + e^{\gamma} \frac{C_1}{1 + \eta}\right]\right) -\frac{1 - \eta}{2\gamma}\log\left(\frac{1}{1 + e^{\gamma}}\left[1 + e^{\gamma} \frac{C_2}{1 - \eta}\right]\right).
\]
%----------------------------------------------------------------------------------------------------------------------------
%
%----------------------------------------------------------------------------------------------------------------------------
\subsection{Proof of Theorem~\ref{thm:psi-simple}}
Since the expression for $\psi(\eta;\gamma)$ is symmetric in terms $\eta$, we will only consider the case when $\eta > 0$. First, we have
\[
\frac{C_1e^{\gamma}}{1 + \eta} = \frac{1 - \eta}{\eta + \sqrt{\eta^2 + (1 - \eta^2)e^{-2\gamma}}} \leq \frac{1 - \eta}{2\eta}.
\]
Similarly, we have
\[
\frac{C_2 e^{\gamma}}{1 - \eta} = \frac{e^{\gamma}}{1 - \eta}\left(\eta e^{\gamma} + \sqrt{\eta^2e^{2\gamma} + 1 - \eta^2}\right) \leq \frac{1+\eta}{1 - \eta} e^{2\gamma}
\]
Thus, we have
\begin{eqnarray*}
\psi(\eta; \gamma) & \geq & \frac{1+\eta}{2\gamma}\log(1 + e^{\gamma}) - \frac{1 + \eta}{2\gamma}\log\left( \frac{1-\eta}{2\eta}\right) - \frac{1 - \eta}{2\gamma}\log\left(\frac{1+\eta}{1 - \eta} e^{\gamma}\right) \\
& \geq & \eta - \frac{1 + \eta}{2\gamma}\log\left( \frac{1-\eta}{2\eta}\right)-\frac{1 - \eta}{2\gamma}\log\left(\frac{1+\eta}{1 - \eta} \right) \\
& \geq & \eta - \frac{1}{\gamma}\log\left(\frac{1 - \eta^2}{4\eta} + \frac{1 + \eta}{2} \right)  =  \eta - \frac{1}{\gamma}\log\left(\frac{1}{4\eta} + \frac{\eta}{4} + \frac{1}{2}\right)
\end{eqnarray*}
where the last inequality follows from the concaveness of $\log(\cdot)$ function. As a result when $\eta \in (-1, 1)$ we have
\[
\frac{1}{4\eta} + \frac{\eta}{4} + \frac{1}{2} \leq \frac{1}{\eta},
\]
which completes the proof.

%----------------------------------------------------------------------------------------------------------------------------
%
%%----------------------------------------------------------------------------------------------------------------------------
\subsection{Proof of Theorem~\ref{thm:combined-1}}
Applying  Lemmas~\ref{lemma:opt} and~\ref{lemma:generalization} to the solution to  the  empirical convex risk minimizer $\fh$, we have
\begin{eqnarray}
\lefteqn{\Ldphi(\fh) \leq \Lsphi(\fh) + K_1\left(\frac{(B + \gamma B^2)t}{n} + \sqrt{\Lsphi(\fh)\frac{(B + \gamma B^2)t}{n}}\right)}  \label{eqn:bound-1} \\
& \leq & \Lsphi(h_{\gamma}^*) + \frac{\gamma B^2}{(k+2)^2} + K_1\left(\frac{(B + \gamma B^2)t}{n} + \sqrt{\Lsphi(h_{\gamma}^*)\frac{(B + \gamma B^2)t}{n}} + \sqrt{\frac{\gamma B^2(B + \gamma B^2)t}{(k+2)^2n}}\right) \nonumber
\end{eqnarray}
On the other hand, by the application of the Bernstein's inequality~\cite{boucheron2004concentration}, with probability at least $1-\delta$ we have
\begin{eqnarray}
\begin{aligned}
\Lsphi(h_{\gamma}^*) - \Ldphi(h_{\gamma}^*) &\leq \frac{4B\log \frac{1}{\delta}}{n} + \sqrt{ \frac{4 \E_{(\x,y)\sim\Pxy} \left[ \left(\phi(yh_{\gamma}^*(\x);\gamma) - \Ldphi(h_{\gamma}^*)\right)^2\right]\log\frac{1}{\delta}}{n}}{} \\
&\leq \frac{4B \log \frac{1}{\delta}}{n} + \sqrt{\frac{8B\Ldphi(h_{\gamma}^*) \log\frac{1}{\delta}}{n}}. \label{eqn:bound-2}
\end{aligned}
\end{eqnarray}
We conclude the proof by plugging in (\ref{eqn:bound-1}) with (\ref{eqn:bound-2}), replacing the constants  with a new universal constant $K$, and noting that $t = \log \frac{1}{\delta} + \log^3n$ .

%----------------------------------------------------------------------------------------------------------------------------
%
%----------------------------------------------------------------------------------------------------------------------------
\section{Summary}\label{sec:conclusion-hinge}
In this chapter we have investigated how the smoothness of loss function being used as the surrogate of 0-1 loss function in empirical risk minimization affects the excess binary risk.  While the relation between convex excess risk and binary excess risk being provably established previously under weakest possible condition such as differentiability, it was not immediately obvious how smoothness of convex surrogate  relates to  statistical consequences. This chapter made first step towards understanding this affect.   In particular, in contrast  to optimization and generalization analysis that favor smooth surrogate losses, our results revealed that smoothness  degrades the binary excess risk. To investigate guarantees on which the smoothness would be a desirable property, we proposed a unified analysis that integrates errors in optimization, generalization, and translating convex excess risk into binary excess risk. Our result shows that under favorable conditions and with appropriate choice of smoothness parameter, a smoothed hinge loss can achieve a binary excess risk that is better than $O(1/\sqrt{n})$.

\part{~Sequential Prediction/Online Learning}\label{part-online}
\chapter {Regret Bounded by Gradual Variation} \label{chap:regret}
\makeatletter
\def \f {\mathbf{f}}
\makeatother
\def \x {\mathbf{w}}
\def \P {\mathcal{P}}
\def \VAR {\text{Variation}}
\def \H {\mathbf{H}}
\def \V {\mathbf{V}}
\def \M {\mathbf{M}}

The focus so far in this thesis has been on statistical learning where we assumed that the learner is provided with a pool of i.i.d training examples according to a fixed and unknown distribution $\D$ from the instance space $\Xi = \X \times \Y$ and is asked to output a hypothesis $h \in \mathcal{H}$ which achieves a good generalization performance. 
This  statistical assumption permits the estimation of the generalization error and the uniform convergence theory  provides basic guarantees on the correctness of future predictions.

We turn now to the sequential prediction setting in which  no statistical assumption is made about the sequence of observations. In particular,  we consider the online convex optimization problem introduced in Chapter~\ref{chap-background} where the ultimate goal  is to devise efficient algorithms in adversarial environments with sub-linear  regret bounds in terms of the number of rounds the game proceeds. We have seen  a wide variety of algorithms such as Follow The Perturbed Leader (FTPL) for linear and combinatorial online learning problems, and a simple Online Gradient Descent (OGD), Follow The Regularized Leader (FTRL), and Online Mirror Descent (OMD)  algorithms for general  convex functions which attain  an $O(\sqrt{T})$ and $O(\log T)$ regret bounds for Lipschitz continuous and strongly convex functions, respectively.

Most previous works, including those discussed above, considered the most general setting in which the loss functions could be arbitrary and possibly chosen in an adversarial way. However, the environments around us may not always be fully adversarial, and the loss functions may have some patterns which can be exploited for achieving a smaller regret. For example, the weather condition or the stock price at one moment may have some correlation with the next and their difference is usually small, while abrupt changes only occur sporadically.  Consequently, it is objected that requiring an algorithm to have a small regret for all sequences leads to results that are too loose to be practically interesting,  and  the bounds obtained for worst case scenarios become pessimistic for these regular sequences. Recently, it has been shown that the regret of the FTRL algorithm for online \textit{linear} optimization can be bounded by the total variation of the cost vectors rather than the number of rounds.  This result is appealing for the scenarios where the sequence of loss functions have a pattern and are not fully adversarial.

In this chapter we extend this result to general online convex optimization  and introduce a new measure referred to as \textit{gradual variation} to capture the variation of consecutive convex functions. We  show that the total variation bound  is not necessarily small when the cost functions change slowly, and the gradual variation lower bounds the total variation.  To establish the main results, we discuss a lower bound on the performance of the FTRL that maintains only one sequence of solutions, and  a necessary condition on smoothness of the cost functions for obtaining a gradual variation bound.  We then present two novel algorithms, improved FTRL and Online Mirror Prox (OMP), that bound the regret by the gradual  variation of cost functions. Unlike previous approaches that maintain a single sequence of solutions, the proposed algorithms  maintain two sequences of solutions that makes it possible to achieve a gradual variation-based regret bound for online convex optimization.  We also extend the main results  two-fold: (i) we  present a general method to obtain  a gradual variation bound measured by general norms rather than the $\ell_2$ norm  and specialize it to three online learning settings, namely online linear optimization, prediction with expert advice, and online strictly convex optimization;  (ii) we develop a deterministic algorithm for online bandit optimization in multipoint bandit setting based on the proposed OMP algorithm.

%%%%%%%%%%%%%%%%%%%%%%%%%%%%%%%%%%%%%%%%%%%%%%%%

\section{Variational Regret Bounds}\label{sec:4:setup}
Recall that in online convex optimization problem, at each trial $t$, the learner is asked to predict the decision vector $\w_t$ that belongs to a bounded closed convex set $\W\subseteq\mathbb R^d$; it then receives a cost function $f_t:\W \rightarrow \R_{+}$ from a family of convex functions $\mathcal{F}$ and incurs a cost of $f_t(\w_t)$ for the submitted solution. The goal of online convex optimization is to come up with a sequence of solutions $\w_1, \ldots, \w_T \in \W$ that minimizes the regret, which is defined as the difference in the  cost of the sequence of decisions accumulated up to the trial $T$ made by the learner and the cost of the best fixed decision in hindsight, i.e.
\begin{equation}
\begin{aligned}
\regret_T = \sum_{t=1}^T f_t(\w_t) -\min_{\w \in \W}\sum_{t=1}^T f_t(\x).
\end{aligned}
\end{equation}
The goal of online convex optimization is to design algorithms that predict, with a small regret, the solution $\w_t$ at the $t$th trial given the (partial) knowledge about the past cost functions $f_1, f_2, \cdots, f_{t-1} \in \mathcal{F}$.

As already mentioned, generally most previous studies of online convex optimization bound the regret  in terms  of the number of trials $T$. In particular for general convex Lipschitz continuous  and strongly convex functions regret bounds of $O(\sqrt{T})$ and $O(\log T)$ have been established, respectively, which are known to minimax optimal.  However,  it is expected that the regret should be low in an unchanging environment or when the cost functions are somehow correlated.  Ideally, the tightest rate for the regret  should depend on the variance of the sequence of cost functions rather than the number of rounds $T$.  Consequently,  the bounds obtained for worst case scenarios in terms of number of iterations  become pessimistic for these regular sequences and  too loose to be practically interesting. Therefore, it is of great interest to derive a variation-based regret bound for online convex optimization in an adversarial setting. 

Recently~\cite{Hazan-2008-extract} made a substantial progress in this route and proved a variation-based regret bound for online \textit{linear} optimization by tight analysis of  FTRL algorithm with an appropriately chosen step size. A similar regret bound is shown in the same paper for prediction from expert advice by slightly modifying the multiplicative weights algorithm.  In this chapter, we  take one step further and contribute to this research direction by developing
algorithms for general framework of online \textit{convex} optimization with variation-based regret bounds.

When all the cost functions are linear, i.e., $f_t(\x)= \dd{\f_t}{\x}$, where $\f_t \in \R^d$ is the cost vector in trial $t$, online convex optimization becomes online linear optimization. Many decision problems can be cast into online linear optimization problem, such as prediction from expert advice~\cite{cesa1997use}, online shortest path problem~\cite{Takimoto:2003:PKM:945365.964295}. The first variation-based regret bound for online linear optimization problems in an adversarial setting has been shown in~\cite{Hazan-2008-extract}.  Hazan and Kale's  algorithm for online \textit{linear} optimization is based on the framework of FTRL. At each trial, the decision vector $\w_t$ is given by solving the following optimization problem:

\begin{equation}\label{eqn:3:ftrl}
\begin{aligned}
\w_t = \arg\min_{\w \in \W} \sum_{\tau=1}^{t-1}\dd{\f_\tau}{\w} +  \frac{1}{2\eta}\|\w\|_2^2,
\end{aligned}
\end{equation}
where $\f_t$ is the cost vector received at trial $t$ after predicting the decision $\w_t$, and $\eta$ is a step size.  They bound the regret by the variation of cost vectors defined as
\begin{equation}
    \VAR_T=\sum_{t=1}^T\|\f_t-\boldsymbol \mu\|_2^2,\label{eqn:var-linear}
\end{equation}
where $\boldsymbol \mu=1/T\sum_{t=1}^T \f_t$. By assuming $\|\f_t\|_2\leq 1, \forall t$ and setting the value of $\eta$ to $\eta=\min(2/\sqrt{\text{Variation}_T}, 1/6)$, they showed that the regret of FTRL can be bounded by
\begin{equation}
\begin{aligned}\label{eqn:boundvar}
\sum_{t=1}^T \dd{\f_t}{\w_t} -\min_{\w \in \W}\sum_{t=1}^T \dd{\f_t}{\w} \leq \left\{\begin{array}{lc} 15\sqrt{\text{Variation}_T}& \text{ if }\sqrt{\text{Variation}}_T\geq 12\\ 150& \text{ if } \sqrt{\text{Variation}}_T\leq 12\end{array}\right..
\end{aligned}
\end{equation}
From (\ref{eqn:boundvar}), we can see that when the variation of the cost vectors is small (less than $12$), the regret is a constant, otherwise it is bounded by the variation $O(\sqrt{\text{Variation}}_T)$. This result indicates that online linear optimization in the adversarial setting is as efficient as in the stationary stochastic setting. 
%---------------------------------------------------------------------------------------------------------
%
%---------------------------------------------------------------------------------------------------------
\section{Gradual Variation and Necessity of Smoothness}\label{sec:4:ana}

 Here we introduce a new measure to characterize the efficiency of online learning algorithms in evolving environments which is termed as  gradual variation. The motivation of defining gradual variation  stems from two observations: one is practical and the other one is technical raised by the limitation of extending the results in~\cite{Hazan-2008-extract} to general convex functions.   From practical point of view,  we are interested in a more general scenario, in which the environment may be evolving but in a somewhat gradual way. For example, the weather condition or the stock price at one moment may have some correlation with the next and their difference is usually small, while abrupt changes only occur sporadically.

 In order to understand the limitation of extending the results in~\cite{Hazan-2008-extract}, let us apply the results  to general convex loss functions. This is an important problem in its own as online convex optimization generalizes online linear optimization by replacing linear cost functions with non-linear convex cost functions and covers many other sequential decision making problems. For instance, it has found applications in  portfolio management~\cite{Agarwal:2006:APM:1143844.1143846} and online classification~\cite{onlinekernellearning}. In online portfolio management problem, an investor wants to distribute his wealth over a set of stocks without knowing the market output in advance. If we let $\mathbf{w}_t$ denote the distribution on the stocks and $\mathbf r_t$ denote the price relative vector, i.e., $r_t[i]$ denote the  the ratio of the closing price of stock $i$ on day $t$ to the closing price on day $t-1$,  then an interesting function is the logarithmic growth ratio, i.e. $\sum_{t=1}^T \log(\dd{\mathbf \w_t}{\mathbf r_t})$, which is a concave function  to be maximized.

\begin{algorithm}[t]
\center \caption{Linearalized Follow The Regularized Leader for OCO}
\begin{algorithmic}[1] \label{alg:11}
    \STATE {\textbf{Input}}: $\eta>0$
    \STATE {\textbf{Initialize}}: $\w_1=0$

    \FOR{$t = 1, \ldots, T$}
        \STATE  Predict $\w_t$ by $$\w_{t} =  \mathop{\arg\min}\limits_{\w \in \W}  \sum_{\tau=1}^{t-1} \dd{\f_{\tau}}{\w} +\frac{1}{2\eta}\|\w\|_2^2$$
        \STATE Receive cost function $f_t(\cdot)$ and incur  loss $f_t(\w_t)$
        \STATE Compute $\f_t = \nabla f_t(\w_t)$
    \ENDFOR
\end{algorithmic}
\end{algorithm}
%---------------------------------------------------------------------------------------------------------
Since the results in~\cite{Hazan-2008-extract}  were developed for linear loss functions,  a straightforward approach is to use the first order approximation for convex loss functions, i.e., $f_t(\w)\simeq f_t(\w_t) + \langle \nabla f_t(\x_t),\x-\x_t\rangle$, and replace the linear loss  vector with the gradient of the loss function $f_t(\x)$ at $\x_t$. The resulting algorithm is shown in Algorithm~\ref{alg:11}.  Using  the convexity of loss function $f_t(\x)$, we have
\begin{equation}
\begin{aligned}\label{eqn:boundvar2}
\sum_{t=1}^T f_t(\w_t) -\min_{\w \in \W}\sum_{t=1}^T f_t(\w) \leq \sum_{t=1}^T \dd{\f_t}{\w_t} -\min_{\w \in \W}\sum_{t=1}^T \dd{\f_t}{\w}.
\end{aligned}
\end{equation}
If we assume $\|\nabla f_t(\w)\|_2\leq 1, \forall t\in[T], \forall \w \in \W$, we can  apply Hazan and Kale's variation-based bound in~(\ref{eqn:boundvar}) to bound the regret in~(\ref{eqn:boundvar2}) by the variation of the cost functions as:
\begin{equation}
\begin{aligned}\label{eqn:var2}
{\VAR}_{T}=\sum_{t=1}^T\|\f_t -\boldsymbol \mu\|_2^2 = \sum_{t=1}^T\left \|\nabla f_t(\w_t) - \frac{1}{T}\sum_{\tau=1}^T \nabla f_\tau(\w_\tau)\right\|_2^2.
\end{aligned}
\end{equation}
To better understand $\VAR_T$ in~(\ref{eqn:var2}), we rewrite it as
\begin{equation}
\begin{aligned}\label{eqn:var}
\VAR_T & = \sum_{t=1}^T\left \|\nabla f_t(\w_t) - \frac{1}{T}\sum_{\tau=1}^T \nabla f_\tau(\w_\tau)\right\|_2^2  \\
& =
\frac{1}{2T} \sum_{t,\tau = 1}^T \|\nabla f_t(\w_t) - \nabla
f_\tau(\w_\tau)\|^2 \notag\\
& \leq  \frac{1}{T} \sum_{t=1}^T \sum_{\tau=1}^T \|\nabla f_t(\w_t) -
\nabla f_t(\w_\tau)\|_2^2 + \frac{1}{T} \sum_{t=1}^T \sum_{\tau=1}^T
\|\nabla f_t(\w_\tau) - \nabla f_\tau(\w_\tau)\|_2^2 \notag\\
& =  \VAR^1_T +  \VAR_T^2.
\end{aligned}
\end{equation}
We see that the variation $\VAR_T$ is bounded by two parts: $\VAR^1_T$ essentially measures the smoothness of individual cost functions, while $\VAR^2_T$ measures the variation in the gradients of cost functions.  Let us consider an easy setting when all cost functions are identical.  In this case, $\VAR^2_T$ vanishes, and $\VAR_T$ is equal to  $\VAR^1_T/2$, i.e.,
\begin{equation*}
\begin{aligned}
\VAR_T  &= \frac{1}{2T} \sum_{t,\tau = 1}^T \|\nabla f_t(\w_t) - \nabla
f_\tau(\w_\tau)\|^2 \\
&= \frac{1}{2T} \sum_{t,\tau = 1}^T \|\nabla f_t(\w_t) - \nabla
f_t(\w_\tau)\|^2\notag\\
& = \frac{\VAR_T^1}{2}.
\end{aligned} 
\end{equation*}
As a result, the regret of the FTRL algorithm for online convex optimization may still be bounded by $O(\sqrt{T})$ regardless of the smoothness of the cost functions.

%---------------------------------------------------------------------------------------------------------
To address this challenge, we develop two novel algorithms for online convex optimization that bound the regret by the variation of cost functions.  In particular, we would like to bound the regret of online convex optimization by the variation of cost functions defined as follows:
\begin{equation}
    \text{GradualVariation}_{T} = \sum_{t=1}^{T-1} \max\limits_{\w \in \W} \|\nabla f_{t+1}(\w) - \nabla f_{t}(\w)\|_2^2. \label{eqn:var-1}
\end{equation}
Note that the variation in~(\ref{eqn:var-1}) is defined in terms of gradual difference between individual cost function to its previous one, while the variation in~(\ref{eqn:var-linear}) is defined in terms of total difference between individual cost vectors to their mean. Therefore we refer to the variation defined in~(\ref{eqn:var-1})  as \textit{gradual variation}, and to the variation defined in~(\ref{eqn:var-linear}) as \textit{total variation}. It is straightforward to show that when $f_t(\w) = \dd{\f_t}{\w}$, the gradual variation $\text{GradualVariation}_T$  is upper bounded  by the total variation $\VAR_T$ defined  with a constant factor:

\begin{equation*}
\begin{aligned}
\sum_{t=1}^{T-1}\|\f_{t+1} - \f_t\|_2^2 \leq \sum_{t=1}^{T-1} 2\|\f_{t+1} - \boldsymbol \mu \|_2^2 + 2\|\f_t - \boldsymbol \mu \|_2^2 \leq 4 \sum_{t=1}^{T}\|\f_t-\boldsymbol \mu\|_2^2.
\end{aligned}
\end{equation*}

On the other hand, we can not bound the total variation by the gradual variation up to a constant. This is verified by the following example. Let us assume that the adversary plays a fixed function $\f$ for the first half of the iterations and another different function $\mathbf g$ for the second half of the iterations, i.e., $\f_1=\cdots=\f_{T/2} = \f$ and $\f_{T/2+1}=\cdots=\f_T =\mathbf g \neq \f$. Then, in this simple scenario the total variation of the sequence of cost functions  in~(\ref{eqn:var-linear}) is given by
\[
\VAR_T= \sum_{t=1}^T\|\f_t -\boldsymbol \mu\|_2^2 = \frac{T}{2}\left\|\f- \frac{\f + \mathbf g}{2}\right\|_2^2 + \frac{T}{2}\left\|\mathbf g - \frac{\f+ \mathbf g}{2}\right\|_2^2  = \Omega(T),
\]
while the gradual variation defined in~(\ref{eqn:var-1}) is a constant given by
\begin{equation*}
\begin{aligned}
\text{GradualVariation}_T = \sum_{t=1}^{T-1} \|\f_{t+1} - \f_t\| _ 2^2  =  \|\f- \mathbf g\|_2^2  = O(1).
\end{aligned}
\end{equation*}

Based on the above analysis, we claim that the regret bound by gradual variation is usually tighter than  total variation. In particular, the following theorem shows a lower bound on the performance of the  FTRL in terms of gradual variation. Unlike the standard setting of online learning where the FTLR achieves the optimal regret bound for Lipschitz continuous and strongly convex losses, it is not capable of achieving regret bounded by gradual variation. The result of this theorem motivates us to develop new algorithms for online convex optimization to achieve a gradual variation bound of $O(\sqrt{\text{GradualVariation}_T})$. 
For the ease of exposition we use $\text{GV}_{T}$ to denote the gradual variation after $T$ iterations.
\begin{theorem}\label{thm-grad-impossibility}
The regret of FTRL is at least $\Omega(\min(\text{GV}_{T}, \sqrt{T}))$. 
\end{theorem}
\begin{proof}
Let $\f$ be any unit vector passing through $\w_1$. Let $s = \floor{1/\eta}$, so that if we use $\f_t = \f$ for every $t \le s$, each such $\z_{t+1} = \w_1 - t \eta \f$ still remains in $\W$ and thus $\w_{t+1} = \z_{t+1}$. Next, we analyze the regret by considering the following three cases depending on the range of $s$.

\noindent {\textbf{Case I:} $s \ge \sqrt{T}$.} \\
First, when $s \ge \sqrt{T}$, we choose $\f_t=\f$ for $t$ from $1$ to $\floor{s/2}$ and $\f_t = 0$ for the remaining $t$. Clearly, the best strategy of the offline algorithm is to play $\w =-\f$. On the other hand, since the learning rate $\eta$ is too small, the strategy $\w_t$ played by GD, for $t \le \floor{s/2}$, is far away from $\w$, so that $\dd{\f_t}{\w_t - \w}  \ge 1-t\eta \ge 1/2$. Therefore, the regret is at least $\floor{s/2}(1/2) = \Omega(\sqrt{T})$.

\noindent {\textbf{Case II:} $0< s < \sqrt{T}$.} \\
Second, when $0< s < \sqrt{T}$, the learning rate is high enough so that FTRL may overreact to each loss vector, and we make it pay by flipping the direction of loss vectors frequently. More precisely, we use the vector $\f$ for the first $s$ rounds so that $\w_{t+1} = \w_1 - t \eta \f$ for any $t \le s$, but just as $\w_{s+1}$ moves far enough in the direction of $-\f$, we make it pay by switching the loss vector to $-\f$, which we continue to use for $s$ rounds. Note that $\w_{s+1+r} = \w_{s+1-r}$ but $\f_{s+1+r} = -\f_{s+1-r}$ for any $r \le s$, so $\sum^{2s}_{t=1} \dd{\f_t} {\w_t - \w_1} =  \dd{\f_{s+1}}{\w_{s+1} - \w_1} \ge \Omega(1)$.
As $\w_{2s+1}$ returns back to $\w_1$, we can see the first $2s$ rounds as a period, which only contributes $\| 2f\|_2^2 =4$ to the deviation. Then we repeat the period for $\tau$ times, where $\tau=\floor{\text{GV}_T/4}$ if there are enough rounds, with $\floor{T/(2s)} \ge \floor{\text{GV}_T/4}$, to use up the gradual variation $\text{GV}_T$, and $\tau=\floor{T/(2s)}$ otherwise. For any remaining round $t$, we simply choose $\f_t=0$. As a result, the total regret is at least $\Omega(1) \cdot \tau = \Omega(\min\{\text{GV}_T/4, T/(2s)\}) = \Omega(\min\{\text{GV}_T, \sqrt{T}\})$.

\noindent {\textbf{Case III:} $s=0$.} \\
Finally, when $s=0$, the learning rate is so high that we can easily make GD pay by flipping the direction of the loss vector in each round. More precisely, by starting with $\f_1 = -\f$, we can have $\w_2$ on the boundary of $\W$, which means that if we then alternate between $\f$ and $-\f$, the strategies FTRL plays will alternate between $\w_3$ and $\w_2$ which have a constant distance from each other. Then following the analysis in the second case, one can show that the total regret is at least $\Omega(\min\{\text{GV}_T, T\})$.
\end{proof}
%---------------------------------------------------------------------------------------------------------
\begin{assumption}
In this study, we assume smooth cost functions with Lipschitz continuous gradients, i.e., there exists a constant $L>0$ such that
\begin{equation}\label{eqn:smooth}
\begin{aligned}
\|\nabla f_t(\w)-\nabla f_t(\w')\|_2\leq L\|\w-\w'\|_2, \forall \w, \w'\in\W, \forall t.
\end{aligned}
\end{equation}
\end{assumption}
%---------------------------------------------------------------------------------------------------------
We would like to emphasize that our assumption about the smoothness of cost functions is necessary to achieve the variation-based bound stated in this chapter. To see this, consider the special case of $f_1(\w) = \cdots = f_T(\w) = f(\w)$. If we are able to achieve a regret bound which scales as the square roof of the gradual variation, for any sequence of convex functions, then for the special case where all the  cost functions are identical, we  have
\begin{equation*}
\begin{aligned}
\sum_{t=1}^T f(\w_t) \leq \min_{\w \in \W} \sum_{t=1}^T f(\x)  + O(1),
\end{aligned}
\end{equation*}
implying that $\widehat \w_T = \sum_{t=1}^T \w_t/T$ approaches the optimal solution at the rate of $O(1/T)$. This contradicts the lower complexity bound (i.e., $O(1/\sqrt{T})$) for any  optimization method {which only uses first order information about the cost functions}~\cite[Theorem 3.2.1]{nemircomp1983} (see also Table~\ref{table:lower}). This analysis  indicates that the smoothness assumption is necessary to attain variation based regret bound for general online convex optimization problem. {We would like to emphasize the fact   that this contradiction  holds  when only  the gradient information about the cost functions is provided to the learner and the learner  \textit{may} be able to achieve a variation-based bound using second  order  information about the cost functions, which is not the focus of this chapter.}
%-------------------------------------------------------------------------------------------------------------------%

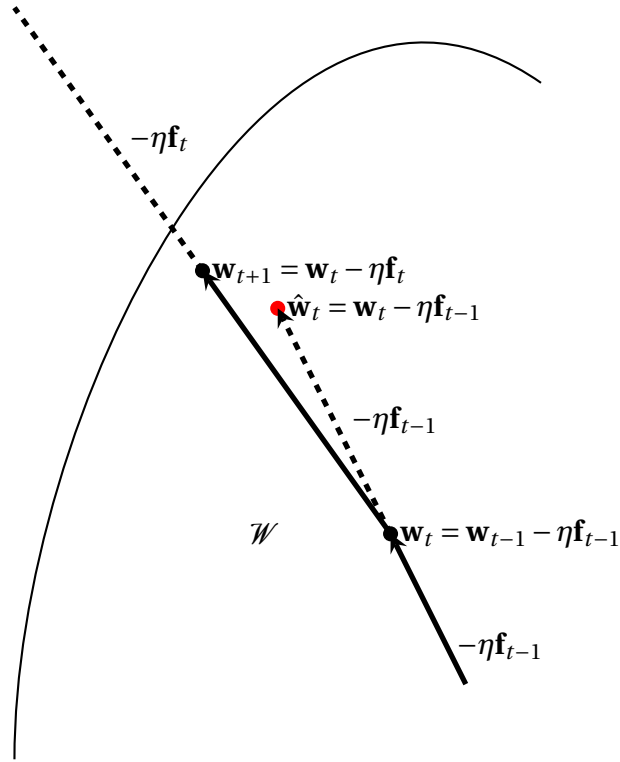
\begin{figure}[t]
 \begin{center}
  \begin{tikzpicture}[node distance=4.5cm, auto, >=stealth]

    \fill[black] (4,2) circle (1mm);
    \fill[red] (2.5,5) circle (1mm);
    \fill[black] (1.5,5.5) circle (1mm);

    \draw[line width=2,->] (5,0) -- (4,2);
    \draw[dashed,line width=2,->] (4,2) -- (2.5,5);

    \draw[line width=2,->] (4,2) -- (1.5,5.5);
    \draw[dashed,line width=2,-] (1.5,5.5) -- (-1,9);
    
    \draw[thick] (-1,-1) to [out=90,in=145] (6,8);
    
    \node [right] at (4,2) {$\w_{t} = \w_{t-1} - \eta \mathbf{f}_{t-1}$};
    \node [right] at (2.5,5) {$\hat{\w}_{t} = \w_{t} - \eta \mathbf{f}_{t-1}$};
    \node [right] at (1.5,5.5) {$\w_{t+1} = \w_{t} - \eta \mathbf{f}_{t}$};
    \node [right] at (3.35,3.5) {$-\eta \mathbf{f}_{t-1}$};
    \node [right] at (0.4,7.25) {$-\eta \mathbf{f}_{t}$};
    \node [right] at (4.75,0.5) {$-\eta \mathbf{f}_{t-1}$};
    \node [right] at (2,2) {$\mathcal{W}$};

  \end{tikzpicture}
  \label{flowchart}
 \end{center}
\caption[The intuition behind the improved FTRL and OMP algorithms]{Illustration of the main idea behind the proposed improved FTRL and online mirror prox methods to attain regret bounds in terms of gradual variation for linear loss functions. The learner plays the decision $\hat{\w}_t$ instead of $\w_t$ to  suffer less regret when the consecutive loss functions are gradually evolving.}
\label{fig:gradual}
\end{figure}

%--------------------------------------------------------------------------------------------------------------------%
\section{The Improved FTRL Algorithm}
\label{sec:algorithm}
As mentioned earlier, the ultimate goal of this chapter is to have  algorithms that can take advantage of benign  sequences in gradually evolving environments  and at the same time protect against the adversarial sequences. However, the impossibility result we showed in the previous section and in particular Theorem~\ref{thm-grad-impossibility},  demonstrated that the existing algorithms such as  OGD and in general the family of follow the regularized leader algorithms  fail to attain a regret bounded by the gradual variation of the loss functions. Motivated by this negative result, we now turn to proposing two algorithms for online convex optimization that are able to attain regret bounds in terms of gradual variation.  The first algorithm is an improved FTRL and the second one is based on the mirror prox method introduced in Chapter~\ref{chap-background}. One common feature shared by the two algorithms is that both of them maintain two sequences of solutions: decision vectors $\w_{1:T}=(\w_1,\cdots,\w_T)$ and searching vectors $\z_{1:T}=(\z_1,\cdots, \z_T)$ that facilitate the updates of decision vectors. Both algorithms share almost the same regret bound except for a constant factor.

All of our algorithms are based on the following idea, which we illustrate using the online linear optimization problem as an example which is graphically depicted in Figure~\ref{fig:gradual}. For general linear functions, the online gradient descent algorithm is known to achieve an optimal regret, which plays in round $t$ the point $\x_{t} = \Pi_{\W}\left({\x_{t-1} - \eta \f_{t-1}}\right)$. Now, if the loss functions have a small deviation, $\f_{t-1}$ may be close to $\f_{t}{}$, so in round $t$, it may be a good idea to play a point which moves further in the direction of $-\f_{t-1}{}$ as it may make its inner product with $\f_{t}{}$ (which is its loss with respect to $\f_{t}{}$) smaller. In fact, it can be shown that if one could play the point $\x_{t+1}{} = \Pi_{\W}\left({\x_{t}{} - \eta \f_{t}{}}\right)$ in round $t$, a very small regret could be achieved, i.e., $\sum_{t=1}^{T}{\dd{\w_{t+1}}{\f_t}} - \min_{\w \in \W} \sum_{t=1}^{T}{\dd{\w}{\f_t}} \leq O(1)$, but in reality one does not have $\f_{t}{}$ available before round $t$ to compute $\x_{t+1}{}$. On the other hand, if $\f_{t-1}{}$ is a good estimate of $\f_{t}{}$, the point $\hat{\x}_{t} = \Pi_{\W}\left({\x_{t}{} - \eta \f_{t-1}{}}\right)$ should be a good estimate of $\x_{t+1}{}$ too. The point $\hat{\x}_{t}$ can actually be computed before round $t$ since $\f_{t-1}{}$ is available, so our algorithm plays $\hat{\x}_{t}$ in round $t$. As it will be clear later in this chapter, our algorithms for the prediction with expert advice problem and the online convex optimization problem use the same idea to be able to achieve regret bounds stated in terms of   gradual variation of the sequence of losses.

To facilitate the discussion, besides the variation of cost functions defined in (\ref{eqn:var-1}), we define another variation, named \textit{extended gradual variation}, as follows
\begin{equation}
\begin{aligned}\label{eqn:vard}
{\rm{EGV}}_{T,2}(\y_{1:T})=\sum_{t=0}^{T-1}\|\nabla f_{t+1}(\y_t)-\nabla f_{t}(\y_t)\|_2^2 \leq \|\nabla f_1(\y_0)\|_2^2 + \text{GV}_T,
\end{aligned}
\end{equation}
where $f_0(\x) = 0$, the sequence  $(\y_0, \ldots, \y_T)$ is either $(\z_0,\ldots, \z_T)$ (as in the improved FTRL)  or $(\w_0, \ldots, \w_T)$ (as in the online mirror prox method) and the subscript $2$ means the variation is defined with respect to $\ell_2$ norm. When all cost functions are identical, $\text{GV}_T$ becomes zero and the extended variation ${\rm{EGV}}_{T,2}(\y_{1:T})$ is reduced to $\|\nabla f_1(\y_0)\|_2^2$, a constant independent from the number of trials. In the sequel, we use the notation ${\rm{EGV}}_{T,2}$ for simplicity. 

%---------------------------------------------------------------------------------------------------------
Our results show that for online convex optimization with $L$-smooth cost functions, the regrets of the proposed algorithms can be bounded as follows
\begin{equation}
\begin{aligned}\label{eqn:var-main}
\sum_{t=1}^T f_t(\x_t) -\min_{\w \in \W} \sum_{t=1}^T f_t(\w) \leq O \left(\sqrt{{\rm{EGV}}_{T,2}}\right) + \text{constant}.
\end{aligned}
\end{equation}

%\subsection{An Improved FTRL Algorithm}\label{sec:IFTRL}

\begin{algorithm}[t]
\center \caption{Improved FTRL  (\texttt{IFTRL}) Algorithm}
\begin{algorithmic}[1] \label{alg:2-iftrl}
    \STATE {\textbf{Input}}: $\eta\in(0,1]$

    \STATE {\textbf{Initialize:}}: $\z_0 = \mathbf{0}$ and $f_0(\x) =
    0$

    \FOR{$t = 1, \ldots, T$}
        \STATE Predict $\x_t$ by
        \[
            \x_t = \mathop{\arg\min}\limits_{\x \in \W} \left\{\dd{\x}{\nabla f_{t-1}(\z_{t-1}})+
            \frac{L}{2\eta}\|\x - \z_{t-1}\|_2^2
            \right\}
        \]
        \STATE Receive  cost function $f_t(\cdot)$ and incur  loss $f_t(\x_t)$
        \STATE Update  $\z_{t}$ by
        \[
           \z_{t} = \mathop{\arg\min}\limits_{\x \in \W} \left\{ \dd{\x}{\sum_{\tau=1}^t \nabla f_{\tau}(\z_{\tau-1})} +\frac{L}{2\eta}\|\x\|_2^2\right\}\]
    \ENDFOR
\end{algorithmic}
\end{algorithm}

We now turn to presenting  our first algorithm, dubbed IFTRL, which is a simple modification of the FTRL algorithm and show that its regret bounded by the gradual variation. The improved FTRL algorithm for online convex optimization is presented in Algorithm~\ref{alg:2-iftrl}.  Without loss of generality, we assume that the decision set $\W$ is contained in a unit ball $\mathbb{B} =\{\mathbf{x} \in \R^d: \|\mathbf{x}\| \leq 1\}$, i.e., $\W \subseteq\mathbb{B}$, and $0 \in \W$.  Note that in step 6, the searching vectors $\z_t$ are updated according to the FTRL algorithm after receiving the cost function $f_t(\w)$. To understand the updating procedure for the decision vector $\x_t$ specified in step 4, we rewrite it as
\begin{equation}
        \begin{aligned}
            \x_t
            & = \mathop{\arg\min}\limits_{\x \in \W} \left\{ f_{t-1}(\z_{t-1})+ \dd{\x - \z_{t-1}}{\nabla f_{t-1}(\z_{t-1})} +
            \frac{L}{2\eta}\|\x - \z_{t-1}\|_2^2
            \right\}. \label{eqn:update-x}
        \end{aligned}
\end{equation}
        
Notice that
\begin{equation}
\begin{aligned}\label{eqn:upc}
f_t(\x)&\leq f_t(\z_{t-1}) + \dd{\x-\z_{t-1}}{\nabla f_{t}(\z_{t-1})} + \frac{L}{2}\|\x-\z_{t-1}\|_2^2\nonumber\\
&\leq f_t(\z_{t-1}) + \dd{\x-\z_{t-1}}{\nabla f_{t}(\z_{t-1})} + \frac{L}{2\eta}\|\x-\z_{t-1}\|_2^2,
\end{aligned}
\end{equation}
where the first inequality follows the smoothness condition in~(\ref{eqn:smooth}) and the second inequality follows from the fact $\eta\leq 1$. The inequality~(\ref{eqn:upc}) provides an upper bound for $f_t(\x)$ and therefore can be used as an approximation of $f_t(\x)$ for predicting $\x_t$. However, since $\nabla f_t(\z_{t-1})$ is unknown before the prediction, we use $\nabla f_{t-1}(\z_{t-1})$ as a surrogate for $\nabla f_t(\z_{t-1})$, leading to the updating rule in (\ref{eqn:update-x}). It is this approximation that leads to the variation  bound.  The following theorem states the  regret bound of Algorithm~\ref{alg:2-iftrl}.

\begin{theorem} \label{thm:1-grad}
Let $f_1, f_2, \ldots, f_T$ be a sequence of  convex functions with $L$-Lipschitz continuous gradients.  By setting $\eta = \min\left\{1, L/\sqrt{{\rm{EGV}}_{T,2}} \right\}$, we have the
following regret bound for the IFTRL in  Algorithm~\ref{alg:2-iftrl}:
\[
\sum_{t=1}^Tf_t(\x_t) -\min_{\x\in\W}\sum_{t=1}^Tf_t(\x) \leq \max\left(L, \sqrt{{\rm{EGV}}_{T,2}}\right).
\]
\end{theorem}

\begin{remark}
\label{remark:2} Comparing with the variation bound in~(\ref{eqn:var}) for the FTRL algorithm, the smoothness parameter  $L$ plays the same role as $\VAR^1_T$ that accounts for the smoothness of cost functions, and term ${\rm{EGV}}_{T,2}$ plays the same role as $\VAR^2_T$ that accounts for the variation in the cost functions. Compared to the FTRL algorithm, the key advantage of the improved FTRL algorithm is that the regret bound is reduced to a constant  when the cost functions change only by a constant number of times along the horizon. Of course, the extended variation ${\rm{EGV}}_{T,2}$ may not be known apriori for setting the optimal $\eta$, {we can apply the standard doubling trick~\cite{bianchi-2006-prediction} to obtain a bound  that holds uniformly over time and is a factor at most 8 from the bound obtained  with the optimal choice of $\eta$. The details are provided later in  this chapter.}
\end{remark}

%%%%%%%%%%%%%%%%%%%%%%%%%%%%%%
\section{The Online Mirror Prox Algorithm}
\label{sec:PM}
The second algorithm we present to attain regret bounds in terms of gradual variation is based on the prox method we introduced in  Chapter~\ref{chap-background}~for non-smooth convex optimization.  We generalize the  prox method  for online convex optimization that shares the same order of regret bound as the improved FTRL algorithm.  The detailed steps of the Online Mirror Prox (OMP) method are shown in Algorithm~\ref{alg:3}, where we use an equivalent form of updates for $\x_t$ and $\z_t$ in order to compare to Algorithm~\ref{alg:2-iftrl} .   The OMP method  is closely related to the prox method in~\cite{Nemirovski2005} by maintaining two sets of vectors $\x_{1:T}$ and $\z_{1:T}$, where $\x_t$ and $\z_t$ are computed by \textit{gradient mappings} using $\nabla f_{t-1}(\x_{t-1})$, and $\nabla f_t(\x_t)$, respectively, as
\begin{equation*}
\begin{aligned}
\x_t&=\arg\min_{\x\in\W}\frac{1}{2}\left\|\x- \left(\z_{t-1}-\frac{\eta}{L}\nabla f_{t-1}(\x_{t-1})\right)\right\|_2^2\\
\z_t&=\arg\min_{\x\in\W}\frac{1}{2}\left\|\x- \left(\z_{t-1}-\frac{\eta}{L}\nabla f_{t}(\x_t)\right)\right\|_2^2
\end{aligned}
\end{equation*}

\begin{algorithm}[t]
\center \caption{Online Mirror Prox  (\texttt{OMP}) Algorithm}
\begin{algorithmic}[1] \label{alg:3}
    \STATE {\textbf{Input}}: $\eta >0$

    \STATE {\textbf{Initialize:}}: $\z_0 =\x_0= \mathbf{0}$ and $f_0(\x) =
    0$

    \FOR{$t = 1, \ldots, T$}
        \STATE Predict $\x_t$ by
        $$
            \x_t = \mathop{\arg\min}\limits_{\x \in \W} \left\{ \dd{\x}{\nabla f_{t-1}(\x_{t-1})}+
            \frac{L}{2\eta}\|\x - \z_{t-1}\|_2^2
            \right\}
        $$
        \STATE Receive  cost function $f_t(\cdot)$ and incur  loss $f_t(\x_t)$
       \STATE Update $\z_t$ by
        $$
            \z_{t} =\mathop{\arg\min}\limits_{\x \in \W} \left\{\dd{\x}{ \nabla f_{t}(\x_{t})}+\frac{L}{2\eta}\|\x-\z_{t-1}\|_2^2 \right\}
        $$
    \ENDFOR
\end{algorithmic}
\end{algorithm}

The OMP differs from the IFTRL algorithm: (i) in updating the searching points $\z_t$,  Algorithm~\ref{alg:2-iftrl} updates $\z_t$ by the FTRL scheme using \textit{all} the gradients of the cost functions at $\{\z_{\tau}\}_{\tau = 1}^{t - 1}$, while OMP updates $\z_t$ by a prox method using a \textit{single} gradient $\nabla f_t(\x_t)$, and (ii) in updating the decision vector $\x_t$, OMP uses the gradient $\nabla f_{t-1}(\x_{t-1})$ instead of $\nabla f_{t-1}(\z_{t-1})$. The advantage of  OMP algorithm compared to the IFTRL algorithm is that it only requires to compute one gradient $\nabla f_t(\x_t)$ for each loss function; in contrast, the improved FTRL algorithm in Algorithm~\ref{alg:2-iftrl} needs to compute the gradients of $f_t(\x)$ at two searching points $\z_t$ and $\z_{t-1}$. It is these differences  that make it easier to extend the OMP to a bandit setting, which will be discussed in Section~\ref{sec:conc}.  

The following theorem states the regret bound of the online mirror prox  method for online convex optimization.
\begin{theorem} \label{thm:2-grad}
Let $f_1, f_2, \ldots, f_T$ be a sequence of convex functions with L-Lipschitz continuous gradients.  By setting
$\eta = (1/2)\min\left\{1/\sqrt{2}, L/\sqrt{{\rm{EGV}}_{T,2}} \right\}$, we have the
following regret bound for OMP in Algorithm~\ref{alg:3}
\[
\sum_{t=1}^Tf_t(\x_t)-\min_{\x\in\W}\sum_{t=1}^Tf_t(\x) \leq 2\max\left(\sqrt{2}L, \sqrt{{\rm{EGV}}_{T,2}}\right).
\]
\end{theorem}
We note that compared to Theorem~\ref{thm:1-grad}, the regret bound in Theorem~\ref{thm:2-grad} is slightly worse by a factor of $2$. 

\subsection{Online Mirror Prox Method  with General Norms}
In this subsection, we first present a general OMP method to obtain a variation bound defined in a general norm. Then we discuss three special cases: online  linear optimization, prediction with expert advice, and online strictly  convex optimization.  The omitted proofs in this subsection can be easily duplicated by mimicking the proof of Theorem~\ref{thm:2-grad}, if necessary with the help of previous analysis as mentioned in the appropriate text. 

%\subsection{A General Prox Method}\label{subsubsec:gpm}
\begin{algorithm}[t]
\center \caption{ Online Mirror Prox Method for General Norms}
\begin{algorithmic}[1] \label{alg:3-1}
    \STATE {\textbf{Input}}: $\eta >0, \Phi(\z)$

    \STATE {\textbf{Initialize:}}: $\z_0 =\x_0= \min_{\z\in\W}\Phi(\z)$ and $f_0(\x) =
    0$

    \FOR{$t = 1, \ldots, T$}
        \STATE Predict $\x_t$ by
        $$
           \displaystyle \x_t = \mathop{\arg\min}\limits_{\x \in \W} \left\{\dd{\x}{\nabla f_{t-1}(\x_{t-1})}+
           \frac{L}{\eta}\mathsf{B}(\x, \z_{t-1})
            \right\}
        $$
        \STATE Receive  cost function $f_t(\cdot)$ and incur  loss $f_t(\x_t)$
       \STATE Update $\z_t$ by
        $$
           \displaystyle \z_{t} =\mathop{\arg\min}\limits_{\x \in \W} \left\{\dd{\x}{\nabla f_{t}(\x_{t})}+ \frac{L}{\eta}\mathsf{B}(\x, \z_{t-1})\right\}
        $$
    \ENDFOR
\end{algorithmic}
\end{algorithm}

To  adapt OMP to general norms other than the Euclidean norm, let $\|\cdot\|$ denote a general norm, $\|\cdot\|_*$ denote its dual norm, $\Phi(\z)$ be a  $\alpha$-strongly convex function with respect to the norm $\|\cdot\|$,  and $\mathsf{B}(\x,\z) = \Phi(\x)- \left(\Phi(\z) + \dd{\x-\z}{\Phi'(\z)}\right)$ be the Bregman distance induced by function $\Phi(\x)$.  Let $f_1, f_2, \cdots, f_T$ be a sequence of smooth functions with Lipschitz continuous  gradients bounded by $L$ with respect to norm $\|\cdot\|$, i.e.,

\begin{equation}
\begin{aligned}
\|\nabla f_t(\x) - \nabla f_t(\w')\|_* \leq L \|\x- \w'\|.
\end{aligned}
\end{equation}
Correspondingly, we define the extended gradual variation based on the general norm as follows:
\begin{equation}
\begin{aligned}\label{eqn:ge}
{\rm{EGV}}_T &= \sum_{t=0}^{T-1}\|\nabla f_{t+1}(\x_t) - \nabla f_t(\x_t)\|_*^2.
\end{aligned}
\end{equation}
Algorithm~\ref{alg:3-1} gives the detailed steps for the general framework. %, and the following theorem states the general norm variation bound.
We note that the key differences from Algorithm~\ref{alg:3} are:  $\z_0$ is set to $\min_{\z\in\W}\Phi(\z)$, and the Euclidean distances in steps 4 and 6 are replaced by Bregman distances, i.e.,  
%change the key steps 4, 6 in Algorithm~\ref{alg:3} to

\begin{equation}
 \begin{aligned}
           &\displaystyle \x_t = \mathop{\arg\min}\limits_{\x \in \W} \left\{ \dd{\x}{\nabla f_{t-1}(\x_{t-1})}+
           \frac{L}{\eta}\mathsf{B}(\x, \z_{t-1}) \right\},\\
           &\displaystyle \z_{t} =\mathop{\arg\min}\limits_{\x \in \W} \left\{\dd{\x}{\nabla f_{t}(\x_{t})}+ \frac{L}{\eta}\mathsf{B}(\x, \z_{t-1})\right\}.
\end{aligned}
\end{equation}

The following theorem states the variation-based regret bound for the general norm  framework, where $R$ measure the size of $\W$ defined as 
\[
R = \sqrt{2(\max_{\x\in\W} \Phi(\x) - \min_{\x\in\W}\Phi(\x))}.
\]
\begin{theorem} \label{thm:2-1}
Let $f_1, f_2, \cdots, f_T$ be a sequence of convex functions whose gradients are L-smooth continuous, $\Phi(\z)$ be a $\alpha$-strongly convex function, both with respect to norm $\|\cdot\|$, and ${\rm{EGV}}_T$ be defined in~(\ref{eqn:ge}). By setting
$\eta = (1/2)\min\left\{\sqrt{\alpha}/\sqrt{2}, LR/\sqrt{{\rm{EGV}}_T} \right\}$, we have the
following regret bound
\[
\sum_{t=1}^Tf_t(\x_t)-\min_{\x\in\W}\sum_{t=1}^Tf_t(\x) \leq 2R\max\left(\sqrt{2}LR/\sqrt{\alpha}, \sqrt{{\rm{EGV}}_T}\right).
\]
\end{theorem}

In the following subsections we specialize the proposed general method to few specific online learning settings. 
\subsection{Online Linear Optimization}

Here we consider online linear optimization and   present the algorithm and the gradual  variation bound for this setting as a special case of proposed algorithm.  In particular, we are interested in bounding the regret by the gradual variation 
\[
{\rm{EGV}}^{f}_{T,2}=\sum_{t=0}^{T-1}\|\f_{t+1} - \f_t\|_2^2,
\] where $\f_t, t=1,\ldots, T$ are the linear cost vectors and $\f_0=0$. Since linear functions are smooth functions that satisfy the inequality in (\ref{eqn:smooth}) for any positive $L>0$, therefore we can apply  Algorithm~\ref{alg:3} to online linear optimization with any positive value for $L$~\footnote{We simply set $L=1$ for online linear optimization and prediction with expert advice.}. The regret bound of Algorithm~\ref{alg:3} for online linear optimization is presented in the following corollary.  
\begin{corollary}\label{cor:olo}
Let $f_t(\x) = \dd{\f_t}{\x}, t=1,\ldots, T$ be a sequence of linear functions. By setting $\eta= \displaystyle \sqrt{1/\left(2{\rm{EGV}}^{f}_{T,2}\right)}$ and $L=1$ in Algorithm~\ref{alg:3}, then we have
\[
\sum_{t=1}^T\f_t^{\top}\x_t - \min_{\x\in\W}\sum_{t=1}^T \dd{\f_t}{\x}\leq\sqrt{2{\rm{EGV}}^{f}_{T,2}}. 
\]
\end{corollary}
{\begin{remark}
Note that the regret bound in Corollary~\ref{cor:olo} is  stronger than  the regret bound obtained in~\cite{hazan2010extracting} for online linear optimization  due to the fact that the gradual variation is smaller than the total variation. 
\end{remark}
}
\subsection{Prediction with Expert Advice}
In the problem of prediction with expert advice, the decision vector $\x$ is a distribution over $m$ experts, i.e., $\x\in\W=\{\x\in\mathbb R^m_+:  \sum_{i=1}^m w_i=1\}$.  Let $\f_t\in \mathbb R^m$ denote the costs for $m$ experts in trial $t$.  Similar to~\cite{Hazan-2008-extract}, we would like to bound the regret of prediction from expert advice by the gradual variation defined in infinite norm, i.e., 
\[
{\rm{EGV}}^{f}_{T, \infty}= \sum_{t=0}^{T-1}\|\f_{t+1} -\f_t\|_\infty^2.
\] Since it is a special online linear optimization problem, we can apply Algorithm~\ref{alg:3} to obtain a regret bound as in Corollary~\ref{cor:olo}, i.e., 
\[
\sum_{t=1}^T\f_t^{\top}\x_t - \min_{\x\in\W}\sum_{t=1}^T\dd{\f_t}{\x}\leq\sqrt{2{\rm{EGV}}_{T,2}^{f}}\leq \sqrt{2m{\rm{EGV}}^{f}_{T,\infty}}. 
\]
However, the above regret bound scales badly with the number of experts. We can  obtain a better regret bound in $O(\sqrt{{\rm{EGV}}^{f}_{T, \infty}}\ln m)$ by applying the general prox method in Algorithm~\ref{alg:3-1} with $\Phi(\x)= \sum_{i=1}^m w_i\ln w_i$ and $\mathsf{B}(\x, \z)= \sum_{i=1}^m w_i\ln(z_i/w_i)$. The two updates in Algorithm~\ref{alg:3-1} become 
\begin{equation}
\begin{aligned}
w_t^i  &= \frac{z^i_{t-1}\exp([\eta/L]f_{t-1}^i)}{\sum_{j=1}^mz^j_{t-1}\exp([\eta/L]f_{t-1}^j)}, i=1,\ldots, m\\
z_t^i &  = \frac{z^i_{t-1}\exp([\eta/L]f_{t}^i)}{\sum_{j=1}^mz^j_{t-1}\exp([\eta/L]f_{t}^j)}, i=1,\ldots, m.
\end{aligned}
\end{equation}

The resulting  regret bound is formally stated in the following Corollary. 
\begin{corollary}\label{cor:pea}
Let $f_t(\x) = \dd{\f_t}{\x}, t=1,\ldots, T$ be a sequence of linear functions in prediction with expert advice. By setting $\eta = \displaystyle \sqrt{(\ln m)/{\rm{EGV}}^{f}_{T,\infty}}$, $L=1$, $\Phi(\x)= \sum_{i=1}^m w_i\ln w_i$ and $\mathsf{B}(\x, \z)= \sum_{i=1}^m w_i\ln(w_i/z_i)$  in Algorithm~\ref{alg:3-1}, we have
\[
\sum_{t=1}^T \dd{\f_t}{\x_t} - \min_{\x\in\W}\sum_{t=1}^T \dd{\f_t}{\x}\leq \sqrt{2{\rm{EGV}}^{f}_{T,\infty}\ln m}. 
\]
\end{corollary}

By noting the definition of $\rm{EGV}^f_{T,\infty}$,  the regret bound  in Corollary~\ref{cor:pea} is  $O\left(\sqrt{\sum_{t=0}^{T-1}\max_i |f_{t+1}^i - f_t^i|\ln m}\right)$, which is similar to the regret bound obtained in~\cite{Hazan-2008-extract} for  prediction with expert advice. However,  the definitions of the variation are not exactly the same. In \cite{Hazan-2008-extract}, the authors bound the regret of prediction with expert advice by $O\left(\sqrt{\ln m\max_i \sum_{t=1}^T |f_t^i -  \mu_t^i|^2}+ \ln m\right)$, where the variation is the maximum total variation over all experts. To compare the two regret bounds, we first consider two extreme cases. When the costs of all experts are the same, then the variation in Corollary~\ref{cor:pea} is a standard gradual variation, while the variation in  \cite{Hazan-2008-extract} is a standard total variation. According to the previous analysis, a gradual variation is smaller than a total variation, therefore the regret bound in Corollary~\ref{cor:pea} is better than that in \cite{Hazan-2008-extract}. In another extreme case when the costs at all iterations of each expert are the same, both regret bounds are constants. More generally, if we assume the maximum total variation is small (say a constant), then $\sum_{t=0}^{T-1}|f_{t+1}^i - f^i_t|$ is also a constant for any $i\in[m]$. By a trivial analysis  $\sum_{t=0}^{T-1}\max_i |f^i_{t+1} - f^i_t|\leq  m \max_i \sum_{t=0}^{T-1}|f_{t+1}^i - f_t^i|$,  the regret bound in Corollary~\ref{cor:pea} might be worse up to a factor $\sqrt{m}$ than  that in \cite{Hazan-2008-extract}.

\begin{remark}
It was shown in~\cite{chiang2012online}, both the regret bounds in Corollary~\ref{cor:olo} and Corollary~\ref{cor:pea} are optimal because they match  the lower bounds for a special sequence of loss functions.  In particular, for online linear optimization if  all  loss functions but the  first $T_k=\sqrt{{\rm{EGV}}^{f}_{T,2}}$ are all-0 functions, then the known lower bound $\Omega(\sqrt{T_k})$ matches the upper bound in Corollary~\ref{cor:olo}. Similarly, for prediction from expert advice if all  loss functions but the  first $T'_k=\sqrt{{\rm{EGV}}^{f}_{T,\infty}}$ are all-0 functions, then the known lower bound $\Omega(\sqrt{T'_k\ln m})$~\cite{bianchi-2006-prediction} matches the upper bound in Corollary~\ref{cor:pea}. 
\end{remark}

\subsection{Online Strictly Convex Optimization}
In this subsection, we present an algorithm to achieve a logarithmic variation bound for online strictly convex optimization. In particular, we assume the cost functions $f_t(\x)$ are not only smooth but also strictly convex  defined formally in the following. 
\begin{definition} For $\beta>0$, a function $f(\x):\W\rightarrow\mathbb R$ is $\beta$-strictly convex if for any $\x, \z\in\W$
\begin{equation}
\begin{aligned}\label{eqn:strict-r}
f(\x)\geq f(\z) + \nabla \dd{f(\z)}{\x-\z} + \beta (\x-\z)^{\top}\nabla f(\z)\nabla f(\z)^{\top} (\x-\z)
\end{aligned}
\end{equation}

\end{definition}
It is known that such  a defined strictly convex function include strongly convex function and exponential concave function as special cases as long as the gradient of the function is bounded. {To see this, if $f(\x)$ is a $\beta'$-strongly convex function with a bounded gradient $\|\nabla f(\x)\|_2\leq G$,  then 
\begin{equation}
\begin{aligned}
f(\x)&\geq f(\z) + \dd{\nabla f(\z)}{\x-\z} + \beta' \dd{\x-\z}{\x-\z}\\
&\geq f(\z) + \dd{\nabla f(\z)}{\x-\z} + \frac{\beta'}{G^2}(\x-\z)^{\top}\nabla f(\z)\nabla f(\z)^{\top} (\x-\z), 
\end{aligned}
\end{equation}

thus $f(\x)$ is  a $(\beta'/G^2)$ strictly convex. Similarly if $f(\x)$ is exp-concave, i.e., there exists $\alpha>0$ such that $h(\x)=\exp(-\alpha f(\x))$ is concave, then $f(\x)$ is a $\beta = 1/2\min(1/(4GD), \alpha)$ strictly convex (c.f. Lemma 2 in~\cite{hazan-log-newton}), where $D$ is defined as the diameter of the domain. Therefore, in addition to smoothness and strict convexity we also assume all the cost functions have bounded gradients, i.e., $\|\nabla f_t(\x)\|_2\leq G$. }

We now turn to  deriving a logarithmic  gradual variation bound for  online strictly convex optimization. To this end,  we need to change the  Euclidean distance function in Algorithm~\ref{alg:3}  to a generalized Euclidean distance function.  Specifically, at trial $t$, we let  $\H_t= \mathbf{I} + \beta G^2 \mathbf{I}  + \beta\sum_{\tau=0}^{t-1}\nabla f_\tau(\x_\tau)\nabla f_\tau(\x_\tau)^{\top}$ and use the generalized Euclidean distance $\mathsf{B}_t(\x,\z) = \frac{1}{2}\|\x-\z\|^2_{\H_t}=\frac{1}{2}(\x-\z)^{\top}\H_t(\x-\z)$ in updating $\x_t$ and $\z_t$, i.e., 

\begin{equation}\label{eqn:ust}
\begin{aligned}
           \x_t &= \mathop{\arg\min}\limits_{\x \in \W} \left\{\dd{\x}{\nabla f_{t-1}(\x_{t-1})}+
           \frac{1}{2}\|\x-\z_{t-1}\|^2_{\H_t}
            \right\}\\
           \z_t &= \mathop{\arg\min}\limits_{\x \in \W} \left\{\dd{\x}{\nabla f_{t}(\x_{t})}+
           \frac{1}{2}\|\x-\z_{t-1}\|^2_{\H_t}
            \right\},
\end{aligned}
\end{equation}
To prove the regret bound, we can prove a similar inequality as in Lemma~\ref{lem:6} by applying $\Phi(\x)= 1/2\|\x\|^2_{\H_t}$, which is stated as follows

\begin{equation}
 \begin{aligned}
\nabla f_t(\x_t)^{\top}(\x_t-\z)&\leq  \mathsf{B}_t(\z,\z_{t-1}) - \mathsf{B}_t(\z, \z_t) \nonumber\\
&\hspace*{-1in} + \|\nabla f_t(\x_t)-\nabla f_{t-1}(\x_{t-1})\|_{\H_t^{-1}}^2- \frac{1}{2}\left[\|\x_t-\z_{t-1}\|_{\H_t}^2 + \|\x_t-\z_t\|_{\H_t}^2\right].
 \end{aligned}
 \end{equation}

 Then by applying inequality in~(\ref{eqn:strict-r}) for strictly convex functions, we obtain the following 

\begin{equation}\label{eqn:keyg}
 \begin{aligned}
f_t(\x_t)-f_t(\z)&\leq  \mathsf{B}_t(\x,\z_{t-1}) - \mathsf{B}_t(\x, \z_t) - \beta\|\x_t-\z\|^2_{\M_t}  \\
&\hspace*{-0.8in} +\|\nabla f_t(\x_t)-\nabla f_{t-1}(\x_{t-1})\|_{\H_t^{-1}}^2- \frac{1}{2}\left[\|\x_t-\z_{t-1}\|_{\H_t}^2 + \|\x_t-\z_t\|_{\H_t}^2\right],
 \end{aligned}
 \end{equation}

where $\M_t=\nabla f_t(\x_t)\nabla f_t(\x_t)^{\top}$ and $\H_t= \mathbf{I} + \beta G^2 \mathbf{I}  + \beta\sum_{\tau=0}^{t-1}\M_{\tau}$ as defined above. The  following corollary shows that general OMP method attains  a logarithmic gradual variation bound and its proof is  deferred to later.

\begin{corollary}\label{cor:strict}
Let $f_1, f_2, \ldots, f_T$ be a sequence of $\beta$-strictly convex and $L$-smooth functions with gradients bounded by $G$. We assume $8dL^2\geq 1$, otherwise we can set $L=\sqrt{1/(8d)}$. An algorithm that adopts the updates in~(\ref{eqn:ust}) has a regret bounded by 
\[
\sum_{t=1}^Tf_t(\x_t) - \min_{\x\in\W}\sum_{t=1}^Tf_t(\x)\leq \frac{1 + \beta G^2}{2} +\displaystyle\frac{8d}{\beta} \ln\max(16dL^2, \beta{\rm{EGV}}_{T,2}),
\]
where ${\rm{EGV}}_{T,2}=\sum_{t=0}^{T-1}\|\nabla f_{t+1}(\x_t) - \nabla f_t(\x_t)\|_2^2$ and $d$ is the dimension of $\x\in\W$.
\end{corollary}
%%%%%%%%%%%%%%%%%%%%%%%%%%%%%%%%%%%%%%%%%%%%%%%%%%

\subsection{Gradual Variation Bounds which Hold Uniformly over Time }
As mentioned in Remark~\ref{remark:2}, the algorithms presented in this chapter  rely on the previous
knowledge of the gradual variation ${\rm{EGV}}_{T,2}$ to tune the learning rate $\eta$ to obtain the optimal bound. Here, we show that  the  Algorithm~\ref{alg:2-iftrl} can be used as a black-box to achieve the same regret  bound but without any prior knowledge of the ${\rm{EGV}}_{T,2}$.  We note that the analysis here is not specific to Algorithm~\ref{alg:2-iftrl} and it is general enough to  be adapted to other algorithms in the chapter too.

The main idea is to run the algorithm in epochs with a fixed learning rate $\eta_k = \eta_0/2^k$ for $k$th epoch where $\eta_0$ is a fixed constant and will be decided by analysis. We denote the number of epochs by $K$ and let $b_{k}$ denote the start of $k$th epoch. We note that $b_{K+1} = T+1$. Within $k$th epoch, the algorithm  ensures that the inequality $\eta_k \sum_{t=b_k}^{b_{k+1}-1}{\|\nabla f_{t+1}(\z_t) - \nabla f_{t}(\z_t)\|_2^2} \leq L^2 \eta_k^{-1}$ holds.  To this end, the algorithm  computes and maintains the quantity $\sum_{s=b_k}^{t}{\|\nabla f_{s+1}(\z_s) - \nabla f_{s}(\z_s)\|_2^2}$ and sets the beginning of new epoch to be $b_{k+1} = \min_{t} \sum_{s=b_k}^{t}{\|\nabla f_{s+1}(\z_s) - \nabla f_{s}(\z_s)\|_2^2} > L^2 \eta_k^{-1}$, i.e., the first iteration for which the invariant  is violated. We note that this decision can only be made after seeing the $t$th cost function. Therefor, we burn the first iteration of each epoch which causes an extra regret of $K L^2$ in the total regret. From the analysis we have:

\begin{equation}
 \begin{aligned}
 \label{eqn:dt1}
 \sum_{t=1}^Tf_t(\x_t)-\min_{\x\in\W}\sum_{t=1}^T f_t(\x) &\leq \sum_{k=1}^{K}{ \left[ \sum_{t=b_k}^{b_{k+1}-1}{f_t(\x_t)-\min_{\x\in\W}\sum_{t=b_k}^{b_{k+1}-1} f_t(\x)} \right]}\\ \nonumber
 &\leq \sum_{k=1}^{K}{\frac{L}{2 \eta_k} + \frac{\eta_k}{2L} {\rm{EGV}}_{b_k : b_{k+1}-1}} + K L^2 \\ \nonumber
 &\leq \sum_{k=1}^{K}{\frac{L}{2\eta_k} + \frac{L}{2\eta_k}} + K =  2L \sum_{k=1}^{K}{\eta_k^{-1}} + K L^2
 \end{aligned}
 \end{equation}

where the first inequality follows the analysis of algorithm for each epoch, the last inequality follows the invariant maintained within each phase and the constant $K L^2$ is due to burning the first iteration of each epoch. We now try to upper bound the last term.  We first note that $\sum_{k=1}^{K}{\eta_k^{-1}} = \sum_{k=1}^{K-1}{\eta_0^{-1} 2^{k}} + \eta_0^{-1} 2^{K} = \eta_0^{-1}(2^K - 1) + \eta_0^{-1} 2^{K} \leq \eta_0^{-1} 2^{K+1}$. Furthermore, from $b_K$, we know that $\eta_{K-1} \sum_{t=b_{K-1}}^{b_K}{\|\nabla f_{t+1}(\z_t) - \nabla f_{t}(\z_t)\|_2^2} \geq L^2 \eta_{K-1}^{-1}$ since the $b_K$ is the first iteration within epoch $K-1$ which violates the invariant. Also, from the monotonicity of  gradual variation one can obtain that $\eta_{K-1} {\rm{EGV}}_{T,2} \geq \eta_{K-1} \sum_{t=b_{K-1}}^{b_K}{\|\nabla f_{t+1}(\z_t) - \nabla f_{t}(\z_t)\|_2^2} \geq  L^2 \eta_{K-1}^{-1}$ which indicates $\eta_{K-1}^{-1} \leq \sqrt{{\rm{EGV}_{T,2}}}/L$.  Putting these together, from~(\ref{eqn:dt1}) we obtain:

\begin{equation}
 \begin{aligned}
 \sum_{t=1}^Tf_t(\x_t)-\min_{\x\in\W}\sum_{t=1}^T f_t(\x)  \leq 2L \frac{2^{K+1}}{\eta_0} + K L^2   & \leq 8{\sqrt{{\rm{EGV}}_{T,2}}}  + K L^2.
 \end{aligned}
 \end{equation}

It remains to bound the number of epochs in terms of ${\rm{EGV}}_{T,2}$.  A simple idea would be to set $K$ to be $\lfloor \log {\rm{EGV}}_{T,2} \rfloor + 1 $, since it is the maximum number of epochs that could exists.  Alternatively, we can also bound  $K$ in terms of $\sqrt{{\rm{EGV}}_{T,2}}$ which worsen the  constant  factor  in the bound but results in  a bound  similar to one obtained by setting optimal $\eta$. 

%%%%%%%%%%%%%%%%%%%%%%%%%%%%%%%%%%%%%%%%%%%%%%%%%%

%%%%%%%%%%%%%%%%%%%%%%%%%%%%%%%%%%%%%%%%%%%%%%%%%%
%%%%%%%%%%%%%%%%%%%%%%%%%%%%%%%%%%%%%%%%%%%%%%%%%%
\section{Bandit Online Mirror Prox with Gradual Variation Bounds}
\label{sec:bandit}

Online convex optimization becomes more challenging when the learner only receives partial feedback about the cost functions. One common scenario of partial feedback is that the learner only receives the cost $f_t(\x_t)$ at the predicted point $\x_t$ but without observing the entire cost function $f_t(\cdot)$. This setup is usually referred to as  bandit setting, and the related online learning problem is called online bandit convex optimization.

Recently  Hazan et al~\cite{DBLP:conf/soda/HazanK09} extended the FTRL algorithm to online bandit \textit{linear} optimization and obtained a variation-based regret bound in the form of $O(poly(d)\sqrt{\VAR_T\log(T)}+poly(d\log(T)))$, where $\VAR_T$ is the total variation of the cost vectors. We continue this line of work by proposing  algorithms for general  online bandit convex optimization with a variation-based regret bound.  We  present a  deterministic  algorithm for online bandit convex optimization by extending the OPM algorithm to a multi-point bandit setting, and prove the variation-based regret bound, which is optimal when the variation is independent of the number of trials. In our  bandit setting , we assume we are allowed to query $d+1$ points around the decision point $\x_t$.  

To develop a variation bound for online bandit convex optimization, we follow~\cite{agarwal-2010-optimal} by considering the multi-point bandit setting, where at each trial the player is allowed to query the cost functions at multiple points. We propose a deterministic algorithm to compete against the \textit{completely adaptive} adversary that can choose the cost function $f_t(\x)$ with the knowledge of $\x_1,\cdots, \x_{t}$. To approximate the gradient $\nabla f_t(\x_t)$, we query the cost function to obtain the cost values at $f_t(\x_t)$, and $f_t(\x_t+\delta\e_i),i=1,\cdots, d$, where $\e_i$ is the  $i$th standard base in $\mathbb R^d$. Then we compute the estimate of the gradient $\nabla f_t(\x_t)$ by

\begin{equation}
\begin{aligned}
\mathbf{g}_t= \frac{1}{\delta}\sum_{i=1}^d \left(f_t(\x_t+\delta \e_i)- f_t(\x_t)\right)\e_i.
\end{aligned}
\end{equation}

It can be shown that~\cite{agarwal-2010-optimal}, under the smoothness assumption in~(\ref{eqn:smooth}),

\begin{equation}
\begin{aligned}\label{eqn:gb}
\|\mathbf{g}_t-\nabla f_t(\x_t)\|_2\leq \frac{\sqrt{d}L\delta}{2}.
\end{aligned}
\end{equation}

To prove the regret bound, besides the smoothness assumption of the cost functions, and the boundness assumption about the domain $\W\subseteq\mathcal B$,  we further assume that (i) there exists $r\leq 1$ such that $
r\mathcal B\subseteq \W\subseteq \mathcal B$, and (ii) the cost function themselves are Lipschitz continuous, i.e., there exists a constant $G$ such that

\begin{equation}
\begin{aligned}
|f_t(\x)-f_t(\w')|\leq G\|\x-\w'\|_2, \forall \x, \w'\in\W, \forall t.
\end{aligned}
\end{equation}

For our purpose, we define another gradual variation of cost functions by

\begin{equation}
\begin{aligned}
{\rm{EGV}}^{c}_T=\sum_{t=0}^{T-1} \max_{\x\in\W}|f_{t+1}(\x)-f_t(\x)|. \label{eqn:var-d}
\end{aligned}
\end{equation}

Unlike the gradual variation defined in~(\ref{eqn:vard}) that uses the gradient of the cost functions, the gradual variation in (\ref{eqn:var-d}) is defined according to the values of cost functions. The reason why we bound the regret  by the gradual variation defined in~(\ref{eqn:var-d})  by the values of the cost functions rather than the one defined in~(\ref{eqn:vard}) by the gradient of the cost functions is that in the bandit setting, we only have point evaluations of the cost functions. The following theorem states the regret bound for Algorithm~\ref{alg:8}. 

%The following theorem states the regret bound for Algorithm~\ref{alg:4} where the proof is deferred to Appendix B.
\begin{algorithm}[t]
\center \caption{Deterministic Online Bandit Convex Optimization}
\begin{algorithmic}[1] \label{alg:8}
    \STATE {\textbf{Input}}: $\eta$, $\alpha$, $\delta>0$

    \STATE {\textbf{Initialize:}}: $\z_0 = \mathbf{0}$ and $f_0(\x) =
    0$
    \FOR{$t = 1, \ldots, T$}
        \STATE Compute $\x_t$ by
        $$
            \x_t = \mathop{\arg\min}\limits_{\x \in(1-\alpha) \W} \left\{\dd{\x}{\mathbf{g}_{t-1}}+
            \frac{G}{2\eta}\|\x - \z_{t-1}\|_2^2
            \right\}
        $$
        \STATE Observe $f_t(\x_t), f_t(\x_t+\delta\e_i),i=1,\cdots, d$
       \STATE Update $\z_t$ by
        $$
           \z_{t} = \mathop{\arg\min}\limits_{\x \in (1-\alpha)\W} \left\{\dd{\x}{ \mathbf{g}_{t}}+\frac{G}{2\eta}\|\x-\z_{t-1}\|_2^2 \right\}
       $$
        %\STATE Observe $f_t(\z_t), f_t(\z_t+\delta\e_i),i=1,\cdots, d$
    \ENDFOR
\end{algorithmic}
\end{algorithm}

\begin{theorem} 
\label{thm:3}
Let $f_t(\cdot), t=1, \ldots, T$ be a sequence of $G$-Lipschitz continuous  convex
functions, and their gradients are $L$-Lipschitz continuous. By setting $\displaystyle\delta =\sqrt{\frac{4d\max(\sqrt{2}G,\sqrt{{\rm{EGV}}^{c}_T})}{(\sqrt{d}L+G(1+1/r))T}}$,
$\displaystyle\eta = \frac{\delta}{4d}\min\left\{\frac{1}{\sqrt{2}}, \frac{G}{\sqrt{{\rm{EGV}}^{c}_T}} \right\}$,  and $\alpha=\delta/r$, we have the
following regret bound for
Algorithm~\ref{alg:8}

\begin{equation}
\begin{aligned}
\sum_{t=1}^Tf_t(\x_t)-  \min\limits_{\x
\in \W} \sum_{t=1}^T f_t(\x)  \leq 4\sqrt{\max\left(\sqrt{2}G, \sqrt{{\rm{EGV}}^{c}_T}\right)d\left(dL+G/r\right)T.}
\end{aligned}
\end{equation}
\end{theorem}
%Note that in the regret bound in Theorem~\ref{thm:3}, the loss in each trial  is averaged  over the decision vector $\x_t$ and its perturbations $\x_t+\delta\e_i,i=1,\cdots, d$.

\begin{remark} Similar to the regret bound in~\cite{agarwal-2010-optimal}(Theorem 9), Algorithm~\ref{alg:8} also gives the optimal regret bound $O(\sqrt{T})$ when the variation is independent of the number of trials. Our regret bound has a better dependence on $d$ (i.e., $d$) compared with the regret bound in~\cite{agarwal-2010-optimal} (i.e., $d^2$). %Finally, we note that Algorithm~\ref{alg:4} needs to query $2(d+1)$ points. We can reduce the number of query points to $d+3$ by random sampling as shown in the next subsection.
\end{remark}

\section{Proofs of Gradual Variation}
%\addcontentsline{toc}{section}{\protect\numberline{}Appendix}

%%%%%%%%%%%%%%%%%%%%%%%%%%%%%%%%%%%%%%%%%%%%%%%%%%%
\subsection{Proof of Theorem~\ref{thm:1-grad}}
To prove Theorem~\ref{thm:1-grad}, we first present the following lemma.
\begin{lemma}\label{lem:1g}
Let $f_1, f_2,  \ldots, f_T$ be a sequence of convex functions with $L$-Lipschitz continuous gradients. By running Algorithm~\ref{alg:2-iftrl} over $T$ trials, we have

\begin{equation*}
\begin{aligned}
\sum_{t=1}^T f_t(\x_t) &\leq \min\limits_{\x \in \W}
\left[\frac{L}{2\eta}\|\x\|_2^2 + \sum_{t=1}^T f_t(\z_{t-1}) + \dd{\x -
\z_{t-1}}{\nabla f_t(\z_{t-1})}\right]\\
& \hspace*{0.2in}+ \frac{\eta}{2L}
\sum_{t=0}^{T-1} \|\nabla f_{t+1}(\z_t) - \nabla f_{t}(\z_t)\|_2^2.
\end{aligned}
\end{equation*}

\end{lemma}
With this lemma, we can easily prove Theorem~\ref{thm:1-grad} by exploring the convexity of $f_t(\x)$. 
\begin{proof}[Proof of Theorem~\ref{thm:1-grad}]
By using $\|\x\|_2\leq 1,\forall \x\in\W \subseteq \mathcal B$,  and the convexity of $f_t(\x)$, we have
 \[
 \min\limits_{\x \in \W}
\left\{\frac{L}{2\eta}\|\x\|_2^2 + \sum_{t=1}^T f_t(\z_{t-1}) + \dd{\x -
\z_{t-1}}{\nabla f_t(\z_{t-1})}\right\}\leq \frac{L}{2\eta}+\min_{\x\in\W}\sum_{t=1}^Tf_t(\x).
 \]
 Combining the above result with Lemma~\ref{lem:1g}, we have

\begin{equation*}
 \begin{aligned}
 \sum_{t=1}^Tf_t(\x_t)-\min_{\x\in\W}\sum_{t=1}^T f_t(\x)\leq \frac{L}{2\eta} + \frac{\eta}{2L}{\rm{EGV}}_{T,2}.
 \end{aligned}
 \end{equation*}

 By choosing $\eta=\min(1, L/\sqrt{{\rm{EGV}}_{T,2}})$, we have the regret bound claimed in Theorem~\ref{thm:1-grad}.
\end{proof}
The Lemma~\ref{lem:1g} is proved by induction. The key to the proof is that $\z_t$ is the optimal solution to the strongly convex minimization problem in Lemma~\ref{lem:1g}, i.e., 

\begin{equation*}
\begin{aligned}
\z_t = \arg\min\limits_{\x \in \W}
\left[\frac{L}{2\eta}\|\x\|_2^2 + \sum_{\tau=1}^t f_\tau(\z_{\tau-1}) + \dd{\x -
\z_{\tau-1}}{\nabla f_\tau(\z_{\tau-1})}\right]
\end{aligned}
\end{equation*}

\begin{proof}[Proof of Lemma~\ref{lem:1g}]
We prove the inequality by induction. When $T = 1$, we have $\x_1 =
\z_0 = 0$ and

\begin{equation*}
\begin{aligned}
&
\min\limits_{\x \in \W} \left[\frac{L}{2\eta}\|\x\|_2^2 +f_1(\z_0)+ \dd{\x -
\z_0}{\nabla f_1(\z_0)} \right] + \frac{\eta}{2L}\|\nabla f_1(\z_0)\|_2^2 \\
& \geq  f_1(\z_0) + \frac{\eta}{2L}\|\nabla f_1(\z_0)\|_2^2 +
\min\limits_{\x} \left\{\frac{L}{2\eta}\|\x\|_2^2 + \dd{\x -
\z_0}{\nabla f_1(\z_0)} \right\} \\
&= f_1(\z_0) = f_1(\x_1).
\end{aligned}
\end{equation*}

where the inequality follows that by relaxing  the minimization domain $\x\in\W$ to the whole space.  We assume the inequality holds for $t$ and aim to prove it for $t+1$. To this end, we define

\begin{equation*}
\begin{aligned}
 \psi_t(\x) &=  \left[\frac{L}{2\eta}\|\x\|_2^2 +
\sum_{\tau=1}^{t} f_\tau(\z_{\tau-1}) + \dd{\x - \z_{\tau-1}}{\nabla
f_\tau(\z_{\tau-1})}\right] \\ &\hspace*{0.15in}+ \frac{\eta}{2L} \sum_{\tau=0}^{t - 1} \|\nabla
f_{\tau+1}(\z_\tau) - \nabla f_{\tau}(\z_\tau)\|_2^2.
\end{aligned}
\end{equation*}

According to the updating procedure for $\z_t$ in step 6, we have $\z_{t}=\arg\min_{\x\in\W}\psi_{t}(\x)$. Define $\phi_t = \psi_t(\z_t) = \min_{\x\in\W}\psi_t(\x)$. Since $\psi_t(\x)$ is a $(L/\eta)$-strongly convex function, we  have
\begin{equation*}
\begin{aligned}
\psi_{t+1}(\x)-\psi_{t+1}(\z_{t})&\geq \frac{L}{2\eta}\|\x-\z_{t}\|_2^2 + \dd{\x-\z_t}{ \nabla\psi_{t+1}(\z_t)}\\
&=\frac{L}{2\eta}\|\x-\z_{t}\|_2^2 + \dd{\x-\z_t}{\nabla\psi_{t}(\z_t)+\nabla f_{t+1}(\z_t)}.
\end{aligned}
\end{equation*}

Setting $\x=\z_{t+1}=\arg\min_{\x\in\W}\psi_{t+1}(\x)$ in the above inequality results in

\begin{equation}
\begin{aligned}
\psi_{t+1}(\z_{t+1}) - \psi_{t+1}(\z_t) &= \phi_{t+1} - (\phi_t + f_{t+1}(\z_t)  + \frac{\eta}{2L}\|\nabla f_{t+1}(\z_t)
- \nabla f_t(\z_t)\|_2^2)\\
&\geq \frac{L}{2\eta}\|\z_{t+1}-\z_{t}\|_2^2 + \dd{\z_{t+1}-\z_t}{\nabla\psi_{t}(\z_t)+\nabla f_{t+1}(\z_t)}\\
&\geq \frac{L}{2\eta}\|\z_{t+1}-\z_{t}\|_2^2 + \dd{\z_{t+1}-\z_t}{\nabla f_{t+1}(\z_t)},
\end{aligned}
\end{equation}

where the second inequality follows from the fact $\z_t=\arg\min_{\x\in\W}\psi_t(\x)$, and therefore $(\x-\z_t)^{\top}\nabla\psi_t(\z_t)\geq 0, \forall\x\in\W$. Moving $f_{t+1}(\z_t)$ in the above inequality  to the right hand side,  we have

\begin{equation}
\begin{aligned}
\lefteqn{\phi_{t+1} - \phi_t - \frac{\eta}{2L}\|\nabla f_{t+1}(\z_t)
- \nabla f_t(\z_t)\|_2^2} \label{eqn:indg}\\
&\geq \frac{L}{2\eta}\|\z_{t+1} -\z_t\|_2^2 + \dd{\z_{t+1} - \z_{t}}{\nabla f_{t+1}(\z_t)}+ f_{t+1}(\z_t)\\
& \geq  \min\limits_{\x \in \W} \left\{ \frac{L}{2\eta}\|\x -\z_t\|_2^2 + \dd{\x - \z_{t}}{\nabla f_{t+1}(\z_t)}+ f_{t+1}(\z_t)\right\} \\
& =  \min\limits_{\x\in \W}\left\{  \underbrace{\frac{L}{2\eta}\|\x - \z_t\|_2^2 + \dd{\x - \z_{t}}{\nabla f_{t}(\z_t)}}\limits_{\rho(\x)} + f_{t+1}(\z_t)+ \underbrace{\dd{\x -
\z_{t}}{\nabla f_{t+1}(\z_t) - \nabla f_{t}(\z_t)}}\limits_{r(\x)}\right\}.
\end{aligned}
\end{equation}

To bound the right hand side, we note that $\x_{t+1}$ is the minimizer of $\rho(\x)$ by step 4 in Algorithm~\ref{alg:2-iftrl}, and $\rho(\x)$ is a $L/\eta$-strongly convex function, so we have

\begin{equation*}
\begin{aligned}
\rho(\x)\geq \rho(\x_{t+1}) + \underbrace{\dd{\x-\x_{t+1}}{\nabla\rho(\x_{t+1})}}\limits_{\geq 0}+ \frac{L}{2\eta}\|\x-\x_{t+1}\|_2^2\geq \rho(\x_{t+1}) +  \frac{L}{2\eta}\|\x-\x_{t+1}\|_2^2.
\end{aligned}
\end{equation*}

Then we have
\begin{equation*}
\begin{aligned}
&\rho(\x) + f_{t+1}(\z_t) + r(\x) \geq  \rho(\x_{t+1})  + f_{t+1}(\z_t) +  \frac{L}{2\eta}\|\x-\x_{t+1}\|_2^2+ r(\x)\\
&= \underbrace{ \frac{L}{2\eta}\|\x_{t+1} - \z_t\|^2_2 +
\dd{\x_{t+1} - \z_t}{\nabla f_t(\z_t)}}\limits_{\rho(\x_{t+1})} +f_{t+1}(\z_t) +\frac{L}{2\eta}\|\x-\x_{t+1}\|_2^2+ r(\x)
\end{aligned}
\end{equation*}

Plugging above inequality into the inequality in~(\ref{eqn:indg}), we have

\begin{equation*}
\begin{aligned}
\phi_{t+1} &- \phi_t - \frac{\eta}{2L}\|\nabla f_{t+1}(\z_t)
- \nabla f_t(\z_t)\|_2^2\\
\geq&  \frac{L}{2\eta}\|\x_{t+1} - \z_t\|^2_2 +
\dd{\x_{t+1} - \z_t}{\nabla f_t(\z_t)+ f_{t+1}(\z_t)}  \\
&+ \min\limits_{\x\in \W}\left\{ \frac{L}{2\eta}\|\x -
\x_{t+1}\|_2^2 + \dd{\x - \z_t}{\nabla f_{t+1}(\z_t) - \nabla
f_{t}(\z_t)}\right\} \\
\end{aligned}
\end{equation*}

To continue the bounding, we  proceed as follows

\begin{equation*}
\begin{aligned}
\phi_{t+1} &- \phi_t - \frac{\eta}{2L}\|\nabla f_{t+1}(\z_t)
- \nabla f_t(\z_t)\|_2^2\\
\geq&  \frac{L}{2\eta}\|\x_{t+1} - \z_t\|^2_2 +
\dd{\x_{t+1} - \z_t}{\nabla f_t(\z_t)}+ f_{t+1}(\z_t)  \\
&+ \min\limits_{\x\in \W}\left\{ \frac{L}{2\eta}\|\x -
\x_{t+1}\|_2^2 + \dd{\x - \z_t}{\nabla f_{t+1}(\z_t) - \nabla
f_{t}(\z_t)}\right\} \\
=&  \frac{L}{2\eta}\|\x_{t+1} - \z_t\|^2_2  +\dd{\x_{t+1} - \z_t}{\nabla f_{t+1}(\z_t)}  + f_{t+1}(\z_t) \\
&  + \min\limits_{\x\in \W}\left\{ \frac{L}{2\eta}\|\x -\x_{t+1}\|_2^2 + \dd{\x - \x_{t+1}}{\nabla f_{t+1}(\z_t) - \nabla f_{t}(\z_t)}\right\}\\
\geq&  \frac{L}{2\eta}\|\x_{t+1} - \z_t\|^2_2  +\dd{\x_{t+1} - \z_t}{\nabla f_{t+1}(\z_t)}+ f_{t+1}(\z_t) \\
&+ \min\limits_{\x}\left\{ \frac{L}{2\eta}\|\x -\x_{t+1}\|_2^2 + \dd{\x - \x_{t+1}}{\nabla f_{t+1}(\z_t) - \nabla f_{t}(\z_t)}\right\} \\
=&  \frac{L}{2\eta}\|\x_{t+1} - \z_t\|^2_2 +
\dd{\x_{t+1} - \z_t}{\nabla f_{t+1}(\z_t)} + f_{t+1}(\z_t) -
\frac{\eta}{2L}\|\nabla
f_{t+1}(\z_t) - \nabla f_{t}(\z_t)\|_2^2 \\
\geq &  f_{t+1}(\x_{t+1}) - \frac{\eta}{2L}\|\nabla f_{t+1}(\z_t) -
\nabla f_{t}(\z_t)\|_2^2,
\end{aligned}
\end{equation*}

where the first equality follows by writing $\dd{\x_{t+1} - \z_t}{\nabla f_t(\z_t)} =
\dd{\x_{t+1} - \z_t}{\nabla f_{t+1}(\z_t)} -
\dd{\x_{t+1} - \z_t}{\nabla f_{t+1}(\z_t)-\nabla f_t(\z_t)}$ and combining with $\dd{\x-\z_{t}}{\nabla f_{t+1}(\z_t)-\nabla f_t(\z_t)}$, and the last inequality follows from the smoothness condition of $f_{t+1}(\x)$.
Since by induction $\phi_t \geq \sum_{\tau=1}^t f_{\tau}(\x_\tau)$, we have
$\phi_{t+1} \geq \sum_{\tau=1}^{t+1} f_\tau(\x_\tau)$.
\end{proof}

%%%%%%%%%%%%%%%%%%%%%%%%%%%%%%%%%%%%%%%%%%%%%%%%
\subsection{Proof of Theorem~\ref{thm:2-grad}}
To prove Theorem~\ref{thm:2-grad}, we need the following lemma, which is the Lemma 3.1 in~\cite{Nemirovski2005} stated in our notations.

 \begin{lemma}[\textbf{Lemma 3.1}~\cite{Nemirovski2005}]\label{lem:6}
Let $\Phi(\z)$ be a  $\alpha$-strongly convex function with respect to the norm $\|\cdot\|$, whose dual norm is denoted by $\|\cdot\|_*$,  and $\mathsf{B}(\x,\z) = \Phi(\x)- (\Phi(\z) + (\x-\z)^{\top}\Phi'(\z))$ be the Bregman distance induced by function $\Phi(\x)$. Let $\mathscr{Z}$ be a convex compact set, and $\mathscr{U}\subseteq \mathscr{Z}$ be convex and closed.  Let $\z\in \mathscr{Z}$, $\gamma>0$, Consider the points,

\begin{equation}
\begin{aligned}
 \x &= \arg\min_{\u\in \mathscr{U}} \gamma\u^{\top}\xi + \mathsf{B}(\u, \z)\label{eqn:project1},\\
\end{aligned}
\end{equation}
\begin{equation}
\begin{aligned}
 \z_+&=\arg\min_{\u\in \mathscr{U}} \gamma\u^{\top}\zeta + \mathsf{B}(\u,\z),\label{eqn:project2}
\end{aligned}
\end{equation}

 then for any $\u\in \mathscr{U}$, we have
\begin{equation*}
 \begin{aligned}\label{eqn:ineq}
 \gamma\zeta^{\top}(\x-\u)\leq  \mathsf{B}(\u,\z) - \mathsf{B}(\u, \z_+) + \frac{\gamma^2}{\alpha}\|\xi-\zeta\|_*^2 - \frac{\alpha}{2}\left[\|\x-\z\|^2 + \|\x-\z_+\|^2\right].
 \end{aligned}
 \end{equation*}
 \end{lemma}
 In order not to have readers  struggle with complex notations in~\cite{Nemirovski2005} for the proof of Lemma~\ref{lem:6}, we present a detailed proof  later in Appendix~\ref{chap:appendix-technical} which is an adaption of the original proof to our notations.

Theorem~\ref{thm:2-grad} can be proved by using the above lemma, because the updates of $\x_t, \z_t$ can be  written equivalently as~(\ref{eqn:project1}). The proof below starts from~(\ref{eqn:ineq}) and bounds the summation of each term over $t=1,\ldots, T$, respectively. 

\begin{proof}[Proof of Theorem~\ref{thm:2-grad}]
First, we note that the two updates in step 4 and step 6 of Algorithm~\ref{alg:3} fit in the Lemma~\ref{lem:6}  if we let $\mathscr{U}= \mathscr{Z}=\W$,  $\z=\z_{t-1}$, $\x=\x_t$, $\z_+=\z_t$,  and $\Phi(\x)=\frac{1}{2}\|\x\|_2^2$, which is $1$-strongly convex function with respect to $\|\cdot\|_2$.  Then $\mathsf{B}(\u,\z)=\frac{1}{2}\|\u-\z\|_2^2$. As a result, the two updates for $\x_t, \z_t$ in Algorithm~\ref{alg:3} are exactly the updates in~(\ref{eqn:project1})  with $\z=\z_{t-1}, \gamma=\eta/L$, $\xi=\nabla f_{t-1}(\z_{t-1})$, and $\zeta=\nabla f_t(\x_t)$.  Replacing these into~(\ref{eqn:ineq}), we have the following inequality for any $\u\in\W$, 

\begin{equation}
\begin{aligned}
\frac{\eta}{L} \dd{\x_t - \u}{\nabla f_t(\x_t)} &\le
\frac{1}{2}\left(\|\u-\z_{t-1} \|_2^2 - \|\u-\z_{t} \|_2^2\right)\\
& \hspace*{-0.4in}+
\frac{\eta^2}{L^2}\|\nabla f_{t}(\x_t) - \nabla f_{t-1}(\x_{t-1})\|_2^2 -
\frac{1}{2}\left(\|\x_t - \z_{t-1}\|_2^2 + \|\x_t-\z_t\|_2^2\right)
\end{aligned}
\end{equation}
Then we have
\begin{equation}
\begin{aligned}
&\frac{\eta}{L}(f_t(\x_t) - f_t(\u))\leq \frac{\eta}{L} (\x_t - \u)^{\top}\nabla f_t(\x_t) \leq  \frac{1}{2}\left(\|\u-\z_{t-1} \|_2^2 - \|\u-\z_{t} \|_2^2\right)\\
&+ \frac{2\eta^2}{L^2} \|\nabla
f_{t}(\x_{t-1}) - \nabla f_{t-1}(\x_{t-1})\|_2^2+ \frac{2\eta^2}{L^2}\|\nabla
f_{t}(\x_t) - \nabla f_{t}(\x_{t-1})\|_2^2  \\
&- \frac{1}{2}\left(\|\x_t - \z_{t-1}\|_2^2+\|\x_t-\z_t\|_2^2\right)  \\
& \leq  \frac{1}{2}\left(\|\u-\z_{t-1}\|_2^2 - \|\u-\z_{t}\|_2^2\right)+ \frac{2\eta^2}{L^2}\|\nabla f_{t}(\x_{t-1}) - \nabla f_{t-1}(\x_{t-1})\|_2^2 \\ & +  2\eta^2\|\x_t - \x_{t-1}\|_2^2 - \frac{1}{2}\left(\|\x_t - \z_{t-1}\|_2^2+\|\x_t-\z_t\|_2^2\right),
\end{aligned}
\end{equation}
where the first inequality follows the convexity of $f_t(\x)$, and the third inequality follows the smoothness  of $f_t(\x)$.  

By taking the summation over $t=1,\cdots, T$  with  $\z^*=\arg\min\limits_{\u\in\W}\sum_{t=1}^Tf_t(\u)$, and dividing both sides by $\eta/L$, we have
\begin{equation}\label{eqn:cb}
\begin{aligned}
\sum_{t=1}^T f_t(\w_t) - \min_{\w \in \W }\sum_{t=1} f_t(\w) &\leq  \frac{L}{2\eta}+ \frac{2\eta}{L}\sum_{t=0}^{T-1}\|\nabla f_{t+1}(\w_{t}) - \nabla f_{t}(\w_{t})\|_2^2\nonumber\\
&\hspace*{-0.8in} +\sum_{t=1}^T2\eta^2\|\w_t-\w_{t-1}\|_2^2 - \underbrace{\sum_{t=1}^T \frac{1}{2}\left(\|\w_t - \z_{t-1}\|_2^2+\|\w_t-\z_t\|_2^2\right)}\limits_{\triangleq B_T}
\end{aligned}
\end{equation}

We can bound $B_T$ as follows: 

\begin{equation}
\begin{aligned}
B_T&= \frac{1}{2}\sum_{t=1}^T\|\x_t-\z_{t-1}\|_2^2  +\frac{1}{2}\sum_{t=2}^{T+1}\|\x_{t-1}-\z_{t-1}\|_2^2\notag\\
&\geq  \frac{1}{2}\sum_{t=2}^T\left(\|\x_t-\z_{t-1}\|_2^2  + \|\x_{t-1}-\z_{t-1}\|_2^2\right)\notag\\
&\geq \frac{1}{4}\sum_{t=2}^T\|\x_t-\x_{t-1}\|_2^2= \frac{1}{4}\sum_{t=1}^T\|\x_t-\x_{t-1}\|_2^2\label{eqn:Bt}
\end{aligned}
\end{equation}

where the last equality follows that $\x_1=\x_0$.  Plugging the above bound into~(\ref{eqn:cb}), we have
\begin{equation}
\begin{aligned}\label{eqn:cb}
\sum_{t=1}^Tf_t(\x_t) - \min_{\x\in\W}\sum_{t=1}f_t(\x) &\leq  \frac{L}{2\eta}+ \frac{2\eta}{L}\sum_{t=0}^{T-1}\|\nabla f_{t+1}(\x_{t}) - \nabla f_{t}(\x_{t})\|_2^2\nonumber\\
&\hspace*{-0.8in} +\sum_{t=1}^T\left(2\eta^2-\frac{1}{4}\right)\|\x_t-\x_{t-1}\|_2^2 
\end{aligned}
\end{equation}

We complete the proof by plugging the value of $\eta$.
\end{proof}

%%%%%%%%%%%%%%%%%%%%%%%%%%%%%%%%%%%%%%%%%%%%%%%%

\subsection{Proof of Corollary~\ref{cor:strict}}
We first have the key inequality in~(\ref{eqn:keyg}): for any $\z\in\W$

\begin{equation}
\begin{aligned}
f_t(\x_t)-f_t(\z)&\leq  \mathsf{B}_t(\z,\z_{t-1}) - \mathsf{B}_t(\z, \z_t) - \beta\|\x_t-\z\|^2_{\M_t}  \nonumber\\
&\hspace*{-0.8in} +\|\nabla f_t(\x_t)-\nabla f_{t-1}(\x_{t-1})\|_{\H_t^{-1}}^2- \frac{1}{2}\left[\|\x_t-\z_{t-1}\|_{\H_t}^2 + \|\x_t-\z_t\|_{\H_t}^2\right]. 
 \end{aligned}
\end{equation}

 Taking summation over $t=1,\ldots, T$, we have

\begin{equation}
 \begin{aligned}
\sum_{t=1}^Tf_t(\x_t)-\sum_{t=1}^Tf_t(\z)&\leq  \sum_{t=1}^T\underbrace{\left(\mathsf{B}_t(\z,\z_{t-1}) - \mathsf{B}_t(\z, \z_t)\right)}\limits_{\triangleq A_t} - \sum_{t=1}^T\underbrace{\beta\|\x_t-\z\|^2_{\M_t}}\limits_{\triangleq C_t}  \nonumber\\
&\hspace*{-0.8in} +\sum_{t=1}^T\underbrace{\|\nabla f_t(\x_t)-\nabla f_{t-1}(\x_{t-1})\|_{\H_t^{-1}}^2}\limits_{\triangleq S_t}- \sum_{t=1}^T\underbrace{\frac{1}{2}\left[\|\x_t-\z_{t-1}\|_{\H_t}^2 + \|\x_t-\z_t\|_{\H_t}^2\right]}\limits_{\triangleq B_t}
 \end{aligned}
 \end{equation}

Next we bound each term individually. 
First, 
\begin{equation}
\begin{aligned}
\sum_{t=1}^TA_t  = \mathsf{B}_1(\z, \z_0) - \mathsf{B}_{T}(\z, \z_{T})  + \sum_{t=1}^{T-1}\left(\mathsf{B}_{t+1}(\z, \z_{t}) - \mathsf{B}_t(\z, \z_t)\right)
\end{aligned}
\end{equation}

Note that $\mathsf{B}_1(\z, \z_0) = \frac{1}{2}(1+\beta G^2) \|\z\|_2^2\leq  \frac{1}{2}(1+\beta G^2)$ for any $\z\in\W$, and $\mathsf{B}_{t+1}(\z, \z_{t}) - \mathsf{B}_t(\x, \z_t) = \frac{\beta}{2}\|\z-\z_t\|^2_{\H_t}$, therefore 

\begin{equation}
\begin{aligned}
\sum_{t=1}^TA_t \leq  \frac{1}{2}(1+\beta G^2) + \sum_{t=1}^T\frac{\beta}{2}\|\z-\z_t\|^2_{\H_t} 
\end{aligned}
\end{equation}

Then
\begin{equation}
\begin{aligned}
\sum_{t=1}^T(A_t - C_t)&\leq  \frac{1}{2}(1+\beta G^2) + \sum_{t=1}^T\left[\frac{\beta}{2}\|\z-\z_t\|^2_{\M_t}  - \beta\|\x_t -\z\|_{\M_t}^2\right]\\
&\leq  \frac{1}{2}(1+\beta G^2) + \sum_{t=1}^T\left[\beta\|\z-\x_t\|_{\M_t}^2+\beta\|\x_t -\z_t\|^2_{\M_t}  - \beta\|\x_t -\z\|_{\M_t}^2\right]\\
&\leq \frac{1}{2}(1+\beta G^2) + \sum_{t=1}^T\beta\|\x_t -\z_t\|^2_{\M_t} \leq \frac{1}{2}(1+\beta G^2) + \sum_{t=1}^T\|\x_t -\z_t\|^2_{\H_t} 
\end{aligned}
\end{equation}

Noting the updates in~(\ref{eqn:ust}) and  from inequality in~(\ref{eqn:b3C}) in the proof of Lemma~\ref{thm:2-grad} in Appendix~\ref{chap:appendix-technical},  we can get
\[
\|\x_t - \z_t\|_{\H_t}\leq \|\nabla f_t(\x_t) - \nabla f_{t-1}(\x_{t-1})\|_{\H_t^{-1}}
\]
Next, we bound $\sum_{t=1}^TB_t$. 

\begin{equation}
\begin{aligned}
\sum_{t=1}^T B_t &= \frac{1}{2}\sum_{t=0}^{T-1}\|\x_{t+1} - \z_t\|_{\H_{t+1}}^2  + \frac{1}{2}\sum_{t=1}^T\|\x_t -\z_t\|_{\H_t}^2\\
&\geq \frac{1}{2}\sum_{t=1}^{T-1}\|\x_{t+1} - \z_t\|_{\H_{t}}^2  + \frac{1}{2}\sum_{t=1}^{T-1}\|\x_t -\z_t\|_{\H_t}^2\\
&\geq \frac{1}{4}\sum_{t=1}^{T-1}\|\x_{t+1} -\x_t\|^2_{\H_t}\geq  \frac{1}{4}\sum_{t=0}^{T-1}\|\x_{t+1} -\x_t\|^2_2 - \frac{1}{4}\|\x_1 - \x_0\|_2^2\\
&\geq \frac{1}{4}\sum_{t=0}^{T-1}\|\x_{t+1} -\x_t\|^2_2,
\end{aligned}
\end{equation}

where the last inequality follows  that $\x_0=\x_1=0$. Therefore, 

\begin{equation}
\begin{aligned}
\sum_{t=1}^Tf_t(\x_t)-\sum_{t=1}^Tf_t(\z)&\leq\frac{1}{2}(1+\beta G^2) + 2\sum_{t=1}^T\|\nabla f_t(\x_t) - \nabla f_{t-1}(\x_{t-1})\|_{\H_t^{-1}} \\
&- \frac{1}{4}\sum_{t=0}^{T-1}\|\x_{t+1} -\x_t\|^2_2
\end{aligned}
\end{equation}

To proceed, we need the following lemma. 
\begin{lemma}\label{lem:fl}
We have
\begin{equation}
\begin{aligned}
\sum_{t=1}^T\|\nabla f_t(\x_t) - \nabla f_{t-1}(\x_{t-1})\|_{\H_t^{-1}} \leq \frac{4d}{\beta}\ln\left(1 + \frac{\beta}{4}\sum_{t=1}^T\|\nabla f_t(\x_t) - \nabla f_{t-1}(\x_{t-1})\|_2^2\right)
\end{aligned}
\end{equation}

\end{lemma}
Thus, 
\begin{equation}
\begin{aligned}
&\sum_{t=1}^T\|\nabla f_t(\x_t) - \nabla f_{t-1}(\x_{t-1})\|_{\H_t^{-1}}\leq   \frac{4d}{\beta}\ln\left(1 + \frac{\beta}{4}\sum_{t=1}^T\|\nabla f_t(\x_t) - \nabla f_{t-1}(\x_{t-1})\|_2^2\right)\\
&\leq  \frac{4d}{\beta}\ln\left(1 + \frac{\beta}{4}\sum_{t=1}^T\|\nabla f_t(\x_t) -\nabla f_t(\x_{t-1}) + \nabla f_t(\x_{t-1})-\nabla f_{t-1}(\x_{t-1})\|_2^2\right)\\
&\leq \frac{4d}{\beta}\ln\left(1 + \frac{\beta}{2}\sum_{t=1}^TL^2\|\x_t -\x_{t-1}\|_2^2 + \frac{\beta}{2}\sum_{t=1}^T\|\nabla f_t(\x_{t-1})-\nabla f_{t-1}(\x_{t-1})\|_2^2\right)\\
&\leq \frac{4d}{\beta}\ln\left(1 + \frac{\beta}{2}\sum_{t=0}^{T-1}L^2\|\x_{t+1} -\x_{t}\|_2^2 + \frac{\beta}{2}\sum_{t=0}^{T-1}\|\nabla f_{t+1}(\x_{t})-\nabla f_{t}(\x_{t})\|_2^2\right)\\
&\leq \frac{4d}{\beta}\ln\left(1 + \frac{\beta}{2}\sum_{t=0}^{T-1}L^2\|\x_{t+1} -\x_{t}\|_2^2 + \frac{\beta}{2}\rm{EGV}_{T,2}\right)
\end{aligned}
\end{equation}

Then, 
\begin{equation}
\begin{aligned}
\sum_{t=1}^Tf_t(\x_t)-\sum_{t=1}^Tf_t(\z)&\leq\frac{1}{2}(1+\beta G^2) + \frac{8d}{\beta}\ln\left(1 + \frac{\beta}{2}\sum_{t=0}^{T-1}L^2\|\x_{t+1} -\x_{t}\|_2^2 + \frac{\beta}{2}\rm{EGV}_{T,2}\right)
 \\
&- \frac{1}{4}\sum_{t=0}^{T-1}\|\x_{t+1} -\x_t\|^2_2
\end{aligned}
\end{equation}

Without loss of generality we assume $8dL^2\geq1$. 
Next, let us consider two cases. In the first case, we assume $\beta{\rm EGV}_{T,2}\leq 16dL^2$. Then 

\begin{equation}
\begin{aligned}
&\frac{8d}{\beta}\ln\left(1 + \frac{\beta}{2}\sum_{t=0}^{T-1}L^2\|\x_{t+1} -\x_{t}\|_2^2+ \frac{\beta}{2}\rm{EGV}_{T,2}\right)\\
&\leq \frac{8d}{\beta}\ln\left(1 + \frac{\beta}{2}\sum_{t=0}^{T-1}L^2\|\x_{t+1} -\x_{t}\|_2^2+ 8dL^2\right)\\
&\leq \frac{8d}{\beta}\ln\left( \frac{\beta}{2}\sum_{t=0}^{T-1}L^2\|\x_{t+1} -\x_{t}\|_2^2+ 16dL^2\right)\\
&= \frac{8d}{\beta}\left[\ln 16dL^2 + \ln \left(\frac{\frac{\beta}{2}\sum_{t=0}^{T-1}L^2\|\x_{t+1} -\x_{t}\|_2^2}{16dL^2}+ 1\right)\right]\\
&\leq \frac{8d}{\beta}\ln 16dL^2 + \frac{8d}{\beta}\frac{\frac{\beta}{2}\sum_{t=0}^{T-1}L^2\|\x_{t+1} -\x_{t}\|_2^2}{16dL^2}= \frac{8d}{\beta}\ln 16dL^2 + \frac{\sum_{t=0}^{T-1}\|\x_{t+1} -\x_{t}\|_2^2}{4}
\end{aligned}
\end{equation}

where the last inequality follows $\ln(1+x)\leq x$ for $x\geq 0$. 
Then we get 
\begin{equation}
\begin{aligned}
\sum_{t=1}^Tf_t(\x_t)-\sum_{t=1}^Tf_t(\z)&\leq \frac{1}{2}(1+\beta G^2) + \frac{8d}{\beta}\ln 16dL^2 
\end{aligned}
\end{equation}

In the second case, we assume  $\beta{\rm EGV}_{T,2}\geq 16dL^2\geq 2$, then we have

\begin{equation}
\begin{aligned}
&\frac{8d}{\beta}\ln\left(1 + \frac{\beta}{2}\sum_{t=0}^{T-1}L^2\|\x_{t+1} -\x_{t}\|_2^2+ \frac{\beta}{2}\rm{EGV}_{T,2}\right)\\
&\leq \frac{8d}{\beta}\ln\left( \frac{\beta}{2}\sum_{t=0}^{T-1}L^2\|\x_{t+1} -\x_{t}\|_2^2+ \beta{\rm EGV}_{T,2}\right)\\
&\leq \frac{8d}{\beta}\left[\ln (\beta{\rm EGV}_{T,2})+ \ln \left(\frac{\frac{\beta}{2}\sum_{t=0}^{T-1}L^2\|\x_{t+1} -\x_{t}\|_2^2}{\beta{\rm EGV}_{T,2}}+ 1\right)\right]\\
&= \frac{8d}{\beta}\ln (\beta{\rm EGV}_{T,2})+ \frac{8d}{\beta}\frac{\frac{\beta}{2}\sum_{t=0}^{T-1}L^2\|\x_{t+1} -\x_{t}\|_2^2}{\beta{\rm EGV}_{T,2}}\\
&= \frac{8d}{\beta}\ln (\beta{\rm EGV}_{T,2}) + \frac{4dL^2\sum_{t=0}^{T-1}\|\x_{t+1} -\x_{t}\|_2^2}{\beta{\rm EGV}_{T,2}}\\
&\leq \frac{8d}{\beta}\ln (\beta{\rm EGV}_{T,2}) + \frac{\sum_{t=0}^{T-1}\|\x_{t+1} -\x_{t}\|_2^2}{4}\
\end{aligned}
\end{equation}

where the last inequality follows $\beta{\rm EGV}_{T,2}\geq 16dL^2$. 
Then we get 
\begin{equation}
\begin{aligned}
\sum_{t=1}^Tf_t(\x_t)-\sum_{t=1}^Tf_t(\z)&\leq \frac{1}{2}(1+\beta G^2)  +\frac{8d}{\beta}\ln  (\beta{\rm EGV}_{T,2})
\end{aligned}
\end{equation}

Thus, we complete the proof by combining the two cases. 

Next, we prove Lemma~\ref{lem:fl}. We need the following lemma, which can be proved by using  Lemma 6~\cite{hazan-log-newton} and noting that $|\mathbf{I}  + \sum_{\tau=1}^t\u_\tau\u_{\tau}^{\top}|\leq (1+ \sum_{t=1}^T\|\u_t\|_2^2)^d$, where $|\cdot|$ denotes the determinant of a matrix. 
\begin{lemma}
Let $\u_1, \u_2, \cdots, \u_T \in\mathbb R^d$ be a sequence of vectors. Let $\V_t = \mathbf{I}  + \sum_{\tau=1}^t\u_\tau\u_{\tau}^{\top}$. Then, 

\begin{equation}
\begin{aligned}
\sum_{t=1}^T\u_t^{\top}\V_t^{-1}\u_t\leq d \ln \left(1 + \sum_{t=1}^T\|\u_t\|_2^2\right)
\end{aligned}
\end{equation}

\end{lemma}
To prove Lemma~\ref{lem:fl}, we let $\v_t = \nabla f_t(\x_t), t=1,\ldots, T$ and $\v_0=0$.  Then $\H_t = \mathbf{I} + \beta G^2 \mathbf{I} + \beta\sum_{\tau=0}^{t-1}\v_\tau\v_{\tau}^{\top}$. Note that we assume $\|\nabla f_t(\x)\|_2\leq G$, therefore
\begin{equation}
\begin{aligned}
\H_t &\geq \mathbf{I}  + \beta\sum_{\tau=1}^t\v_\tau\v_{\tau}^{\top} \geq \mathbf{I} + \frac{\beta}{2}\sum_{\tau=1}^t(\v_{\tau}\v_\tau^{\top} + \v_{\tau-1}\v_{\tau-1}^{\top})\\
&\geq \mathbf{I} + \frac{\beta}{4}\sum_{\tau=1}^t(\v_\tau - \v_{\tau-1})(\v_\tau - \v_{\tau-1})^{\top} = \V_t
\end{aligned}
\end{equation}

Let $\u_t = (\sqrt{\beta}/2) (\v_t - \v_{t-1})$, then $\V_t  = \mathbf{I} + \sum_{\tau=1}^t\u_\tau\u_{\tau}^{\top}$. 
By applying the above lemma, we have 

\begin{equation}
\begin{aligned}
\frac{\beta}{4}\sum_{t=1}^T(\v_t-\v_{t-1})^{\top}\V_t^{-1}(\v_t-\v_{t-1}) \leq d \ln\left(1 +\frac{\beta}{4} \sum_{t=1}^T\|\v_t-\v_{t-1}\|_2^2\right)
\end{aligned}
\end{equation}

Thus, 
\begin{equation}
\begin{aligned}
\sum_{t=1}^T(\v_\tau-\v_{\tau-1})^{\top}\H_t^{-1}(\v_\tau-\v_{\tau-1})&\leq \sum_{t=1}^T(\v_\tau-\v_{\tau-1})^{\top}\V_t^{-1}(\v_\tau-\v_{\tau-1})\\
& \leq \frac{4d}{\beta} \ln\left(1 +\frac{\beta}{4} \sum_{t=1}^T\|\v_t-\v_{t-1}\|_2^2\right)
\end{aligned}
\end{equation}

%%%%%%%%%%%%%%%%%%%%%%%%%%%%%%%%%%%%%%%%%%%%%%%%%%%
\subsection{Proof of Theorem~\ref{thm:3}}
Let $h_t(\x)=f_t(\x) + \dd{\mathbf{g}_t-\nabla f_t(\x_t)}{\x}$.  It is easy seen that $\nabla h_t(\x_t)=\mathbf{g}_t$.
Followed by Lemma \ref{lem:6}, we have for any  $\z\in(1-\alpha)\W$,

\begin{equation}
\begin{aligned}
&\frac{\eta}{G}\nabla h_t(\x_t)^{\top}(\x_t - \z) \leq
\frac{1}{2}\left(\|\z-\z_{t-1}\|_2^2 - \|\z-\z_{t}\|_2^2 \right) +
\frac{\eta^2}{G^2}\|\mathbf{g}_t - \mathbf{g}_{t-1}\|_2^2\\
&- \frac{1}{2}\left(\|\x_t - \z_{t-1}\|_2^2+\|\x_t - \z_t\|_2^2\right)
\end{aligned}
\end{equation}

Taking summation over $t=1,\ldots, T$, we have, 

\begin{equation}
\begin{aligned}
&\sum_{t=1}^T\frac{\eta}{G}\nabla h_t(\x_t)^{\top}(\x_t - \z) \leq \frac{\|\z-\z_0\|_2^2}{2} + \sum_{t=1}^T\frac{\eta^2}{G^2}\|\mathbf{g}_t - \mathbf{g}_{t-1}\|_2^2 \\
&\hspace{0.5cm} - \sum_{t=1}^T \frac{1}{2}\left(\|\x_t - \z_{t-1}\|_2^2+\|\x_t - \z_t\|_2^2\right)\\
&\leq  \frac{\|\z-\z_0\|_2^2}{2} + \sum_{t=1}^T\frac{\eta^2}{G^2}\|\mathbf{g}_t - \mathbf{g}_{t-1}\|_2^2 - \sum_{t=1}^T\frac{1}{4}\|\x_t-\x_{t-1}\|_2^2\\
&\leq \frac{1}{2}+ \sum_{t=1}^T\frac{\eta^2}{G^2}\|\mathbf{g}_t - \mathbf{g}_{t-1}\|_2^2 - \sum_{t=1}^T\frac{1}{4}\|\x_t-\x_{t-1}\|_2^2\\
&\leq \frac{1}{2}+ \sum_{t=1}^T\frac{2\eta^2}{G^2}\|\mathbf{g}_t - \mathbf{g}_{t-1}\|_2^2 + \sum_{t=1}^T\frac{2\eta^2}{G^2}\|\mathbf{g}_{t} - \mathbf{g}_{t-1}\|_2^2 \\
&\hspace{0.5cm} - \sum_{t=1}^T\frac{1}{4}\|\x_t-\x_{t-1}\|_2^2\\
& \leq \frac{1}{2} + \frac{2\eta^2}{\delta^2G^2}\sum_{t=1}^T\left\|\sum_{i=1}^d (f_t(\x_t+\delta \e_i)-f_t(\x_{t-1}+\delta \e_i))\e_i- (f_t(\x_t)-f_t(\x_{t-1}))\e_i\right \|_2^2 \\
&\hspace{0.5cm} +\frac{2\eta^2}{\delta^2G^2}\sum_{t=1}^T\left\|\sum_{i=1}^d (f_t(\x_{t-1}+\delta \e_i)-f_{t-1}(\x_{t-1}+\delta \e_i))\e_i -(f_t(\x_{t-1})-f_{t-1}(\x_{t-1}))\e_i\right\|_2^2 \\
&\hspace{0.5cm} - \sum_{t=1}^T\frac{1}{4}\|\x_t-\x_{t-1}\|_2^2,
\end{aligned}
\end{equation}
where the second inequality follows~(\ref{eqn:Bt}). Next, we bound the middle two terms in right hand side  of the above inequality.

\begin{equation}
\begin{aligned}
&\sum_{t=1}^T\left\|\sum_{i=1}^d (f_t(\x_t+\delta \e_i)-f_t(\x_{t-1}+\delta \e_i))\e_i- (f_t(\x_t)-f_t(\x_{t-1}))\e_i\right \|_2^2\\
&\leq \sum_{t=1}^T2d\sum_{i=1}^d\left(|f_t(\x_t+\delta \e_i)-f_t(\x_{t-1}+\delta \e_i)|^2 + |f_t(\x_t)-f_t(\x_{t-1})|^2\right)\\
&\leq \sum_{t=1}^T4d^2 G^2\|\x_t - \x_{t-1}\|_2^2, 
\end{aligned}
\end{equation}

and
\begin{equation}
\begin{aligned}
&\sum_{t=1}^T\left\|\sum_{i=1}^d (f_t(\x_{t-1}+\delta \e_i)-f_{t-1}(\x_{t-1}+\delta \e_i))\e_i -(f_t(\x_{t-1})-f_{t-1}(\x_{t-1}))\e_i\right\|_2^2\\
&\leq \sum_{t=1}^T2d\sum_{i=1}^d\left(|f_t(\x_{t-1}+\delta \e_i)-f_{t-1}(\x_{t-1}+\delta \e_i)|^2 + |f_t(\x_{t-1})-f_{t-1}(\x_{t-1})|^2\right)\\
&\leq \sum_{t=1}^T4d^2 \max_{\x\in\W}|f_t(\x) - f_{t-1}(\x)|^2.
\end{aligned}
\end{equation}

Then we have
\begin{equation}
\begin{aligned}
&\sum_{t=1}^T\frac{\eta}{G}\nabla h_t(\x_t)^{\top}(\x_t - \z) \leq \frac{1}{2} + \frac{8d^2\eta^2}{\delta^2}\sum_{t=1}^T\|\x_t-\x_{t-1}\|_2^2 +\frac{8d^2\eta^2}{\delta^2G^2}\sum_{t=1}^T\max_{\x\in\W}|f_t(\x) - f_{t-1}(\x)|^2 \\
&- \sum_{t=1}^T\frac{1}{4}\|\x_t-\x_{t-1}\|_2^2\\
&\leq  \frac{1}{2}  +\frac{8d^2\eta^2}{\delta^2G^2}\sum_{t=1}^T\max_{\x\in\W}|f_t(\x) - f_{t-1}(\x)|^2
\end{aligned}
\end{equation}

where the last inequality follows that $\eta \leq \delta/(4\sqrt{2}d)$. Then by using the convexity of $h_t(\x)$ and dividing both sides by $\eta/G$, we have

\begin{equation}
\begin{aligned}
\sum_{t=1}^Th_t(\x_t) -\min_{\x\in\W}h_t((1-\alpha)\x)&\leq \frac{G}{2\eta} + \frac{8\eta  d^2}{G\delta^2}\text{EVAR}^{c}_T\leq \frac{4d}{\delta}\max\left(\sqrt{2}G, \sqrt{\text{EVAR}^{c}_T}\right)
\end{aligned}
\end{equation}

Following the the proof of Theorem 8 in~\cite{agarwal-2010-optimal},
we have
\begin{equation}
\begin{aligned}
&\sum_{t=1}^{T}f_t(\x_t) -\sum_{t=1}^T f_t(\x) \leq \sum_{t=1}^T h_t(\x_t) -\sum_{t=1}^T h_t(\x) + \sum_{t=1}^T f_t(\x_t) - h_t(\x_t) - f_t(\x) + h_t(\x)\\
 &\leq \sum_{t=1}^T h_t(\x_t) -\sum_{t=1}^T h_t(\x) + \sum_{t=1}^T \dd{\mathbf{g}_t- \nabla f_t(\x_t)}{\x-\x_t}\\
 &\leq \sum_{t=1}^T h_t(\x_t) -\sum_{t=1}^T h_t(\x) +  \sqrt{d}L\delta T
\end{aligned}
\end{equation}

where the last inequality follows from the following facts:

\begin{equation}
\begin{aligned}
&\|\mathbf{g}_t - \nabla f_t(\x_t)]\|_2\leq \frac{\sqrt{d}L\delta}{2}\\
& \|\x-\x_t\|\leq 2
\end{aligned}
\end{equation}
%\[
%\sum_{t=1}^Tf_t(\x_t) - \min_{\x\in\W}\sum_{t=1}^Tc((1-\alpha)\x) \leq \frac{2\sqrt{d^2+1}}{\delta}\max\left(G, \sqrt{\text{EVAR}^{\max}_T}%\right) + \frac{\delta\sqrt{d}L}{2} T
%\]
%and
Then we have
\begin{equation}
\begin{aligned}
\sum_{t=1}^Tf_t(\x_t) -\min_{\x\in\W}\sum_{t=1}^T f_t((1-\alpha)\x)\leq  \frac{4d}{\delta}\max\left(\sqrt{2}G, \sqrt{\text{EVAR}^{c}_T}\right) +  \sqrt{d}L\delta T
\end{aligned}
\end{equation}

By the Lipschitz continuity of $f_t(\x)$, we have

\begin{equation}
\begin{aligned}
%&\sum_{t=1}^T\frac{1}{d+1}\left(f_t(\x_t)  +\sum_{i=1}^df_t(\x_t+\delta\e_i)\right)\leq \sum_{t=1}^T f_t(\x_t) + G\delta T\\
&\sum_{t=1}^Tf_t((1-\alpha)\x)\leq \sum_{t=1}^T f_t(\x)  + G \alpha T
\end{aligned}
\end{equation}

The we get
\begin{equation}
\begin{aligned}
\sum_{t=1}^Tf_t(\x_t)-\min\limits_{\x \in \W}
\sum_{t=1}^T f_t(\x)& \leq  \frac{4d}{\delta}\max\left(\sqrt{2}G, \sqrt{\text{EVAR}^{c}_T}\right)+ \delta\sqrt{d}LT
+ \alpha GT
\end{aligned}
\end{equation}

Plugging the stated values of $\delta$ and $\alpha$ completes the proof.

%%%%%%%%%%%%%%%%%%%%%%%%%%%%%%%%%%%%%%%%%%%%%%%%%%

\section{Summary}\label{sec:conc}
In this chapter, we proposed two novel algorithms for online convex optimization that bound the regret by the gradual variation of consecutive cost functions. The first algorithm is an improvement of the FTRL algorithm, and the second algorithm is based on the mirror prox method. Both algorithms maintain two sequence of solution points, a sequence of decision points and  a sequence of searching points, and share the same order of regret bound up to a constant.  The online mirror prox method only requires to keep tracking of  a single gradient of each cost function, while the improved FTRL algorithm needs to evaluate the gradient of each cost function at two points and maintain a sum of up-to-date gradients of the cost functions.

We note that  a very recent work  Chiang et al.~\cite{Chiang13}  extends the prox method into a two-point bandit setting and achieves a similar regret bound in expectation as that in the full setting, i.e.,  $O\left(d^2\sqrt{\rm EGV_{T,2}\ln T}\right)$ for smooth convex cost functions and $O\left(d^2\ln(\rm EGV_{T,2}+\ln T)\right)$ for smooth and strongly convex cost functions, where $\rm EGV_{T,2}$ is the gradual variation defined on the gradients of the cost functions. We would like to make a thought-provoking comparison between our regret bound and their regret bound for online bandit convex optimization with smooth cost functions. First, the gradual variation in our bandit setting is defined on the values of the cost functions in contrast to that defined on the gradients of the cost functions. Second, we query the cost function $d$ times in contrast to $2$ times in their algorithms, and as a tradeoff our regret bound has a better dependence on the number of dimensions (i.e., $O(d)$) than that (i.e., $O(d^2)$) of their regret bound. Third, our regret bound has an annoying factor of $\sqrt{T}$ in comparison  with $\sqrt{\ln T}$ in theirs. Therefore, some open problems are how to  achieve a lower order of dependence on $d$ than $d^2$ in the two-point bandit setting, and how to remove the factor of $\sqrt{T}$ while keeping a small order of dependence on $d$ in our multi-point bandit setting; and studying the two different types of gradual variations for bandit settings is a future work as well.
 
%an open problem to see how to achieve a variation-based regret bound  with a constant number of queries independent from the dimension $d$ and how to reduce the dependence on $T$ in the regret bound for online bandit convex optimization.

\section{Bibliographic Notes} \label{sec:3:related-gradual}

As is well known, a wide range of literature deals with the online decision making problem  and there exist a number of regret-minimizing algorithms that have the optimal regret bound. 
The first distribution-free framework for sequential decision making was proposed by Hannan~\cite{hannan1957approximation} which was rediscovered in~\cite{Kalai:2005:EAO:1113185.1113189}. Blackwell in his seminal paper~\cite{blackwell-approach} generalized the Hannan's result and concerned the problem of playing a repeated game with a vector-valued payoff function and gave a precise necessary and sufficient condition for when a set is approachable. The most well-known and successful work is probably the Hedge algorithm~\cite{freund1995desicion}, which was a direct generalization of Littlestone and Warmuth's Weighted Majority (WM) algorithm \cite{littlestone1994weighted}. Another algorithm for online decision making problem is the Vovk's   so-called aggregating
strategies~\cite{vovk1990aggregating}. Other recent studies include the improved theoretical bounds and the
parameter-free hedging algorithm \cite{chaudhuri2009parameter} and adaptive Hedge \cite{erven2011adaptive} for decision-theoretic online learning. We refer readers to the \cite{bianchi-2006-prediction} for an  in-depth discussion of this subject.

As we already discussed in Chapter~\ref{chap-background}, over the past decade many algorithms have been proposed for online convex optimization, especially for online linear optimization. As the first seminal paper in online convex optimization, Zinkevich~\cite{DBLP:conf/icml/Zinkevich03} proposed a gradient descent algorithm  with a regret bound of $O(\sqrt{T})$. When cost functions are strongly convex, the regret bound of the online gradient descent algorithm is reduced to $O(\log T)$ with appropriately chosen step size~\cite{hazan-log-newton}. Another common methodology for online convex optimization, especially for online linear optimization, is based on the framework of Follow the Leader (FTL). FTL chooses $\w_t$ by minimizing the cost incurred by $\w_t$ in all previous trials. Since the naive FTL algorithm fails to achieve  sublinear regret in the worst case, many variants have been developed to fix the problem, including Follow The Perturbed Leader (FTPL)~\cite{Kalai:2005:EAO:1113185.1113189}, Follow The Regularized Leader (FTRL)~\cite{interior-ieee-2012}, and Follow The Approximate Leader (FTAL)~\cite{hazan-log-newton}. Other methodologies for online convex optimization introduce a potential function (or link function) to map solutions between the space of primal variables and the space of dual variables, and carry out primal-dual update based on the potential function.  The well-known Exponentiated Gradient (EG) algorithm~\cite{Kivinen:1995:AVE:225058.225121}  or Multiplicative Weights algorithm~\cite{littlestone1994weighted,freund1995desicion} belong to this category. We note that these different algorithms are closely related. For example,  in online linear optimization, the potential-based primal-dual algorithm is equivalent to FTRL algorithm.

%%%%%%%%%%%%%%%%%%%%%%%%%%%%%%%%%%%%%%%%%

\def \x {\mathbf{w}}
\def \P {\mathcal{P}}
\def \VAR {\text{Variation}}

\chapter [Gradual Variation for Composite Losses]{Gradual Variation for Composite Losses} \label{chap:non-smooth}
This chapter continues our investigation and analysis of online learning methods which can lead to better regret bounds in gradually evolving environments.  The results we have obtained in Chapter~\ref{chap:regret}  rely on the assumption that the cost functions are smooth. Additionally,  we showed that for general non-smooth functions {when the only information presented to the learner is the first order information about the cost functions}, it is impossible to obtain a regret bounded by gradual variation.  However, in this chapter, we show that  a gradual  variation  bound is achievable  for a special class of non-smooth functions that is composed of a  smooth component and a  non-smooth component. 

We consider two categories for the non-smooth component.  In the first category, we assume that the non-smooth component is a fixed function  and is relatively easy such that the composite gradient mapping can be solved without too much computational overhead compared to gradient mapping. A common example that falls into this category is to consider a non-smooth regularizer. For example, in addition to the basic domain $\W$, one would enforce the sparsity constraint on the decision vector $\x$, i.e., $\|\x\|_0\leq k<d$, which is important  in feature selection. However, the sparsity constraint $\|\x\|_0\leq k$ is a non-convex function, and is usually implemented by adding a $\ell_1$ regularizer $\lambda\|\x\|_1$ to the objective function, where $\lambda>0$ is a regularization parameter. Therefore,  at each iteration  the cost function is given by $f_t(\x) + \lambda \|\x\|_1$.  To prove a regret bound by gradual variation for  this type of non-smooth optimization, we first present a simplified version of the general online mirror prox method from Chapter~\ref{chap:regret} and show that it has the exactly same regret bound  as stated in Chapter~\ref{chap:regret}, and then extend the algorithm to the non-smooth optimization with a fixed non-smooth component. 

In  the second category, we assume that the non-smooth component can be written as an explicit maximization structure. In general, we consider a time-varying non-smooth component,  present a primal-dual  prox method, and prove a min-max regret bound by gradual variation. When the non-smooth components are equal across all trials, the usual regret is bounded by the min-max bound plus a variation in the non-smooth component. To see an application of min-max regret bound, we consider the problem of online classification with hinge loss and show that the number of mistakes can be bounded by a variation in sequential examples. 

Before moving to the detailed analysis,   it is worth mentioning  that several pieces of  works have proposed algorithms for optimizing the two types  of non-smooth functions as described above to obtain an optimal convergence rate of $O(1/T)$~\cite{nesterov2005excessive,nesterov2005smooth}. Therefore, the existence of a regret bound by gradual variation for these two types of non-smooth optimization does not violate the  contradictory argument in Section~\ref{sec:algorithm}. 

\section{Composite Losses with a Fixed Non-smooth Component}

\subsection{A Simplified Online Mirror Prox Algorithm }
In this subsection, we present a simplified version of online mirror prox (OMP) method algorithm proposed in Chapter~\ref{chap:regret}, which is  the foundation for us to develop the algorithm for non-smooth optimization. 

The key trick is to replace the domain constraint $\x\in \W$ with a non-smooth function in the objective. Let $\delta_\W(\x)$ denote the indicator function of the domain $\W$, i.e., 
\begin{equation*}
\begin{aligned}
\delta_\W(\x)=\left\{\begin{array}{cc}0,&\x\in\W\\\\\infty,&\text{otherwise}\end{array}\right.
\end{aligned}
\end{equation*}

Then the proximal gradient mapping for updating $\x$ (step 4) in Algorihtm~\ref{alg:3-1}   is equivalent to 
\begin{equation*}
\begin{aligned}
\x_t = \arg\min_{\x} \dd{\x}{\nabla f_{t-1}(\x_{t-1})} + \frac{L}{\eta}\mathsf{B}(\x, \z_{t-1}) + \delta_\W(\x). 
\end{aligned}
\end{equation*}

By the first order optimality condition, there exists a sub-gradient $\v_t\in\partial\delta_\W(\x_t)$ such that 
\begin{equation}
\begin{aligned}\label{eqn:opd}
\nabla f_{t-1}(\x_{t-1}) + \frac{L}{\eta}(\nabla \Phi(\x_t) - \nabla\Phi(\z_{t-1})) + \v_t = 0. 
\end{aligned}
\end{equation}

Thus, $\x_t$ is equal to 

\begin{equation}
\begin{aligned}\label{eqn:step41}
\x_t = \arg\min_{\x} \dd{\x}{\nabla f_{t-1}(\x_{t-1}) + \v_t}+ \frac{L}{\eta}\mathsf{B}(\x, \z_{t-1}). 
\end{aligned}
\end{equation}

Then we can change the update for $\z_t$ to 
\begin{equation}
\begin{aligned}\label{eqn:step42}
\hspace*{-0.3in}\z_t = \arg\min_{\x} \dd{\x}{\nabla f_{t}(\x_{t}) + \v_t}+ \frac{L}{\eta}\mathsf{B}(\x, \z_{t-1}). 
\end{aligned}
\end{equation}

The key ingredient of above update compared to step 6 in Algorithm~\ref{alg:3-1} is that we explicitly use the sub-gradient $\v_t$ that satisfies the optimality condition for $\x_t$ instead of solving a domain constrained optimization problem.  The advantage of updating $\z_t$ by~(\ref{eqn:step42})  is that we can easily compute $\z_t$ by  the first order optimality condition, i.e.,

\begin{equation}
\begin{aligned}\label{eqn:opd2}
\nabla f_{t}(\x_{t}) + \v_t + \frac{L}{\eta}(\nabla \Phi(\z_t) - \nabla\Phi(\z_{t-1})) = 0. 
\end{aligned}
\end{equation}

Note that Eq.~(\ref{eqn:opd}) indicates  $\v_t = -\nabla f_{t-1}(\x_{t-1}) - \nabla \Phi(\x_t) + \nabla\Phi(\z_{t-1})$. By plugging  this into~(\ref{eqn:opd2}), we reach to the following simplified update for $\z_t$, 
\begin{equation*}
\begin{aligned}
 \nabla \Phi(\z_t) =  \nabla\Phi(\x_{t})  + \frac{\eta}{L}(\nabla f_{t-1}(\x_{t-1}) - \nabla f_t(\x_t)). 
\end{aligned}
\end{equation*}

The simplified version of Algorithm~\ref{alg:3-1} is presented in Algorithm~\ref{alg:3-2}. 

\begin{algorithm}[t]
\center \caption{A Simplified General Online Mirror Prox Method}
\begin{algorithmic}[1] \label{alg:3-2}
    \STATE {\textbf{Input}}: $\eta >0, \Phi(\z)$

    \STATE {\textbf{Initialize:}}: $\z_0 =\x_0= \min_{\z\in\W}\Phi(\z)$ and $f_0(\x) =
    0$

    \FOR{$t = 1, \ldots, T$}
        \STATE Predict $\x_t$ by
        $$
           \displaystyle \x_t = \mathop{\arg\min}\limits_{\x \in \W} \left\{\dd{\x}{\nabla f_{t-1}(\x_{t-1})}+
           \frac{L}{\eta}\mathsf{B}(\x, \z_{t-1})
            \right\}
        $$
        \STATE Receive  cost function $f_t(\cdot)$ and incur  loss $f_t(\x_t)$
       \STATE Update $\z_t$ by solving
        $$
           \nabla \Phi(\z_t) =  \nabla\Phi(\x_{t})  + \frac{\eta}{L}(\nabla f_{t-1}(\x_{t-1}) - \nabla f_t(\x_t)). 
        $$
    \ENDFOR
\end{algorithmic}
\end{algorithm}
\begin{remark}
We make three remarks for Algorithm~\ref{alg:3-2}. First, the searching point $\z_t$ does not necessarily belong to the domain $\W$, which is usually not a problem given that the decision point $\x_t$ is always in $\W$. Nevertheless,  the update can  be followed by a projection step $\z_t=\min_{\x\in\W}\mathsf{B}(\x,\z'_t)$ to ensure the searching point  also stay in the domain $\W$, where we slightly abuse the notation $\z_t'$ in $\nabla \Phi(\z_t') =  \nabla\Phi(\x_{t})  + \frac{\eta}{L}(\nabla f_{t-1}(\x_{t-1}) - \nabla f_t(\x_t))$.\\

 Second, the update in step 6 can be implemented by~\cite[Chapter 11]{bianchi-2006-prediction}: $$\z_t = \nabla\Phi^*(  \nabla\Phi(\x_{t})  + \frac{\eta}{L}(\nabla f_{t-1}(\x_{t-1}) - \nabla f_t(\x_t))),$$ where $\Phi^*(\cdot)$ is the Legendre-Fenchel conjugate of $\Phi(\cdot)$ (see Appendix~\ref{chap:appendix-convex} for definition). For example, when $\Phi(\x) = 1/2\|\x\|_2^2$, $\Phi^*(\x)=  1/2\|\x\|_2^2$ and the update for the searching point  is given by $$\z_t = \x_t + ( \eta/L)(\nabla f_{t-1}(\x_{t-1})-\nabla f_t(\x_t));$$ when $\Phi(\x) = \sum_i w_i\ln w_i$, $\Phi^*(\x)= \log\left[\sum_i\exp(w_i)\right]$ and the update for  the searching point can be computed by 
 
 $$[\z_t]_i \propto [\x_t]_i \exp\left(\eta/L[\nabla f_{t-1}(\x_{t-1})-\nabla f_t(\x_t)]\right), \quad s.t.\quad \sum_i[\z_t]_i=1.$$

Third, the key inequality in~(\ref{eqn:ineq}) for proving the regret bound still hold for $\zeta= \nabla f_{t}(\x_t) + \v_t$, $\xi=\nabla f_{t-1}(\x_{t-1}) + \v_t$ by noting the equivalence between the pairs (\ref{eqn:step41}, \ref{eqn:step42}) and (\ref{eqn:project1}, \ref{eqn:project2}), which is given below:

\begin{equation*}
\begin{aligned}
&\frac{\eta}{L}\dd{\nabla f_t(\x_t) + \v_t}{\x_t-\x}\leq  \mathsf{B}(\x,\z_{t-1}) - \mathsf{B}(\x, \z_t) \\
&+ \frac{\gamma^2}{\alpha}\|\nabla f_t(\x_t)-\nabla f_{t-1}(\x_{t-1})\|_*^2 - \frac{\alpha}{2}[\|\x_t-\z_{t-1}\|^2 + \|\x-\z_{t}\|^2], \forall \x\in\W,
\end{aligned}
\end{equation*}

where $\v_t\in\partial\delta_\W(\x_t)$.  As a result, we can apply the same analysis as in the proof of Theorem~\ref{thm:2-grad} to obtain  the same regret bound in Theorem~\ref{thm:2-1} for Algorithm~\ref{alg:3-2}.  Note that the above inequality remains valid even if we take a projection step after the update for $\z'_t$ due to the generalized pythagorean inequality $\mathsf{B}(\x, \z_t)\leq \mathsf{B}(\x,\z'_t), \forall \x\in\W$~\cite{bianchi-2006-prediction}. 
\end{remark}

\subsection{A Gradual Variation Bound for Online Non-Smooth Optimization}
In spirit of Algorithm~\ref{alg:3-2}, we present an algorithm for online non-smooth optimization of functions $f_t(\x)=\widehat f_t(\x)+ g(\x)$ with a regret bound by gradual variation ${\rm{EGV}}_T= \sum_{t=0}^{T-1}\|\nabla f_{t+1}(\x_t) - \nabla f_t(\x_{t})\|_*^2$. The trick is to  solve the composite gradient mapping:
\begin{equation*}
\begin{aligned}
\x_t= \arg\min_{\x\in\W}   \dd{\x}{\nabla f_{t-1}(\x_{t-1})} + \frac{L}{\eta}\mathsf{B}(\x, \z_{t-1})  + g(\x)
\end{aligned}
\end{equation*}

and update $\z_t$ by 

\begin{equation*}
\begin{aligned}
 \nabla \Phi(\z_t) =  \nabla\Phi(\x_{t})  + \frac{\eta}{L}(\nabla f_{t-1}(\x_{t-1}) - \nabla f_t(\x_t)). 
\end{aligned}
\end{equation*}

\begin{algorithm}[t]
\center \caption{Online Mirror Prox Method with a Fixed Non-Smooth Component}
\begin{algorithmic}[1] \label{alg:3-3}
    \STATE {\textbf{Input}}: $\eta >0, \Phi(\z)$

    \STATE {\textbf{Initialization}}: $\z_0 =\x_0= \min_{\z\in\W}\Phi(\z)$ and $f_0(\x) =
    0$

    \FOR{$t = 1, \ldots, T$}
        \STATE Predict $\x_t$ by
        \[
           \displaystyle \x_t = \mathop{\arg\min}\limits_{\x \in \W} \left\{\dd{\x}{\nabla f_{t-1}(\x_{t-1})}+
           \frac{L}{\eta}\mathsf{B}(\x, \z_{t-1})  + g(\x)
            \right\}
        \]
        \STATE Receive  cost function $f_t(\cdot)$ and incur  loss $f_t(\x_t)$
       \STATE Update $\z_t$ by
        $$
           \z'_t=  \nabla\Phi^*\left(\nabla\Phi(\x_{t})  + \frac{\eta}{L}(\nabla f_{t-1}(\x_{t-1}) - \nabla f_t(\x_t))\right) 
        $$ and $\z_t= \min_{\x\in\W}\mathsf{B}(\x, \z'_t)$
    \ENDFOR
\end{algorithmic}
\end{algorithm}
Algorithm~\ref{alg:3-3} shows the detailed steps and Corollay~\ref{thm:non1} states the regret bound, which can be proved similarly. 
\begin{corollary} \label{thm:non1}
Let $f_t(\w)=\widehat f_t(\w) + g(\w), t=1, \ldots, T$ be a sequence of convex
functions where $\widehat f_t(\w)$ are L-smooth continuous w.r.t $\|\cdot\|$ and $g(\w)$ is a non-smooth function, $\Phi(\z)$ be a $\alpha$-strongly convex function w.r.t  $\|\cdot\|$, and ${\rm{EGV}}_T$ be defined in~(\ref{eqn:ge}). By setting
$\eta = (1/2)\min\left\{\sqrt{\alpha}/\sqrt{2}, LR/\sqrt{{\rm{EGV}}_T} \right\}$, we have the
following regret bound
\[
\sum_{t=1}^Tf_t(\x_t)-\min_{\x\in\W}\sum_{t=1}^Tf_t(\x) \leq 2R\max\left(\sqrt{2}LR/\sqrt{\alpha}, \sqrt{{\rm{EGV}}_T}\right).
\]
\end{corollary}

\section{Composite Losses with an Explicit Max Structure}
 In previous subsection, we assume the composite gradient mapping with the non-smooth component can be efficiently solved. Here,  we replace this assumption with an explicit max structure of the non-smooth component. 
 
 In what follows, we present a primal-dual prox method for such non-smooth cost functions and prove its regret bound.  We  consider a general setting,  where  the non-smooth functions $f_t(\x)$ has the following structure:
\begin{equation}
\begin{aligned}\label{eqn:pd}
f_t(\x) = \widehat f_t(\x) + \max_{\mathbf u\in \mathcal Q}\langle \mathbf{A}_t\x, \mathbf u\rangle - \widehat \phi_t(\u),
\end{aligned} 
\end{equation}

where $\widehat f_t(\x)$ and $\widehat \phi_t(\u)$ are  $L_1$-smooth and $L_2$-smooth functions, respectively,  and $\mathbf{A}_t\in\mathbb R^{m\times d}$ is a matrix used to characterize the non-smooth component of $f_t(\x)$ with $-\widehat\phi_t(\u)$  by  maximization. Similarly, we define a dual cost function $\phi_t(\u)$ as 

\begin{equation}
\begin{aligned}\label{eqn:dp}
\phi_t(\u) =- \widehat \phi_t(\u) + \min_{\mathbf x\in \W}\langle \mathbf{A}_t\x, \mathbf u\rangle + \widehat f_t(\x).
\end{aligned}
\end{equation}
 
We refer to $\x$ as the primal variable and to  $\mathbf u$ as the dual variable. To motivate the setup, let us consider online classification with hinge loss $\ell_t(\mathbf w)=\max(0, 1-y_t \dd{\mathbf w}{\mathbf{x}_t})$, where we slightly abuse the notation $(\mathbf x_t, y_t)$ to denote the attribute and label pair received at trial $t$. It is straightforward to see that $\ell_t(\mathbf w)$ is a non-smooth function and can be cast into the form in~(\ref{eqn:pd}) by 

\begin{equation*}
\begin{aligned}
\ell_t(\mathbf w)  = \max_{\alpha\in[0, 1]}\alpha(1-y_t\mathbf x_t^{\top}\x) = \max_{\alpha\in[0,1]} -\alpha y_t\mathbf x_t^{\top}\x  + \alpha.
\end{aligned}
\end{equation*}

To present the algorithm and analyze its regret bound, we introduce some notations. Let $F_t(\x,\u) = \widehat f_t(\x) + \langle \mathbf{A}_t\x, \u\rangle - \widehat\phi_t(\u)$, $\Phi_1(\x)$ be a $\alpha_1$-strongly convex function defined on the primal variable $\x$ w.r.t a norm $\|\cdot\|_p$ and $\Phi_2(\u)$ be a $\alpha_2$-strongly convex function defined on the dual variable $\u$ w.r.t a norm $\|\cdot\|_q$.   Correspondingly, let $\mathsf{B}_1(\x,\z)$ and $\mathsf{B}_2(\u,\v)$ denote the induced Bregman distance, respectively. We assume the domains $\W$, $\Q$ are bounded and matrices $\mathbf{A}_t$ have a bounded norm, i.e.,
\begin{equation}\label{eqn:cond}
\begin{aligned}
&\max_{\x\in\W}\|\x\|_p\leq R_1,\quad \max_{\u\in\Q}\|\u\|_q\leq R_2\\
&\max_{\x\in\W}\Phi_1(\x) -\min_{\x\in\W}\Phi_1(\x)\leq M_1\\
&\max_{\u\in\Q}\Phi_2(\u) -\min_{\u\in\Q}\Phi_2(\u)\leq M_2\\
&\|\mathbf{A}_t\|_{p, q}=\max_{\|\x\|_p\leq 1, \|\u\|_q\leq 1}\u^{\top}\mathbf{A}_t\x\leq \sigma.
\end{aligned}
\end{equation}
Let $\|\cdot\|_{p,*}$ and $\|\cdot\|_{q,*}$ denote the dual norms to $\|\cdot\|_p$ and $\|\cdot\|_{q}$, respectively. To prove a variational regret bound, we define a gradual variation as follows: 

\begin{equation}
\begin{aligned}\label{eqn:negv}
{\rm{EGV}}_{T, p, q} &= \sum_{t=0}^{T-1}\|\nabla\widehat f_{t+1}(\x_t) - \nabla \widehat f_{t}(\x_t)\|^2_{p,*}  + (R_1^2+R_2^2)\sum_{t=0}^{T-1}\|\mathbf{A}_t-\mathbf{A}_{t-1}\|_{p, q}^2
\nonumber\\
&+ \sum_{t=0}^{T-1}\|\nabla\widehat \phi_{t+1}(\u_t) - \nabla \widehat \phi_{t}(\u_t)\|^2_{q,*}.
\end{aligned}
\end{equation}

\begin{algorithm}[t]
\center \caption{Online Mirror Prox Method with an Explicit Max Structure }
\begin{algorithmic}[1] \label{alg:5}
    \STATE {\textbf{Input}}: $\eta >0, \Phi_1(\z), \Phi_2(\v)$

    \STATE {\textbf{Initialization}}: $\z_0 =\x_0= \min_{\z\in\W}\Phi_{1}(\z)$, $\v_0=\u_0=\min_{\v\in\Q}\Phi_2(\v)$ and $\widehat f_0(\x)=\widehat \phi_0(\u)=0$

    \FOR{$t = 1, \ldots, T$}
    \STATE Update $\u_t$ by $$\u_t =    \displaystyle  \mathop{\arg\max}\limits_{\u \in \Q} \left\{\dd{\u}{\mathbf{A}_{t-1}\x_{t-1}-\nabla \widehat \phi_{t-1}(\u_{t-1})} -
           \frac{L_2}{\eta}\mathsf{B}_2(\u, \v_{t-1})
            \right\}
$$
        \STATE Predict $\x_t$ by
        $$
           \displaystyle \x_t = \mathop{\arg\min}\limits_{\x \in \W} \left\{ \dd{\x}{\nabla \widehat f_{t-1}(\x_{t-1}) + \mathbf{A}_{t-1}\mathbf{A}_{t-1}^{\top}\u_{t-1}} +
           \frac{L_1}{\eta}\mathsf{B}_1(\x, \z_{t-1})
            \right\}
        $$
        \STATE Receive  cost function $f_t(\cdot)$ and incur  loss $f_t(\x_t)$
        \STATE Update $\v_t$ by
        $$
           \displaystyle \v_{t} =\mathop{\arg\max}\limits_{\u \in \Q} \left\{\dd{\u}{\mathbf{A}_t\x_t - \nabla \widehat \phi_{t}(\u_{t})}- \frac{L_2}{\eta}\mathsf{B}_2(\u, \v_{t-1})\right\}
        $$      
        \STATE Update $\z_t$ by
        $$
           \displaystyle \z_{t} =\mathop{\arg\min}\limits_{\x \in \W} \left\{\dd{\x}{\nabla \widehat f_{t}(\x_{t}) +\mathbf{A}_t^{\top}\u_t}+ \frac{L_1}{\eta}\mathsf{B}_1(\x, \z_{t-1})\right\}
        $$
  
    \ENDFOR
\end{algorithmic}
\end{algorithm}
Given above notations, Algorithm~\ref{alg:5} shows the detailed steps and Theorem~\ref{thm:non} states a min-max bound. 

\begin{theorem}\label{thm:non}
Let $f_t(\x) = \widehat f_t(\x) + \max_{\mathbf u\in \mathcal Q}\langle \mathbf{A}_t\x, \mathbf u\rangle - \widehat \phi_t(\u), t=1,\ldots, T$ be a sequence of non-smooth functions. Assume $\widehat f_t(\x), \widehat \phi(\u)$ are $L=\max(L_1,L_2)$-smooth functions and the domain $\W, \Q$ and $\mathbf{A}_t$ satisfy the boundness condition as in~(\ref{eqn:cond}).  Let $\Phi(\x)$ be a $\alpha_1$-strongly convex function w.r.t the norm $\|\cdot\|_p$, $\Phi(\u)$ be a $\alpha_2$-strongly convex function w.r.t. the norm $\|\cdot\|_q$, and $\alpha=\min(\alpha_1,\alpha_2)$. By setting $\eta= \min\left(\sqrt{M_1+M_2}/(2\sqrt{{\rm{EGV}}_{T,p,q}}), \sqrt{\alpha}/(4\sqrt{\sigma^2+L^2})\right)$ in Algorithm~\ref{alg:5}, then we have

\begin{equation*}
\begin{aligned}
&\max_{\u\in\Q}\sum_{t=1}^TF_t(\x_t, \u) - \min_{\x\in\W}\sum_{t=1}^TF_t(\x, \u_t)\\
&\leq4\sqrt{M_1+M_2}\max\left(2\sqrt{\frac{(M_1+M_2)(\sigma^2+L^2)}{\alpha}}, \sqrt{{\rm{EGV}}_{T,p, q}}\right).
\end{aligned}
\end{equation*}
\end{theorem}

To facilitate understanding, we break the proof  into several lemmas. The following lemma is by analogy with Lemma~\ref{lem:6}. 
\begin{lemma}~\label{lem:pdl2} Let   $\theta=\displaystyle\left(\x\atop \u\right)$ denote a single vector with a norm $\|\theta\|=\sqrt{\|\x\|^2_p+\|\u\|_q^2}$  and a dual norm $\|\theta\|_*=\sqrt{\|\x\|^2_{p*}+\|\u\|^2_{q*}}$. Let $\Phi(\theta)=\Phi_1(\x)+ \Phi_2(\u)$, $\mathsf{B}(\theta, \zeta)=\mathsf{B}_1(\x,\u)+ \mathsf{B}_2(\z,\v)$. Then 

\begin{equation*}
\begin{aligned}
&\eta\left(\nabla_\x F_t(\x_t, \u_t)\atop -\nabla_\u F_t(\x_t,\u_t)\right)^{\top}\left(\x_t-\x\atop\u_t-\u\right)\leq \mathsf{B}(\theta, \zeta_{t-1}) - \mathsf{B}(\theta, \zeta_t) \\
& + \eta^2\left(\|\nabla_\x F_t(\x_t,\u_t) - \nabla_\x F_{t-1}(\x_{t-1}, \u_{t-1})\|^2_{p,*}\right)\\
&+\eta^2\left(\|\nabla_\u F_t(\x_t,\u_t) - \nabla_\u F_{t-1}(\x_{t-1}, \u_{t-1})\|^2_{q,*}\right)\\
& - \frac{\alpha}{2}\left(\|\x_t-\z_t\|^2_{p} +\|\u_t-\v_t\|_q^2 + \|\x_t-\z_{t-1}\|_p^2 + \|\u_t-\v_{t-1}\|_q^2\right).
\end{aligned}
\end{equation*}

\end{lemma}
\begin{proof}
%Let  $\theta=\displaystyle\left(\x\atop \u\right)$ denote a single vector with norm $\|\theta\|=\sqrt{\|\x\|^2_p+\|\u\|_q^2}$  and dual norm $\|\theta\|_*=\sqrt{\|\x\|^2_{p*}+\|\u\|^2_{q*}}$. 
The updates of $(\x_t, \u_t)$ in Algorithm~\ref{alg:5} can be seen as applying the updates in Lemma~\ref{lem:6}  with  $\theta_t=\displaystyle\left(\x_t\atop \u_t\right)$ in place of $\x$,  $\zeta_{t}=\displaystyle \left(\z_t\atop\v_t\right)$ in place of $\z_+$, $\zeta_{t-1}=\displaystyle \left(\z_{t-1}\atop \v_{t-1}\right)$ in place of $\z$. %, and $\Phi(\theta)=\Phi_1(\x)+ \Phi_2(\u)$, $\mathsf{B}(\theta, \zeta)=\mathsf{B}_1(\x,\u)+ \mathsf{B}_2(\z,\v)$. 
Note that $\Phi(\theta)$ is a $\alpha=\min(\alpha_1,\alpha_2)$-strongly convex function with respect to the norm $\|\theta\|$. 
Then applying the results in Lemma~\ref{lem:6} we can complete the proof. 
%\begin{align*}
%&\eta\left(\nabla_\x F_t(\x_t, \u_t)\atop -\nabla_\u F_t(\x_t,\u_t)\right)^{\top}\left(\x_t-\x\atop\u_t-\u\right)\leq \mathsf{B}(\theta, \zeta_{t-1}) - \mathsf{B}(\theta, \zeta_t) \\
%& + \eta^2\left(\|\nabla_\x F_t(\x_t,\u_t) - \nabla_\x F_{t-1}(\x_{t-1}, \u_{t-1})\|^2_{p,*}\right)\\
%&+\eta^2\left(\|\|\nabla_\u F_t(\x_t,\u_t) - \nabla_\u F_{t-1}(\x_{t-1}, \u_{t-1})\|^2_{q,*}\|\right)\\
%& - \frac{\alpha}{2}\left(\|\x_t-\z_t\|^2_{p} +\|\u_t-\v_t\|_q^2 + \|\x_t-\z_{t-1}\|_p^2 + \|\u_t-\v_{t-1}\|_q^2\right).
%\end{align*}
\end{proof}
Applying the convexity of $F_t(\x,\u)$ in terms of $\x$ and the concavity of $F_t(\x,\u)$ in terms of $\u$ to the result in Lemma~\ref{lem:pdl2}, we have

\begin{equation}\label{eqn:F}
\begin{aligned}
&\eta \left( F_t(\x_t, \u) - F_t(\x, \u_t) \right)\\
%&\leq \frac{1}{2}\left\|\begin{pmatrix}\w-\u_{t-1}\\ \alpha-\beta_{t-1}\end{pmatrix}\right\|_2^2 - \frac{1}{2}\left\|\begin{pmatrix}\w-\u_t\\ \alpha -\beta_t\end{pmatrix}\right\|_2^2 \\
%&+{\gamma^2}\left\|G_\alpha(\w_t, \alpha_t)- G_\alpha(\u_{t-1}, \beta_{t-1})\right\|_2^2\\
%&- \frac{1}{2}\|\w_t-\u_{t-1}\|_2^2\\
& \leq \mathsf{B}_1(\x, \z_{t-1}) -\mathsf{B}_1(\x,\z_t) +  \mathsf{B}_2(\u, \v_{t-1}) -\mathsf{B}_2(\u,\v_t)\\
& \hspace{0.5cm}+ {\eta^2}\|\nabla\widehat f_t(\x_t) - \nabla\widehat f_{t-1}(\x_{t-1}) + \mathbf{A}_t^{\top}\u_t - \mathbf{A}_{t-1}^{\top}\u_{t-1}\|_{p,*}^2 \\
& \hspace{0.5cm}+ {\eta^2}\|\nabla\widehat \phi_t(\u_t) - \nabla\widehat \phi_{t-1}(\u_{t-1}) + \mathbf{A}_t\x_t - \mathbf{A}_{t-1}\x_{t-1}\|_{q,*}^2 \\
&\hspace{0.5cm}- \frac{\alpha}{2}\left(\|\x_t-\z_{t-1}\|_p^2+\|\x_t-\z_{t}\|_p^2\right)- \frac{\alpha}{2}\left(\|\u_t-\v_{t-1}\|_q^2 + \|\u_t-\v_{t}\|_q^2\right) \\
&\leq \mathsf{B}_1(\x, \z_{t-1}) -\mathsf{B}_1(\x,\z_t) +  \mathsf{B}_2(\u, \v_{t-1}) -\mathsf{B}_2(\u,\v_t)\\
& \hspace{0.5cm}+ 2{\eta^2}\|\nabla\widehat f_t(\x_t) - \nabla\widehat f_{t-1}(\x_{t-1})\|_{p,*} ^2+ 2\eta^2\|\mathbf{A}_t^{\top}\u_t - \mathbf{A}_{t-1}^{\top}\u_{t-1}\|_{p,*}^2 \\
& \hspace{0.5cm}+2 {\eta^2}\|\nabla\widehat \phi_t(\u_t) - \nabla\widehat \phi_{t-1}(\u_{t-1})\|_{q,*}^2 + 2\eta^2\|\mathbf{A}_t\x_t - \mathbf{A}_{t-1}\x_{t-1}\|_{q,*}^2 \\
&\hspace{0.5cm}- \frac{\alpha}{2}\left(\|\x_t-\z_{t-1}\|_p^2+\|\x_t-\z_{t}\|_p^2\right)- \frac{\alpha}{2}\left(\|\u_t-\v_{t-1}\|_q^2 + \|\u_t-\v_{t}\|_q^2\right)
\end{aligned}
\end{equation}
The following lemma provides tools for proceeding the bound. 
\begin{lemma}\label{prop:1}

\begin{equation*}
\begin{aligned}
\|\nabla\widehat f_t(\x_t) - \nabla\widehat f_{t-1}(\x_{t-1})\|_{p,*} ^2\leq 2\|\nabla\widehat f_t(\x_t) - \nabla\widehat f_{t-1}(\x_{t})\|_{p,*} ^2 + 2L^2\|\x_t -\x_{t-1}\|_{p}^2
\end{aligned}
\end{equation*}

\begin{equation*}
\begin{aligned}
\|\nabla\widehat \phi_t(\u_t) - \nabla\widehat \phi_{t-1}(\u_{t-1})\|_{q,*} ^2\leq 2\|\nabla\widehat \phi_t(\u_t) - \nabla\widehat \phi_{t-1}(\u_{t})\|_{q,*} ^2 + 2L^2\|\u_t -\u_{t-1}\|_{q}^2
\end{aligned}
\end{equation*}

\begin{equation*}
\begin{aligned}
\|\mathbf{A}_t\x_t - \mathbf{A}_{t-1}\x_{t-1}\|_{q,*}^2\leq 2R_1^2\|\mathbf{A}_t - \mathbf{A}_{t-1}\|_{p,q}^2 + 2\sigma^2 \|\x_t - \x_{t-1}\|_p^2
\end{aligned}
\end{equation*}

\begin{equation*}
\begin{aligned}
\|\mathbf{A}_t^{\top}\u_t - A^{\top}_{t-1}\u_{t-1}\|_{p,*}^2\leq 2R_2^2\|\mathbf{A}_t - \mathbf{A}_{t-1}\|_{p,q}^2 + 2\sigma^2 \|\u_t - \u_{t-1}\|_q^2
\end{aligned}
\end{equation*}

\end{lemma}
\begin{proof} We prove the first and the third inequalities. Another two inequalities can be proved similarly. 
\begin{equation*}
\begin{aligned}
&\|\nabla\widehat f_t(\x_t) - \nabla\widehat f_{t-1}(\x_{t-1})\|_{p,*} ^2\\
&\leq 2\|\nabla\widehat f_t(\x_t) - \nabla\widehat f_{t-1}(\x_{t})\|_{p,*} ^2 + 2\|\nabla\widehat f_{t-1}(\x_t) - \nabla\widehat f_{t-1}(\x_{t-1})\|_{p,*} ^2\\
&\leq 2\|\nabla\widehat f_t(\x_t) - \nabla\widehat f_{t-1}(\x_{t})\|_{p,*} ^2 + 2L^2\|\x_t-\x_{t-1}\|_{p} ^2
\end{aligned}
\end{equation*}
where we use the smoothness of $\widehat f_t(\x).$
\begin{equation*}
\begin{aligned}
\|\mathbf{A}_t\x_t - \mathbf{A}_{t-1}\x_{t-1}\|_{q,*}^2&\leq 2\|(\mathbf{A}_t - \mathbf{A}_{t-1})\x_t\|_{q,*}^2 + 2 \|\mathbf{A}_{t-1}(\x_t - \x_{t-1})\|_p^2\\
&\leq 2R_1^2\|\mathbf{A}_t - \mathbf{A}_{t-1}\|^2_{p,q} + 2\sigma^2 \|\x_t - \x_{t-1}\|_p^2
\end{aligned}
\end{equation*}
\end{proof}

\begin{lemma}\label{lem:F2} For any $\x\in\W$ and $\u\in\Q$, we have
\begin{equation*}
\begin{aligned}
&\eta \left( F_t(\x_t, \u) - F_t(\x, \u_t) \right)\\
&\leq \mathsf{B}_1(\x, \z_{t-1}) -\mathsf{B}_1(\x,\z_t) +  \mathsf{B}_2(\u, \v_{t-1}) -\mathsf{B}_2(\u,\v_t)\\
&\hspace{0.5cm}+4\eta^2\bigg(\|\nabla\widehat f_t(\x_t) - \nabla\widehat f_{t-1}(\x_{t})\|_{p,*} ^2 +\|\nabla\widehat \phi_t(\u_t) - \nabla\widehat \phi_{t-1}(\u_{t})\|_{q,*}^2\\
&\hspace*{0.6in}+ (R_1^2+R_2^2)\|\mathbf{A}_t-\mathbf{A}_{t-1}\|_{p,q}^2 \bigg)\\
&\hspace{0.5cm}+4\eta^2\sigma^2\|\x_t-\x_{t-1}\|_p^2 + 4\eta^2L^2\|\x_t-\x_{t-1}\|_p^2 - \frac{\alpha}{2}(\|\x_t-\z_t\|_p^2+ \|\x_t-\z_{t-1}\|_p^2)\\
&\hspace{0.5cm}+4\eta^2\sigma^2\|\u_t-\u_{t-1}\|_p^2 + 4\eta^2L^2\|\u_t-\u_{t-1}\|_q^2 - \frac{\alpha}{2}(\|\u_t-\v_t\|_q^2+ \|\u_t-\v_{t-1}\|_q^2). 
\end{aligned}
\end{equation*}

\end{lemma}
\begin{proof} The lemma can be proved by combining  the results in Lemma~\ref{prop:1} and the inequality in~(\ref{eqn:F}). 
%By applying the convexity of $F_t(\x,\u)$ in terms of $\x$ and the concavity of $F_t(\x,\u)$ in terms of $\u$ to the result in Lemma~\ref{lem:pdl2}, we have
\begin{equation*}
\begin{aligned}
&\eta \left( F_t(\x_t, \u) - F_t(\x, \u_t) \right)\\
&\leq \mathsf{B}_1(\x, \z_{t-1}) -\mathsf{B}_1(\x,\z_t) +  \mathsf{B}_2(\u, \v_{t-1}) -\mathsf{B}_2(\u,\v_t)\\
& \hspace{0.5cm}+ 2{\eta^2}\|\nabla\widehat f_t(\x_t) - \nabla\widehat f_{t-1}(\x_{t-1})\|_{p,*} ^2+ 2\eta^2\|\mathbf{A}_t^{\top}\u_t - \mathbf{A}_{t-1}^{\top}\u_{t-1}\|_{p,*}^2 \\
& \hspace{0.5cm}+2 {\eta^2}\|\nabla\widehat \phi_t(\u_t) - \nabla\widehat \phi_{t-1}(\u_{t-1})\|_{q,*}^2 + 2\eta^2\|\mathbf{A}_t\x_t - \mathbf{A}_{t-1}\x_{t-1}\|_{q,*}^2 \\
&\hspace{0.5cm}- \frac{\alpha}{2}\left(\|\x_t-\z_{t-1}\|_p^2+\|\x_t-\z_{t}\|_p^2\right)- \frac{\alpha}{2}\left(\|\u_t-\v_{t-1}\|_q^2 + \|\u_t-\v_{t}\|_q^2\right)\\
\hspace*{0.1in}&\leq \mathsf{B}_1(\x, \z_{t-1}) -\mathsf{B}_1(\x,\z_t) +  \mathsf{B}_2(\u, \v_{t-1}) -\mathsf{B}_2(\u,\v_t)\\
& \hspace{0.5cm}+ 4{\eta^2}\|\nabla\widehat f_t(\x_t) - \nabla\widehat f_{t-1}(\x_{t})\|_{p,*} ^2+4{\eta^2}L^2\|\x_t-\x_{t-1}\|_p^2  + 4\eta^2\sigma^2\|\x_t-\x_{t-1}\|_p^2\\
&\hspace{0.5cm}+ 4\eta^2R_1^2\|\mathbf{A}_t - \mathbf{A}_{t-1}\|_{p,q}^2 + 4\eta^2R_2^2\|\mathbf{A}_t-\mathbf{A}_{t-1}\|_{p,q}^2\\
&\hspace{0.5cm} +4 {\eta^2}\|\nabla\widehat \phi_t(\u_t) - \nabla\widehat \phi_{t-1}(\u_{t})\|_{q,*}^2 + 4\eta^2L^2\|\u_t - \u_{t-1}\|_{q}^2+ 4\eta^2\sigma^2\|\u_t-\u_{t-1}\|_q^2 \\
&\hspace{0.5cm}- \frac{\alpha}{2}\left(\|\x_t-\z_{t-1}\|_p^2+\|\x_t-\z_{t}\|_p^2\right)- \frac{\alpha}{2}\left(\|\u_t-\v_{t-1}\|_q^2 + \|\u_t-\v_{t}\|_q^2\right)
\end{aligned}
\end{equation*}

\begin{equation*}
\begin{aligned}
\hspace*{0.3cm}&\leq \mathsf{B}_1(\x, \z_{t-1}) -\mathsf{B}_1(\x,\z_t) +  \mathsf{B}_2(\u, \v_{t-1}) -\mathsf{B}_2(\u,\v_t)\\
&\hspace{0.5cm}+4\eta^2\bigg(\|\nabla\widehat f_t(\x_t) - \nabla\widehat f_{t-1}(\x_{t})\|_{p,*} ^2 +\|\nabla\widehat \phi_t(\u_t) - \nabla\widehat \phi_{t-1}(\u_{t})\|_{q,*}^2\\
&\hspace*{0.6in}+ (R_1^2+R_2^2)\|\mathbf{A}_t-\mathbf{A}_{t-1}\|_{p,q}^2 \bigg)\\
&\hspace{0.5cm}+4\eta^2\sigma^2\|\x_t-\x_{t-1}\|_p^2 + 4\eta^2L^2\|\x_t-\x_{t-1}\|_p^2 - \frac{\alpha}{2}(\|\x_t-\z_t\|_p^2+ \|\x_t-\z_{t-1}\|_p^2)\\
&\hspace{0.5cm}+4\eta^2\sigma^2\|\u_t-\u_{t-1}\|_p^2 + 4\eta^2L^2\|\u_t-\u_{t-1}\|_q^2 - \frac{\alpha}{2}(\|\u_t-\v_t\|_q^2+ \|\u_t-\v_{t-1}\|_q^2). 
\end{aligned}
\end{equation*}

\end{proof}
\begin{proof}[Proof of Theorem~\ref{thm:non}]
Taking summation of the inequalities in Lemma~\ref{lem:F2} over $t=1,\ldots, T$, applying the inequality in~(\ref{eqn:Bt}) twice and using $\eta\leq \sqrt{\alpha}/(4\sqrt{\sigma^2+L^2})$, we have
\begin{equation}
\begin{aligned}\label{eqn:minmax}
&\sum_{t=1}^TF_t(\x_t, \u)-\sum_{t=1}^TF_t(\x,\u_t)\leq 4\eta{\rm{EGV}}_{T,p,q}  + \frac{M_1+M_2}{\eta}\nonumber\\
&= 4\sqrt{M_1+M_2}\max\left(2\sqrt{\frac{(M_1+M_2)(\sigma^2+L^2)}{\alpha}}, \sqrt{{\rm{EGV}}_{T,p, q}}\right).
\end{aligned}
\end{equation}
We complete the proof by using $\x^*=\arg\min_{\x\in\W}\sum_{t=1}^TF_t(\x_t,\u)$ and $\u^*=\arg\max_{\u\in\Q}\sum_{t=1}^TF_t(\x,\u_t)$. 
\end{proof}

As an immediate  result of Theorem~\ref{thm:non},  the following Corollary states a regret bound for  non-smooth optimization with a fixed non-smooth component that can be written as a max structure, i.e., $f_t(\x) =\widehat f_t(\x) + [g(\x)=\max_{\u\in\Q}\langle \mathbf{A}\x, \u\rangle - \widehat\phi(\u)]$. 
\begin{corollary}\label{thm:non2}
Let $f_t(\x) = \widehat f_t(\x) + g(\x), t=1,\ldots, T$ be a sequence of non-smooth functions, where $g(\x)=\max_{\mathbf u\in \mathcal Q}\langle \mathbf{A}\x, \mathbf u\rangle - \widehat \phi(\u)$, and the gradual variation ${\rm{EGV}}_T$ be defined in~(\ref{eqn:ge}) w.r.t the dual norm $\|\cdot\|_{p,*}$. Assume $\widehat f_t(\x)$ are $L$-smooth functions w.r.t $\|\cdot\|$,  the domain $\W, \Q$ and $\mathbf{A}$ satisfy the boundness condition as in~(\ref{eqn:cond}).  If we set $\eta=\min\left(\sqrt{M_1+M_2}/(2\sqrt{{\rm{EGV}}_{T}}), \sqrt{\alpha}/(4\sqrt{\sigma^2+L^2})\right)$ in Algorithm~\ref{alg:5}, then we have the following regret bound

\begin{equation*}
\begin{aligned}
&\sum_{t=1}^Tf_t(\x_t)- \min_{\x\in\W}\sum_{t=1}^Tf_t(\x)\\
&\leq4\sqrt{M_1+M_2}\max\left(2\sqrt{\frac{(M_1+M_2)(\sigma^2+L^2)}{\alpha}}, \sqrt{{\rm{EGV}}_{T}}\right) + \text{V}(g, \x_{1:T}),
\end{aligned}
\end{equation*}

where $\widehat\x_T= \sum_{t=1}^T\x_t/T$ and $\text{V}(g, \x_{1:T})=\sum_{t=1}^T|g(\x_t)-g(\widehat\x_T)|$ measures the variation in the non-smooth component. 
\end{corollary}
\begin{proof}
In the case of fixed non-smooth component, the gradual variation defined in~(\ref{eqn:negv}) reduces the one defined in~(\ref{eqn:ge}) w.r.t the dual norm $\|\cdot\|_{p,*}$. By using the bound in~(\ref{eqn:minmax}) and noting  that $f_t(\x)=\max_{\u\in\Q}F_t(\x,\u)\geq F_t(\x, \u_t)$, we have 

\begin{equation*}
\begin{aligned}
\sum_{t=1}^T\bigg( \widehat f_t(\x_t) + \langle \mathbf{A}\x_t, \u\rangle  - \widehat \phi(\u)\bigg)& \leq \sum_{t=1}^T f_t(\x) + \\
&\hspace*{-0.6in}4\sqrt{M_1+M_2}\max\left(2\sqrt{\frac{(M_1+M_2)(\sigma^2+L^2)}{\alpha}}, \sqrt{{\rm{EGV}}_{T}}\right).
\end{aligned}
\end{equation*}
Therefore 
\begin{equation*}
\begin{aligned}
\sum_{t=1}^T \bigg(\widehat f_t(\x_t) + g(\widehat \x_T)\bigg)& \leq \sum_{t=1}^T f_t(\x) + \\
&\hspace*{-0.3in}4\sqrt{M_1+M_2}\max\left(2\sqrt{\frac{(M_1+M_2)(\sigma^2+L^2)}{\alpha}}, \sqrt{{\rm{EGV}}_{T}}\right).
\end{aligned}
\end{equation*}
We complete the proof by complementing $\widehat f_t(\x_t)$  with $g(\x_t)$ to obtain  $f_t(\x_t)$ and  moving the additional term $\sum_{t=1}^T(g(\widehat \x_T)-g(\x_t))$ to the right hand side .  
\end{proof}
\begin{remark}
Note that the regret bound in Corollary~\ref{thm:non2} has an additional term $V(g,\x_{1:T})$ compared to the regret bound in Corollary~\ref{thm:non1}, which constitutes  a tradeoff between the reduced computational cost in solving a composite gradient mapping. 
\end{remark}

To see an application of Theorem~\ref{thm:non} to an online non-smooth optimization with time-varying non-smooth components, let us consider the example of online classification  with hinge loss. At each trial, upon receiving an example $\mathbf{x}_t$, we need to make a prediction based on the current model $\w_t$, i.e., $\widehat y_t=\dd{\w_t}{\mathbf{x}_t}$, then we receive the true label of $\mathbf{x}_t$ denoted by $y_t\in\{+1, -1\}$.  The goal is to minimize the total  number of mistakes across the time line $M_T=\sum_{t=1}^T{\mathbb{I}}[\widehat y_ty_t\leq 0]$.  Here we are interested in a scenario that the data sequence  $(\mathbf{x}_t, y_t),t=1,\ldots, T$ has a small gradual variation in terms of $y_t\mathbf{x}_t$.   To obtain such a gradual variation based mistake bound, we can apply Algorithm~\ref{alg:5}. For the purpose of deriving the mistake bound,  we need to make a small change to Algorithm~\ref{alg:5}. At the beginning of each trial, we first make a prediction $\widehat y_t=\dd{\w_{t}}{\mathbf{x}_t}$, and if we make a mistake $\mathbb{I}[\widehat y_ty_t\leq 0]$ the we proceed to update the auxiliary primal-dual pair $(\w'_t, \beta_t)$  similar to $(\z_t,\v_t)$ in Algorithm~\ref{alg:5}  and the primal-dual pair $(\w_{t+1}, \alpha_{t+1})$  similar to $(\x_{t+1},\u_{t+1})$ in Algorithm~\ref{alg:5}, which are given explicitly as follows: 
\begin{equation*}
\begin{aligned}
\beta_{t} &= \prod_{[0, 1]}\left(\beta_{t-1} +\eta(1-\w_{t}^{\top}y_{t}\mathbf{x}_{t}) \right),\quad \w'_{t}   = \prod_{\|\w\|_2\leq R}(\w'_{t-1} + \eta\alpha_{t}y_{t}\mathbf{x}_{t})\\
\alpha_{t+1} &= \prod_{[0, 1]}\left(\beta_t +\eta(1-\w_{t}^{\top}y_t\mathbf{x}_t) \right),\quad \w_{t+1}   = \prod_{\|\w\|_2\leq R}(\w'_t + \eta\alpha_ty_t\mathbf{x}_t).
\end{aligned}
\end{equation*}
Without loss of generality, we let $(\mathbf{x}_t, y_t),t=1,\ldots, M_T$ denote the examples that are predicted incorrectly. The  function $F_t(\cdot,\cdot)$ is written as $F_t(\w, \alpha) = \alpha (1-y_t\dd{\w}{\mathbf{x}_t})$.  Then for a total sequence of $T$ examples, we have the following bound by assuming $\|\x_t\|_2\leq 1$ and $\eta\leq {1}/{2\sqrt{2}}$
\begin{equation*}
\begin{aligned}
\sum_{t=1}^{M_T}F_t(\w_{t}, \alpha) \leq \sum_{t=1}^{M_T}\ell(y_t\w^{\top}_t\mathbf{x}_t)  + \eta \sum_{t=0}^{M_T-1}(R^2+1)\|y_{t+1}\mathbf{x}_{t+1} - y_t\mathbf{x}_t\|_2^2 + \frac{R^2 + \alpha^2}{2\eta}.
\end{aligned}
\end{equation*}

Since $y_t\w_{t}^{\top}\mathbf{x}_t$ is less than  $0$ for the incorrectly predicted examples, if we set $\alpha=1$ in the above inequality, we have
\begin{equation*}
\begin{aligned}
M_T&\leq \sum_{t=1}^{M_T}\ell(y_t\w^{\top}_t\mathbf{x}_t)  + \eta \sum_{t=0}^{M_T-1}(R^2+1)\|y_{t+1}\mathbf{x}_{t+1} - y_t\mathbf{x}_t\|_2^2 + \frac{R^2 + 1}{2\eta}\\
&\leq \sum_{t=1}^{M_T}\ell(y_t\w^{\top}_t\mathbf{x}_t)  + \sqrt{2}(R^2+1)\max(2, \sqrt{{\rm{EGV}}_{T,2}}).
\end{aligned}
\end{equation*}
which results in a gradual variational mistake bound, where $\sqrt{\text{EVG}_{T,2}}$ measures the gradual variation in the incorrectly predicted examples. {To end the discussion, we note that one may find applications of a small gradual variation of $y_t\mathbf{x}_t$ in time series classification. For instance,  if $\mathbf{x}_t$ represent some medical measurements of a person  and $y_t$ indicates whether the person observes a disease, since the health conditions usually change slowly then it is expected that the gradual variation of $y_t\mathbf{x}_t$ is small. Similarly,  if $\mathbf{x}_t$ are some sensor measurements of an equipment  and $y_t$ indicates whether the equipment fails or not, we would also observe a small gradual variation of  the sequence $y_t\mathbf{x}_t$  during a time period. 
}

\section{Summary}
In this chapter we developed a simplified online mirror prox method using a composite gradient mapping for non-smooth optimization with a fixed non-smooth component  and a primal-dual prox method for non-smooth optimization with the non-smooth component written as a  max structure.  Despite the impossibility result in Chapter~\ref{chap:regret} which demonstrated that smoothness of loss functions in necessary to obtain gradual variation bounds, we showed that a simplified version of online mirror prix method is able to attain regret bounded by gradual variation for loss functions with a smooth component and two types of mentioned non-smooth components.

\part{~Stochastic Optimization}\label{part-stochastic}
\def \F {\mathcal{F}}
\def \G {\mathcal{F}}
\def \gh {\widehat{f}}
\chapter{Mixed Optimization for Smooth Losses}\label{chap:mixed}
In this part of the thesis, we consider  stochastic convex optimization problem and show that  leveraging the smoothness of functions allows us  to devise stochastic optimization algorithms that enjoy faster convergence rate. 

The focus of this chapter is on stochastic smooth optimization. The motivation for exploiting smoothness in stochastic optimization stems from the observation that the optimal convergence rate for stochastic optimization of smooth functions is $O(1/\sqrt{T})$, which is same as stochastic optimization of Lipschitz continuous convex functions. This is in contrast to optimizing smooth functions using full gradients, which yields a convergence rate of $O(1/T^2)$. Therefore, it is of great interest to exploit smoothness in stochastic setting as well. In particular,  we are interested in designing an efficient algorithm that is in the same spirit of  the stochastic gradient descent method, but can effectively leverage the smoothness of the loss function to achieve a significantly faster convergence rate.

We introduce a new setup for optimizing  convex functions, termed as \textbf{mixed optimization}, which allows to access  both a stochastic oracle  and a full gradient oracle to take advantages of their individual merits. Our goal is to significantly improve the convergence rate of stochastic optimization of smooth functions by having an additional small number of accesses to the full gradient oracle. We show that, with an $O(\ln T)$ calls to the full gradient oracle  and an $O(T)$ calls to the stochastic oracle, the proposed mixed optimization algorithm is able to achieve an optimization error of $O(1/T)$. The key insight underlying the mixed optimization paradigm is that by  infrequent use of full gradients at specific points we are able to progressively reduce the variance of stochastic gradients which leads to  faster convergence rates. 

The rest of this chapter is organized as follows.  In Section~\ref{sec:7-motiv} we motivate the problem.  Section~\ref{sec:mixedgrad} describes the {{MixedGrad}} algorithm, discusses the main intuition behind it, and states the main result on its convergence rate. The proof of convergence rate  is given in Section~\ref{sec:7-analysis} and the omitted proofs are deferred to Section~\ref{mixed-7-omitted-proofs}.  Section~\ref{sec:7-conclusion} concludes the chapter and discusses few open questions. Finally, Section~\ref{sec:7-related} briefly reviews the literature on deterministic and stochastic optimization.

\section{Motivation}\label{sec:7-motiv}
As it has been shown in Chapter~\ref{chap-background},   many practical machine learning algorithms follow the framework of empirical risk minimization, which often can be cast into the following generic optimization problem:
\begin{eqnarray}
\min\limits_{\w \in \W} \; \G(\w) := \frac{1}{n}\sum_{i=1}^n f_i(\w),
 \label{eqn:1}
\end{eqnarray}
where $n$ is the number of training examples, $f_i(\w)$ encodes the loss function related to the $i$th training example $(\mathbf{x}_i, y_i)$, and $\W$ is a bounded convex domain that is introduced to regularize the solution $\w \in \W$ (i.e., the smaller the size of $\W$, the stronger the regularization is). In this chapter, we focus on the learning problems for which the loss function $f_i(\w)$ is smooth. Examples of smooth loss functions include least square with $f_i(\w) = (y_i - \dd{\w}{\mathbf{x}_i})^2$  and logistic regression with $f_i(\w) = \log \left(1 + \exp(-y_i \dd{\w}{\mathbf{x}_i}) \right)$. Since the regularization is enforced through the restricted domain $\W$, we did not introduce a $\ell_2$ regularizer ${\lambda}\|\w\|^2/2$ into the optimization problem and as a result, we do not assume the loss function to be strongly convex. We note that a small $\ell_2$ regularizer does NOT improve the convergence rate of stochastic optimization. More specifically, the convergence rate for stochastically optimizing a $\ell_2$ regularized loss function remains as $O(1/\sqrt{T})$ when $\lambda = O(1/\sqrt{T})$~\cite{hazan-2011-beyond}, a scenario that is often encountered in real-world applications.

 A preliminary approach for solving the optimization problem in~(\ref{eqn:1}) is the batch gradient descent (GD) algorithm~\cite{nesterov2004introductory}. It starts with some initial point, and iteratively updates the solution using the equation $\w_{t+1} = \Pi_{\W}(\w_t - \eta \nabla \G(\w_t))$, where $\Pi_{\W}(\cdot)$ is the orthogonal projection onto the convex domain $\W$. It has been shown that for smooth objective functions, the convergence rate of standard GD is $O(1/T)$~\cite{nesterov2004introductory}, and can be improved to $O(1/T^2)$ by an accelerated GD algorithm~\cite{nesterov1983method,nesterov2004introductory,nesterov2005smooth}. The main shortcoming of GD method is its high cost in computing the full gradient $\nabla \G(\w_t)$ when the number of training examples is large, i.e., it requires $O(n)$ gradient computations per iteration. Stochastic gradient descent (SGD) alleviates this limitation of GD by sampling one (or a small set of) examples and computing a stochastic (sub)gradient at each iteration based on the sampled examples~\cite{NIPS2007_726,nemirovski2009robust,Shalev-Shwartz:2007:PPE:1273496.1273598}. Since the computational cost of SGD per iteration is independent of the size of the data (i.e., $n$), it is usually appealing for large-scale learning and optimization.

While SGD enjoys a high computational efficiency per iteration, it suffers from a slow convergence rate for optimizing smooth functions. It has been shown in~\cite{nemircomp1983} that the effect of the stochastic noise cannot be decreased with a better rate than $O(1/\sqrt{T})$ which is significantly worse than GD that uses the full gradients for updating the solutions and this limitation is also valid when the target function is smooth. %It has been shown that the {\textit optimal} convergence rate for stochastic optimization of smooth functions is only $O(1/\sqrt{T})$~\cite{nemirovsky1983problem}, which is significantly worse than GD that uses the full gradients for updating the solutions. 
 In addition  for general Lipschitz continuous convex functions, SGD exhibits the same convergence rate as that for the smooth functions, implying that smoothness of the loss function is essentially not very useful and can not be exploited in stochastic optimization. The slow convergence rate for stochastically optimizing smooth loss functions is mostly due to the variance in stochastic gradients: unlike the full gradient case where the norm of a gradient approaches to zero when the solution is approaching to the optimal solution, in stochastic optimization, the norm of a stochastic gradient is constant even when the solution is close to the optimal solution. It is the variance in stochastic gradients that makes the convergence rate $O(1/\sqrt{T})$ unimprovable  for stochastic  smooth optimization~\cite{nemircomp1983,sgd-lower-bounds}.
%\begin{savenotes}

%\begin{table}
%\centering
%\begin{tabular}{|c||c|c|c||c|c|c||c|c|c|}
%\hline
%& \multicolumn{3}{c||}{\textbf{Full (GD)}} & \multicolumn{3}{c||}{\textbf{Stochastic (SGD)}}& \multicolumn{3}{c|}{\parbox{3.2cm}{\textbf{Mixed Optimization}}}\\
%\hline
%\textbf{Setting} & {\small Convergence} & $\O_s$ & $\O_f$ & {\small Convergence} & $\O_s$ & $\O_f$ & {\small Convergence} & $\O_s$ & $\O_f$ \\
%\hline\hline
%Lipschitz  & $\frac{1}{\sqrt{T}}$~\footnote{The convergence rate can be improved to $O(1/T)$ when the structure of the objective function is provided.} & 0 & $T $ & $\frac{1}{\sqrt{T}}$ & $T$ & 0 & --- & --- & ---\\
%\hline\hline
%Smooth & $\frac{1}{T^2}$ & 0 &  $T$ & $\frac{1}{\sqrt{T}}$  & $T$ & 0 & $\frac{1}{T}$ & $T$ & $ \log T$\\
%\hline
%\end{tabular}
%\caption{The convergence rate ($O$), number of calls to stochastic oracle ($\O_s$), and number of calls to full gradient oracle ($\O_f$)  for optimizing Lipschitz continuous and smooth convex functions, using full GD,  SGD, and mixed optimization methods, measured in the number of iterations $T$. }
%\label{table:results}
%\end{table}

In this chapter, we are interested in designing an efficient algorithm that is in the same spirit of SGD but can effectively leverage the smoothness of the loss function to achieve a significantly faster convergence rate. To this end, we consider a new setup for optimization that allows us to interplay between stochastic and deterministic gradient descent methods. In particular, we assume that the optimization algorithm has an access to two oracles:
\begin{itemize}
\item A stochastic oracle $\O_s$ that returns the  loss function $f_i(\w)$ based on the sampled training example $(\mathbf{x}_i, y_i)$~\footnote{We note that the stochastic oracle assumed in our study is slightly stronger than the stochastic gradient oracle as it returns the sampled function instead of the stochastic gradient.}, and
\item A full gradient oracle $\O_f$ that returns the  gradient $\nabla \G(\w)$ for any given solution $\w \in \W$.
\end{itemize}
We refer to this new setting as  \textit{mixed optimization} in order to distinguish it from both stochastic and  full gradient optimization models. Obviously, the challenging issue in this regard is to minimize the number of full gradients to be as minimum as possible while having the same number of stochastic gradient accesses.  The key  question we examined in this chapter is:

 \begin{myquotation}
{\em Is it possible to improve the convergence rate for stochastic optimization of smooth functions by having a small number of calls to the full gradient oracle $\O_f$}?
 \end{myquotation}

In this  chapter we give an affirmative answer to this question. In particular, we show that with an additional $O(\ln T)$ accesses to the full gradient oracle $\O_f$, the proposed algorithm, referred to as {{MixedGrad}}, can improve the convergence rate for stochastic optimization of smooth functions to $O(1/T)$, the same rate for stochastically optimizing a strongly convex function~\cite{hazan-2011-beyond,ICML2012Rakhlin,ohad-2013}.   The {{MixedGrad}} algorithm builds off  on multi-stage methods~\cite{hazan-2011-beyond} and  operates in epochs, but involves novel ingredients so as to obtain an $O(1/T)$ rate for smooth losses. In particular, we form a sequence of strongly convex objective functions to be optimized at each epoch and  decrease the amount of regularization and shrink the domain as the  algorithm proceeds. The full gradient oracle $\O_f$ is only called at the beginning of each epoch. In this chapter our focus is only smooth functions and in the coming chapters, we show that  it is further possible to develop faster optimization schemes when the target function is both smooth and strongly convex by making the number of accesses to the full gradients independent of the condition number, leading to more efficient optimization algorithms for ill-conditioned problems.

\section{The MixedGrad Algorithm}\label{sec:mixedgrad}

In stochastic first-order optimization setting, instead of having direct access to $\G(\w)$, we only have access to a stochastic gradient oracle, which given a solution $\w \in\W$, returns the gradient $\nabla f_i(\w)$ where $i$ is sampled  uniformly at random from $\{1, 2, \cdots, n\}$. The goal of stochastic optimization to use a bounded number $T$ of oracle calls, and compute some $\bar{\w}\in \W$ such that the optimization error, $\F(\bar{\w})-\F(\w^*)$, is as small as possible.

In the mixed optimization model introduced here, we first relax the stochastic oracle $\O_s$ by assuming that it will return a randomly sampled loss function $f_i(\w)$, instead of the gradient $\nabla f_i(\w)$ for a given solution $\w$  may be applied to achieve $O(1/T)$ convergence~\footnote{The audience may feel that this relaxation of stochastic oracle could provide significantly more information, and second order methods such as Online Newton~\cite{hazan-log-newton}. We note (i) the proposed algorithm is a first order method, and (ii) although the Online Newton method yields a regret bound of $O(1/T)$, its convergence rate for optimization can be as low as $O(1/\sqrt{T})$ due to the concentration bound for Martingales. In addition, the Online Newton method is only applicable to exponential concave function, not any smooth loss function.}. Second, we assume that the learner also has an  access to the full gradient oracle $\O_f$. Our goal is to significantly improve the convergence rate of stochastic gradient descent (SGD) by making a small number of calls to the full gradient oracle $\O_f$. In particular, we show that by having only $O(\log T)$ accesses to the full gradient oracle and $O(T)$ accesses to the stochastic oracle, we can tolerate the noise in stochastic gradients and attain an $O(1/T)$ convergence rate for optimizing smooth functions.

We now turn to describe the proposed mixed optimization algorithm and state its  convergence rate. The detailed steps of {{MixedGrad}} algorithm are shown in Algorithm~\ref{alg:MixedGrad}.  It follows the epoch gradient descent algorithm proposed in~\cite{hazan-2011-beyond} for stochastically minimizing strongly convex functions and divides the optimization process into $m$ epochs, but involves novel ingredients so as to obtain an $O(1/T)$ convergence rate. The key idea is to introduce a $\ell_2$ regularizer into the objective function to make it strongly convex, and gradually reduce the amount of regularization over the epochs. We also shrink the domain as the  algorithm proceeds.  We note that reducing the amount of regularization over time is closely-related to the classic proximal-point algorithms. Throughout the chapter, we will use the subscript for the index of each epoch, and the superscript for the index of iterations within each epoch. Below, we describe the key idea behind the {{MixedGrad}} algorithm.

\begin{algorithm}[t]
\caption{{MixedGrad} Algorithm}%: Mixed Gradient Method}%Epoch Dual Coordinate Descent}
\begin{algorithmic}[1]
\STATE {\textbf{Input}}: 
\begin{mylist}
\item step size $\eta_1$
\item domain size $\Delta_1$
\item the number of iterations $T_1$ for the first epoch
\item the number of epochs $m$
\item regularization parameter $\lambda_1$
\item shrinking parameter $\gamma > 1$
\end{mylist}
\STATE \textbf{Initialize}: $\wb_1 = \bz$
\FOR{$k = 1, \ldots, m$}
     \STATE Construct the domain $\W_k = \{\w: \w+\w_k \in \W, \|\w\| \leq \Delta_k\}$
    \STATE Call the full gradient oracle $\O_f$ for $\nabla \G(\wb_k)$
    \STATE Compute $\g_k = \lambda_k \wb_k + \nabla \G(\wb_k) = \lambda_k \wb_k + \frac{1}{n}\sum_{i=1}^n \nabla f_i(\wb_k)$
    \STATE Initialize $\w_k^1 = \bz$
    \FOR{$t = 1, \ldots, T_k$}
        \STATE Call stochastic oracle $\O_s$ to return a randomly selected loss function $f_{i_k^t}(\w)$
        \STATE Compute the stochastic gradient as $\gb_k^t = \g_k + \nabla f_{i_k^t}(\w_k^t + \wb_k) - \nabla f_{i_k^t}(\wb_k)$
        \STATE Update the solution by
        \[
            \w_k^{t+1} = \mathop{\arg\max}\limits_{\w \in \W_k} \eta_k\langle \w - \w^t_k, \gb_k^t + \lambda_k \w_k^t\rangle + \frac{1}{2}\|\w - \w_k^t\|^2
        \]
    \ENDFOR
    \STATE Set $\wt_{k+1} = \frac{1}{T+1} \sum_{t=1}^{T+1} \w_k^t$ and $\wb_{k+1} = \wb_{k} + \wt_{k+1}$
    \STATE Set $\Delta_{k+1} = \Delta_k/\gamma$, $\lambda_{k+1} = \lambda_k /\gamma$, $\eta_{k+1} = \eta_k/\gamma$, and $T_{k+1} = \gamma^2 T_k$
\ENDFOR
\end{algorithmic} \label{alg:MixedGrad}
{\textbf{Return}} $\wb_{m+1}$
\end{algorithm}

Let $\wb_k$ be the solution obtained before the $k$th epoch, which is initialized to be $\bz$ for the first epoch. Instead of searching for $\w_*$ at the $k$th epoch, our goal is to find $\w_* - \wb_k$, resulting in the following optimization problem for the $k$th epoch
\begin{eqnarray}
\min\limits_{\small \begin{array}{c} \w + \w_k \in \W \\ \|\w\| \leq \Delta_k\end{array}} \; \frac{\lambda_k}{2}\|\w + \wb_k\|^2 + \frac{1}{n}\sum_{i=1}^n f_i(\w + \wb_k), \label{eqn:opt-epoch-1}
\end{eqnarray}

where $\Delta_k$ specifies the domain size of $\w$ and $\lambda_k$ is the regularization parameter introduced at the $k$th epoch. By introducing the $\ell_2$ regularizer, the objective function in (\ref{eqn:opt-epoch-1}) becomes strongly convex, making it possible to exploit the technique for stochastic optimization of strongly convex function in order to improve the convergence rate. The domain size $\Delta_k$ and the regularization parameter $\lambda_k$ are initialized to be $\Delta_1 >0$ and $\lambda_1 > 0$, respectively, and are reduced by a constant factor $\gamma > 1$ every epoch, i.e., $\Delta_k = \Delta_1 /\gamma^{k-1}$ and $\lambda_k = \lambda_1/\gamma^{k-1}$. By removing the constant term ${\lambda_k}\|\wb_k\|^2/2$ from the objective function in (\ref{eqn:opt-epoch-1}), we obtain the following  optimization problem for the $k$th epoch
\begin{eqnarray}
\min\limits_{\w \in \W_k} \;  \left[\F_k(\w) = \frac{\lambda_k}{2}\|\w\|^2 + \lambda_k \langle \w, \wb_k\rangle + \frac{1}{n}\sum_{i=1}^n f_i(\w + \wb_k) \right] ,\label{eqn:opt-epoch}
\end{eqnarray}
where $\W_k = \{\w: \w + \w_k \in \W,\;  \|\w\| \leq \Delta_k \}$. We rewrite the objective function $\F_k(\w)$ as
\begin{eqnarray}
\F_k(\w) & = & \frac{\lambda_k}{2}\|\w\|^2 + \lambda_k \langle \w, \wb_k\rangle + \frac{1}{n}\sum_{i=1}^n f_i(\w + \wb_k) \nonumber \\
& = & \frac{\lambda_k}{2}\|\w\|^2 + \left\langle \w, \lambda_k\wb_k + \frac{1}{n}\sum_{i=1}^n \nabla f_i(\wb_k) \right \rangle + \frac{1}{n}\sum_{i=1}^n f_i(\w + \wb_k) - \langle \w, \nabla f_i(\wb_k) \rangle \nonumber \\
& = & \frac{\lambda_k}{2}\|\w\|^2 + \langle \w, \g_k \rangle + \frac{1}{n}\sum_{i=1}^n \gh^k_i(\w) \label{eqn:fmixed}
\end{eqnarray}
where
\[
\g_k = \lambda_k\wb_k + \frac{1}{n}\sum_{i=1}^n \nabla f_i(\wb_k) \;\;\text{and} \;\; \gh^k_i(\w) = f_i(\w + \wb_k) - \langle \w, \nabla f_i(\wb_k) \rangle.
\]

The main reason for using $\gh^k_i(\w)$ instead of $f_i(\w)$ is to tolerate the variance in the stochastic gradients. To see this,  from the smoothness assumption of  $f_i(\w)$ we obtain the following inequality  for the norm of $\gh^k_i(\w)$ as:
\[
\left\|\nabla \gh^k_i(\w) \right\| = \left\|\nabla f_i(\w+\wb_k) - \nabla f_i(\wb_k) \right\| \leq \beta\|\w\|.
\]
As a result, since $\|\w\| \leq \Delta_k$ and $\Delta_k$ shrinks over epochs, then  $\|\w\|$ will approach to zero over epochs and consequentially $\|\nabla \gh^k_i(\w)\|$ approaches to zero, which allows us to effectively control the variance in stochastic gradients, a key to improving the convergence of stochastic optimization for smooth functions to $O(1/T)$.

Using $\F_k(\w)$ in (\ref{eqn:fmixed}), at the $t$th iteration of the $k$th epoch, we call the stochastic oracle $\O_s$ to randomly select a loss function $f_{i_t^k}(\w)$ and update the solution by following the standard paradigm of SGD by

\begin{eqnarray}
\w_k^{t+1} & = & \Pi_{\w \in \W_k}\left( \w_k^t - \eta_k(\lambda_k \w_k^t + \g_k + \nabla \gh^k_{i_k^t}(\w_k^t)) \right) \nonumber \\
& = & \Pi_{\w \in \W_k}\left( \w_k^t - \eta_k(\lambda_k \w_k^t + \g_k + \nabla f_{i_k^t}(\w_k^t + \wb_k) - \nabla f_{i_k^t}(\wb_k)) \right), \label{eqn:update-1}
\end{eqnarray}
where $\Pi_{\w \in \W_k}(\cdot)$ projects the solution $\w$ into the domain $\W_k$ that shrinks over epochs.

At the end of each epoch, we compute the average solution $\wt_k$, and update the solution from $\wb_k$ to $\wb_{k+1} = \wb_k + \wt_k$. Similar to the epoch gradient descent algorithm~\cite{hazan-2011-beyond}, we increase the number of iterations by a constant $\gamma^2$ for every epoch, i.e. $T_k = T_1 \gamma^{2(k-1)}$.

In order to perform stochastic gradient updating given in (\ref{eqn:update-1}), we need to compute vector $\g_k$ at the beginning of the $k$th epoch, which requires an access to the full gradient oracle $\O_f$. By  infrequent use of full gradients at the beginning of each epoch we are able to progressively reduce the variance of stochastic gradients which leads to  faster convergence rates.  It is easy to count that the number of accesses to the full gradient oracle $\O_f$ is $m$, and the number of accesses  to the stochastic oracle $\O_s$ is
\[
T = T_1\sum_{i=1}^m \gamma^{2(i - 1)} = \frac{\gamma^{2m} - 1}{\gamma^2 - 1} T_1.
\]
Thus, if the total number of accesses to the stochastic gradient oracle is $T$, the number of access to the full gradient oracle required by MixedGrad algorithm is $O(\ln T)$, consistent with our goal of making a small number of calls to the full gradient oracle.

 The theorem below shows that for smooth objective functions, by having $O(\ln T)$ access to the full gradient oracle $\O_f$ and $O(T)$ access to the stochastic oracle $\O_s$, by running {MixedGrad} algorithm, we achieve an optimization error of $O(1/T)$.
\begin{theorem} \label{theorem:1}
Let $\delta \leq e^{-9/2}$ be the failure probability. Set $\gamma = 2$, $\lambda_1 = 16\beta$ and
\[
    T_1 = 300\ln\frac{m}{\delta}, \quad \eta_1 = \frac{1}{2\beta \sqrt{3T_1}},\;\;\text{and}\;\; \Delta_1 = R.
\]
Define $T = T_1\left(2^{2m} - 1\right)/3$. Let $\wb_{m+1}$ be the solution returned by MixedGrad method in Algorithm~\ref{alg:MixedGrad} after $m$ epochs with $m = O(\ln T)$ calls to the full gradient oracle $\O_f$ and $T$ calls to the stochastic oracle $\O_s$. Then, with a probability $1 - 2\delta$, we have
\[
\F(\wb_{m+1}) - \min\limits_{\w \in \W} \F(\w) \leq \frac{80\beta R^2}{2^{2m - 2}} = O\left(\frac{\beta}{T}\right).
\]
\end{theorem}

%\paragraph{Remark} We argue that it is impossible to achieve a convergence rate better than $O(\beta/T)$ by making $O(T)$ calls to the stochastic oracle $O_s$ and $O(\ln T)$ calls to the full gradient oracle. We prove our statement by contradiction. Assume we are able to achieve a converge rate than $O(\beta/T^q)$, with $q > 1$, for any smooth objective function by making $O(T)$ calls to the stochastic oracle and $O(\ln T)$ calls to the . Then, following the smoothing technique used in~\cite{}, we could

\section{Analysis of Convergence Rate}\label{sec:7-analysis}
Now we turn to proving the main theorem. The proof will be given in a series of lemmas
and theorems where the proof of few are given in Section~\ref{mixed-7-omitted-proofs}. The  proof of main theorem is based on induction. To this end, let $\wh_*^k$ be the optimal solution that minimizes $\F_k(\w)$ defined in~(\ref{eqn:opt-epoch}). The key to our analysis is show that when $\|\wh_*^k\| \leq \Delta_k$,  with a high probability, it holds that $\|\wh_*^{k+1}\| \leq \Delta_k/\gamma$, where $\wh_*^{k+1}$ is the optimal solution that minimizes $\F_{k+1}(\w)$, as revealed by the following theorem.
\begin{theorem} \label{theorem:2}
Let $\wh^k_*$ and $\wh^{k+1}_*$ be the optimal solutions that minimize $\F_k(\w)$ and $\F_{k+1}(\w)$, respectively, and $\wt_{k+1}$ be the average solution obtained at the end of $k$th epoch of {MixedGrad} algorithm. Suppose $\|\wh^k_*\| \leq \Delta_k$. By setting the step size $\eta_k = 1/\left(2\beta\sqrt{3T_k}\right)$, we have, with a probability $1 - 2\delta$,
\[
    \|\wh^{k+1}_*\| \leq \frac{\Delta_k}{\gamma} \;\;\text{and}\;\; \F_{k}(\wt_{k+1}) - \min\limits_{\w} \F_{k}(\w) \leq \frac{\lambda_k\Delta_k^2}{2\gamma^4}
\]
provided that $\delta \leq e^{-9/2}$ and
\[
T_k \geq \frac{300\gamma^8\beta^2}{\lambda_k^2}\ln\frac{1}{\delta}.
\]
\end{theorem}
Taking this statement as given for the moment, we proceed with the proof of Theorem~\ref{theorem:1},
returning later to establish the claim stated in Theorem~\ref{theorem:2}.
\begin{proof}[Proof of Theorem~\ref{theorem:1}]
It is easy to check that for the first epoch, using the fact $\W \in \mathbb{B}_R$, we have
\[
\|\w^1_*\| = \|\w_*\| \leq R := \Delta_1.
\]
Let $\w_*^m$ be the optimal solution that minimizes $\F_m(\w)$ and let $\wh^{m+1}_*$ be the optimal solution obtained in the last epoch. Using Theorem~\ref{theorem:1}, with a probability $1 - 2m\delta$, we have
\[
\|\wh_*^m\| \leq \frac{\Delta_1}{\gamma^{m - 1}}, \quad \F_m(\wt_{m+1}) - \F_m(\wh_*^m) \leq \frac{\lambda_m\Delta_m^2}{2\gamma^4} = \frac{\lambda_1\Delta_1^2}{2\gamma^{3m+1}}
\]
Hence,
\begin{eqnarray*}
\frac{1}{n}\sum_{i=1}^n f_i(\wb_{m+1}) & \leq & \F_m(\wh_*^m) + \frac{\lambda_1\Delta_1^2}{2\gamma^{3m+1}} - \frac{\lambda_1}{\gamma^{m-1}}\langle \wt_{m+1}, \wb_m\rangle \\
& \leq & \F_m(\wh_*^m) + \frac{\lambda_1\Delta_1^2}{2\gamma^{3m+1}} + \frac{\lambda_1 \|\wb_m\|\Delta_1}{\gamma^{2m - 2}}
\end{eqnarray*}
where the last step uses the fact $\|\wh_*^{m+1}\| \leq \Delta_m = \Delta_1\gamma^{1-m}$. Since
\[
\|\wb_m\| \leq \sum_{i=1}^{m} |\wt_i| \leq \sum_{i=1}^{m} \Delta_i \leq \frac{\gamma\Delta_1}{\gamma - 1} \leq 2\Delta_1
\]
where in the last step holds  under the condition $\gamma \geq 2$. By combining above inequalities, we obtain
\[
\frac{1}{n}\sum_{i=1}^n f_i(\wb_{m+1}) \leq \F_m(\wh_*^m) + \frac{\lambda_1\Delta_1^2}{2\gamma^{3m+1}} + \frac{2\lambda_1 \Delta^2_1}{\gamma^{2m - 2}}.
\]
Our final goal is to relate $\F_m(\w)$ to $\min_{\w} \L(\w)$. Since $\wh_*^m$ minimizes $\F_m(\w)$, for any $\w_* \in \mathop{\arg\min} \L(\w)$, we have
\begin{eqnarray}
\F_m(\w_*^m) \leq \F_m(\w_*) = \frac{1}{n}\sum_{i=1}^n f_i(\w_*) + \frac{\lambda_1}{2\gamma^{m-1}}\left(\|\w_* - \wb_m\|^2 + 2\langle \w_* - \wb_m, \wb_m\rangle\right). \label{eqn:bound-2}
\end{eqnarray}
Thus, the key to bound $|\F(\w_*^m) - \G(\w_*)|$ is to bound $\|\w_* - \wb_m\|$. To this end, after the first $m$ epoches, we run Algorithm~\ref{alg:MixedGrad} with {\textit full gradients}. Let $\wb_{m+1}, \wb_{m+2}, \ldots$ be the sequence of solutions generated by Algorithm~\ref{alg:MixedGrad} after the first $m$ epochs. For this sequence of solutions, Theorem~\ref{theorem:2} will hold deterministically as we deploy the full gradient for updating, i.e., $\|\wt_{k}\| \leq \Delta_k$ for any $k \geq m+1$. Since we reduce $\lambda_k$ exponentially, $\lambda_k$ will approach to zero and therefore the sequence $\{\wb_{k}\}_{k=m+1}^{\infty}$ will converge to $\w_*$, one of the optimal solutions that minimize $\L(\w)$. Since $\w_*$ is the limit of sequence $\{\wb_k\}_{k=m+1}^{\infty}$ and $\|\wb_{k}\| \leq \Delta_k$ for any $k \geq m + 1$, we have
\[
\|\w_* - \wb_m\| \leq \sum_{i=m+1}^{\infty} |\wt_i| \leq \sum_{k=m+1}^{\infty}\Delta_k \leq \frac{\Delta_1}{\gamma^{m}(1 - \gamma^{-1})} \leq \frac{2\Delta_1}{\gamma^m}
\]
where the last step follows from the condition $\gamma \geq 2$. Thus,
\begin{eqnarray}
\F_m(\w_*^m) & \leq & \frac{1}{n}\sum_{i=1}^n f_i(\w_*) + \frac{\lambda_1}{2\gamma^{m-1}}\left(\frac{4\Delta_1^2}{\gamma^{2m}} + \frac{8\Delta_1^2}{\gamma^m} \right) \nonumber \\
& = & \frac{1}{n}\sum_{i=1}^n f_i(\w_*) + \frac{2\lambda_1\Delta_1^2}{\gamma^{2m-1}}\left(2+\gamma^{-m}\right) \leq \frac{1}{n}\sum_{i=1}^n f_i(\w_*) + \frac{5\lambda_1\Delta_1^2}{\gamma^{2m-1}} \label{eqn:bound-1}
\end{eqnarray}
By combining the bounds in (\ref{eqn:bound-2}) and (\ref{eqn:bound-1}), we have, with a probability $1 - 2m\delta$,
\[
\frac{1}{n}\sum_{i=1}^n f_i(\wb_{m+1}) - \frac{1}{n}\sum_{i=1}^n f_i(\w_*) \leq \frac{5\lambda_1\Delta_1^2}{\gamma^{2m - 2}} = O(1/T)
\]
where
\[
T = T_1\sum_{k=0}^{m-1}\gamma^{2k} = \frac{T_1 \left(\gamma^{2m} - 1\right)}{\gamma^2 - 1} \leq \frac{T_1}{3}\gamma^{2m}.
\]
We complete the proof by plugging in the stated values  for $\gamma$, $\lambda_1$ and $\Delta_1$.
\end{proof}

%\subsection{Proof of Theorem~\ref{theorem:2}}

We turn now to proving the Theorem~\ref{theorem:2}.  For the convenience of discussion, we drop the subscript $k$ for epoch just to simplify our notation. Let $\lambda = \lambda_k$, $T = T_k$, $\Delta = \Delta_k$, $\g = \g_k$. Let $\wb = \wb_k$ be the solution obtained before the start of the epoch $k$, and let $\wb' = \wb_{k+1}$ be the solution obtained after running through the $k$th epoch. We denote by $\F(\w)$ and $\F'(\w)$ the objective functions $\F_k(\w)$ and $\F_{k+1}(\w)$. They are given by
\begin{eqnarray}
\F(\w) & = & \frac{\lambda}{2}\|\w\|^2 + \lambda \langle \w, \wb\rangle + \frac{1}{n}\sum_{i=1}^n f_i(\w + \wb) \\
\F'(\w) & = & \frac{\lambda}{2\gamma}\|\w\|^2 + \frac{\lambda}{\gamma}\langle \w, \wb'\rangle + \frac{1}{n}\sum_{i=1}^n f_i(\w + \wb')
\end{eqnarray}
Let $\wh_* = \wh^k_*$ and $\wh'_* = \wh^{k+1}_*$ be the optimal solutions that minimize $\F(\w)$ and $\F'(\w)$ over the domain $\W_k$ and $\W_{k+1}$, respectively. Under the assumption that $\|\wh_*\| \leq \Delta$, our goal is to show
\[
\|\wh_*'\| \leq \frac{\Delta}{\gamma}, \quad \F(\wb') - \F(\wh_*) \leq \frac{\lambda \Delta^2}{2\gamma^4}
\]
%For each iteration $t$ in the $k$th epoch, from the strong convexity of $\F(\w)$ we have
%\begin{eqnarray*}
%\lefteqn{\F(\w_t) - \F(\wh_*) \leq \langle \nabla \F(\w_t), \w_t - \wh_* \rangle - \frac{\lambda}{2}\|\w_t - \wh_*\|^2} \\
%& = & \langle \g + \nabla \gh_{i_t}(\w_t) + \lambda \w_t, \w_t - \wh_* \rangle + \left\langle -\nabla \gh_{i_t}(\w_t) + \nabla \Fh(\w_t), \w_t - \wh_* \right\rangle - \frac{\lambda}{2}\|\w_t - \wh_*\|^2,
%\end{eqnarray*}
%where $\Fh(\w) = \frac{1}{n}\sum_{i=1}^n \gh_i(\w)$.
The following lemma bounds $\F(\w_t) - \F(\wh_*)$ where the proof is deferred to Section~\ref{mixed-7-omitted-proofs}.
\begin{lemma} \label{lem:1mixed}
\begin{eqnarray*}
\lefteqn{\F(\w_t) - \F(\wh_*) \leq \frac{\|\w_t - \wh_*\|^2}{2\eta} - \frac{\|\w_{t+1} - \wh_*\|^2}{2\eta} + \frac{\eta}{2}\left\|\nabla \gh_{i_t}(\w_t) + \lambda\w_t\right\|^2 + \langle \g, \w_t - \w_{t+1} \rangle } \\
&      & + \left\langle \nabla \Fh(\wh_*) - \nabla \gh_{i_t}(\wh_*), \w_t - \wh_* \right\rangle + \left\langle -\nabla \gh_{i_t}(\w_t) + \nabla \gh_{i_t}(\wh_*) - \nabla \Fh(\wh_*) + \nabla \Fh(\w_t), \w_t - \wh_* \right\rangle
\end{eqnarray*}
\end{lemma}

By adding the inequality in Lemma~\ref{lem:1mixed} over all iterations, using the fact $\wb_1 = \bz$, we have
\begin{eqnarray*}
\lefteqn{\sum_{t=1}^T \F(\w_t) - \F(\wh_*) \leq \frac{\|\wh_*\|^2}{2\eta} - \frac{\|\w_{T+1} - \wh_*\|^2}{2\eta} - \langle \g, \w_{T+1} \rangle} \\
&   & + \frac{\eta}{2}\underbrace{\sum_{t=1}^T \|\nabla \gh_{i_t}(\w_t) + \lambda\w_t\|^2}_{\triangleq A_T}  + \underbrace{\sum_{t=1}^T \langle \nabla \Fh(\wh_*) - \nabla \gh_{i_t}(\wh_*), \w_t - \wh_* \rangle}_{\triangleq B_T} \\
&   & + \underbrace{\sum_{t=1}^T \left\langle -\nabla \gh_{i_t}(\w_t) + \nabla \gh_{i_t}(\wh_*) - \nabla \Fh(\wh_*) + \nabla \Fh(\w_t), \w_t - \wh_* \right\rangle}_{\triangleq C_T}.
\end{eqnarray*}
Since $\g = \nabla \F(\bz)$ and
\begin{eqnarray*}
\F(\w_{T+1}) - \F(\bz) \leq \langle \nabla \F(\bz), \w_{T+1}\rangle + \frac{\beta}{2}\|\w_{T+1}\|^2 = \langle \g, \w_{T+1}\rangle + \frac{\beta}{2}\|\w_{T+1}\|^2
\end{eqnarray*}
using the fact $\F(\bz) \leq \F(\w_*) + \frac{\beta}{2}\|\w_*\|^2$ and $\max(\|\w_*\|, \|\w_{T+1}\|) \leq \Delta$, we have
\[
- \langle \g, \w_{T+1} \rangle \leq \F(\bz) - \F(\w_{T+1}) + \frac{\beta}{2}\Delta^2 \leq \beta\Delta^2 - (\F(\w_{T+1}) - \F(\wh_*))
\]
and therefore
\begin{eqnarray}\label{eqn:FABC}
\sum_{t=1}^{T+1} \F(\w_t) - \F(\wh_*) \leq \Delta^2\left(\frac{1}{2\eta} + \beta\right)
+ \frac{\eta}{2}A_T + B_T + C_T.
\end{eqnarray}

The following lemmas bound $A_T$, $B_T$ and $C_T$.\\
\linespread{0.8}

\begin{lemma} \label{lem:AT} For $A_T$ defined above we have  $A_T \leq 6\beta^2\Delta^2 T$.
\end{lemma}
The following lemma upper bounds $B_T$ and $C_T$. The proof is based on the Bernstein's inequality for martingales and is given later .
\begin{lemma} \label{lem:BCT}
With a probability $1 - 2\delta$, we have
\[
B_T \leq \beta\Delta^2 \left(\ln\frac{1}{\delta} + \sqrt{2T\ln\frac{1}{\delta}}\right) \;\;\text{and}\;\; C_T \leq 2\beta\Delta^2 \left(\ln\frac{1}{\delta} + \sqrt{2T\ln\frac{1}{\delta}}\right).
\]
\end{lemma}
Using Lemmas~\ref{lem:AT} and~\ref{lem:BCT}, by substituting the uppers bounds for $A_T$, $B_T$, and $C_T$ in~(\ref{eqn:FABC}), with a probability $1 - 2\delta$, we obtain
\begin{eqnarray*}
\sum_{t=1}^{T+1} \F(\w_t) - \F(\wh_*) \leq \Delta^2\left(\frac{1}{2\eta} + \beta + 6\beta^2\eta T + 3\beta\ln\frac{1}{\delta} + 3\beta\sqrt{2T\ln\frac{1}{\delta}} \right)
\end{eqnarray*}
By choosing $\eta = 1/[2\beta\sqrt{3T}]$, we have
\begin{eqnarray*}
\sum_{t=1}^{T+1} \F(\w_t) - \F(\wh_*) \leq \Delta^2\left(2\beta\sqrt{3T} + \beta + 3\beta\ln\frac{1}{\delta} + 3\beta\sqrt{2T\ln\frac{1}{\delta}} \right)
\end{eqnarray*}
and using the fact $\wt = \sum_{i=1}^{T+1} \w_t/(T+1)$, we have
\[
\F(\wt) - \F(\wh_*) \leq \Delta^2\frac{5\beta\sqrt{3\ln[1/\delta]}}{\sqrt{T+1}},\;\;\text{and}\;\;  \widehat{\Delta}^2 = \|\wt - \wh_*\|^2 \leq \Delta^2\frac{5\beta\sqrt{3\ln[1/\delta]}}{\lambda\sqrt{T+1}}.
\]
Thus, when $T \geq [300\gamma^8\beta^2 \ln\frac{1}{\delta}]/{\lambda^2}$,
we have, with a probability $1 - 2\delta$,
\begin{eqnarray}
\widehat{\Delta}^2 \leq \frac{\Delta^2}{\gamma^4}, \;\; \text{and}\;\;  {|\F(\wt) - \F(\wh_*)| \leq \frac{\lambda}{2\gamma^4}\Delta^2.} \label{eqn:bound-3}
\end{eqnarray}
The next lemma relates $\|\wh'_*\|$ to $\|\wt - \wh_*\|$.
\begin{lemma} \label{lem:4mixed} We have $\|\wh'_*\| \leq \gamma \|\wt - \wh_*\|$.
\end{lemma}
Combining the bound in (\ref{eqn:bound-3}) with Lemma~\ref{lem:4mixed}, we have $\|\wh'_*\| \leq \Delta/\gamma$.

%%%%%%%%%%%%%%%%%%%%%%%%%%%%%%%%%%%%%%%%%%%%%%%%%%%
\section{Proofs of Convergence Rate}\label{mixed-7-omitted-proofs}
%\addcontentsline{toc}{section}{\protect\numberline{}Appendix}

\subsection{Proof of Lemma~\ref{lem:1mixed}}
Before proving the lemmas we recall the  definition of $\F(\w)$, $\F'(\w)$, $\g$, and $\gh_i(\w)$ as:
\begin{eqnarray*}
\begin{aligned}
\F(\w) & =  \frac{\lambda}{2}\|\w\|^2 + \lambda \langle \w, \wb\rangle + \frac{1}{n}\sum_{i=1}^n f_i(\w + \wb), \\
\F'(\w) & =  \frac{\lambda}{2\gamma}\|\w\|^2 + \frac{\lambda}{\gamma}\langle \w, \wb'\rangle + \frac{1}{n}\sum_{i=1}^n f_i(\w + \wb'), \\
\g &= \lambda \wb  + \frac{1}{n}\sum_{i=1}^n \nabla f_i(\wb), \\
\gh_i(\w) &= f_i(\w + \wb) - \langle \w, \nabla f_i(\wb) \rangle.
\end{aligned}
\end{eqnarray*}

We also recall that  $\wh_* $ and $\wh'_* $ are the optimal solutions that minimize $\F(\w)$ and $\F'(\w)$ over the domain $\W_k$ and $\W_{k+1}$, respectively.  \\

\begin{lemma} \label{lem:1mixed}
\begin{eqnarray*}
\lefteqn{\F(\w_t) - \F(\wh_*) \leq \frac{\|\w_t - \wh_*\|^2}{2\eta} - \frac{\|\w_{t+1} - \wh_*\|^2}{2\eta} + \frac{\eta}{2}\left\|\nabla \gh_{i_t}(\w_t) + \lambda\w_t\right\|^2 + \langle \g, \w_t - \w_{t+1} \rangle } \\
&      & + \left\langle \nabla \gh_{i_t}(\wh_*) - \nabla \Fh(\wh_*), \w_t - \wh_* \right\rangle + \left\langle -\nabla \gh_{i_t}(\w_t) + \nabla \gh_{i_t}(\wh_*) - \nabla \Fh(\wh_*) + \nabla \Fh(\w_t), \w_t - \wh_* \right\rangle
\end{eqnarray*}
\end{lemma}
\begin{proof}
For each iteration $t$ in the $k$th epoch, from the strong convexity of $\F(\w)$ we have
\begin{eqnarray*}
\lefteqn{\F(\w_t) - \F(\wh_*) \leq \langle \nabla \F(\w_t), \w_t - \wh_* \rangle - \frac{\lambda}{2}\|\w_t - \wh_*\|^2} \\
& = & \langle \g + \nabla \gh_{i_t}(\w_t) + \lambda \w_t, \w_t - \wh_* \rangle + \left\langle -\nabla \gh_{i_t}(\w_t) + \nabla \Fh(\w_t), \w_t - \wh_* \right\rangle - \frac{\lambda}{2}\|\w_t - \wh_*\|^2,
\end{eqnarray*}
where $\Fh(\w) = \frac{1}{n}\sum_{i=1}^n \gh_i(\w)$. We now try to upper bound the first term in the right hand side.  Since
\begin{eqnarray*}
&   & \langle \g + \nabla \gh_{i_t}(\w_t)+ \lambda\w_t, \w_t - \wh_* \rangle \\
& = & \langle \g + \nabla \gh_{i_t}(\w_t)+\lambda\w_t, \w_t - \wh_* \rangle - \frac{\|\w_t - \wh_*\|^2}{2\eta} + \frac{\|\w_t - \wh_*\|^2}{2\eta} \\
& \leq & \langle \g + \nabla \gh_{i_t}(\w_t)+\lambda\w_t, \w_t - \w_{t+1} \rangle - \frac{\|\w_t - \w_{t+1}\|^2}{2\eta} - {\frac{\|\w_{t+1} - \wh_*\|^2}{2\eta}} + \frac{\|\w_t - \wh_*\|^2}{2\eta} \\
& \leq & \langle \g, \w_t - \w_{t+1} \rangle - \frac{\|\w_{t+1} - \wh_*\|^2}{2\eta} + \frac{\|\w_t - \wh_*\|^2}{2\eta} + \max_{\w} \left[\langle \nabla \gh_{i_t}(\w_t)+\lambda\w_t, \w_t - \w \rangle - \frac{\|\w_t - \w\|^2}{2\eta} \right]\\
& = & \langle \g, \w_t - \w_{t+1} \rangle - \frac{\|\w_{t+1} - \wh_*\|^2}{2\eta} + \frac{\|\w_t - \wh_*\|^2}{2\eta} + \frac{\eta}{2} \|\nabla \gh_{i_t}(\w_t)+\lambda\w_t\|^2
\end{eqnarray*}
where the first inequality follows from the fact  that $\w_{t+1}$ in the minimizer of the following optimization problem:
\[
    \w_{t+1} = \mathop{\arg\min}\limits_{\w \in \W \cap \|\w - \wb\| \leq \Delta} \; \langle \g + \nabla \gh_{i_t}(\w_t)+\lambda\w_t, \w - \w_t \rangle + \frac{\|\w - \w_{t}\|^2}{2\eta}.
\]
Therefore, we obtain
\begin{eqnarray*}
\lefteqn{\F(\w_t) - \F(\wh_*)} \\
& \leq & \frac{\|\w_t - \wh_*\|^2}{2\eta} - \frac{\|\w_{t+1} - \wh_*\|^2}{2\eta} - \frac{\lambda}{2}\|\w_t - \wh_*\|^2 \\
&      & + \langle \g, \w_t - \w_{t+1} \rangle + \frac{\eta}{2}\left\|\nabla \gh_{i_t}(\w_t) + \lambda\w_t\right\|^2 + \left\langle \nabla \Fh(\wh_*)-\nabla \gh_{i_t}(\wh_*), \w_t - \wh_* \right\rangle \\
&      & + \left\langle -\nabla \gh_{i_t}(\w_t) + \nabla \gh_{i_t}(\wh_*) - \nabla \Fh(\wh_*) + \nabla \Fh(\w_t), \w_t - \wh_* \right\rangle,
\end{eqnarray*}
as desired.
\end{proof}

We now turn to prove the upper bound on   $A_T$.
%%%%%%%%%%%%%%%%%%%%%%%%%%%%%%%%%%%%%%%%%%%%%%%%%%%
\subsection{Proof of Lemma~\ref{lem:AT}}
\begin{lemma} \label{lem:AT}
\[
A_T \leq 6\beta^2\Delta^2 T
\]
\end{lemma}
\begin{proof}
We bound $A_T$ as
\begin{eqnarray*}
A_T &=& \sum_{t=1}^T \|\nabla \gh_{i_t}(\w_t) + \lambda\w_t\|^2 \\
& \leq & \sum_{t=1}^T 2\|\nabla \gh_{i_t}(\w_t)\|^2 + 2\lambda^2\|\w_t\|^2 \\
& \leq & \sum_{t=1}^T 2\lambda^2\Delta^2 + 2\|\nabla \gh_{i_t}(\w_t) - \nabla \gh_{i_t}(\wh_*) + \nabla \gh_{i_t}(\wh_*)\|^2 \\
& \leq & 6\beta^2\Delta^2 T
\end{eqnarray*}
where the second inequality follows $(a+b)^2 \leq 2(a^2 + b^2)$ and the last inequality follows from the smoothness assumption.
\end{proof}

%%%%%%%%%%%%%%%%%%%%%%%%%%%%%%%%%%%%%%%%%%%%%%%%%%%%%
\begin{lemma} \label{lem:BCT}
With a probability $1 - 2\delta$, we have
\[
B_T \leq \beta\Delta^2 \left(\ln\frac{1}{\delta} + \sqrt{2T\ln\frac{1}{\delta}}\right)\;\; \text{and} \;\; C_T \leq 2\beta\Delta^2 \left(\ln\frac{1}{\delta} + \sqrt{2T\ln\frac{1}{\delta}}\right)
\]
\end{lemma}

%~\cite{boucheron2004concentration} which is restated here for completeness.
%\begin{theorem} \label{theorem:bernstein} (Bernstein's inequality for martingales). Let $X_1, \ldots , X_n$ be a bounded martingale difference sequence with respect to the filtration $\F = (\F_i)_{1\leq i\leq n}$ and with $\|X_i\| \leq K$. Let
%\[
%S_i = \sum_{j=1}^i X_j
%\]
%be the associated martingale. Denote the sum of the conditional variances by
%\[
%    \Sigma_n^2 = \sum_{t=1}^n \E\left[X_t^2|\F_{t-1}\right],
%\]
%Then for all constants $t$, $\nu > 0$,
%\[
%\Pr\left[ \max\limits_{i=1, \ldots, n} S_i > t \mbox{ and } \Sigma_n^2 \leq \nu \right] \leq \exp\left(-\frac{t^2}{2(\nu + Kt/3)} \right),
%\]
%and therefore,
%\[
%    \Pr\left[ \max\limits_{i=1,\ldots, n} S_i > \sqrt{2\nu t} + \frac{\sqrt{2}}{3}Kt \mbox{ and } \Sigma_n^2 \leq \nu \right] \leq e^{-t}.
%\]
%\end{theorem}
The proof is based on the Berstein inequality for martingales stated in Theorem~\ref{theorem:bernsteinB}. Equipped with this concentration inequality, we are now in a position to upper bound $B_T$ and $C_T$ as follows.

%\subsection{Proof of Lemma~\ref{lem:BCT} }
\begin{proof}[Proof of Lemma~\ref{lem:BCT}]  Denote $X_t =  \langle \nabla \gh_{i_t}(\wh_*) - \nabla \Fh(\wh_*), \w_t - \wh_* \rangle$.  We have that the conditional expectation of $X_t$, given randomness in previous rounds, is $\E_{t-1} [X_t] = 0$.  We now apply Theorem~\ref{theorem:bernsteinB} to the sum of martingale differences.  In particular, we have, with a probability $1 - e^{-t}$,
\[
B_T \leq \frac{\sqrt{2}}{3}Kt + \sqrt{2\Sigma t}
\]
where
\begin{eqnarray*}
K & = & \max\limits_{1\leq t \leq T} \langle \nabla \gh_{i_t}(\wh_*) - \nabla \Fh(\wh_*), \w_t - \wh_* \rangle \leq 2\beta\Delta^2 \\
\Sigma & = & \sum_{t=1}^T \E_t\left[|\langle \nabla \gh_{i_t}(\wh_*) - \nabla \Fh(\wh_*), \w_t - \wh_*\rangle|^2\right] \leq \beta^2\Delta^4T
\end{eqnarray*}
Hence, with a probability $1 - \delta$, we have
\[
B_T \leq \beta\Delta^2 \left(\ln\frac{1}{\delta} + \sqrt{2T\ln\frac{1}{\delta}}\right)
\]
Similar, for $C_T$, we have, with a probability $1 - \delta$,
\[
C_T \leq 2\beta\Delta^2 \left(\ln\frac{1}{\delta} + \sqrt{2T\ln\frac{1}{\delta}}\right)
\]

\end{proof}
%%%%%%%%%%%%%%%%%%%%%%%%%%%%%%%%%%%%%%%%%%%%%%%%%%%
\begin{lemma} \label{lem:4mixed} $\|\wh'_*\| \leq \gamma \|\wt - \wh_*\|$.

\end{lemma}
\begin{proof}
We rewrite $\F(\w)$ as
\begin{eqnarray*}
\F(\w) & = & \frac{\lambda}{2}\|\w\|^2 + \lambda\langle \w, \wb \rangle +  \frac{1}{n}\sum_{i=1}^n f_i(\w + \wb) \\
& = & \frac{\lambda}{2}\|\w-\wt + \wt\|^2 + \lambda\langle \w - \wt + \wt, \wb \rangle + \frac{1}{n}\sum_{i=1}^n f_i(\w - \wt + \wb')
\end{eqnarray*}
Define $\z = \w - \wt$. We have
\begin{eqnarray*}
\F(\w) & = & \frac{\lambda}{2}\|\z + \wt\|^2 + \lambda \langle \z, \wb\rangle + \lambda \langle\wt, \wb\rangle + \frac{1}{n}\sum_{i=1}^n f_i(\z + \wb') \\
& = & \frac{\lambda}{2}\|\z\|^2 + \lambda \langle \z, \wb' \rangle + \frac{1}{n}\sum_{i=1}^n f_i(\z + \wb') + \frac{\lambda}{2}\|\wt\|^2 + \lambda \langle \wt, \wb \rangle \\
& = & \Ft(\z) + \frac{\lambda}{2}\|\wt\|^2 + \lambda \langle \wt, \wb \rangle
\end{eqnarray*}
where
\[
\Ft(\z) = \frac{\lambda}{2}\|\z\|^2 + \lambda \langle \z, \wb' \rangle + \frac{1}{n}\sum_{i=1}^n f_i(\z + \wb')
\]
Define $\wt_* = \wh_* - \wt$. Evidently, $\wt_*$ minimizes $\Ft(\w)$. The only difference between $\Ft(\w)$ and $F'(\w)$ is that they use different modulus of strong convexity $\lambda$. Thus, following~\cite{zhang2012recovering}, we have
\[
\|\wt_* - \wh_*'\| \leq \frac{1 - \gamma^{-1}}{\gamma^{-1}}\|\wt_*\| \leq (\gamma - 1)\|\wt_*\|
\]
Hence,
\[
\|\wh_*'\| \leq \gamma \|\wt_*\| = \gamma \|\wh_* - \wt \|
\]
which completes the proofs.
\end{proof}

%%%%%%%%%%%%%%%%%%%%%%%%%%%%%%%%%%%%%%%%%%%%%%%%%%%
\section{Summary}\label{sec:7-conclusion}
In this chapter we presented a new paradigm for optimization, termed as mixed optimization, that aims to improve the convergence rate of stochastic optimization by making a small number of calls to the full gradient oracle. We proposed the {{MixedGrad}} algorithm and showed that it is able to achieve an $O(1/T)$ convergence rate by accessing stochastic and full gradient oracles for  $O(T)$ and $O(\log T)$  times, respectively.  We showed that the {{MixedGrad}} algorithm is able to exploit the {\textit smoothness} of the function, which is believed to be not very useful in stochastic optimization. The key insight behind the MixedGrad  algorithm is to use infrequent full gradients to progressively reduce the variance of stochastic gradients as the optimization proceeds. 

There are few directions that are worthy of investigation. First, it would be interesting to examine the optimality of our algorithm, namely if it is possible to achieve a better convergence rate for stochastic optimization of smooth functions using $O(\ln T)$ accesses to the full gradient oracle. Furthermore, to alleviate the computational cost caused by $O(\log T)$ accesses to the  full gradient oracle, it would be interesting to empirically evaluate the proposed algorithm in a distributed framework by distributing  the individual functions among processors to parallelize the full gradient computation at the beginning of each epoch which requires $O(\log T)$ communications between the processors in total. Lastly, it is very interesting to check whether an $O(1/T^2)$ rate could be achieved by an accelerated method in the mixed optimization scenario, and whether linear convergence rates could be achieved in the strongly-convex case.

%%%%%%%%%%%%%%%%%%%%%%%%%%%%%%%%%%%%%%%%%%%%%%%%%%%
\section{Bibliographic Notes}\label{sec:7-related}

\noindent \textbf{Deterministic Smooth Optimization.}{ The convergence rate of gradient based methods usually depends on the analytical properties of the objective function to be optimized. When the objective function is strongly convex and smooth, it is well known that  a simple GD method can achieve a linear convergence rate~\cite{boyd-convex-opt,nesterov2004introductory}. For a non-smooth Lipschitz-continuous function, the optimal rate for the first order method is only $O({1}/{\sqrt{T}})$~\cite{nesterov2004introductory}. Although $O({1}/{\sqrt{T}})$ rate is not improvable in general, several recent studies are able to improve this rate to $O(1/T)$ by exploiting the special structure of the objective function~\cite{nesterov2005smooth,nesterov2005excessive}. In the full gradient based convex optimization, smoothness is a highly desirable property. It has been shown that a simple GD achieves a convergence rate of $O(1/T)$ when the objective function is smooth, which is further can be improved to $O({1}/{T^2})$ by using the accelerated gradient methods~\cite{nesterov1983method,nesterov2005smooth,nesterov2004introductory}.}\\

\noindent \textbf{Stochastic Smooth Optimization.}{ Unlike the optimization methods based on full gradients, the smoothness assumption was not exploited by most stochastic optimization methods. In fact, it was shown in~\cite{nemircomp1983} that the $O(1/\sqrt{T})$ convergence rate for stochastic optimization cannot be improved even when the objective function is smooth. This classical result is further confirmed by the recent studies of composite bounds for the first order optimization methods~\cite{mirror-beck-2003,lin2010smoothing}. The smoothness of the objective function is exploited extensively in mini-batch stochastic optimization~\cite{mini-batch-2011,dekel2012optimal}, where the goal is not to improve the convergence rate but to reduce the variance in stochastic gradients and consequentially the number of times for updating the solutions~\cite{zhang2013logt}. We finally note that the smoothness assumption coupled with the strong convexity of function is  beneficial in stochastic setting and yields a geometric convergence in {\textit expectation} using Stochastic Average Gradient (SAG) and Stochastic Dual Coordinate Ascent (SDCA) algorithms proposed in~\cite{NIPS2012_SGM} and \cite{shalev2012stochastic}, respectively.}

 Finally, we would like to distinguish  mixed optimization from hybrid methods that use growing sample-sizes as  optimization method proceeds to gradually transform the iterates into the full gradient method~\cite{friedlander2012hybrid}, which  makes the iterations to be dependent to the sample size  $n$ as opposed to SGD. In contrast, {{MixedGrad}} is as an alternation of deterministic  and stochastic gradient steps, with different of frequencies for each type of steps. Our result for mixed optimization is useful for the scenario when the full gradient of the objective function can be computed relatively efficient although it is still significantly more expensive than computing a stochastic gradient. An example of such a scenario is distributed computing where the computation of full gradients can be speeded up by having it run in parallel on many machines with each machine containing a relatively small subset of the entire training data. Of course, the latency due to the communication between machines will result in an additional cost for computing the full gradient in a distributed fashion.

\chapter[Mixed Optimization for Smooth and Strongly Convex Losses]{Mixed Optimization for Smooth and Strongly Convex Losses
}\label{chap:mixed-strong}

\def \R {\mathbb{R}}
\def \D {\mathcal{W}}
\def \y {\mathbf{y}}
\def \M {\mathcal{M}}
\def \E {\mathbb{E}}
\def \x {\mathbf{x}}
\def \a {\mathbf{a}}
\def \L {\mathcal{L}}
\def \H {\mathcal{H}}
\def \fh {\widehat{f}}
\def \f {\mathbf{f}}
\def \Hk {\H_{\kappa}}
\def \P {\mathcal{P}}
\def \v {\mathbf{v}}
\def \c {\mathbf{c}}
\def \X {\mathcal{X}}
\def \fhv {\widehat{\f}}
\def \S {\mathcal{S}}
\def \eh {\widehat{\epsilon}}
\def \z {\mathbf{z}}
\def \t {\mathbf{t}}
\def \ph {\widehat{\phi}}
\def \lh {\widehat{\lambda}}
\def \gt {\widetilde{g}}
\def \et {\widetilde{\varepsilon}}
\def \zt {\widetilde{\z}}
\def \Zt {\widetilde{Z}}
\def \gammat {\widetilde{\gamma}}
\def \etat {\widehat{\eta}}
\def \ab {\overline{\alpha}}
\def \Mh {\widehat{M}}
\def \veh {\widehat{\varepsilon}}
\def \gammah {\widehat{\gamma}}
\def \etah {\widehat{\eta}}
\def \xh {\widehat{\x}}
\def \Kh {\widehat{K}}
\def \rh {\widehat{r}}
\def \vet {\widetilde{\varepsilon}}
\def \hh {\widehat{h}}
\def \htt {\widetilde{h}}
\def \Er {\mathcal{E}}
\def \u {\mathbf{u}}
\def \Eh {\widehat{\Er}}
\def \v {\mathbf{v}}
\def \w {\mathbf{w}}
\def \Hb {\overline{\H}}
\def \DDh {\widehat{\D}}
\def \B {\mathcal{B}}
\def \Phl {\overline{\Phi}}
\def \R {\mathbb{R}}
\def \Kt {\widetilde{K}}
\def \G {\mathcal{G}}
\def \kt {\widetilde{\kappa}}
\def \b {\mathbf{b}}
\def \zt {\widetilde{\z}}
\def \h {\mathbf{h}}
\def \vt {\widetilde{\v}}
\def \Vt {\widetilde{V}}
\def \A {\mathcal{A}}
\def \B {\mathcal{B}}
\def \Hh {\widehat{H}}
\def \Ht {\widetilde{H}}
\def \alt {\widetilde{\alpha}}
\def \at {\widetilde{\a}}
\def \bt {\widetilde{\b}}
\def \hht {\widetilde{\h}}
\def \Dt {\widetilde{D}}
\def \Zt {\widetilde{Z}}
\def \lt {\widetilde{\lambda}}
\def \rb {\bar{r}}
\def \Dh {\widehat{D}}
\def \tr {\mbox{tr}}
\def \Ah {\widehat{A}}
\def \r {\mathcal{R}}
\def \Zh {\widehat{Z}}
\def \m {\mathbf{m}}
\def \Ub {\bar{U}}
\def \Vb {\bar{V}}
\def \Pb {\bar{\P}}
\def \e {\mathbf{e}}
\def \rank {\mbox{rank}}
\def \F {\mathcal{F}}
\def \Uh {\widehat{U}}
\def \uh {\widehat{\u}}
\def \Vh {\widehat{V}}
\def \vh {\widehat{\v}}
\def \sigmah {\widehat{\sigma}}
\def \mh {\widehat{\m}}
\def \m {\mathbf{m}}
\def \xt {\widetilde{\x}}
\def \Bh {\widehat{B}}
\def \Mt {\widetilde{M}}
\def \d {\mathbf{d}}
\def \vec {\mbox{vec}}
\def \k {\mathbf{k}}
\def \At {\widetilde{A}}
\def \Bt {\widetilde{B}}
\def \N {\mathcal{N}}
\def \wh {\widehat{\w}}
\def \ah {\widehat{\alpha}}
\def \wt {\widetilde{\w}}
\def \Gh {\widehat{G}}
\def \rb {\bar{r}}
\def \etah {\widehat{\eta}}
\def \s {\mathbf{s}}
\def \Xh {\widehat{X}}
\def \Ph {\widehat{P}}
\def \lb {\L}
\def \Lh {\widehat{\L}}

\def \K {\mathcal{K}}
\def \clip {\mbox{clip}}
\def \nt {\widetilde{\nabla}}
\def \Xt {\widetilde{X}}
\def \lh {\widehat{\ell}}
\def \eh {\widehat{\ell}}

\def \F {\mathcal{F}}

\def \wb {\bar{\w}}
\def \g {\mathbf{g}}
\def \gt {\tilde{\g}}
\def \gh {\widehat{g}}
\def \Fh {\widehat{\F}}
\def \O {\mathcal{O}}

In the preceding chapter,  we presented a new paradigm for stochastic optimization that allowed us to leverage the smoothness of objective function to devise faster algorithms.  In this chapter of thesis, we continue our study of efficient  optimization algorithms in mixed optimization regime and show that we may leverage the smoothness assumption of loss functions to devise algorithms with iteration complexities  that are independent of condition number in accessing the  full gradient oracle.  To motivate the setting considered in this chapter,  consider the optimization of  smooth and strongly convex functions where the optimal iteration complexity of the gradient-based algorithm is $O(\sqrt{\kappa}\log 1/\epsilon)$, where $\kappa$ is the condition number of the objective function to be optimized. In the case that the optimization problem is ill-conditioned, we need to evaluate a larger number of full gradients, which could be computationally expensive despite the linear convergence rate of the algorithm in terms of the target accuracy $\epsilon$. 

In this chapter, we propose to reduce the number of full gradients required by allowing the algorithm to access the stochastic gradients of the objective function. To this end, we present an  algorithm named Epoch Mixed Gradient Descent (EMGD) that is able to utilize two kinds of gradients similar to Chapter~\ref{chap:mixed}.  Similar to the  MixedGrad algorithm  a distinctive step in EMGD is the mixed gradient descent, where we use a combination of the full gradient and the stochastic gradient to update the intermediate solutions. By performing a fixed number of mixed gradient descents, we are able to improve the sub-optimality of the solution by a constant factor, and thus achieve a linear convergence rate. Theoretical analysis shows that EMGD is able to find an $\epsilon$-optimal solution by computing $O(\log 1/\epsilon)$ full gradients and $O(\kappa^2\log 1/\epsilon)$ stochastic gradients. We also provide experimental evidence complementing our theoretical results for classification problem on few medium-sized data sets.

\section{Introduction}
The optimal iteration complexities for some popular optimization methods considering different combinations of  characteristics of the objective function are shown in Table~\ref{sample-table}. We observe that when the objective function is smooth (and strongly convex), the convergence rate for full gradient descent is much faster than that for stochastic gradient descent. On the other hand, the evaluation of a stochastic gradient is usually significantly more efficient than that for a full gradient. Thus, replacing full gradients with stochastic gradients essentially trades the number of iterations with a low computational cost per iteration.

In this chapter, we consider the case when the objective function is both smooth and strongly convex, where the \textit{optimal} iteration complexity is $O(\sqrt{\kappa}\log\frac{1}{\epsilon})$ if the optimization method is first order and has access to the full gradients. For the optimization problems that are ill-conditioned, the condition number $\kappa$ can be very large, leading to many evaluations of full gradients, an operation that is computationally expensive for large data sets. To reduce the computational cost, we are interested in the possibility of making the number of full gradients required independent from $\kappa$. Although the $O(\sqrt{\kappa}\log\frac{1}{\epsilon})$ rate is in general not improvable for any first order method, we bypass this difficulty by allowing the algorithm to have access to both full and stochastic gradients. Our objective is to reduce the iteration complexity from $O(\sqrt{\kappa}\log\frac{1}{\epsilon})$ to $O(\log\frac{1}{\epsilon})$ by replacing most of the evaluations of full gradients with the evaluations of stochastic gradients. Under the assumption that stochastic gradients can be computed efficiently, this tradeoff could lead to a significant improvement in computational efficiency.

We propose an efficient algorithm, dubbed Epoch Mixed Gradient Descent (EMGD), which fits the mixed optimization regime introduced in Chapter~\ref{chap:mixed}. The proposed EMGD algorithm  divides the optimization process into a sequence of different epochs, an idea that is borrowed from the epoch gradient descent~\cite{hazan-2011-beyond}. In each epoch, the proposed algorithm performs \emph{mixed gradient descent} by evaluating \textit{one} full gradient and $O(\kappa^2)$ stochastic gradients. It achieves a constant reduction in the optimization error for every epoch, leading to a linear convergence rate. Our analysis shows that EMGD is able to find an $\epsilon$-optimal solution by computing $O(\log \frac{1}{\epsilon})$ full gradients and $O(\kappa^2\log \frac{1}{\epsilon})$ stochastic gradients. In other words, with the help of stochastic gradients, the number of full gradients required is reduced from $O(\sqrt{\kappa} \log \frac{1}{\epsilon})$ to $O(\log\frac{1}{\epsilon})$, independent from the condition number.

\begin{table}[t]
\begin{center}
\begin{tabular}{lccc}
 & Lipschitz continuous &  Smooth & Smooth \& Strongly Convex
\\ \hline \\ Full Gradient         & $O\left(\frac{1}{\epsilon^2}\right)$  &  $O\left(\frac{L}{\sqrt{\epsilon}}\right)$  & $O\left(\sqrt{\kappa} \log \frac{1}{\epsilon}\right) $ \\
Stochastic Gradient    & $O\left(\frac{1}{\epsilon^2}\right)$  & $O\left(\frac{1}{\epsilon^2}\right)$ &  $O\left(\frac{1}{\lambda \epsilon}\right)$ \\ \\\hline
\end{tabular}
\caption[The optimal iteration complexity of convex optimization]{The optimal iteration complexity of convex optimization. $L$ and $\lambda$ are the moduli of smoothness and strongly convexity, respectively. $\kappa=L/\lambda$ is the conditional number.}
\label{sample-table}

\end{center}
\end{table}

%%%%%%%%%%%%%%%%%%%%%%%%%%%%%%%%%%%%%%%%%%%%%%%%%%%
\section{The Epoch Mixed Gradient Descent Algorithm}
First, we recall from Chapter~\ref{chap-background} that we wish to solve the following optimization problem
\begin{eqnarray}
\begin{aligned}\label{eqn:8:sum}
\min_{\w\in \W} \mathcal{F}(\w) \quad \text{for} \quad \F(\w) = \E[{f}(\w, \xi)] = \int_{\Xi}^{}{{f}(\w, \xi)dP(\xi)},
\end{aligned}
\end{eqnarray}
where $\D$ is a convex domain, and $f(\w,\xi)$ is a convex function with respect to the first argument. An special setting which is more appropriate for machine learning tasks is the case  when the objective function can be written as a sum of finite number of convex functions, i.e.,
\begin{eqnarray}\label{eqn:sum-empirical}
\mathcal{F}(\w) = \frac{1}{n}\sum_{i=1}^{n}{f(\w, \xi_i)}.
\end{eqnarray}

For learning problems such as  classification and regression, each individual function in the summand can be considered as the  prediction loss on the $i$th training example $\xi_i = (\x_i, y_i)$ for a fixed loss function, i.e., $f(\w, \xi_i) = \ell(\w;(\x_i,y_i))$. For simplicity of exposition, we absorb the randomness in the individual functions and use $f_i(\w)$ to denote the loss on random sample $\xi_i$ instead of $f(\w, \xi_i)$.  We note that although the formulation in~(\ref{eqn:sum-empirical}) seems attractive from a practical point of view, but the proposed algorithm is general enough to solve any stochastic optimization problem  formulated in~(\ref{eqn:8:sum}). As a result, in the  remainder of this chapter we base the randomness on sampling the individual  functions according to the unknown distribution defined over functions.

Similar to the setting introduced in Chapter~\ref{chap:mixed},  we assume there exist two oracles.
 \begin{enumerate}
  \item The first one is a full gradient oracle $\O_f$, which for a given input point $\w$ returns the gradient $\nabla \F(\w)$, that is,
  \[
  \O_f(\w)=\nabla \F(\w).
  \]
  \item The second one is a function oracle or stochastic oracle $\O_s$, each call of which returns a random function $f(\w)$, such that
    \[
  \F(\w) = \mathbb{E}_f [f(\w)], \forall \w \in \W,
  \]
  and $f(\w)$ has Lipschitz continuous gradients with constant $L$, that is,
 \begin{equation} \label{eqn:f:smooth}
  \| \nabla f(\w) - \nabla f(\w') \| \leq L \|\w - \w'\|, \ \forall \w, \w' \in \D.
\end{equation}
 \end{enumerate}
Although we do not define a stochastic gradient oracle directly, the function oracle $\O_s$ allows us to evaluate the stochastic gradient of $\F(\w)$ at any point $\w \in \D$.

Notice that the assumption about the function oracle $\O_s$ implies that the objective function $\F(\cdot)$ is also $L$-smooth. To see this, since $\nabla \F(\w) = \mathbb{E}_f[f(\w)]$, by Jensen's inequality we have:
\begin{equation} \label{eqn:F:smooth}
  \| \nabla \F(\w) - \nabla \F(\w') \| \leq \mathbb{E}_{f} \| \nabla f(\w)- \nabla f(\w') \| \leq L \|\w - \w'\|, \ \forall \w, \w' \in \D.
\end{equation}
Besides, we further assume $\F(\cdot)$ is $\lambda$-strongly convex, that is,
\begin{equation} \label{eqn:F:convex}
\| \nabla \F(\w) - \nabla \F(\w') \| \geq \lambda \|\w - \w'\|, \ \forall \w, \w' \in \D.
\end{equation}

From~(\ref{eqn:F:smooth}) and~(\ref{eqn:F:convex}) it is straightforward to see that $L \geq \lambda$. The condition number $\kappa$ is defined as the ratio between these two parameters, i.e., $\kappa = L/\lambda \geq 1$.

%\subsection{The Algorithm}
The detailed steps of the proposed Epoch Mixed Gradient Descent (EMGD) are shown in Algorithm~\ref{alg:emgd}, where we use the superscript for the index of epochs, and the subscript for the index of iterations in each epoch. Similar to the MixedGrad algorithm, we divided the optimization  process into a sequence of epochs (step~\ref{step:epoch:start} to step~\ref{step:epoch:end}). While in the MixedGrad algorithm the size of epochs increases exponentially  and the full gradient oracle is called at the beginning of each epoch, the size of epochs  and the number of access to the two types of  oracles in EMGD is fixed. 

At the beginning of each epoch, we initialize the solution $\w^k_1$ to be the average solution $\wb^k$ obtained from the last epoch, and then call the gradient oracle $\O_f$ to obtain $\nabla \F(\wb^k)$. At each iteration $t$ of epoch $k$, we call the function oracle $\O_s$ to obtain a random function $f^k_t(\w)$ and define the \emph{mixed gradient} at the current solution $\w^k_t$ as
\[
  \gt^k_t = \nabla \F(\wb^k) + \nabla f^k_t(\w^k_t) - \nabla f^k_t(\wb^k),
\]
which involves both the full gradient and the stochastic gradient. The mixed gradient can be divided into two parts: the deterministic part $\nabla \F(\wb^k)$ and the stochastic part $\nabla f^k_t(\w^k_t) - \nabla f^k_t(\wb^k)$. Due to the smoothness property of $f^k_t(\cdot)$, the norm of the stochastic part is well bounded, which facilitates the convergence analysis.

Based on the mixed gradient, we update $\w^k_t$ by a gradient mapping over a shrinking domain (i.e., $\D \cap \|\w - \wb^k\| \leq \Delta^k$) in step~\ref{step:epoch:update}. Since the updating is similar to the standard gradient descent except for the domain constraint, we refer to it as mixed gradient descent for short. At the end of the iterations for epoch $k$, we compute the average value of $T+1$ solutions, instead of $T$ solutions, and update the domain size by reducing a factor of $\sqrt{2}$.

\begin{algorithm}[t]
\caption{Epoch Mixed Gradient Descent (\texttt{EMGD}) Algorithm}
\begin{algorithmic}[1]
\STATE {\textbf{Input}}: 
\begin{mylist}
\item step size $\eta$
\item the initial domain size $\Delta^1$
\item  the number of iterations $T$ per epoch
\item the number of epochs $m$
\end{mylist}
\STATE \textbf{Initialize}: $\wb^1 = \mathbf{0}$
\FOR{$k = 1, \ldots, m$}
    \STATE Set $\w^k_1 = \wb^k$\label{step:epoch:start}
    \STATE Call the full gradient oracle $\O_f$ to obtain $\nabla \F(\wb^k)$  \label{step:epoch:grad}
    \FOR{$t = 1, \ldots, T$}
        \STATE Call the stochastic function oracle $\O_s$ to obtain a random function $f^k_t(\cdot)$
        \STATE Compute the mixed gradient as
        \[
            \gt^k_t = \nabla \F(\wb^k) + \nabla f^k_t(\w^k_t) - \nabla f^k_t(\wb^k)
        \]
        \STATE Update the solution by
        \[
            \w^k_{t+1} = \arg \min\limits_{\w \in \D \cap \|\w - \wb^k\| \leq \Delta^k} \eta\langle \w - \w_t^k, \gt^k_t \rangle + \frac{1}{2}\|\w - \w^k_t\|^2
        \]  \label{step:epoch:update}
    \ENDFOR
    \STATE Set $\wb^{k+1} = \frac{1}{T+1} \sum_{t=1}^{T+1} \w^k_t$ and $\Delta^{k+1} = \Delta^k/ \sqrt{2}$ \label{step:epoch:end}
\ENDFOR
\end{algorithmic} \label{alg:emgd}
{\bf Return} $\wb^{m+1}$
\end{algorithm}

The following theorem shows the convergence rate of the proposed algorithm.
\begin{theorem} \label{thm:1emgd}
Assume
\begin{equation} \label{eqn:thm:ass}
\delta \leq e^{-1/2}, \ T \geq \frac{1152 L^2}{\lambda^2}\ln\frac{1}{\delta}, \textrm{ and } \Delta^1 \geq \max \left(\sqrt{\frac{2}{\lambda} (\F(\mathbf{0})-\F(\w_*))},   \|\w_* \|\right).
\end{equation}
Set $\eta = 1/[L \sqrt{T}]$. Let $\wb^{m+1}$ be the solution returned by Algorithm~\ref{alg:emgd} after $m$ epoches that has $m$ access to oracle $\O_f$ and $mT$ access to oracle $\O_s$. Then, with a probability at least $1 - m \delta$, we have
\[
\F(\wb^{m+1})-\F(\w_*) \leq  \frac{\lambda [\Delta^1]^2}{2^{m+1}}   , \textrm{ and } \|\wb^{m+1} - \w_*\|^2 \leq   \frac{[\Delta^1]^2}{2^m}.
\]
\end{theorem}
Theorem~\ref{thm:1emgd} immediately implies that EMGD is able to achieve an $\epsilon$ optimization error by computing $O(\log\frac{1}{\epsilon})$ full gradients and $O(\kappa^2\log\frac{1}{\epsilon})$ stochastic gradients.

\section{Analysis of Convergence Rate}
The proof of Theorem~\ref{thm:1emgd} is based on induction. From the assumption about $\Delta^1$ in (\ref{eqn:thm:ass}), we have
\[
\F(\wb^{1})-\F(\w_*) \leq  \frac{\lambda [\Delta^1]^2}{2}   , \textrm{ and } \|\wb^{1} - \w_*\|^2 \leq [\Delta^1]^2,
\]
which means Theorem~\ref{thm:1emgd} is true for $m=0$. Suppose Theorem~\ref{thm:1emgd} is true for $m=k$. That is, with a probability at least $1 - k \delta$, we have
\[
\F(\wb^{k+1})-\F(\w_*) \leq  \frac{\lambda [\Delta^1]^2}{2^{k+1}}   , \textrm{ and } \|\wb^{k+1} - \w_*\|^2 \leq   \frac{[\Delta^1]^2}{2^k}.
\]
Our goal is to show that after running the $k{+}1$-th epoch, with a probability at least $1 - (k+1)\delta$, we have
\[
\F(\wb^{k+2})-\F(\w_*) \leq  \frac{\lambda [\Delta^1]^2}{2^{k+2}}   , \textrm{ and } \|\wb^{k+2} - \w_*\|^2 \leq   \frac{[\Delta^1]^2}{2^{k+1}}.
\]
For the simplicity of presentation, we drop the index $k$ for epoch. Let $\wb$ be the solution obtained from the epoch $k$. Given the condition \begin{equation}\label{eqn:input}
\F(\wb) -\F(\w_*) \leq \frac{\lambda}{2} \Delta^2, \textrm{ and }  \| \wb - \w_*\|^2 \leq \Delta^2,
\end{equation}
we will show that that after running the $T$ iterations in one epoch, the new solution, denoted by $\wh$, satisfies
\begin{equation}\label{eqn:output}
\F(\wh) -\F(\w_*) \leq \frac{\lambda}{4} \Delta^2, \textrm{ and }  \|\wh - \wb\|^2 \leq \frac{1}{2}\Delta^2,
\end{equation}
with a probability at least $1 - \delta$.

In the proof, we frequently use the following property of strongly convex function (see Appendix~\ref{chap:appendix-convex}).
\begin{lemma} Let $\F(\w)$ be a $\lambda$-strongly convex function over the domain $\W$, and $\w_*= \arg \min_{\w \in \W} \F(\w)$. Then, for any $\w \in \W$, we have
\begin{equation} \label{eqn:lem:1}
\F(\w)-\F(\w_*) \geq \frac{\lambda}{2} \|\w-\w_*\|^2.
\end{equation}
\end{lemma}

Define
\begin{equation}\label{eqn:proof:definition}
\g = \nabla \F(\wb), \   \Fh(\w) =\F(\w)- \langle \w, \g \rangle, \textrm{ and } g_t(\w)=f_t(\w)-\langle \w, \nabla f_t(\wb) \rangle.
\end{equation}
The objective function can be rewritten as
\begin{equation} \label{eqn:newF}
\F(\w) =  \langle \w, \g \rangle + \Fh(\w).
\end{equation}
And the mixed gradient can be rewritten as
\[
 \gt^k= \g + \nabla g_t (\w_t).
\]
Then, the updating rule given in Algorithm~\ref{alg:emgd} becomes
\begin{equation} \label{eqn:update}
\w_{t+1} = \arg\min \limits_{\w \in \D \cap \|\w - \wb\| \leq \Delta} \eta \langle \w - \w_{t}, \g + \nabla g_t (\w_t) \rangle + \frac{1}{2}\|\w - \w_t\|^2.
\end{equation}

For each iteration $t$ in the current epoch, we have
\begin{equation} \label{eqn:proof:1}
\begin{aligned}
& \F(\w_t) - \F(\w_*) \\
\overset{\text{(\ref{eqn:F:convex})}}{\leq} &  \langle \nabla \F(\w_t), \w_t - \w_* \rangle - \frac{\lambda}{2}\|\w_t - \w_*\|^2\\
\overset{\text{(\ref{eqn:newF})}}{=} & \langle \g + \nabla g_t (\w_t), \w_t - \w_* \rangle + \left\langle \nabla \Fh(\w_t) -\nabla g_t(\w_t) , \w_t - \w_* \right\rangle - \frac{\lambda}{2}\|\w_t - \w_*\|^2,
\end{aligned}
\end{equation}
and
\begin{equation}\label{eqn:proof:2}
\begin{aligned}
  & \langle \g + \nabla g_t (\w_t), \w_t - \w_* \rangle \\
= & \langle \g + \nabla g_t (\w_t), \w_t - \w_* \rangle - \frac{\|\w_t - \w_*\|^2}{2\eta} + \frac{\|\w_t - \w_*\|^2}{2\eta}  \\
 \overset{\text{(\ref{eqn:update}), (\ref{eqn:proof:1})}}{\leq} & \langle \g + \nabla g_t (\w_t), \w_t - \w_{t+1} \rangle - \frac{\|\w_t-\w_{t+1}\|^2}{2\eta}- \frac{\|\w_{t+1}-\w_*\|^2}{2\eta}+ \frac{\|\w_t - \w_*\|^2}{2\eta} \\
 \leq & \langle \g, \w_t - \w_{t+1} \rangle  + \frac{\|\w_t - \w_*\|^2}{2\eta} -\frac{\|\w_{t+1}-\w_*\|^2}{2\eta}  \\
& \hspace{0.5cm}+  \max_{\w} \left(\langle \nabla g_t (\w_t), \w_t - \w \rangle - \frac{\|\w_t - \w\|^2}{2\eta}\right) \\
 = & \langle \g, \w_t - \w_{t+1} \rangle + \frac{\|\w_t - \w_*\|^2}{2\eta} -\frac{\|\w_{t+1}-\w_*\|^2}{2\eta}+ \frac{\eta}{2} \|\nabla g_t (\w_t)\|^2.
\end{aligned}
\end{equation}
Combining (\ref{eqn:proof:1}) and (\ref{eqn:proof:2}), we have
\[
\begin{aligned}
& \F(\w_t) - \F(\w_*) \\
\leq & \frac{\|\w_t - \w_*\|^2}{2\eta} - \frac{\|\w_{t+1} - \w_*\|^2}{2\eta} - \frac{\lambda}{2}\|\w_t - \w_*\|^2 \\
      & + \langle \g, \w_t - \w_{t+1} \rangle + \frac{\eta}{2}\|\nabla  g_t (\w_t)\|^2 + \left\langle \nabla \Fh(\w_t) -\nabla g_t(\w_t), \w_t - \w_* \right\rangle.
\end{aligned}
\]
By adding the inequalities of all iterations, we have
\begin{equation} \label{eqn:proof:3}
\begin{aligned}
& \sum_{t=1}^T \F(\w_t) - \F(\w_*) \\
\leq &  \frac{\|\wb - \w_*\|^2}{2\eta} - \frac{\|\w_{T+1} - \w_*\|^2}{2\eta} - \frac{\lambda}{2}\sum_{t=1}^T \|\w_t - \w_*\|^2 + \langle \g, \wb - \w_{T+1} \rangle \\
& + \frac{\eta}{2}\underbrace{\sum_{t=1}^T \|\nabla g_t(\w_t)\|^2}_{\triangleq A_T}  + \underbrace{\sum_{t=1}^T \langle \nabla \Fh(\w_t) -\nabla g_t(\w_t), \w_t - \w_* \rangle}_{\triangleq B_T}.
\end{aligned}
\end{equation}

Since $\F(\cdot)$ is $L$-smooth, we have
\[
\F(\w_{T+1}) - \F(\wb) \leq \langle \nabla \F(\wb), \w_{T+1} - \wb \rangle + \frac{L}{2}\|\wb - \w_{T+1}\|^2,
\]
which implies
\begin{equation} \label{eqn:proof:4}
\begin{aligned}
& \langle \g, \wb - \w_{T+1} \rangle \\
\leq  & \F(\wb) - \F(\w_{T+1}) + \frac{L}{2}\Delta^2 \\
\overset{\text{(\ref{eqn:input})}}{\leq} &  \F(\w_*) - \F(\w_{T+1}) + \frac{\lambda}{2} \Delta^2 + \frac{L}{2} \Delta^2 \\
\leq &  \F(\w_*) - \F(\w_{T+1}) + L \Delta^2.
\end{aligned}
\end{equation}
From (\ref{eqn:proof:3}) and (\ref{eqn:proof:4}), we have
\begin{equation} \label{eqn:proof:5}
\sum_{t=1}^{T+1} \F(\w_t) - \F(\w_*) \leq \Delta^2\left(\frac{1}{2\eta} + L\right)
 + \frac{\eta}{2} A_T + B_T.
\end{equation}

Next, we consider how to bound $A_T$ and $B_T$. The upper bound of $A_T$ is given by
\begin{equation} \label{eqn:proof:6}
 A_T =  \sum_{t=1}^T \|\nabla g_t(\w_t)\|^2 =  \sum_{t=1}^T \|\nabla f_t(\w_t) - \nabla f_t(\wb) \|^2
\overset{\text{(\ref{eqn:f:smooth})}}{\leq}   L^2 \sum_{t=1}^T \|\w_t -  \wb \|^2  \leq T L^2\Delta^2.
\end{equation}

To bound $B_T$, we need the Hoeffding-Azuma inequality which is stated in Theorem~\ref{thm:hoeffding-azumeB} for completeness. Define
\[
V_t = \langle \nabla \Fh(\w_t) -\nabla g_t(\w_t), \w_t - \w_* \rangle, \ t=1,\ldots,T.
\]
Recall the definition of $\Fh(\w)$ and $g_t(\w)$ in (\ref{eqn:proof:definition}). Based on our assumption about the function oracle $\O_s$, it is straightforward to check that $V_1, \ldots$ is a martingale difference with respect to $g_1,\ldots$. The value of $V_t$ can be bounded by

\begin{eqnarray*}
 |V_t| & \leq &  \left\| \nabla \Fh(\w_t) -\nabla g_t(\w_t) \right\| \left \|\w_t - \w_* \right\| \\
&\leq & 2 \Delta \left( \left\| \nabla \F(\w_t)- \nabla \F(\wb)\right\| + \left\| \nabla f_t(\w_t)- \nabla f_t(\wb) \right \| \right) \\
&\overset{\text{(\ref{eqn:f:smooth}), (\ref{eqn:F:smooth})}}{\leq} & 4L \Delta \|\w_t-\wb\|  \leq 4 L \Delta^2.
\end{eqnarray*}

Following Theorem~\ref{thm:hoeffding-azumeB}, with a probability at least $1-\delta$, we have
\begin{equation} \label{eqn:proof:7}
B_T \leq 4 L \Delta^2 \sqrt{2 T \ln \frac{1}{\delta}}.
\end{equation}

By adding the inequalities in (\ref{eqn:proof:5}), (\ref{eqn:proof:6}) and (\ref{eqn:proof:7}) together, with a probability at least $1 - \delta$,  we have
\begin{eqnarray*}
\sum_{t=1}^{T+1} \F(\w_t) - \F(\w_*) \leq \Delta^2\left(\frac{1}{2\eta} + L + \frac{\eta T L^2}{2} +4L \sqrt{2 T \ln \frac{1}{\delta}} \right).
\end{eqnarray*}
By choosing $\eta = 1/[L\sqrt{T}]$, we have
\begin{eqnarray*}
\sum_{t=1}^{T+1} \F(\w_t) - \F(\w_*) \leq L \Delta^2\left(\sqrt{T}+1 +  4 \sqrt{2 T \log \frac{1}{\delta}}   \right) \leq 6 L \Delta^2 \sqrt{2 T \ln \frac{1}{\delta}}.
\end{eqnarray*}
and therefore
\[
\F(\wh) - \F(\w_*) \leq \Delta^2 \frac{6 L\sqrt{2\ln 1/\delta }}{\sqrt{T+1}}, \textrm{ and } \|\wh - \w_*\|^2 \overset{\text{(\ref{eqn:lem:1})}}{\leq} \Delta^2\frac{12 L\sqrt{2\ln 1/\delta}}{\lambda\sqrt{T+1}}.
\]
Thus, when
\[
T \geq \frac{1152 L^2}{\lambda^2}\ln\frac{1}{\delta},
\]
with a probability at least $1 - \delta$, we have
\[
\F(\wh) - \F(\w_*) \leq  \frac{\lambda}{4} \Delta^2, \textrm{ and } \|\wh - \w_*\|^2  \leq \frac{1}{2}\Delta^2.
\]

%%%%%%%%%%%%%%%%%%%%%%%%%%%%%%%%%%%%%%%%%%%%%%%%%%%
\section{Experiments}
In this section, we provide  experimental evidence complementing  our theoretical results. In particular  we consider solving the regularized logistic regression problem formulated as:
\[
\min_{\w} \frac{1}{n} \sum_{i=1}^n  f_i(\w) + \frac{\lambda}{2} \|\w\|^2
\]
where
\[
f_i(\w)=\log \left(1+\exp(-y_i \dd{\w}{\x_i}) \right), \textrm{ and } \nabla f_i(\w)=\frac{-y_i}{1+\exp(y_i \dd{\w}{\x_i})} \x_i.
\]
We compare the EMGD algorithm to the SGD method. SGD start with solution $\w_1 = \mathbf{0}$ and at each iteration samples a random function  indexed by $i_k$ uniformly at random over all $n$ available functions and updates the solution by:
$$        \w_{t+1}= \w_t - \frac{1}{\lambda t} \left( \nabla  f_{i_k}(\w_t)+  \lambda \w_t \right) = \left (1-\frac{1}{t} \right) \w_t -\frac{1}{\lambda t}  \nabla  f_{i_k}(\w_t).$$

The variance of SGD method  at each iteration  can be computed as:
\begin{equation*}
\begin{aligned}
& \left \| \nabla  f_{i_k}(\w_t)  + \lambda \w_t - \left(\frac{1}{n} \sum_{i=1}^n  \nabla f_i(\w_t)   + \lambda \w_t \right) \right \|^2 = \E_{i_k} \left \| \nabla  f_{i_k}(\w_t) - \frac{1}{n} \sum_{i=1}^n  \nabla f_i(\w_t) \right \|^2 \\
=&  \E_{i_k}  \left \|\nabla  f_{i_k}(\w_t)\right \|^2  -\left \|\frac{1}{n} \sum_{i=1}^n  \nabla f_i(\w_t) \right \|^2 = \frac{1}{n}  \left \|\nabla  f_{i}(\w_t)\right \|^2-\left \|\frac{1}{n} \sum_{i=1}^n  \nabla f_i(\w_t) \right \|^2 \\
\end{aligned}
\end{equation*}
The variance of the mixed gradient in EMGD is given by
%\[
%\begin{split}
\begin{equation*}
\begin{aligned}
& \E_{i_k} \left \| \left(\nabla  f_{i_k}(\w_t) + \lambda \w_t \right)  -  \left(\nabla f_{i_k}(\wb) +\lambda \wb \right) + \left( \frac{1}{n} \sum_{i=1}^n  \nabla f_i(\wb)  - \lambda \wb \right)- \left(\frac{1}{n} \sum_{i=1}^n  \nabla f_i(\w_t)   + \lambda \w_t \right)  \right \|^2 \\
= & \E_{i_k} \left \| \nabla  f_{i_k}(\w_t)  - \nabla f_{i_k}(\wb)  -\left(\frac{1}{n} \sum_{i=1}^n  \nabla f_i(\wb) - \frac{1}{n} \sum_{i=1}^n  \nabla f_i(\w_t)  \right)  \right \|^2 \\
= & \frac{1}{n}  \left \|\nabla  f_{i}(\w_t) - \nabla  f_{i}(\wb) \right \|^2-\left \|\frac{1}{n} \sum_{i=1}^n  \nabla f_i(\w_t) -\frac{1}{n} \sum_{i=1}^n  \nabla f_i(\wb) \right \|^2.\\
\end{aligned}
\end{equation*}

%\end{split}
%\]
We run both algorithms on two well-known Adult and RCV1 data sets.  The data sets used  in	this	  experiment	  were	  sourced	  from	  the	  University	  of	  California	  Irvine	 (UCI) Machine Learning Repository~\cite{Bache+Lichman:2013}	  and	 are referred to as the Adult and RCV1 data sets. The Adult data set	  contain	  information on individuals such as age, level of education and current employment type. The Adult data set contains over 30,000 records of census information taken in 1994 from many diverse demographics. The RCV1 is an archive of over 800,000 manually categorized newswire stories recently made available by Reuters, Ltd. for research purposes. For EMGD, we set the number of stochastic gradient in each epoch to $n$, which is the number of training data.

\begin{figure*}
  \centering
  \subfloat[The training error]{
    \label{fig:subfig1:a} %% label for first subfigure
\includegraphics[width=.48\textwidth]{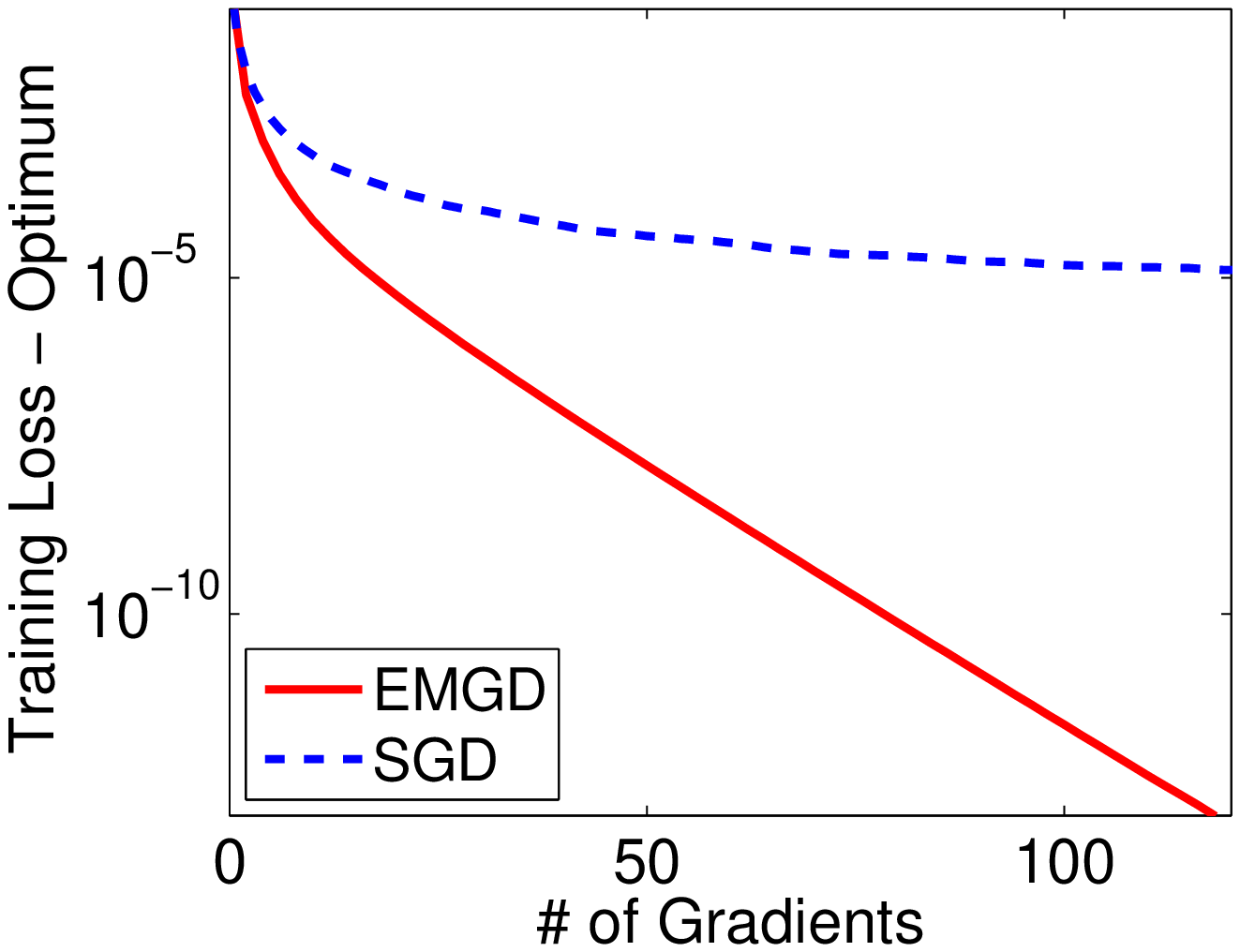}}
  \hspace{1ex}
  \subfloat[The variance of the (stochastic or mixed) gradient]{
    \label{fig:subfig1:b} %% label for second subfigure
\includegraphics[width=.48\textwidth]{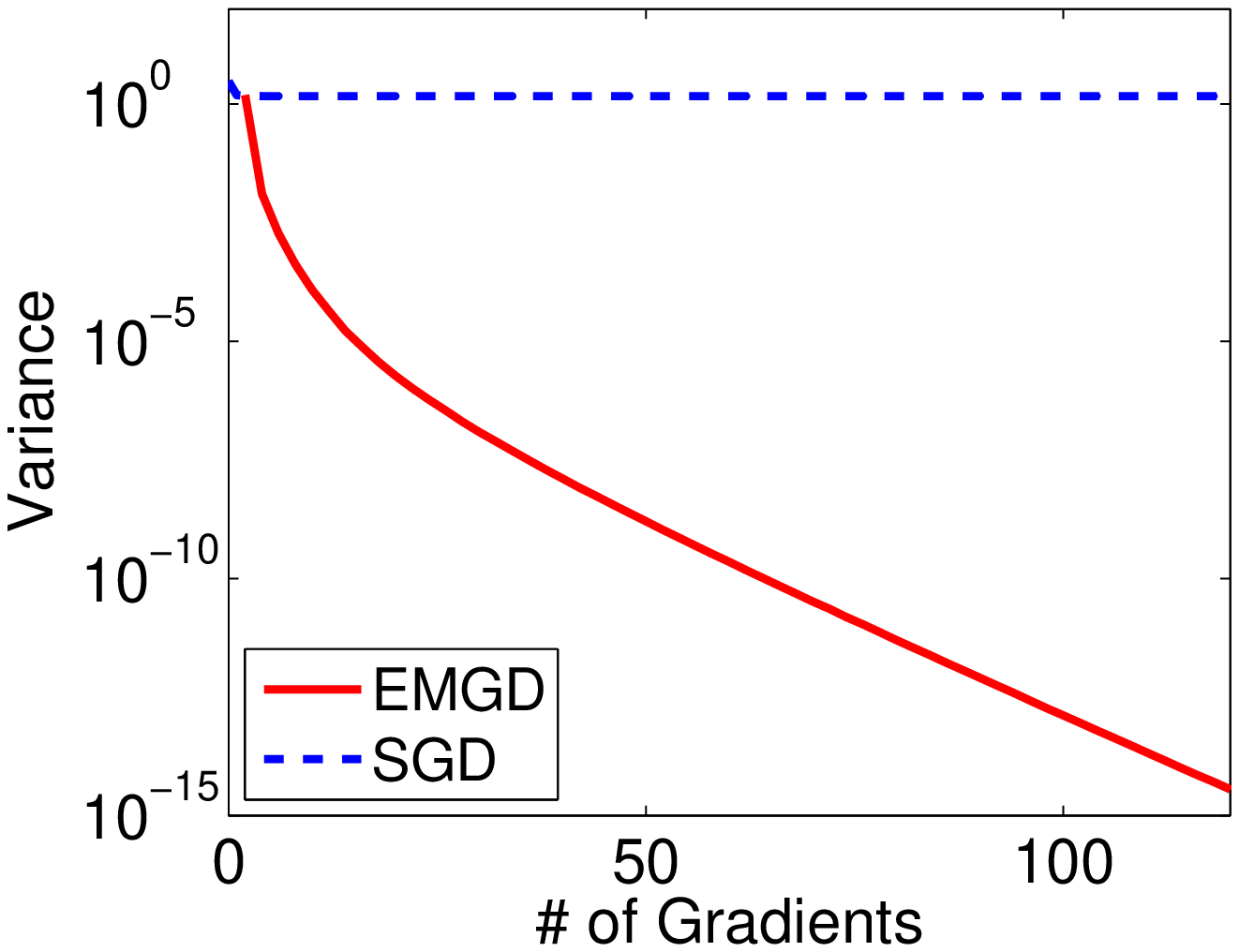}}
  \caption[Experimental results on  Adult data set]{Experimental results on the Adult data set. $\lambda=1e{-}3$ and $\eta=0.01$ in EMGD.}
  \label{fig:subfig1} %% label for entire figure
  \centering
  \subfloat[The training error]{
    \label{fig:subfig2:a} %% label for first subfigure
\includegraphics[width=.48\textwidth]{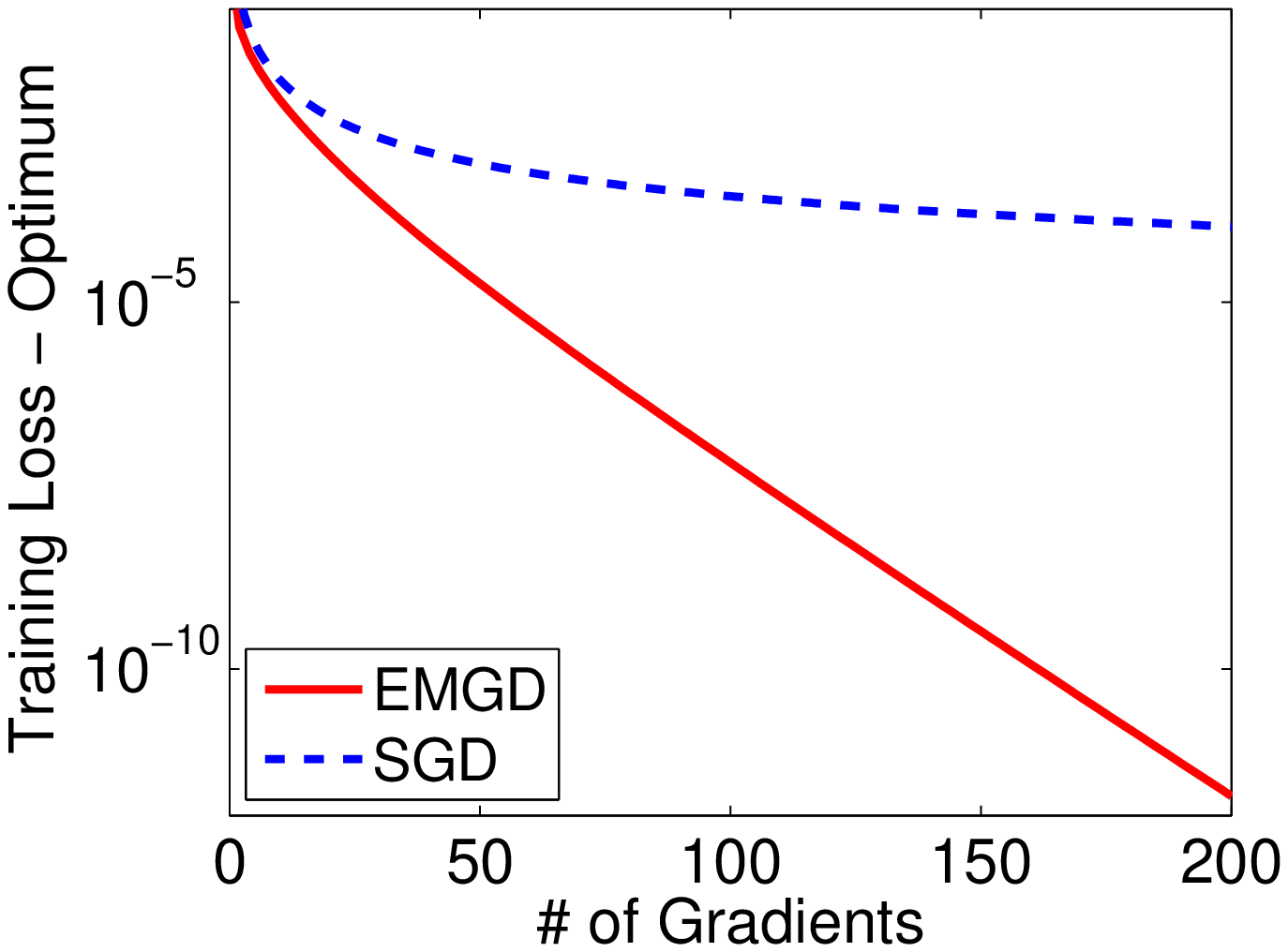}}
  \hspace{1ex}
  \subfloat[The variance of the (stochastic or mixed) gradient]{
    \label{fig:subfig2:b} %% label for second subfigure
\includegraphics[width=.48\textwidth]{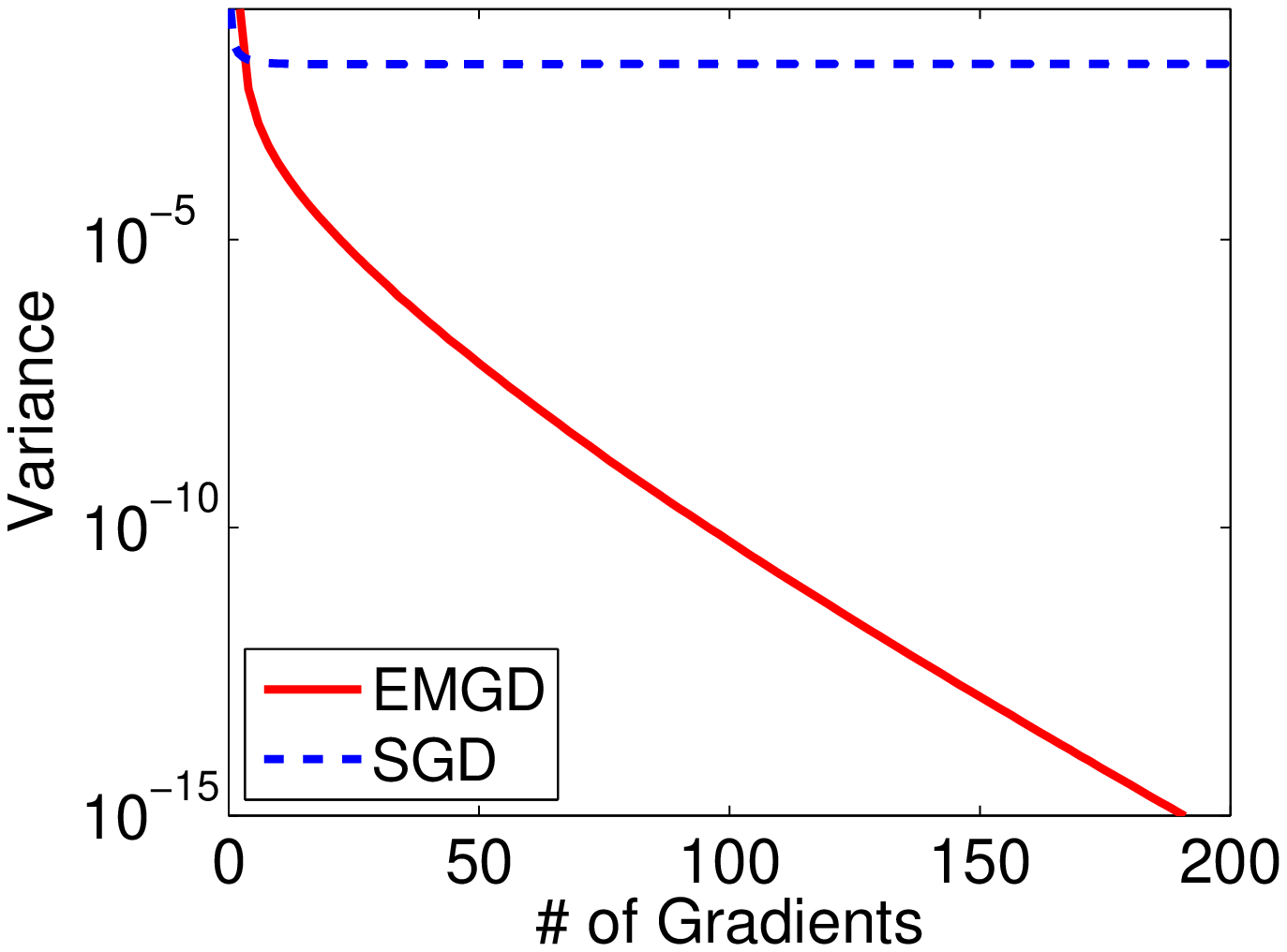}}
  \caption[Experimental results on  RCV1 data set]{Experimental results on the RCV1 data set. $\lambda=1e{-}5$ and $\eta=1$ in EMGD.}
  \label{fig:subfig2} %% label for entire figure
\end{figure*}
The results of both SGD and EMGD algorithms on the Adult data set are provided in Figure~\ref{fig:subfig1}. Figure~\ref{fig:subfig1:a} shows the difference between the current objective value and the optimum (which is obtained by running a batch algorithm for a long time) versus
\[
\textrm{The number of full gradients} + \frac{\textrm{The number of stochastic gradients}}{n}.
\]
Figure~\ref{fig:subfig1:b} shows the variances of the SGD and EMGD, respectively. It can be inferred from the results that the variance of SGD is almost same even when the algorithm approaches the optimal solution. As can be seen, the EMGD method is able to reduce the training error and the variance exponentially. We  obtain similar results on the RCV1 data set, which are shown in Figure~\ref{fig:subfig2}.

\begin{figure*}
  \centering
  \subfloat[The Adult data set]{
    \label{fig:subfig3:a} %% label for first subfigure
\includegraphics[width=.48\textwidth]{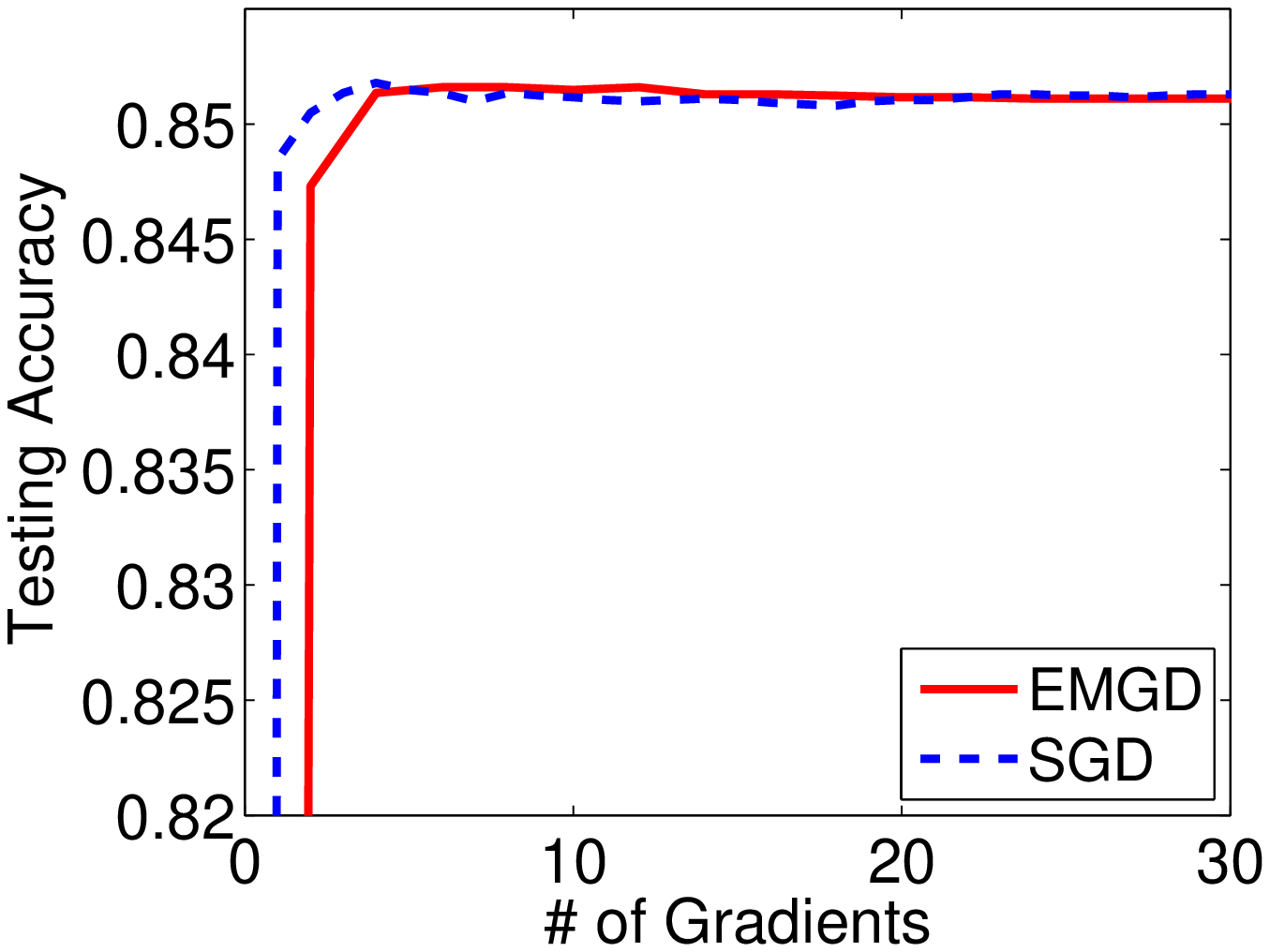}}
  \hspace{1ex}
  \subfloat[The RCV1 data set]{
    \label{fig:subfig3:b} %% label for second subfigure
\includegraphics[width=.48\textwidth]{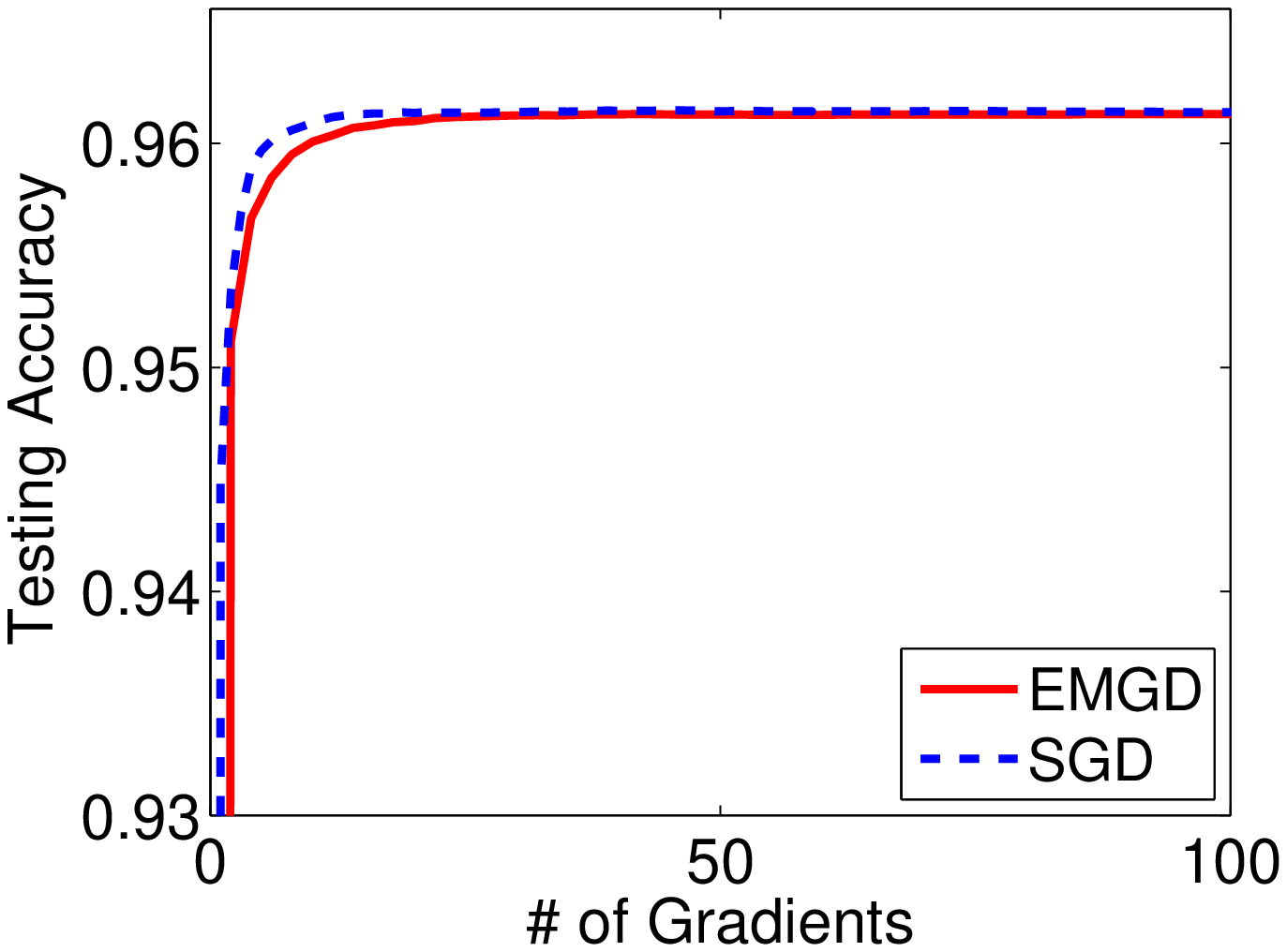}}
  \caption{The testing accuracy on RCV1 and Adult data sets.}
  \label{fig:subfig3} %% label for entire figure
\end{figure*}  
We also examine the testing accuracy of SGD and EMGD, which are depicted in Figure~\ref{fig:subfig3}. As can be seen, in terms of testing, EMGD is slightly worse than SGD at the beginning. This behavior is expected since EMGD needs one  full gradient to initialize.

\section{Discussion}

Compared to the optimization algorithm that only relies on full gradients~\cite{nesterov2004introductory}, the number of full gradients needed in EMGD is $O(\log\frac{1}{\epsilon})$ instead of $O(\sqrt{\kappa} \log \frac{1}{\epsilon})$. Compared to the optimization algorithms that only relies on stochastic gradients~\cite{Juditsky:SCS,hazan-2011-beyond,ICML2012Rakhlin},  EMGD is more efficient since it achieves a linear convergence rate.

The proposed EMGD algorithm can also be applied to the special optimization problem considered in~\cite{NIPS2012_SGM,shalev2012stochastic}, where $\F(\w)=\frac{1}{n} \sum_{i=1}^n f_i(\w)$. To make quantitative comparisons, let's assume the full gradient is $n$ times more expensive to compute than the stochastic gradient. Table~\ref{table:2-emgd} lists the computational complexity of the algorithms that enjoy linear convergence. As can be seen, the computational complexity of EMGD is lower than Nesterov's algorithm~\cite{nesterov2004introductory} as long as the condition number $\kappa \leq n^{2/3}$, the complexity of  SAG~\cite{NIPS2012_SGM} is lower than Nesterov's algorithm if $\kappa \leq n/8$, and the complexity of SDCA~\cite{shalev2012stochastic} is lower than Nesterov's algorithm if $\kappa \leq n^2$.\footnote{In  learning problems, we usually face a regularized optimization problem $\min_{\w \in \D} \frac{1}{n} \sum_{i=1}^n \ell(y_i; \dd{\w}{\x_i})+\frac{\tau}{2} \|\w\|^2$, where $\ell(\cdot;\cdot)$ is some loss fixed function. When the norm of the data is bounded, the smoothness parameter $L$ can be treated as a constant. The strong convexity parameter $\lambda$ is lower bounded by $\tau$. As a result, as long as $\tau > \Omega(n^{-2/3})$, which is a reasonable scenario~\cite{Wu:2005:SSM}, we have $\kappa < O(n^{2/3})$, indicating that our proposed EMGD algorithm can be applied.}  The complexity of EMGD is on the same order as SAG and SDCA when $\kappa \leq n^{1/2}$, but higher in other cases. Thus, in terms of computational cost, EMGD may not be the best one, but it has advantages in other aspects.
 \begin{enumerate}
  \item Unlike SAG and SDCA that only work for unconstrained optimization problem, the proposed algorithm works for both constrained and unconstrained optimization problems, provided the constrained problem in Step~\ref{step:epoch:update} can be solved efficiently.
  \item Unlike the SAG and SDCA that require an $\Omega(n)$ storage space, the proposed algorithm only requires the storage space of $\Omega(d)$, where $d$ is the dimension of $\w$.
  \item The only step in Algorithm~\ref{alg:emgd} that has dependence on $n$ is step~\ref{step:epoch:grad} for computing the gradient $\nabla \F(\wb^k)$. By utilizing distributed computing, the running time of this step can be reduced to $O(n/k)$, where $k$ is the number of computers, and the convergence rate remains the same. For SAG and SDCA , it is unclear whether they can reduce the running time without affecting the convergence rate.
  \item The linear convergence of SAG and SDCA only holds in expectation, whereas the linear convergence of EMGD holds with a high probability, which is much stronger.
\end{enumerate}

\begin{table}[t]
\label{table:2}
\begin{center}
\begin{tabular}{cccc}
Nesterov's algorithm~\cite{nesterov2004introductory} & EMGD & SAG ($n \geq 8 \kappa$)~\cite{NIPS2012_SGM}&  SDCA~\cite{shalev2012stochastic}
\\ \hline \\
$O\left(\sqrt{\kappa} n \log \frac{1}{\epsilon} \right)$ & $O\left((n+\kappa^2) \log \frac{1}{\epsilon} \right)$ & $O\left( n \log \frac{1}{\epsilon} \right)$ & $O\left((n+\kappa)\log \frac{1}{\epsilon} \right)$
 \\ \\\hline
\end{tabular}
\end{center}
\caption[The computational complexity for minimizing empirical error]{The computational complexity for minimizing $(1/n) \sum_{i=1}^n f_i(\w)$ }
\label{table:2-emgd}
\end{table}

\section{Summary}
In this chapter, we considered how to reduce the number of full gradients needed for smooth and strongly convex optimization problems. Under the assumption that both the gradient and the stochastic gradient are available, the EMGD algorithm, with the help of stochastic gradients, is are able to reduce the number of gradients needed from $O(\sqrt{\kappa}\log \frac{1}{\epsilon})$ to $O(\log \frac{1}{\epsilon})$. In the case that the objective function is in the form of (\ref{eqn:8:sum}), i.e., a sum of $n$ smooth functions, EMGD has lower computational cost than the full gradient method~\cite{nesterov2004introductory}, if the condition number $\kappa \leq n^{2/3}$. We validated our theoretical results on the convergence of EMGD algorithm and in particular its ability in reducing the variance during the optimization by some experimental results on two data sets. We note that although EMGD enjoys many nice properties, it is unclear whether it is the optimal algorithm when two kinds of gradients are available. Finally,  we  provided experimental evidence complementing our theoretical results for classification problem on few standard data sets.

%%%%%%%%%%%%%%%%%%%%%%%%%%%%%%%%%%%%%%%%%%%%%%%%%%%
\section{Bibliographic Notes}
During the last three decades, there have been significant advances in convex optimization~\cite{nemircomp1983,nesterov2004introductory,boyd-convex-opt}. In this section, we provide a brief review of the first order optimization methods.

We first discuss deterministic optimization, where the gradient of the objective function is available. For the general convex and Lipschitz continuous optimization problem, the iteration complexity of gradient (subgradient) descent is $O(\frac{1}{\epsilon^2})$, which is optimal up to constant factors~\cite{nemircomp1983}. When the objective function is convex and smooth, the optimal optimization scheme is the accelerated gradient descent developed by Nesterov, whose iteration complexity is $O(\frac{L}{\sqrt{\epsilon}})$~\cite{nesterov1983method,nesterov2005smooth}. With slight modifications, the accelerated gradient descent algorithm can also be applied to optimize the smooth and strongly convex objective function, whose iteration complexity is $O(\sqrt{\kappa} \log \frac{1}{\epsilon})$ and is in general not improvable~\cite{nesterov2004introductory,Nesterov_Composite}. The objective of our work is to reduce the number of access to the full gradients by exploiting the availability of stochastic gradients.

In stochastic optimization, we have the access to the stochastic gradient, which is an unbiased estimate of the full gradient~\cite{nemirovski2009robust}. Similar to the case in deterministic optimization, if the objective function is convex and Lipschitz continuous, stochastic gradient (sub-gradient) descent is the optimal algorithm and the iteration complexity is also $O(\frac{1}{\epsilon^2})$~\cite{nemircomp1983,nemirovski2009robust}. When the objective function is strongly convex, the algorithms proposed in very recent works~\cite{Juditsky:SCS,hazan-2011-beyond,ICML2012Rakhlin} achieve the optimal $O(\frac{1}{\lambda \epsilon})$ iteration complexity~\cite{sgd-lower-bounds}. Since the convergence rate of stochastic optimization is dominated by the randomness in the gradient~\cite{lan2012optimal,Lan:SCSC}, smoothness usually does not lead to a faster convergence rate for stochastic optimization.

From the above discussion, we observe that the iteration complexity in stochastic optimization is polynomial in $\frac{1}{\epsilon}$, making it is difficult to find high-precision solutions. However, when the objective function is strongly convex and can be written as a sum of a finite number of functions, i.e.,
\[
\F(\w)=\frac{1}{n} \sum_{i=1}^n f_i(\w),
\]
where each $f_i(\w)$ is smooth, the iteration complexity of some specific algorithms may exhibit a logarithmic dependence on $\frac{1}{\epsilon}$, i.e., a linear convergence rate. The first such algorithm is the stochastic average gradient (SAG)~\cite{NIPS2012_SGM}, whose iteration complexity is $O( n \log \frac{1}{\epsilon})$, provided $n \geq 8 \kappa$. The second one is the stochastic dual coordinate ascent (SDCA)~\cite{shalev2012stochastic}, whose iteration complexity is $O((n+\kappa) \log \frac{1}{\epsilon})$.
%%%%%%%%%%%%%%%%%%%%%%%%%%%%%%%%%%%%%%%%%%%%%%%%%%%

\chapter{Efficient Optimization with Bounded Projections} \label{chap:projection}

\def \R {\mathbb{R}}
\def \D {\mathcal{D}}
\def \y {\mathbf{y}}
\def \M {\mathcal{M}}
\def \E {\mathrm{E}}
\def \x {\mathbf{w}}
\def \y {\mathbf{y}}
\def \a {\mathbf{a}}
\def \L {\mathcal{L}}
\def \H {\mathcal{H}}
\def \fh {\widehat{f}}
\def \f {\mathbf{f}}
\def \Hk {\H_{\kappa}}
\def \P {\mathcal{P}}
\def \c {\mathbf{c}}
\def \X {\mathcal{X}}
\def \fhv {\widehat{\f}}
\def \S {\mathcal{S}}
\def \eh {\widehat{\delta}}
\def \z {\mathbf{z}}
\def \t {\mathbf{t}}
\def \ph {\widehat{\phi}}
\def \gh {\widehat{g}}
\def \Lh {\widehat{L}}
\def \lh {\widehat{\lambda}}
\def \gt {\widetilde{g}}
\def \et {\widetilde{\varepsilon}}
\def \zt {\widetilde{\z}}
\def \Zt {\widetilde{Z}}
\def \gammat {\widetilde{\gamma}}
\def \etat {\widehat{\eta}}
\def \ab {\overline{\alpha}}
\def \Mh {\widehat{M}}
\def \veh {\widehat{\varepsilon}}
\def \gammah {\widehat{\gamma}}
\def \etah {\widehat{\eta}}
\def \xh {\widehat{\x}}
\def \Kh {\widehat{K}}
\def \rh {\widehat{r}}
\def \vet {\widetilde{\varepsilon}}
\def \hh {\widehat{h}}
\def \htt {\widetilde{h}}
\def \Er {\mathcal{E}}
\def \u {\mathbf{u}}
\def \Eh {\widehat{\Er}}
\def \v {\mathbf{v}}
\def \w {\mathbf{w}}
\def \Hb {\overline{\H}}
\def \DDh {\widehat{\D}}
\def \B {\mathcal{B}}
\def \Phl {\overline{\Phi}}
\def \R {\mathbb{R}}
\def \Kt {\widetilde{K}}
\def \G {\mathcal{G}}
\def \kt {\widetilde{\kappa}}
\def \b {\mathbf{b}}
\def \zt {\widetilde{\z}}
\def \h {\mathbf{h}}
\def \g {\mathbf{g}}
\def \vt {\widetilde{\v}}
\def \Vt {\widetilde{V}}
\def \A {\mathcal{A}}
\def \B {\mathcal{B}}
\def \Hh {\widehat{\H}}
\def \Ht {\widetilde{\H}}
\def \alt {\widetilde{\alpha}}
\def \at {\widetilde{\a}}
\def \bt {\widetilde{\b}}
\def \hht {\widetilde{\h}}
\def \Dt {\widetilde{D}}
\def \Zt {\widetilde{Z}}
\def \lt {\widetilde{\lambda}}
\def \rb {\bar{r}}
\def \Dh {\widehat{D}}
\def \tr {\mbox{tr}}
\def \vh {\widehat{\varphi}}
\def \Nh {\widehat{N}}
\def \e {\mathbf{e}}
\def \vt {\widetilde{\varphi}}
\def \zh {\widehat{\z}}
\def \ft {\widetilde{f}}
\def \lb {\bar{\ell}}
\def \Dh {\widehat{\gamma}}
\def \dt {\tilde{\gamma}}
\def \J {\mathcal{J}}
\def \wh {\widehat{\w}}
\def \wt {\widetilde{\w}}
\def \I {\mathcal{I}}
\def \lh {\widehat{\ell}}
\def \Oh {\widehat{\K}}
\def \F {\mathcal{F}}
\def \Z {\mathcal{Z}}
\def \hb {\bar{h}}
\def \fb {\bar f}
\def \Lb {\bar{\L}}
\def \xb {\widehat{\x}}
\def \K {\mathcal{W}}

\def \fh {\hat{f}}
\def \B {\mathbb{B}}
\def \x {\mathbf{w}}
\def \y {\mathbf{y}}
\def \v {\mathbf{v}}
\def \H {\mathcal{H}_{\kappa}}
\def \R {\mathbb{R}}
\def \w {\mathbf{w}}
\def \wh {\widehat{\w}}
\def \sgn {\mbox{sgn}}
\def \a {\mathbf{a}}
\def \u {\mathbf{u}}
\def \uh {\widehat{\u}}
\def \wt {\widetilde{\w}}
\def \E {\mathrm{E}}
\def \C {\mathcal{C}}
\def \xh {\widehat{\x}}
\def \xt {\widetilde{\x}}
\def \truncate {\mbox{truncate}}
\def \A {\mathcal{A}}
\def \p {\mathbf{p}}
\def \S {\mathcal{S}}
\def \L {\mathcal{L}}
\def \H {\mathcal{H}}
\def \q {\mathbf{q}}
\def \F {\mathcal{F}}
\def \z {\mathbf{z}}
\def \b {\mathbf{b}}
\def \K {\mathcal{K}}
\def \S {\mathcal{S}}
\def \E {\mathbb{E}}
\def \O {O}
\def \V {\mathcal{V}}
\def \BD {\mathsf{B}}
\def \blambda {\boldsymbol{\lambda}}
\def \bmu {\boldsymbol{\mu}}

\def \K {\mathcal{W}}

In this chapter we aim at developing more efficient optimization methods by reducing the number of projection steps. An examination of optimization methods for both online and convex optimization problems introduced in Chapter~\ref{chap-background} reveals that  most of them require projecting the updated solution at \textit{each} iteration to ensure that the obtained solution stays within the feasible domain. For complex domains (e.g., positive semidefinite cone), the projection step can be computationally expensive, making first order optimization methods unattractive for large-scale optimization problems. The broad question in this chapter of the thesis is the extent to which it is possible to reduce the number of projection steps in stochastic and online optimization algorithms.

In stochastic setting, we address this limitation by developing  novel stochastic optimization algorithms that do not need intermediate projections. Instead, only one projection at the last iteration is needed to obtain a feasible solution in the given domain. Our theoretical analysis shows that with a high probability, the proposed algorithms achieve an $O(1/\sqrt{T})$ convergence rate for general convex functions, and an $O(\ln T/T)$  rate for  strongly convex functions under mild conditions about the domain and the objective function. The key insight underlying the proposed method for strongly convex objectives is that by smoothing the objective function and leveraging it in the optimization, we are able to skip the intermediate projections. This is in contrast to other parts of the thesis where we explicitly leveraged the smoothness assumption. To the best of our knowledge, these are the first projection-free stochastic optimization methods.

In online setting, to tackle the computational challenge arising from the projection steps, we consider an alternative online learning problem as follows. Instead of requiring that  each solution obeys the constraints which define the convex domain,  we only require the constraints to be satisfied in a long run. Then, the online learning problem becomes a task to find a sequence of solutions  under the  long term constraints.  In other words, instead of solving the projection  on each round, we allow the learner to make decisions at some rounds which do not belong to the constraint set, but the overall sequence of chosen decisions  must obey the constraints at the end by a vanishing  rate. We refer to this problem as  online learning with soft constraints. By turning the problem into an online convex-concave optimization problem, we propose an efficient algorithm which achieves an $O(\sqrt{T})$  regret bound and an $O(T^{3/4})$ bound on the violation of the constraints. Then, we modify the algorithm in order to guarantee that the constraints are exactly satisfied in the long run. This gain is achieved at the price of getting an $O(T^{3/4})$  regret bound.

\section{Setup and Motivation}\label{sec-one-sgd-setup}

In stochastic  setting, we consider the following convex optimization problem:
\begin{eqnarray}
\begin{aligned}\label{eqn:opt}
\min_{\x\in\K} f(\x),
\end{aligned}
\end{eqnarray}

where $\K$ is a bounded convex domain. We assume that $\K$ can be characterized by an inequality constraint and without loss of generality  is bounded by the unit ball,  i.e.,
\begin{eqnarray}
\begin{aligned}\label{eqn:d}
\K=\{\x\in\R^d: g(\x)\leq 0\}\subseteq \B=\{\x\in\R^d: \|\x\|_2\leq 1\},
\end{aligned}
\end{eqnarray}

where $g(\x)$ is a convex constraint function. We assume that $\K$ has a non-empty interior, i.e., there exists $\x$ such that $g(\x)<0$ and   the optimal solution $\x_*$ to~(\ref{eqn:opt}) is in the interior of the unit ball $\B$, i.e., $\|\x_*\|_2<1$. Note that when a domain is characterized by multiple convex constraint functions, say $g_i(\x) \leq 0, i=1, \ldots, m$, we can summarize them into one constraint $g(\x)\leq 0$, by defining it  as    $g(\x) = \max_{1 \leq i \leq m } g_i(\x)$.

To solve the optimization problem in (\ref{eqn:opt}), we assume that the
only information available to the algorithm is through a stochastic oracle that provides  unbiased estimates of the gradient of $f(\x)$. More precisely, let $\xi_1, \ldots, \xi_T$ be a sequence of independently and identically distributed (i.i.d) random variables sampled from an unknown distribution.  At each iteration $t$, given solution $\x_t$, the oracle returns $\mathbf{g}(\x_t; \xi_t)$, an unbiased estimate of the true gradient $\nabla f(\x_t)$, i.e., $\E_{\xi_t}[\mathbf{g}(\x_t, \xi_t)] = \nabla f(\x_t)$. The goal of the learner is to find an approximate optimal solution by making $T$ calls to this oracle. 

Recall from Chapter~\ref{chap-background} that to find a solution within the domain $\K$ which optimizes the given objective function $f(\x)$, SGD computes an unbiased estimate of the gradient of $f(\x)$, and updates the solution by moving it in the opposite direction of the estimated gradient. To ensure that the solution stays within the domain $\K$, SGD has to project the updated solution back into the $\K$ at \textit{every iteration}.  More precisely, SGD method produces a sequence of solutions   by following updating:
\begin{eqnarray}
\begin{aligned}\label{eqn:updt}
\x_{t+1} = \Pi_{\K}(\x_t -\eta_t \mathbf{g}(\x_t,\xi_t)),
\end{aligned}
\end{eqnarray}
where $\eta_t$ is the step size at iteration $t$, $\Pi_\K(\cdot)$ is a projection operator that projects $\x$ into the domain $\K$, and $\mathbf{g}(\x,\xi_t)$ is an unbiased stochastic gradient of $f(\x)$, for which we further assume bounded gradient variance as
\begin{eqnarray}
\begin{aligned}
&\E_{\xi_t}[\exp(\|\mathbf{g}(\x, \xi_t)- \nabla f(\x)\|_2^2/\sigma^2)]\leq \exp(1).\label{eqn:gb}
\end{aligned}
\end{eqnarray}
For general convex optimization, stochastic gradient descent methods can obtain an $O(1/\sqrt{T})$ convergence rate in expectation or in a high probability provided~(\ref{eqn:gb}) holds~\cite{nemirovski2009robust}.  Although a large number of iterations is usually needed to obtain a solution of desirable accuracy, the lightweight computation per iteration makes SGD attractive for many large-scale learning problems. 

The SGD method is  computationally efficient only when the projection $\Pi_\K(\cdot)$ can be carried out efficiently.  Although efficient algorithms have been developed for projecting solutions into special domains (e.g., simplex and $\ell_1$ ball~\cite{Duchi:2008:EPL:1390156.1390191, DBLP:conf/icml/LiuY09}); for complex domains, such as a positive semidefinite (PSD) cone in metric learning and bounded trace norm matrices in matrix completion (more examples of complex domains can be found in~\cite{hazan-projection-free} and~\cite{thesisMJ}), the projection step requires solving an expensive convex optimization, leading to a high computational cost per iteration and consequently making SGD unappealing for large-scale optimization problems over such domains. For instance, projecting a matrix into a PSD cone requires computing the full eigen-decomposition of the matrix, whose complexity is cubic in the size of the matrix.

The central theme of this chapter in stochastic setting is to develop a SGD based method that does not require projection at each iteration. This problem was first addressed in a very recent work~\cite{hazan-projection-free}, where the authors extended Frank-Wolfe algorithm~\cite{frank56} for online learning. But, one main shortcoming of the algorithm proposed in~\cite{hazan-projection-free} is that it has a slower convergence rate (i.e., $O(T^{-1/3})$) than a standard SGD algorithm (i.e., $O(T^{-1/2})$). In this work, we demonstrate that a properly modified SGD algorithm can achieve the optimal convergence rate of $O(T^{-1/2})$ using only \textbf{ONE} projection for general stochastic convex optimization problem.

We further develop an SGD based algorithm for strongly convex optimization that achieves a convergence rate of $O(\ln T/T)$, which is only a logarithmic factor worse than the optimal rate~\cite{hazan-2011-beyond}. The key idea of both algorithms is to appropriately penalize the intermediate solutions when they are outside the domain. With an appropriate design of penalization mechanism, the average solution $\xh_T$ obtained by the SGD after $T$ iterations will be very close to the domain $\K$, even without intermediate projections. As a result, the final feasible solution $\widetilde{\x}_T$ can be obtained by projecting $\xh_T$ into the domain $\K$, the only projection that is needed for the entire algorithm. We note that our approach is very different from the previous efforts in developing projection free convex optimization algorithms (see \cite{hazan-projection-free,DBLP:conf/icml/JaggiS10,thesisMJ} and references therein), where the key idea is to develop appropriate \textit{updating} procedures to restore the feasibility of solutions at every iteration.

\section{Stochastic Frank-Wolfe Algorithm}
Before presenting the proposed algorithms, here we investigate and analyze a greedy algorithm, which is a slight modification of the Frank-Wolfe (FW) method, when applied to stochastic optimization problem. As discussed in Chapter~\ref{chap-background}, the FW or conditional gradient algorithm is a feasible direction method and at each iteration, it finds the best feasible direction (with respect to the linear approximation of the function) and updates the next solution as a convex combination of the current solution and chosen feasible direction. This algorithm is revisited in~\cite{thesisMJ} for deterministic \textit{smooth} convex function over a convex domain aiming to devise an efficient algorithm without projections. As discussed before, the algorithm in~\cite{thesisMJ} generates a sequence of solutions via the following steps:
\begin{equation}
\begin{aligned}
&\u_t =\arg\max_{\u\in\K} \langle \u, -\nabla f(\x_t) \rangle\label{eqn:lin-9-app}\\
&\x_{t+1} =  (1-\eta_t)\x_t + \eta_t\u_t
\end{aligned}
\end{equation}
The bulk of computation in FW algorithm is step (\ref{eqn:lin-9-app}) and the algorithm is an  attractive method whenever step (\ref{eqn:lin-9-app}) can be achieved with low cost complexity, for otherwise, it is not practically interesting.  This is true for some special domains, such as polyhedron, simplex and $\ell_1$ ball where the linear optimization problem can be efficiently solved. However, this assumption does not hold for general complex domains (e.g. general PSD cones). This limitation makes the FW algorithm computationally unattractive for general complex domains since it simply translates the complexity of quadratic optimization for projection into a linear optimization problem over the same domain. We note that the FW algorithm is also attractive because of its sparsity merits which is not the focus of this chapter and we only consider the computational efficiency of optimization methods.

A trivial modification to extend this algorithm to  stochastic gradients  is to replace the true gradient $\nabla f(\x_t)$ with a stochastic gradient $\mathbf{g}(\x_t,\xi_t)$ in~(\ref{eqn:lin-9-app}). Next, we sketch an analysis of stochastic greedy algorithm and discuss the problem caused by using unbiased estimate of the gradient instead of true gradient in updating solutions.

A key inequality in the convergence analysis of stochastic greedy algorithm is
\begin{equation}\label{eqn:key-9-app}
f(\x_{t+1}) \leq f(\x_t) - \eta_t \underbrace{\max_{\x\in\K}\langle \x_t-\x,\nabla f(\x_t) \rangle}\limits_{\triangleq~g(\x_t, \nabla f(\x_t))} + \eta_t^2L,
\end{equation}
where $f(\x)$ is assumed to be $L$-smooth function, $g(\x, \nabla f(\x))) = \max_{\u\in\K}\langle \x-\u, \nabla f(\x)\rangle$ is the duality gap.  Using $g(\x_t, \nabla f(\x_t))\geq f(\x_t)- f(\x_*)$, we can obtain an $O(1/T)$ convergence bound for the greedy algorithm with $\eta_t=2/(t+1)$ by induction (e.g., see Theorem~\ref{thm-chp-2-FW} in Chapter~\ref{chap-background}). However, when using the stochastic gradient, $\u_t=\arg\max_{\u \in\K}\langle \u, -\mathbf{g}(\x_t,\xi_t) \rangle$, and the key inequality becomes as:

\begin{equation*}
\begin{aligned}
f(\x_{t+1}) 
& \leq f(\x_t) - \eta_t  \langle \x_t - \u_t, \nabla f(\x_t)\rangle+ \eta_t^2L \\
& \leq f(\x_t) - \eta_t  g(\x_t, \nabla f(\x_t))+ \eta_t^2L \\
& \hspace{0.5cm} +  \underbrace{\eta_t(g(\x_t, \nabla f(\x_t)) - g(\x_t, \mathbf{g}(\x_t, \xi_t))) + \eta_t\langle \x_t - \u_t, \mathbf{g}(\x_t,\xi_t)-\nabla f(\x_t)\rangle}\limits_{\triangleq~\zeta_t}
\end{aligned}
\end{equation*}
where the top line recovers the inequality in~(\ref{eqn:key-9-app}). However, the problem is that the quantity $\zeta_t$ in the second line of above inequality  is not a martingale sequence, i.e., $\E_{\xi_t|\xi_{[t-1]}}[\zeta_t]\neq 0$. We can take a conservative analysis to bound $\zeta_t\leq C\eta_t\|\nabla f(\x_t)-\mathbf{g}(\x_t, \xi_t)\|_*\triangleq C\eta_t\delta_t$, where $C$ is a constant.  Assuming $\|\nabla f(\x_t) - \mathbf{g}(\x_t, \xi_t)\|_*\leq \sigma$, we could have, with a probably $1-\epsilon$, $\eta_t\delta_t\leq \eta_t(1+\ln(1/\epsilon))\sigma$ by Markov inequality in Lemma~\ref{lemma:markov}. As a result, we obtain the following recursive inequality
\[
f(\x_{t+1}) - f(\x_*)\leq (1-\eta_t)(f(\x_t) - f(\x_*)) + \eta_t^2L+ {\eta_tC(1+\ln(1/\epsilon))\sigma}
\]
where the last term in the above inequality makes the convergence analysis more involved.  Whether  it is possible to design a sequence of step sizes $\eta_t$ to have a vanishing bound for $f(\x_t)-f(\x_*)$ is unclear for us and remains an open problem, which is beyond the scope of this chapter.

\section{Stochastic Optimization with  Single Projection}
We now turn to extending the SGD method to the setting where only one projection is allowed to perform for the entire sequence of updating. The main idea is to incorporate the constraint function $g(\x)$ into the objective function to penalize the intermediate solutions that are outside the domain. The result of the penalization is that, although the average solution obtained by SGD may not be feasible, it should be very close to the boundary of the domain. A projection is performed at the end of the iterations to restore the feasibility of the average solution. 

Before proceeding, we recall  few definitions from Appendix~\ref{chap:appendix-convex} about convex analysis.
\begin{definition}
A function $f(\x)$ is a Lipschitz continuous function with constant $G$ w.r.t a norm $\|\cdot\|$, if
\begin{eqnarray}
\begin{aligned}
|f(\x_1)- f(\x_2)|\leq G\|\x_1-\x_2\|,\forall \x_1,\x_2\in\B.
\end{aligned}
\end{eqnarray}
\end{definition}
In particular, a convex function $f(\x)$ with a bounded (sub)gradient $\|\partial f(\x)\|_*\leq G$ is $G$-Lipschitz continuous, where $\|\cdot\|_*$ is the dual norm to $\|\cdot\|$.
\begin{definition}
A convex function $f(\x)$ is $\beta$-strongly convex w.r.t a norm $\|\cdot\|$ if there exists a constant $\beta > 0$ (often called the modulus of strong convexity) such that, for any $\alpha\in[0, 1]$, it holds:
\[f(\alpha\x_1+ (1-\alpha)\x_2)\leq \alpha f(\x_1) + (1-\alpha) f(\x_2) - \frac{1}{2}\alpha(1-\alpha)\beta\|\x_1-\x_2\|^2, \forall \x_1,\x_2\in\B.
\]
\end{definition}
When $f(\x)$ is differentiable, the strong convexity is equivalent to
$$f(\x_1) \geq f(\x_2) + \langle \nabla f(\x_2), \x_1-\x_2\rangle + \frac{\beta}{2}\|\x_1-\x_2\|^2, \forall \x_1,\x_2\in\B.$$
In the sequel, we use the standard Euclidean norm to define Lipschitz and strongly convex functions.

\begin{algorithm}[tb]
\caption{SGD with ONE Projection by Primal Dual Updating (\texttt{SGD-PD})}
\begin{algorithmic}[1] \label{alg:1}
    \STATE {\textbf{Input}}: a sequence of step sizes $\{\eta_t\}$,  and a parameter $\gamma > 0$
    \STATE {\textbf{Initialize:}}: $\x_1 = \mathbf{0}$ and $\lambda_1 = 0$
    \FOR{$t = 1, 2, \ldots, T$}
      \STATE Compute $\x'_{t+1} = \x_{t} -  \eta_t (\mathbf{g}(\x_t, \xi_t) +\lambda_t\nabla g(\x_t))$
      \STATE Update $\x_{t+1}=\x'_{t+1}/\max{(\|\x'_{t+1}\|_2,1)}$,
      \STATE Update $\lambda_{t+1} = [(1-\gamma\eta_t)\lambda_{t} + \eta_t g(\x_t)]_+$
      \ENDFOR
      \STATE {\textbf Output:} $\widetilde \x_T = \Pi_{\K}\left(\widehat\x_T\right)$, where $\widehat\x_T=\sum_{t=1}^T\x_t/T$.
      \end{algorithmic}
      \end{algorithm}

The key ingredient of proposed algorithms  is to replace the projection step with the gradient computation of the constraint function defining the domain $\K$, which is significantly cheaper than projection step. As an example, when a solution is restricted to a PSD cone, i.e., $\mathbf{X} \succeq 0$ where $\mathbf{X}$ is a symmetric matrix, the corresponding inequality constraint is $g(\mathbf{X})=\lambda_{\max}(-\mathbf{X}) \leq 0$, where $\lambda_{\max}(\mathbf{X})$ computes the largest eigenvalue of $\mathbf{X}$ and is a convex function.  In this case, $\nabla g(\mathbf{X})$  only requires computing the minimum eigenvector of a matrix, which is  cheaper than a  full eigenspectrum computation required at each iteration of  the standard SGD algorithm to restore feasibility. 

Below, we state a few  assumptions about $f(\x)$ and $g(\x)$ often made in stochastic optimization  as:
\begin{assumption} We assume that:
\begin{eqnarray}
\begin{aligned}
 &\mbox{{\rm{\textbf{A1}}}}\quad\quad\|\nabla f(\x)\|_2\leq G_1, \quad\|\nabla g(\x)\|_2\leq G_2,\quad |g(\x)|\leq C_2,\quad\forall\x\in\B,\label{eqn:bound1}\\
&\mbox{{\rm{\textbf{A2}}}}\quad\quad\E_{\xi_t}[\exp(\|\mathbf{g}(\x, \xi_t)- \nabla f(\x)\|_2^2/\sigma^2)]\leq \exp(1),\quad\forall \x\in\B.\label{eqn:bound2}
\end{aligned}
\end{eqnarray}
\end{assumption}

We also make the following mild assumption about the boundary of the convex domain $\K$ as:
\begin{assumption} We assume that:
\begin{eqnarray}
\begin{aligned}
\hspace*{-0.2in}\mbox{{\rm{{\textbf{A3}}}}}\quad\quad \mbox{there exists a constant $\rho > 0$ such that } \min_{g(\x)=0}\|\nabla g(\x)\|_2\geq \rho. \label{eqn:rho}
\end{aligned}
\end{eqnarray}
\end{assumption}

\begin{remark}
The purpose of introducing assumption \textbf{A3} is to ensure that the optimal dual variable for the constrained optimization problem in (\ref{eqn:opt}) is well bounded from the above, a key factor for our analysis. To see this, we write the problem in (\ref{eqn:opt}) into a convex-concave optimization problem:
\[
    \min\limits_{\x \in \B} \max\limits_{\lambda \geq 0} f(\x) + \lambda g(\x).
\]
Let $(\x_*, \lambda_*)$ be the optimal solution to the above convex-concave optimization problem. Since we assume $g(\x)$ is strictly feasible, $\x_*$ is also an optimal solution to (\ref{eqn:opt}) due to the strong duality theorem~\cite{boyd-convex-opt}. Using the first order optimality condition, we have $\nabla f(\x_*) = -\lambda_* \nabla g(\x_*)$.
Hence, $\lambda_* =0$ when $g(\x_*) < 0$, and $\lambda_* = \|\nabla f(\x_*)\|_2/\|\nabla g(\x_*)\|_2$ when $g(\x_*) = 0$. Under assumption \textbf{A3}, we have $\lambda_* \in [0, G_1/\rho]$. 
\end{remark}

We note that, from a practical point of view, it is straightforward to verify that for many domains including PSD cone and Polytope, the gradient of the constraint function is lower bounded on the boundary and therefore assumption \textbf{A3} does not limit the applicability of the proposed algorithms for stochastic optimization. For the example of $g(\mathbf{X})=\lambda_{\max}(-\mathbf{X})$, the assumption \textbf{A3} implies $\min_{g(\mathbf{X})=0}\|\nabla g(\mathbf{X})\|_{\rm{F}}=\|\u\u^{\top}\|_{\rm{F}}=1$, where $\u$ is an orthonomal vector representing the corresponding eigenvector of the matrix $\mathbf{X}$ whose minimum eigenvalue is zero.

We propose two different ways of incorporating the constraint function into the objective function, which result in two algorithms, one for general convex  and the other for strongly convex functions. 

\subsection{General Convex Functions}
To incorporate the constraint function $g(\x)$, we introduce a regularized Lagrangian function,
\begin{eqnarray}\label{eqn-9-fixed-convex-concave}
\begin{aligned}
\L(\x, \lambda) =  f(\x) + \lambda g(\x) - \frac{\gamma}{2}\lambda^2 , \quad \lambda\geq0.
\end{aligned}
\end{eqnarray}
The summation of the first two terms in $\L(\x,\lambda)$ corresponds to the Lagrangian function in dual analysis and $\lambda$ corresponds to a Lagrangian multiplier. A regularization term $-(\gamma/2)\lambda^2$ is introduced in $\L(\x, \lambda)$ to prevent $\lambda$ from being too large. Instead of solving the constrained optimization problem in (\ref{eqn:opt}), we try to solve the following convex-concave optimization problem
\begin{eqnarray}
\min_{\x\in\B}\max_{\lambda\geq0} \L(\x, \lambda). \label{eqn:convex-concaveone}
\end{eqnarray}
The proposed algorithm for stochastically optimizing the problem in (\ref{eqn:convex-concaveone}) is summarized in Algorithm~\ref{alg:1}. It differs from the existing stochastic gradient descent methods in that it updates both the primal variable $\x$ (steps 4 and 5) and the dual variable $\lambda$ (step 6), which shares the same step sizes.  We note that the parameter $\rho$   is not employed  in the implementation of  Algorithm~\ref{alg:1} and is only required for  the theoretical analysis. 
It is noticeable that a similar  primal-dual updating  will be  explored later in this chapter  to avoid projection in online learning. We note that in online setting the algorithm and analysis only lead to a bound for the regret and the violation of the constraints in a long run, which does not necessarily guarantee the feasibility of final solution. Also our proof techniques differ from~\cite{nemirovski2009robust}, where the convergence rate is obtained for the saddle point; however our goal is to attain bound on the convergence of the primal feasible solution. 
\begin{remark}The convex-concave optimization problem in (\ref{eqn:convex-concaveone}) is equivalent to the following minimization problem:
\begin{eqnarray}\label{eqn:pen}
\min_{\x\in\B} \;f(\x) + \frac{[g(\x)]^2_+}{2\gamma}, \label{eqn:fg}
\end{eqnarray}
where $[z]_+$ outputs $z$ if $z > 0$ and zero otherwise. It thus may seem tempting to directly optimize the penalized function $f(\x) + [g(\x)]_+^2/(2\gamma)$ using the standard SGD method, which unfortunately does not yield a regret of $O(\sqrt{T})$. This is because, in order to obtain a regret of $O(\sqrt{T})$, we need to set $\gamma = \Omega(\sqrt{T})$, which unfortunately will lead to a blowup of the gradients and consequently a poor regret bound. Using a primal-dual updating schema allows us to adjust the penalization term more carefully to obtain an $O(1/\sqrt{T})$ convergence rate.\\
\end{remark}

\begin{theorem}\label{thm:1}
For any general convex function $f(\x)$, if we set $\eta_t=\gamma/(2G_2^2), t=1,\cdots, T$, and $\gamma= G^2_2/\sqrt{(G_1^2+C_2^2 + (1+\ln(2/\delta))\sigma^2)T}$ in Algorithm~\ref{alg:1}, under assumptions \textbf{A1-A3}, we have,  with a probability at least $1-\delta$,  
\[
f(\widetilde\x_T)\leq \min_{\x\in\K}f(\x) + O\left(\frac{1}{\sqrt{T}}\right),
\]
%In the case of $\beta$-strongly convex and $\gamma$-smooth function $f(\x)$, if we set $\gamma=G_2^2/\beta$, $\tau_t=1/(\gamma t)$, and $\eta_t=1/(\beta t)$, we have, with a probability $1-\delta$
%\[
%f(\widetilde \x_T)\leq \min_{\x\in\K} f(\x) + \widetilde O(1/T)
%\]
where $O(\cdot)$ suppresses polynomial factors that depend on $\ln(2/\delta), G_1, G_2, C_2, \rho$, and $\sigma$.
\end{theorem}

\subsection{Strongly Convex Functions with Smoothing}

We first emphasize that it is difficult to extend Algorithm~\ref{alg:1} to achieve an $O(\ln T/T)$ convergence rate for strongly convex optimization. This is because although the function $-\L(\x, \lambda)$ is strongly convex in $\lambda$, its modulus for strong convexity is $\gamma$, which is too small to obtain an $O(\ln T)$ regret bound.

To achieve a faster convergence rate for strongly convex optimization, we change assumptions \textbf{A1} and {\textbf{A2}} as follows.
\begin{assumption}  
\[
\mbox{{\rm{\textbf{A4}}}}\quad \|\mathbf{g}(\x,\xi_t)\|_2\leq G_1, \quad \|\nabla g(\x)\|_2\leq G_2, \quad\forall \x\in\B,
\]
where we slightly abuse the same notation $G_1$.  
\end{assumption}

Note that {\textbf{A1}}  only requires  that $\|\nabla f(\x)\|_2$ is bounded and  {\textbf{A2}} assumes a mild condition on the stochastic gradient.  In contrast,  for strongly convex optimization we need to assume a bound on the stochastic gradient $\|\mathbf{g}(\x,\xi_t)\|_2$. Although assumption \textbf{A4} is stronger than assumptions {\textbf{A1}} and {\textbf{A2}}, however, 
it is always possible to bound the stochastic gradient for machine learning
problems where $f(\x)$ usually consists of a summation of loss functions on training examples,  and the stochastic gradient is computed by sampling over the training examples. Given the bound on $\|\mathbf{g}(\x,\xi_t)\|_2$, we can easily have $\|\nabla f(\x)\|_2=\|\E\mathbf{g}(\x,\xi_t)\|_2\leq \E\|\mathbf{g}(\x,\xi_t)\|_2\leq G_1$, which is used to set  an  input parameter $\lambda_0>G_1/\rho$ to the algorithm. 
According to the discussion in the last subsection, we know that the optimal dual variable $\lambda_*$ is upper bounded by $G_1/\rho$, and consequently is upper bounded by $\lambda_0$. 

Similar to the last approach, we write the optimization problem~(\ref{eqn:opt}) into an equivalent convex-concave optimization problem:
\[
\min_{g(\x)\leq 0} f(\x) = \min_{\x\in\B}\max_{0\leq \lambda\leq \lambda_0} f(\x) + \lambda g(\x) = \min\limits_{\x \in \B} f(\x) + \lambda_0[g(\x)]_+.
\]
To avoid unnecessary complication due to the subgradient of $[\cdot]_+$, following~\cite{nesterov2005smooth}, we introduce a smoothing term $H(\lambda/\lambda_0)$, where $H(p)= -p\ln p - (1-p)\ln(1-p)$ is the entropy function, into the Lagrangian function, leading to the optimization problem $\min\limits_{\x \in \B} \mathcal{F}(\x)$, where $\mathcal{F}(\x)$ is defined as
\begin{eqnarray*}
\begin{aligned}
 \mathcal{F}(\x)& = f(\x) + \max_{0\leq \lambda\leq \lambda_0}\lambda g(\x) + \gamma H(\lambda/\lambda_0)= f(\x) + \gamma \ln \left(1+\exp\left(\frac{\lambda_0g(\x)}{\gamma}\right)\right),
\end{aligned}
\end{eqnarray*}
where $\gamma > 0$ is a parameter whose value will be determined later. Given the smoothed objective function $\mathcal{F}(\x)$, we find the optimal solution by applying SGD to minimize $\mathcal{F}(\x)$, where the gradient of $\mathcal{F}(\x)$ is computed by
\begin{eqnarray}
\begin{aligned}\label{eqn:grad}
\nabla \mathcal{F}(\x) = \nabla f(\x)  + \frac{\exp\left({\lambda_0g(\x)/\gamma}\right)}{1+\exp\left({\lambda_0g(\x)/\gamma}\right)}\lambda_0\nabla g(\x).
\end{aligned}
\end{eqnarray}
Algorithm~\ref{alg:2-sgd-one-prj} gives the detailed steps. Unlike Algorithm~\ref{alg:1}, only the primal variable $\x$ is updated in each iteration using the stochastic gradient computed in~(\ref{eqn:grad}). 
%------------------------------------------------------------------------------------
\begin{algorithm}[t]
\caption{SGD with ONE Projection by a Smoothing Technique (\texttt{SGD-ST})} \label{alg:2-sgd-one-prj}
\begin{algorithmic}[1]
\STATE \textbf{Input}: a sequence of step sizes $\{\eta_t\}$, $\lambda_0$, and $\gamma$

\STATE \textbf{Initialize}: $\x_1 = \mathbf{0}$.

\FOR{$t = 1, \ldots, T$}
    \STATE Compute $$\x'_{t+1} = \x_t - \eta_t \left(\displaystyle\mathbf{g}(\x_t, \xi_t) + \frac{\exp\left({\lambda_0g(\x_t)/\gamma}\right)}{1+\exp(\lambda_0g(\x_t)/\gamma)}\lambda_0\nabla g(\x_t)\right)$$
    \STATE Update $\x_{t+1} = \x'_{t+1}/\max(\|\x'_{t+1}\|_2,1)$
\ENDFOR
\STATE \textbf{Return:} $\widetilde \x_T=\Pi_{\K}\left(\widehat\x_T\right)$, where $\widehat\x_T = \sum_{t=1}^T \x_t/T$.
\end{algorithmic}
\end{algorithm}
%------------------------------------------------------------------------------------
The following theorem shows that Algorithm~\ref{alg:2-sgd-one-prj} achieves an $O(\ln T/ T)$ convergence rate if the cost functions are strongly convex. 
\begin{theorem} \label{thm:2}
For any $\beta$-strongly convex function $f(\x)$, if we set $\eta_t=1/(2\beta t), t=1,\ldots, T$,  $\gamma = \ln T /T$, and $\lambda_0 >G_1/\rho$ in Algorithm~\ref{alg:2-sgd-one-prj}, under assumptions {\textbf{A3}} and {\textbf{A4}},  we have with a probability at least $1 - \delta$,
\[
    f(\widetilde \x_T)  \leq \min_{\x\in\K}f(\x)+O\left(\frac{\ln T}{T}\right),
\]
where $O(\cdot)$ suppresses polynomial factors that depend on $\ln(1/\delta)$, $1/\beta, G_1, G_2, \rho$,  and $\lambda_0$.
\end{theorem}

It is well known that the optimal convergence rate of SGD for strongly convex optimization is $O(1/T)$~\cite{hazan-2011-beyond} which has been proven to be tight in stochastic optimization setting~\cite{sgd-lower-bounds}. According to Theorem~\ref{thm:2}, Algorithm~\ref{alg:2-sgd-one-prj} achieves an almost optimal convergence rate except for the factor of $\ln T$. It is worth mentioning that although it is not explicitly given in Theorem~\ref{thm:2}, the detailed expression for the convergence rate of Algorithm~\ref{alg:2-sgd-one-prj} exhibits a tradeoff in setting $\lambda_0$ (more can be found in the proof of Theorem~\ref{thm:2}).  Finally, under assumptions {\textbf{A1}}-{\textbf{A3}}, Algorithm~\ref{alg:2-sgd-one-prj} can achieve an $O(1/\sqrt{T})$ convergence rate for general convex functions, similar to Algorithm~\ref{alg:1}. 

\section{Analysis of Convergence Rate}
We here present  the proofs of main theorems. The omitted proofs of some results are deferred to  Section~\ref{sec-proofs-convergence}. 

\subsection{Convergence Rate  for General Convex Functions}
To pave the path for the proof of of Theorem~\ref{thm:1}, we present a series of lemmas. The lemma below states two key inequalities for the proof, which follows the standard analysis of gradient descent.

\begin{lemma}\label{lem:1} Under the bounded assumptions in~(\ref{eqn:bound1}) and~(\ref{eqn:bound2}), for any $\x\in\B$ and $\lambda>0$,  we have

\begin{eqnarray*}
\begin{aligned}\dd{\x_t-\x}{\nabla_\x \L(\x_t,\lambda_t)}&\le\frac{1}{2\eta_t}\left(\|\x-\x_t\|_2^2-\|\x-\x_{t+1}\|_2^2\right) + 2\eta_tG^2_1+\eta_tG_2^2\lambda_t^2\\
&\hspace*{0.2in}+ 2\eta_t\underbrace{\|\mathbf{g}(\x_t,\xi_t)- \nabla f(\x_t)\|^2_2}\limits_{\equiv \Delta_t}+\underbrace{\dd{\x-\x_t}{ \mathbf{g}(\x_t,\xi_t)-\nabla f(\x_t)}}\limits_{\equiv \zeta_t(\x)},\\
(\lambda-\lambda_t)\nabla_\lambda \L(\x_t, \lambda_t)&\leq \frac{1}{2\eta_t}\left(|\lambda-\lambda_t|^2 - |\lambda-\lambda_{t+1}|^2\right) + 2\eta_tC_2^2.
\end{aligned}
\end{eqnarray*}

\end{lemma}
An immediate result of  Lemma~\ref{lem:1} is the following  which states a regret-type bound.
\begin{lemma}\label{lem:2}
For any general convex function $f(\x)$, if we set $\eta_t=\gamma/(2G_2^2), t=1,\cdots, T$, we have

\begin{eqnarray}
\begin{aligned}
&\sum_{t=1}^{T}{f(\x_t)-f(\x_*)}+ \frac{[\sum_{t=1}^Tg(\x_t)]^2_+}{2(\gamma T+2G_2^2/\gamma)}\leq \frac{G_2^2}{\gamma}+ \frac{(G_1^2+C_2^2)}{G_2^2} \gamma T + \frac{\gamma}{G_2^2}\sum_{t=1}^T\Delta_t+\sum_{t=1}^T\zeta_t(\x_*),
\end{aligned}
\end{eqnarray}

where $\x_*=\arg\min_{\x\in\K}f(\x)$.
\end{lemma}
\vspace{-3mm}
\begin{proof}[Proof of Theorem~\ref{thm:1}]
First, by martingale inequality (e.g., Lemma 4 in~\cite{lan2012optimal}), with a probability $1-\delta/2$, we have
$\sum_{t=1}^T\zeta_t(\x_*)\leq 2\sigma\sqrt{3\ln(2/\delta)}\sqrt{T}$.
By Markov's inequality (Lemma~\ref{lemma:markov} in Appendix~\ref{chap:appendix-concen}),  with a probability $1-\delta/2$, we have $\sum_{t=1}^T\Delta_t\leq (1+\ln(2/\delta))\sigma^2T$.
Substituting these  inequalities into Lemma \ref{lem:2},  plugging the stated value of $\gamma$, and using  $O(\cdot)$ notation for ease of exposition, we have with a probability $1-\delta$
\[
\sum_{t=1}^{T}{f(\x_t) - f(\x_*)} +  \frac{1}{C\sqrt{T}} \big{[}\sum_{t=1}^Tg(\x_t)\big{]}^2_+\leq O(\sqrt{T}),
\]
where $C= 2G_2(1/\sqrt{G_1^2+C_2^2+(1+\ln(2/\delta))\sigma^2} + 2\sqrt{G_1^2+C_2^2+(1+\ln(2/\delta))\sigma^2})$ and $O(\cdot)$ suppresses polynomial  factors that depend on $\ln(2/\delta), G_1, G_2, C_2, \sigma$. 

Recalling the definition of $\widehat \x_T=\sum_{t=1}^T\x_t/T$ and using the convexity of $f(\x)$ and $g(\x)$, we have

\begin{eqnarray}
\begin{aligned}\label{eqn:key}
f(\widehat\x_T) -f(\x_*) +\frac{\sqrt{T}}{C}[g(\widehat\x_T)]_+^2\leq O\left(\frac{1}{\sqrt{T}}\right).
\end{aligned}
\end{eqnarray}

Assume $g(\widehat\x_T)>0$, otherwise $\widetilde \x_T=\widehat\x_T$ and we easily have $f(\widetilde \x_T)\leq \min_{\x\in\K}f(\x) + O(1/\sqrt{T})$.  Since $\widetilde\x_T$ is the projection of $\widehat \x_T$ into $\K$, i.e.,  $\widetilde \x_T=\arg\min_{g(\x)\leq 0}\|\x-\widehat \x_T\|_2^2$, then by first order optimality condition, there exists a positive constant $s>0$ such that 
\begin{eqnarray*}
\begin{aligned}
g(\widetilde\x_T)=0,\text{ and }  \widehat\x_T-\widetilde\x_T= s \nabla g(\widetilde\x_T)
\end{aligned}
\end{eqnarray*}
which indicates that  $\widehat  \x_T - \widetilde \x_T$ is in the same direction to $\nabla g(\tilde \x_T)$. Hence,

\begin{equation}
\begin{aligned}
g(\widehat \x_T) &= g(\widehat\x_T) - g(\widetilde\x_T)\geq \dd{\widehat\x_T-\widetilde\x_T}{\nabla g(\widetilde \x_T)} \\
& = \|\widehat\x_T-\widetilde\x_T\|_2\|\nabla g(\widetilde\x_T)\|_2 \\
& \geq \rho \|\widehat\x_T-\widetilde\x_T\|_2\label{eqn:g},
\end{aligned}
\end{equation}
where the last inequality follows  the definition of $\min_{g(\x)=0}\|\nabla g(\x)\|_2\geq\rho$. Additionally, we have
\begin{equation}
\begin{aligned}\label{eqn:f}
f(\x_*)- f(\widehat \x_T)\leq f(\x_*) - f(\widetilde\x_T) + f(\widetilde\x_T) - f(\widehat\x_T)\leq G_1\|\widehat\x_T-\widetilde\x_T\|_2,
\end{aligned}
\end{equation}

due to $f(\x_*)\leq f(\widetilde\x_T)$ and Lipschitz continuity of $f(\x)$.  Combining inequalities~(\ref{eqn:key}),~(\ref{eqn:g}), and~(\ref{eqn:f}) yields
\[
\frac{\rho^2}{C}\sqrt{T} \|\widehat\x_T-\widetilde\x_T\|_2^2\leq O(1/\sqrt{T}) + G_1 \|\widehat\x_T-\widetilde\x_T\|_2.
\]
By simple algebra, we have
 $\|\widehat\x_T-\widetilde\x_T\|_2 \leq  \frac{G_1C}{\rho^2\sqrt{T}} + O\left(\sqrt{\frac{C}{\rho^2 T}}\right)$
Therefore

\begin{equation}
\begin{aligned}
f(\widetilde\x_T) & \leq f(\widetilde\x_T)  -  f(\widehat\x_T) + f(\widehat\x_T) \\
&  \leq G_1\|\widehat\x_T-\widetilde\x_T\|_2 + f(\x_*)  + O\left(\frac{1}{\sqrt{T}}\right) \\
& \leq f(\x_*) + O\left(\frac{1}{\sqrt{T}}\right),
\end{aligned}
\end{equation}

where we use the inequality in~(\ref{eqn:key}) to bound $f(\widehat\x_T)$  by $f(\x_*)$ and absorb the dependence on $\rho, G_1, C$ into the $O(\cdot)$ notation. %which gives the bound stated in Theorem~\ref{thm:1}.
\end{proof}
%\vspace{-2mm}
\begin{remark} 
From the proof of Theorem~\ref{thm:1}, we can see that the key inequalities are~(\ref{eqn:key}),~(\ref{eqn:g}), and~(\ref{eqn:f}).  In particular, the regret-type bound in~(\ref{eqn:key}) depends  on the algorithm. If we only update the primal variable using the penalized objective in~(\ref{eqn:pen}), whose gradient depends on $1/\gamma$, it will cause a blowup in the regret bound with ($1/\gamma +\gamma T + T/\gamma$), which leads to a non-convergent bound. 
\end{remark}

\subsection{Convergence Rate for  Strongly Convex Functions}
%\subsection{Proof of Theorem~\ref{thm:2}}
Our proof of Theorem~\ref{thm:2} for the convergence rate of Algorithm~\ref{alg:2-sgd-one-prj} when applied to strongly convex functions starts with the following lemma  by analogy of Lemma~\ref{lem:2}.
\begin{lemma}\label{lem:3}
For any $\beta$-strongly convex function $f(\x)$, if we set $\eta_t=1/(2\beta t)$, we have

\begin{eqnarray*}
\begin{aligned}
\sum_{t=1}^{T}{\mathcal{F}(\x) - \mathcal{F}(\x_*)} \leq \frac
{(G_1^2+\lambda_0^2G_2^2)(1+\ln T)}{2\beta}+{\sum_{t=1}^T\zeta_t(\x_*)-\frac{\beta}{4}\sum_{t=1}^T\|\x_*-\x_t\|_2^2}
\end{aligned}
\end{eqnarray*}

where $\x_*=\arg\min_{\x\in\K}f(\x)$.
\end{lemma}
In order to prove Theorem~\ref{thm:2},  we need  the following  result for an improved martingale inequality.
\begin{lemma}\label{lem:4}
For any fixed $\x \in \B$, define $D_T= \sum_{t=1}^{T} \|\x_t - \x\|_2^2$, $\Lambda_T =\sum_{t=1}^T\zeta_t(\x)$, and $m = \lceil \log_2 T \rceil$. We have
\[
\Pr\left( D_T \leq \frac{4}{T} \right) + \Pr\left( \Lambda_T \leq 4G_1\sqrt{D_T\ln\frac{m}{\delta}} + 4G_1\ln\frac{m}{\delta}\right) \geq 1 - \delta.
\]
%where $m = \lceil \log_2 T \rceil$.
\end{lemma}
%\vspace{-4mm}
\begin{proof}[Proof of Theorem~\ref{thm:2}]
We substitute the bound in Lemma~\ref{lem:4} into the inequality in Lemma~\ref{lem:3} with $\x=\x_*$. We consider two cases. In the first case, we assume $D_T\leq 4/T$. As a result, we have
\begin{eqnarray*}
\sum_{t=1}^T\zeta_t(\x_*)=\sum_{t=1}^{T} \dd{\nabla f(\x_t) - \mathbf{g}(\x_t,\xi_t)}{\x_*-\x_t} \leq 2G_1\sqrt{TD_T} \leq 4G_1,
\end{eqnarray*}
which together with the inequality in Lemma~\ref{lem:3} leads to the bound
\[
    \sum_{t=1}^{T}{\mathcal{F}(\x_t) - \mathcal{F}(\x_*)} \leq 4G_1 +  \frac{(G_1^2+\lambda_0^2G_2^2)(1+\ln T)}{2\beta}.
\]
In the second case, we assume
\[
 \sum_{t=1}^T\zeta_t(\x_*) \leq 4G_1\sqrt{D_T\ln\frac{m}{\delta}} + 4G_1\ln\frac{m}{\delta} \leq \frac{\beta}{4}D_T + \left(\frac{16G_1^2}{\beta} + 4G_1\right)\ln\frac{m}{\delta},
\]
where the last step uses the fact $2\sqrt{ab} \leq a^2 + b^2$. We thus have
\[
 \sum_{t=1}^T\mathcal{F}(\x_t) - \mathcal{F}(\x_*) \leq   \left(\frac{16G_1^2}{\beta} + 4G_1\right)\ln\frac{m}{\delta} + \frac{(G_1^2+\lambda_0^2G_2^2)(1+\ln T)}{2\beta}.
\]
Combing the results of the two cases, we have, with a probability $1 - \delta$,

\begin{eqnarray}
\begin{aligned}
 \sum_{t=1}^T \mathcal{F}(\x_t) - \mathcal{F}(\x_*)&\leq   \underbrace{\left(\frac{16G_1^2}{\beta} + 4G_1\right)\ln\frac{m}{\delta} + 4G_1 +  \frac{(G_1^2+\lambda_0^2G_2^2)(1+\ln T)}{2\beta}}\limits_{= O(\ln T)}.
\end{aligned}
\end{eqnarray}

By convexity of $\mathcal{F}(\x)$, we have
$\mathcal{F}(\widehat\x_T) \leq \mathcal{F}(\x_*) + O\left(\ln T/T\right)$.
Noting that $\x_*\in \K$, $g(\x_*)\leq 0$, we have $\mathcal{F}(\x_*)\leq f(\x_*) + \gamma \ln 2$. On the other hand,
\[
\mathcal{F}(\widehat\x_T)= f(\widehat\x_T) + \gamma \ln \left(1+\exp\left(\frac{\lambda_0g(\widehat\x_T)}{\gamma}\right)\right)\geq f(\widehat\x_T) + \max\left(0, \lambda_0 g(\widehat\x_T)\right).
\]
Therefore, with the value of $\gamma =\ln T/T$,  we have 

\begin{eqnarray}
\begin{aligned}
&f(\widehat\x_T)  \leq f(\x_*) +  O\left(\frac{\ln T}{T}\right),\label{eqn:r1}\\
&f(\widehat\x_T) + \lambda_0 g(\widehat\x_T) \leq f(\x_*)+ O\left(\frac{\ln T}{T}\right).\label{eqn:r2}
\end{aligned}
\end{eqnarray}

Applying the inequalities~(\ref{eqn:g}) and~(\ref{eqn:f}) to~(\ref{eqn:r2}), and noting that $\gamma = \ln T/T$, we have
\[
\lambda_0 \rho \|\widehat\x_T  -\widetilde\x_T\|_2 \leq  G_1\|\widehat\x_T  -\widetilde\x_T\|_2 +  O\left(\frac{\ln T}{T}\right).
\]
For $\lambda_0> G_1/\rho$, we have $\|\widehat\x_T  -\widetilde\x_T\|_2\leq (1/(\lambda_0\rho - G_1)) O(\ln T/T)$.
%\[
%\|\widehat\x_T  -\widetilde\x_T\|_2\leq \frac{1}{\lambda_0\rho - G_1} O\left(\frac{\ln T}{T}\right).
%\]
Therefore
\begin{equation*}
\begin{aligned}
f(\widetilde\x_T)  \leq  & f(\widetilde\x_T) - f(\widehat\x_T) + f(\widehat\x_T) \\
\leq & G_1\|\widehat\x_T-\widetilde\x_T\|_2  +   f(\x_*)  +  O\left(\frac{\ln T}{{T}}\right) \\
\leq& f(\x_*) + O\left(\frac{\ln T}{{T}}\right),
\end{aligned}
\end{equation*}

%\[
%f(\widetilde\x_T) \leq  f(\widetilde\x_T) - f(\widehat\x_T) + f(\widehat\x_T) \leq G_1\|\widehat\x_T-\widetilde\x_T\|_2 + f(\x_*)  +  O\left(\frac{\ln T}{{T}}\right) \leq f(\x_*) + O\left(\frac{\ln T}{{T}}\right),
%\]
where in the second inequality we use inequality~(\ref{eqn:r1}).
\end{proof}

%%%%%%%%%%%%%%%%%%%%%%%%%%%%%%%%%%%%%%%%%%%%%%%%%%

\section{Online Optimization with Soft Constraints}
 
We now turn to making online learning algorithms more efficient by excluding the projection steps. To this end,  we consider an alternative online learning problem in which, instead of requiring that  each solution obeys the constraints, we only require the constraints, which define the convex domain $\W$,  to be satisfied in a long run. Then, the online learning problem becomes a task to find a sequence of solutions  under the  long term constraints.   We present and analyze an online gradient descent based algorithm for online convex optimization  with soft constraints. 

%We first show an impossibility theorem on the performance of a simple penalization idea. Then,  we describe an algorithm which is allowed to violate the constraints and then, by applying a simple trick, we propose a variant of this   algorithm which exactly satisfies the constraints in the long run. 

To facilitate our analysis, similar to the stochastic setting described in Section~\ref{sec-one-sgd-setup}, we assume that the domain $\W$ can be written as an intersection of a finite number of  convex constraints, that is,
\[ \K = \{\x \in \R^d: g_i(\x) \leq 0, i \in [m]\}, \]
where $g_i(\w), i \in [m]$,  are Lipschitz continuous functions. We assume that $\K$  is a bounded domain, that is, there exist  constants $R > 0$ and $r < 1$ such that $\K \subseteq R \mathbb{B}$  and $r \mathbb{B} \subseteq \K$ where $\mathbb{B}$ denotes the unit $\ell_2$ ball centered at the origin.  

We focus on the problem of online convex optimization as  introduced in Chapter~\ref{chap-background}, in which the goal is to achieve a low regret with respect to a fixed decision on a sequence of adversarially chosen cost functions. The difference between the setting considered here and the general online convex optimization is that, in our setting, instead of requiring $\x_t \in \K$, or equivalently $g_i(\x_t) \leq 0, i \in [m]$, for all $t \in [T]$, we only require the constraints to be satisfied in the long run, namely $\sum_{t=1}^T g_i(\x_t) \leq 0, i \in [m]$.
Then, the problem becomes to find a sequence of solutions $\x_t, t \in [T]$ that minimizes the regret, under the  long term constraints $\sum_{t=1}^T g_i(\x_t) \leq 0, i \in [m]$. Formally, we would like to solve the following optimization problem online,
\begin{eqnarray}
\label{eqn:setting1one}
\min_{\x_1, \ldots, \x_T \in \B} \sum_{t=1}^T f_t(\x_t) - \min\limits_{\x \in \K} \sum_{t=1}^T f_t(\x) \;\; \text{s.t.} \;\;\sum_{t = 1}^{T}{g_i(\x_t)} \leq 0\; \text{, } i \in [m].
\end{eqnarray}

For simplicity, we will focus on a finite-horizon setting where the number of rounds $T$ is known in advance. This condition can be relaxed under certain conditions, using standard techniques (see, e.g., \cite{bianchi-2006-prediction}). Note that in (\ref{eqn:setting1one}), (i) the solutions come from the ball $\B \supseteq \K$ instead of $\K$  and (ii) the constraint functions are fixed and are given in advance.

\subsection{An Impossibility Theorem}

Before we state our formulation and algorithms, let us review a few alternative techniques that do not need explicit projection. A straightforward approach is  to introduce an appropriate self-concordant barrier function for the given convex set $\K$ and add it to the objective function such that the barrier diverges at the boundary of the set. Then we can  interpret the resulting optimization problem, on the modified objective functions, as an unconstrained minimization problem that can  be solved without projection steps. Following the analysis in \cite{interior-ieee-2012}, with an appropriately designed procedure for  updating solutions, we could guarantee a regret bound of $O(\sqrt{T})$ without the violation of constraints. A similar idea is used in \cite{AbernethyHR08} for online bandit learning and in \cite{DBLP:conf/nips/NarayananR10} for a random walk approach for regret minimization which, in fact,  translates the issue of projection into the difficulty of sampling. Even for linear Lipschitz cost functions, the random walk approach requires sampling from a Gaussian distribution with covariance given by the Hessian of the self-concordant barrier of the convex set $\K$ that has the same  time complexity as inverting a  matrix. The main limitation with these approaches is that they require computing the Hessian matrix of the objective function in order to guarantee that the updated solution stays within the given domain $\K$. This limitation makes it computationally unattractive when dealing with high dimensional data. In addition, except for  well known cases, it is often unclear  how to efficiently construct a self-concordant barrier function for a general convex domain.

An alternative approach for online convex optimization with long term constraints is to introduce a penalty term in the loss function that penalizes the violation of constraints. More specifically, we can define a new loss function $\fh_t(\w)$ as
\begin{eqnarray}
\label{eqn:fhatone}
    \fh_t(\x) = f_t(\x) + \delta\sum_{i=1}^{m} [g_i(\x)]_+,
\end{eqnarray}
where $[z]_+ = \max(0, z)$  and $\delta > 0$ is a fixed positive constant used to penalize the violation of constraints. We then run the standard OGD algorithm from Chapter~\ref{chap-background} to minimize the modified loss function $\fh_t(\w)$. The following theorem shows that this simple strategy fails to achieve sub-linear bound for both  regret and the long term violation of constraints at the same time.

\begin{theorem} \label{thm:aone}
Given $\delta > 0$, there always exists a sequence of loss functions $f_1, f_2, \cdots, f_T$ and a constraint function $g(\x)$ such that either $\sum_{t=1}^T f_t(\x_t) - \min_{g(\x) \leq 0} \sum_{t=1}^T f_t(\x) = O(T)$ or $\sum_{t=1}^T [g(\x_t)]_+ = O(T)$ holds, where $\w_1, \w_2, \cdots, \w_T$ is the sequence of solutions generated by the OGD algorithm that minimizes the modified loss functions given in (\ref{eqn:fhatone}).
\end{theorem}

 \begin{proof}
 We first show that when $\delta < 1$, there exists a loss function and a constraint function such that the violation of constraint is linear in $T$. To see this, we set $f_t(\x) = \dd{\x}{\mathbf{x}}, t \in [T]$ and $g(\x) = 1 - \dd{\x}{\mathbf{x}}$. Assume we start with an infeasible solution, that is, $g(\x_1) > 0$ or $\dd{\x_1}{\mathbf{x}} < 1$. Given the solution $\x_t$ obtained at $t$th trial, using the standard gradient descent approach, we have $\x_{t+1} = \x_t - \eta (1 - \delta)\mathbf{x}$. Hence, if $\dd{\x_t}{\mathbf{x}} < 1$, since we have $\dd{\x_{t+1}}{\mathbf{x}} < \dd{\x_t}{\mathbf{x}} < 1$,  if we start with an infeasible solution, all the solutions obtained over the trails will violate the constraint $g(\x) \leq 0$, leading to a linear number of violation of constraints. Based on this analysis, we assume $\delta > 1$ in the analysis below.

Given a strongly convex loss function $f(\x)$ with modulus $\gamma$, we consider a constrained optimization problem given by
\begin{eqnarray*}
    \min\limits_{g(\x) \leq 0} f(\x), \label{eqn:constrain-opt}
\end{eqnarray*}
which is equivalent to the following unconstrained optimization problem
\[
    \min\limits_{\x} f(\x) + \lambda [g(\x)]_+,
\]
where $\lambda \geq 0$ is the Lagrangian multiplier. Since we can always scale $f(\x)$ to make $\lambda \leq 1/2$, it is safe to assume $\lambda \leq 1/2 < \delta$. Let $\x_*$ and $\x_a$ be the optimal solutions to the constrained optimization problems $\arg\min_{g(\x) \leq 0} f(\x)$ and $\arg\min\limits_{\x} f(\x) + \delta [g(\x)]_+$, respectively. We choose $f(\x)$ such that $\|\nabla f(\x_*)\| > 0$, which leads to $\x_a \neq \x_*$. This holds because according to the first order optimality condition, we have
\[
    \nabla f(\x_*) = -\lambda\nabla g(\x_*), \; \nabla f(\x_a) = -\delta \nabla g(\x_*),
\]
and therefore $\nabla f(\x_*) \neq \nabla f(\x_a)$ when $\lambda < \delta$. Define $\Delta = f(\x_a) - f(\x_*)$. Since $\Delta \geq \gamma\|\x_a - \x_*\|^2/2$ due to the strong convexity of $f(\x)$, we have $\Delta > 0$.

Let $\{\x_t\}_{t=1}^T$ be the sequence of solutions generated by the OGD algorithm that minimizes the modified loss function $f(\x) + \delta [g(\x)]_+$. We have
\begin{eqnarray*}
&&\sum_{t=1}^T f(\x_t) + \delta[g(\x_t)]_+ \geq T\min\limits_{\x} f(\x) + \delta [g(\x)]_+ \\
& = & T(f(\x_a) + \delta[g(\x_a)]_+) \geq T(f(\x_a) + \lambda[g(\x_a)]_+) \\
& = & T(f(\x_*) + \lambda[g(\x_*)]_+) + T(f(\x_a) + \lambda[g(\x_a)]_+ - f(\x_*) - \lambda [g(\x_*)]) \\
& \geq & T\min\limits_{g(\x) \leq 0} f(\x) + T\Delta.
\end{eqnarray*}
As a result, we have
\[
    \sum_{t=1}^T f(\x_t) + \delta[g(\x_t)]_+ - \min\limits_{g(\x) \leq 0} f(\x) = O(T),
\]
implying that either the regret $\sum_{t=1}^T f(\x_t) - Tf(\x_*)$ or the violation of the constraints $\sum_{t=1}^T [g(\x)]_+$ is linear in $T$.
\end{proof}

To better understand the performance of penalty based approach, here we analyze the performance of the OGD in solving the  online optimization problem in (\ref{eqn:setting1one}). The algorithm is analyzed using the following lemma from \cite{DBLP:conf/icml/Zinkevich03} (see also Theorem~\ref{thm-chapter-2-ogd} in Chapter~\ref{chap-background}).

\begin {lemma} 
\label{lemma:basic-ogd}
Let $\x_1, \x_2, \ldots, \x_T \in \K$ be the sequence of solutions obtained by applying OGD on the sequence of bounded convex functions $f_1, f_2, \ldots, f_T$. Then, for any  solution $\x_* \in \K$ we have
\[ \sum_{t=1}^{T}{f_t(\x_t)} - \sum_{t=1}^{T}{f_t(\x_*)} \leq \frac{R^2}{2\eta} + \frac{\eta}{2} \sum_{t=1}^{T}{\| \nabla f_t(\x_t)\|^2}. \]
\end{lemma}

We apply OGD to  functions $\fh_t(\x), \; t \in [T]$ defined in (\ref{eqn:fhatone}), that is, instead of updating the solution based on the gradient of $f_t(\x)$, we update the solution by the gradient of $\fh_t(\x)$. Using Lemma \ref{lemma:basic-ogd}, by expanding the functions $\fh_t(\x)$ based on (\ref{eqn:fhatone}) and considering the fact that $\sum_{i=1}^{m}{[g_i(\x_*)]_{+}^2} = 0$, we get
\begin{eqnarray}
\label{eqn:basic-inq}
\sum_{t=1}^{T}{f_t(\x_t)} - \sum_{t=1}^{T}{f_t(\x_*)} + \frac{\delta}{2} \sum_{t=1}^{T}{\sum_{i=1}^{m}{[g_i(\x)]_{+}^2}} \leq \frac{R^2}{2\eta} + \frac{\eta}{2} \sum_{t=1}^{T}{\| \nabla \fh_t(\x_t)\|^2}. 
\end{eqnarray}
From the definition of $\fh_t(\x)$, the norm of the gradient $ \nabla \fh_t(\x_t)$ is bounded as follows
\begin{eqnarray}
\label{eqn:gupper}
\| \nabla \fh_t(\x)\|^2 = \| \nabla f_t(\x) + \delta \sum_{i = 1}^{m}{[g_i(\x)]_+ \nabla g_i(\x)}\|^2  
  \leq  2 G^2 (1+m\delta^2 D^2),
\end{eqnarray}
where the inequality holds because $(a_1+a_2)^2 \leq 2(a_1^2+a_2^2)$. By substituting (\ref{eqn:gupper}) into the (\ref{eqn:basic-inq}) we have:
\begin{eqnarray}
\label{eqn:5one}
\sum_{t=1}^{T}{f_t(\x_t)} - \sum_{t=1}^{T}{f_t(\x_*)} + \frac{\delta}{2} \sum_{t=1}^{T}{\sum_{i=1}^{m}{[g_i(\x_t)]_{+}^2}} \leq \frac{R^2}{2\eta} + \eta G^2 (1+m\delta^2 D^2) T.
\end{eqnarray}
Since $[\cdot]_{+}^2$ is a convex function, from Jensen's inequality and following the fact that 
$\sum_{t=1}^T f_t(\x_t) - f_t(\x_*) \geq - FT$, we have:
 \begin{eqnarray*}
\frac{\delta}{2T}  \sum_{i=1}^{m}{ \left [ \sum_{t=1}^{T}{g_i(\x_t)} \right]_{+}^2 } \leq \frac{\delta}{2}\sum_{i=1}^{m}{ \sum_{t=1}^{T}{[g_i(\x_t)]_{+}^2}}  \leq \frac{R^2}{2\eta} + \eta G^2 (1+m\delta^2 D^2) T + FT.
 \end{eqnarray*}
By minimizing the right hand side of (\ref{eqn:5one}) with respect to $\eta$, we get the regret bound as
\begin{eqnarray}
\label{eqn:penalty-regret}
 \sum_{t=1}^{T}{f_t(\x_t)} - \sum_{t=1}^{T}{f_t(\x_*)} \leq RG \sqrt{2(1+m\delta^2D^2)T} = \O(\delta \sqrt{T})
 \end{eqnarray}
and the bound for the violation of constraints as
\begin{eqnarray}
\label{eqn:violation-bound}
\sum_{t=1}^{T}{g_i(\x_t)}  \leq \sqrt{ \left( \frac{R^2}{2\eta} + \eta G^2 (1+m\delta^2 D^2) T + FT \right) \frac{2T}{ \delta} } = \O (T^{1/4} \delta^{1/2} + T \delta^{-1/2}).
 \end{eqnarray}
By examining the bounds obtained in (\ref{eqn:penalty-regret}) and (\ref{eqn:violation-bound}), it turns out that in order to recover $\O(\sqrt{T})$ regret bound, we need to set $\delta$ to be a constant, leading to  $\O(T)$ bound for the violation of constraints in the long run, which is not satisfactory at all.  The analysis shows that in order to obtain $O(\sqrt{T})$ regret bound, linear bound on the long term violation of the constraints is unavoidable. The main reason for the failure of using  modified loss function in (\ref{eqn:fhatone}) is that the weight constant $\delta$ is fixed and independent from the sequence of solutions obtained so far. In the next subsection, we present an online 
convex-concave formulation for online convex optimization with long term constraints, which explicitly addresses the limitation of (\ref{eqn:fhatone}) by automatically adjusting the weight constant based on the violation of the solutions obtained so far.

As mentioned before, our general strategy is to turn online convex optimization with long term constraints into a convex-concave optimization problem. Instead of generating a sequence of solutions that satisfies the long term constraints, we first consider an online optimization strategy that allows  the violation of constraints on some rounds in a controlled way. We then modify the online optimization strategy to obtain a sequence of solutions that obeys the long term constraints. Although the online convex optimization with long term constraints is clearly easier than the standard online convex optimization problem, it is straightforward to see that optimal regret bound for online optimization with long term constraints should be on the order of $\O(\sqrt{T})$, no better than the standard online convex optimization problem.

\subsection[An Online Algorithm with Vanishing Violation of Constraints]{An Efficient Algorithm with $\O(\sqrt{T})$ Regret Bound and $\O(T^{3/4})$ Bound on the Violation of Constraints}
The intuition behind  our approach stems from the observation that the constrained optimization problem $\min_{\x \in \K} \sum_{t=1}^T f_t(\x)$ is equivalent to the following convex-concave optimization problem
\begin{eqnarray}
\label{eqn:saddle}
    \min\limits_{\x \in \B} \max\limits_{\lambda \in \R_+^m} \sum_{t=1}^T f_t(\x) + \sum_{i=1}^m \lambda_i g_i(\x), \label{eqn:convex-concave}
\end{eqnarray}
where  $\blambda = (\lambda_1, \ldots, \lambda_m)^{\top}$ is the vector of Lagrangian multipliers  associated with the constraints $g_i(\cdot), i = 1, \ldots, m$ and belongs to the nonnegative orthant $\R_+^m$.  To solve the online convex-concave optimization problem, we extend the gradient based approach for variational inequality  \cite{nemirovski-efficient} to (\ref{eqn:convex-concave}). To this end, we consider the following regularized convex-concave function  as
\begin{eqnarray}\label{eqn:cox-con-func}
    \L_t(\x, \blambda) = f_t(\x) + \sum_{i=1}^m \left\{\lambda_i g_i(\x) - \frac{\delta \eta}{2} \lambda_i^2\right\},
\end{eqnarray}
where $\delta > 0$  is a constant whose value will be decided by the analysis. Note that in~ (\ref{eqn:cox-con-func}), we introduce a regularizer $\delta \eta \lambda_i^2/2$ to prevent $\lambda_i$ from being too large. This is because, when $\lambda_i$ is large, we may encounter a large gradient for $\x$ because of $\nabla_{\x}\L_t(\x, \blambda) \propto \sum_{i=1}^m \lambda_i \nabla g_i(\x)$, leading to unstable solutions and a poor regret bound. Although we can achieve the same goal by restricting $\lambda_i$ to a bounded domain, using the quadratic regularizer makes it convenient for our analysis. We also note that the convex-concave function defined in~(\ref{eqn:cox-con-func}) is different from the stochastic setting as it changes over the iterations, while in~\ref{eqn-9-fixed-convex-concave} the function is fixed.

Algorithm~\ref{alg:1-soft} shows the detailed steps of the proposed algorithm. Unlike standard online convex optimization algorithms that only update $\x$, Algorithm~\ref{alg:1-soft} updates both $\x$ and $\blambda$. In addition, unlike the modified loss function in (\ref{eqn:fhatone}) where the weights for constraints $\{g_i(\x) \leq 0\}_{i=1}^m$ are fixed, Algorithm~\ref{alg:1-soft} automatically adjusts the weights $\{\lambda_i\}_{i=1}^m$ based on $\{g_i(\x)\}_{i=1}^m$, the violation of constraints, as the game proceeds. It is this property that allows Algorithm~\ref{alg:1-soft} to achieve sub-linear bound for both  regret and  the violation of  constraints. 

To analyze Algorithm \ref{alg:1-soft}, we first state the following lemma, the key to the main theorem on the regret bound and the violation of constraints.

\begin{lemma}
\label{lemma:basic-ineq}
Let $\L_t(\cdot, \cdot)$ be the function defined in (\ref{eqn:cox-con-func}) which is convex in its first argument and concave in its second argument. Then for any $(\x, \blambda)  \in \B \times  \R_+^m$ we have
\begin{eqnarray*}
\begin{aligned}
  \L_t(\x_t, \blambda) - \L_t(\x, \blambda_t) & \leq  \frac{1}{2\eta} (\|\x - \x_t\|^2 + \|\blambda - \blambda_t\|^2 - \|\x - \x_{t+1}\|^2 - \|\blambda - \blambda_{t+1}\|^2) \\ & \hspace{5mm}+ \frac{\eta}{2} ( \|\nabla_{\x}\L_t(\x_t, \blambda_t)\|^2 + \|\nabla_{\blambda}\L_t(\x_t, \blambda_t)\|^2 ).
\end{aligned}
\end{eqnarray*}

\end{lemma}
\begin{proof}Following the analysis of \cite{DBLP:conf/icml/Zinkevich03}, convexity of $\L_t(\cdot, \blambda)$ implies that
\begin{eqnarray}
\L_t(\x_t, \blambda_t) - \L_t(\x, \blambda_t) \leq \dd{\x_{t} - \x}{\nabla_\x \L_t(\x_t, \blambda_t)}
\label{eqn:conv}
\end{eqnarray}
and by concavity of $\L_t( \x, \cdot)$ we have
\begin{eqnarray}
\L_t(\x_t, \blambda) - \L_t(\x_t, \blambda_t) \leq \dd{\blambda - \blambda_{t}}{\nabla_{\blambda} \L_t(\x_t, \blambda_t)}.
\label{eqn:conc}
\end{eqnarray}
Combining the inequalities (\ref{eqn:conv}) and (\ref{eqn:conc}) results in
\begin{eqnarray}
\L_t(\x_t, \blambda) - \L_t(\x, \blambda_t) \leq \dd{\x_{t} - \x}{\nabla_\x \L_t(\x_t, \blambda_t)} - \dd{\blambda - \blambda_{t}}{\nabla_{\blambda} \L_t(\x_t, \blambda_t)}.
\label{eqn:convconc}
\end{eqnarray}
Using the update rule for $\x_{t+1}$ in terms of $\x_t$ and expanding, we get
\begin{eqnarray}
\label{eqn:x-inq}
\|\x-\x_{t+1} \|^2 \leq \|\x-\x_t\|^2 - 2 \eta \dd{\x_{t} - \x}{\nabla_\x \L_t(\x_t, \blambda_t)} + \eta^2 \| \nabla_\x \L_t(\x_t, \blambda_t)\|^2,
\end{eqnarray}
where the first inequality follows from the nonexpansive property of the projection operation (see Lemma~\ref{lemma-app-a-non-exp}). Expanding the inequality for $\|\blambda - \blambda_{t+1} \|^2$ in terms of $\blambda_t$ and plugging back into the (\ref{eqn:convconc}) with (\ref{eqn:x-inq}) establishes the desired inequality.
\end{proof}
\begin{algorithm}[tb]
\caption{Online Gradient Descent  with Soft Constraints}
\begin{algorithmic}[1] \label{alg:1-soft}
    \STATE {\textbf{Input}}: 
    \begin{mylist}
    \item constraints $g_i(\x) \leq 0, i \in [m]$
    \item step size $\eta$
    \item constant $\delta > 0$
    \end{mylist}
    \STATE {\textbf{Initialize}}: $\x_1 = \mathbf{0}$ and $\blambda_1 = \mathbf{0}$
    \FOR{$t = 1, 2, \ldots, T$}
        \STATE Submit solution $\x_t$
        \STATE Receive the convex function $f_t(\cdot)$ and suffer loss $f_t(\x_t)$
        \STATE Compute the gradients as:
        \begin{eqnarray*}
&&        \nabla_\x \L_t(\x_t, \blambda_t) =\nabla f_t(\x_t) + \sum_{i=1}^m \lambda_t^i \nabla g_i(\x_t) \\ 
&&        \nabla_{\lambda_i}\L_t(\x_t, \blambda_t)  =  g_i(\x_t) - \eta \delta \lambda_t^i
                \end{eqnarray*}
        \STATE Update $\x_t$ and $\blambda_t$ by:
        \begin{eqnarray*}
        && \x_{t+1} = \Pi_{\B}\left(\x_t - \eta \nabla_\x \L_t(\x_t, \blambda_t)\right) \\  
        && \blambda_{t+1} = \Pi_{[0, +\infty)^m}(\blambda_t + \eta \nabla_{\blambda} \L_t(\x_t, \blambda_t))
        \end{eqnarray*}
    \ENDFOR
\end{algorithmic}
\end{algorithm}
\begin{proposition}
\label{prop:2}
 Let $\x_t$ and $\blambda_t, t \in [T]$ be the sequence of solutions obtained by Algorithm~\ref{alg:1-soft}. Then for any $\x \in \B$ and $\blambda  \in \R_+^m$, we have

\begin{eqnarray}
\begin{aligned}
\label{eqn:center}
& \lefteqn{\sum_{t=1}^T \L_t(\x_t, \blambda) - \L_t(\x, \blambda_t)}  \\
 & \leq \frac{R^2 + \|\blambda\|^2}{2\eta} + \frac{\eta T}{2 }\left((m+1)G^2 + 2mD^2\right) + \frac{\eta }{2} \left((m +1)G^2 + 2m \delta^2\eta^2\right)\sum_{t=1}^T \|\blambda_t\|^2. \nonumber
\end{aligned}
\end{eqnarray}

\end{proposition}
\begin{proof}
We first bound the gradient terms in the right hand side of  Lemma \ref{lemma:basic-ineq}.  Using the inequality $(a_1+a_2+ \ldots, a_n)^2 \leq n (a_1^2+a_2^2+\ldots+ a_n^2)$, we have $\|\nabla_{\x} \L_t(\x_t, \blambda_t)\|^2 \leq (m + 1) G^2 \left(1 + \|\blambda_t\|^2 \right)$ and $ \|\nabla_{\blambda}\L_t(\x_t, \blambda_t)\|^2 \leq 2m (D^2 + \delta^2\eta^2 \|\blambda_t\|^2)$. In Lemma \ref{lemma:basic-ineq}, by adding the inequalities of all iterations, and using the fact $\|\x\| \leq R$ we complete the proof.
\end{proof}
The following theorem bounds the regret and the violation of the constraints in the long run for Algorithm \ref{alg:1-soft}.
\begin{theorem}
\label{thm:ogd}
 Define $a = R\sqrt{(m+1)G^2 + 2mD^2}$. Set $\eta = R^2/[a\sqrt{T}]$. Assume $T$ is large enough such that $2\sqrt{2}\eta (m+1) \leq 1$. Choose $\delta$ such that $\delta \geq (m +1)G^2 + 2m \delta^2\eta^2$. Let $\x_1,\w_2, \cdots,  \w_T$ be the sequence of solutions obtained by Algorithm~\ref{alg:1-soft}. Then for the optimal solution $\x_* = \min_{\x \in \K} \sum_{t=1}^T f_t(\x)$ we have
\begin{eqnarray*}
&& \sum_{t=1}^T f_t(\x_t) - f_t(\x_*) \leq  a \sqrt{T} = \O(T^{1/2}), \;\text{and}\\  
&& \sum_{t=1}^T g_i(\x_t) \leq \sqrt{2\left(F T + a \sqrt{ T}\right)\sqrt{T}\left(\frac{\delta R^2}{a} + \frac{ma}{R^2}\right)} = \O(T^{3/4}).
\end{eqnarray*}
\end{theorem}
\begin{proof} We begin by expanding (\ref{eqn:center}) using (\ref{eqn:cox-con-func}) and rearranging the terms to get 

\begin{eqnarray*}
\begin{aligned}
\lefteqn{\sum_{t=1}^T \left[f_t(\x_t) - f_t(\x)\right] + \sum_{i=1}^m \left\{\lambda_i\sum_{t=1}^T g_i(\x_t) - \sum_{t=1}^T \lambda^i_t g_i(\x)\right\} - \frac{\delta\eta T}{2}\|\blambda\|^2} \\
& \leq -\frac{\delta\eta}{2}\sum_{t=1}^T\|\blambda_t\|^2 + \frac{R^2 + \|\blambda\|^2}{2\eta} + \frac{\eta T}{2}\left((m+1)G^2 + 2mD^2\right) \\ &\hspace{5mm}+ \frac{\eta }{2} \left((m +1)G^2 + 2m \delta^2\eta^2\right)\sum_{t=1}^T \|\blambda_t\|^2.
\end{aligned}
\end{eqnarray*}

Since $\delta \geq  (m +1)G^2 + 2m \delta^2\eta^2$, we can drop the $\|\blambda_t\|^2$ terms from both sides of the above inequality and obtain
\begin{eqnarray*}
\lefteqn{\sum_{t=1}^T \left[f_t(\x_t) - f_t(\x)\right] + \sum_{i=1}^m \left\{\lambda_i\sum_{t=1}^T g_i(\x_t) - \left(\frac{\delta\eta T}{2} + \frac{m}{2\eta}\right)\lambda_i^2\right\}}  \;\;\;\;\;\;\;\;\;\;\;\;\;\;\;\;\;\;\;\;\;\;\;\;\\
 & \leq &\sum_{i=1}^m \sum_{t=1}^T \lambda^i_t g_i(\x)  + \frac{R^2}{2\eta} + \frac{\eta T}{2}\left((m+1)G^2 + 2mD^2)\right). \!\!\!\!\!\!\!\!\!\!\!\!
\end{eqnarray*}
The left hand side of above inequality consists of two terms. The first term basically measures the difference between the cumulative loss of the Algorithm~\ref{alg:1-soft}  and the optimal solution and the second term includes the constraint functions with corresponding Lagrangian multipliers which will be used to bound the long term violation of the constraints. By taking  maximization for $\blambda$ over the range $(0, +\infty)$, we get
\begin{eqnarray*}
\sum_{t=1}^T \left[f_t(\x_t) - f_t(\x)\right] + \sum_{i=1}^m \left\{\frac{\left[\sum_{t=1}^T g_i(\x_t)\right]_+^2}{2(\delta \eta T + m/\eta)} - \sum_{t=1}^T \lambda^i_t g_i(\x)\right\}  \\ \leq  \frac{R^2}{2\eta} + \frac{\eta T}{2}\left((m+1)G^2 + 2mD^2)\right).
\end{eqnarray*}
Since $\x_* \in \K$, we have $g_i(\x_*) \leq 0, i \in [m]$, and the resulting inequality becomes
\begin{eqnarray*}
\sum_{t=1}^T f_t(\x_t) - f_t(\x_*) + \sum_{i=1}^m \frac{\Big{[}\sum_{t=1}^T g_i(\x_t)\Big{]}_+^2}{2(\delta \eta T + m/\eta)}  \leq  \frac{R^2}{2\eta} + \frac{\eta T}{2}\left((m+1)G^2 + 2mD^2)\right). 
\end{eqnarray*}
The statement of the first part of the theorem follows by using the expression for $\eta$. The second part is proved by  substituting the regret bound by its lower bound as  $\sum_{t=1}^T f_t(\x_t) - f_t(\x_*) \geq - FT$.
\end{proof}
\begin{remark} We observe that the introduction of quadratic regularizer $\delta\eta \|\blambda\|^2/2$ allows us to turn the expression $\lambda_i \sum_{t=1}^T g_i(\x_t)$ into $\left[\sum_{t=1}^T g_i(\x_t)\right]_+^2$, leading to the bound for the violation of the constraints. In addition, the quadratic regularizer defined in terms of $\blambda$ allows us to work with unbounded $\blambda$ because it cancels  the contribution of the $\|\blambda_t\|$ terms from the loss function and the bound on the gradients $\| \nabla_{\x}\L_t(\x, \blambda)\|$. 
Note that the constraint for $\delta$ mentioned in Theorem  \ref{thm:ogd} is equivalent to
\begin{eqnarray}
    \frac{2}{1/(m+1) + \sqrt{(m+1)^{-2} - 8G^2\eta^2}} \leq \delta \leq \frac{1/(m+1) + \sqrt{(m+1)^{-2} - 8 G^2\eta^2}}{4\eta^2}, \label{eqn:delta-condition}
\end{eqnarray}
from which, when $T$ is large enough (i.e., $\eta$ is small enough), we can simply set $\delta = 2 (m+1) G^2$ that will obey  the constraint in (\ref{eqn:delta-condition}).
\end{remark}
By investigating Lemma \ref{lemma:basic-ineq}, it turns out that the boundedness of the gradients is essential to obtain bounds for Algorithm \ref{alg:1-soft} in Theorem \ref{thm:ogd}. Although, at each iteration, $\blambda_t$ is projected onto the $\R_+^m$,  since $\K$ is a compact set and functions $f_t(\x)$ and $g_i(\x), i \in [m]$ are convex, the boundedness of the functions implies that  the gradients are bounded \cite[Proposition 4.2.3]{convex-analysis-book}.
%%%%%%%%%%%%%%%%%%%%%%%%%%%%%%%%%%%%%%%
\subsection [An Online Algorithm without Violation of Constraints]{An Efficient Algorithm with $\O(T^{3/4})$ Regret Bound and without Violation of Constraints}
In this subsection we generalize Algorithm \ref{alg:1-soft}  such that the constrained are satisfied in a long run.  To create a sequence of solutions $\{\x_t, t \in [T]\}$ that satisfies the long term constraints $\sum_{t=1}^T g_i(\x_t) \leq 0, i \in [m]$, we make two modifications to Algorithm \ref{alg:1-soft}. First, instead of handling all of the $m$ constraints, we consider a single constraint defined as $g(\x) = \max_{i \in [m]} g_i(\x)$. Apparently,  by achieving zero violation  for the constraint $g(\x) \leq 0 $, it is guaranteed that all of the constraints $g_i(\cdot), i \in [m]$ are also satisfied in the long term. Furthermore,  we change Algorithm \ref{alg:1-soft}   by modifying  the definition of $\L_t(\cdot, \cdot)$ as
\begin{eqnarray}
\label{eqn:alg2}
\L_t(\x, \lambda) = f_t(\x) +  \lambda (g(\x) + \gamma)- \frac{\eta \delta}{2} \lambda^2,
\end{eqnarray}
where $\gamma > 0$ will be decided later. This modification is equivalent to considering the constraint $g(\x) \leq -\gamma$, a tighter constraint than $g(\x) \leq 0$. The main idea behind this modification is that by using a tighter constraint in our algorithm, the resulting sequence of solutions will satisfy the long term constraint $\sum_{t=1}^T g(\x_t) \leq 0$, even though the tighter constraint is violated in many trials.

Before proceeding, we state a fact about the Lipschitz continuity of the function $g(\x)$ in the following proposition.
 \begin{proposition}
 \label{prop:lip}
Assume that functions $g_i(\cdot), i \in [m]$ are Lipschitz continuous with constant G. Then, function $g(\x) = \max_{i \in [m]}{ g_i(\x)}$ is Lipschitz continuous with constant $G$, that is,
\[ |g(\x) - g(\x')| \leq G \| \x - \x'\| \; \; \text{for any} \; \x \in \B \; \text{and}  \; \x' \in \B.\]
\end{proposition}
\begin{proof} See Appendix~\ref{chap:appendix-convex}. 
\end{proof}
%\begin{proof} First, we  rewrite  $g(\x) = \max_{i \in [m]} g_i(\x)$ as  $g(\x) =  \max_{\alpha \in \Delta_m} \sum_{i=1}^{m}{\alpha_i g_i(\x)}$ where $\Delta_m$ is the $m$-simplex, that is, $\Delta_m= \{ \alpha \in\R_{+}^m; \sum_{i=1}^{m}{\alpha_i} = 1\}$. Then, we have
%\begin{eqnarray*}
%|g(\x)-g(\x')| &=&  \left | \max_{\alpha \in \Delta_m} \sum_{i=1}^{m}{\alpha_i g_i(\x)} - \max_{\alpha \in \Delta_m} \sum_{i=1}^{m}{\alpha_i g_i(\x')}\right | \\
%& \leq& \max_{\alpha \in \Delta_m}  \left | \sum_{i=1}^{m}{\alpha_i g_i(\x)} -  \sum_{i=1}^{m}{\alpha_i g_i(\x')}\right | \\ &\leq& \max_{\alpha \in \Delta_m}  \sum_{i=1}^{m}{\alpha_i \left | g_i(\x) - g_i(\x')\right |}
% \leq G \|\x-\x'\|,
%\end{eqnarray*}
%where  the last inequality follows from the Lipschitz continuity of $g_i(\x), i \in [m]$.
%\end{proof}
To obtain a zero bound on the violation of constraints in the long run,  we make the following assumption about the constraint function $g(\x)$.

\begin{assumption} Let $\K' \subseteq \K$ be the convex set defined as  $\K' = \{\x \in \R^d: g(\x) + \gamma \leq 0\}$ where $\gamma \geq 0$. We assume that the norm of the gradient of the constraint function $g(\x)$ is lower bounded  at the boundary of $\K'$, that is,
$$\mbox{{\rm{\textbf{A5}}}}\quad\quad\quad\quad \min\limits_{g(\x) +\gamma = 0} \|\nabla g(\x)\| \geq \sigma.$$
\label{assumption:1}
\end{assumption}
A direct consequence of {assumption} \textbf{A5} is that by reducing the domain $\K$ to $\K'$, the optimal value of the constrained optimization problem $\min_{\x \in \K} f(\x)$ does not change much, as revealed by the following theorem. 
\begin{theorem}\label{thm:jin-1} Let $\x_*$ and $\x_\gamma$ be the optimal solutions to the constrained optimization problems defined as $\min_{g(\x) \leq 0} f(\x)$ and $\min_{g(\x) \leq - \gamma} f(\x)$, respectively, where $f(\x) = \sum_{t=1}^{T}{f_t(\x)}$ and $\gamma \geq 0$. We have
\begin{eqnarray*}
    |f(\x_*) - f(\x_{\gamma})| \leq \frac{G}{\sigma}\gamma T.
\end{eqnarray*}
\end{theorem}
\begin{proof}
We note that the  optimization problem $\min_{g(\x) \leq - \gamma} f(\x) = \min_{g(\x) \leq - \gamma} \sum_{t=1}^{T}{f_t(\x)}$, can also be written in the minimax form as
\begin{eqnarray}
f(\x_{\gamma}) = \min \limits_{\x \in \B} \max\limits_{\lambda \in \R_+} \sum_{t=1}^T f_t(\x) +  \lambda (g(\x) + \gamma), \label{eqn:minimax}
\end{eqnarray}
where we use the  fact that $\K' \subseteq \K \subseteq \B$. We denote by $\x_{\gamma}$  and $\lambda_{\gamma}$ the optimal solutions to (\ref{eqn:minimax}). We have

\begin{eqnarray*}
\begin{aligned}
f(\x_{\gamma})  & =  \min\limits_{\x \in \B} \max\limits_{\lambda \in \R_+} \sum_{t=1}^T f_t(\x) +  \lambda (g(\x) + \gamma) \\  &=  \min\limits_{\x \in \B} \sum_{t=1}^T f_t(\x) +  \lambda_{\gamma} (g(\x) + \gamma) \nonumber \\
& \leq  \sum_{t=1}^T f_t(\x_*) +  \lambda_{\gamma} (g(\x_*) + \gamma) \nonumber 
 \leq \sum_{t=1}^T f_t(\x_*) +  \lambda_{\gamma} \gamma,
\end{aligned}
\end{eqnarray*}

where the second equality follows the definition of the $\x_{\gamma}$ and the last inequality is due to the optimality of $\x_*$, that is, $g(\x_*) \leq 0$.\\
To bound $|f(\x_{\gamma}) - f(\x_*)|$, we need to bound $ \lambda_{\gamma}$. Since $\x_{\gamma}$ is the minimizer of (\ref{eqn:minimax}), from the optimality condition we have
\begin{eqnarray}
\label{eqn:minusf}
-\sum_{t=1}^T \nabla f_t(\x_{\gamma}) = \lambda_{\gamma} \nabla g(\x_\gamma).
\end{eqnarray}
By setting  $\v = -\sum_{t=1}^T \nabla f_t(\x_{\gamma}) $, we can simplify   (\ref{eqn:minusf}) as  $\lambda_\gamma \nabla g(\x_\gamma) =  \v$. From the KKT optimality condition \cite{boyd-convex-opt}, if $g(\x_\gamma) + \gamma < 0$ then we have $\lambda_\gamma = 0$; otherwise according to Assumption \ref{assumption:1} we can bound $\lambda_\gamma$ by
\begin{eqnarray*}
\label{eqn:14}
\lambda_{\gamma} \leq \frac{\|\v\|}{\| \nabla g(\x_\gamma)\|} \leq  \frac{GT}{\sigma}.
\end{eqnarray*}
We complete the proof by applying the fact $f(\x_*) \leq f(\x_{\gamma}) \leq f(\x_*) + \lambda_{\gamma}\gamma$.
\end{proof}
As indicated by Theorem~\ref{thm:jin-1}, when $\gamma$ is small, we expect the difference between two optimal values $f(\x_*)$ and $f(\x_{\gamma})$ to be small. Using the result from Theorem~\ref{thm:jin-1}, in the following theorem, we show that by running Algorithm \ref{alg:1-soft} on the modified convex-concave functions defined in (\ref{eqn:alg2}), we are able to obtain an $O(T^{3/4})$ regret bound and zero bound on the  violation of  constraints in the long run.

\begin{theorem}
\label{thm:no-violation}
Set  $a = 2R/\sqrt{2G^2+3(D^2+b^2)}$, $\eta = R^2/[a\sqrt{T}]$, and $\delta = 4G^2$. Let $\x_t, t \in [T]$ be the sequence of solutions obtained by Algorithm~\ref{alg:1-soft} with functions defined in (\ref{eqn:alg2}) with $\gamma = b T^{-1/4}$ and $b = 2 \sqrt{F(\delta R^2 a^{-1}+a R^{-2})} $. Let  $\x_*$ be the optimal solution to   $\min_{\x \in \K} \sum_{t=1}^T f_t(\x)$. With sufficiently large $T$, that is, $F T \geq a \sqrt{T} $, and under Assumption~\ref{assumption:1}, we have $\x_t, t\in [T]$ satisfy the global constraint $\sum_{t=1}^T g(\x_t) \leq 0$ and the regret is bounded by
\begin{eqnarray*}
{\rm{Regret}}_T = \sum_{t=1}^T f_t(\x_t) - f_t(\x_*) \leq a \sqrt{T} + \frac{b}{\sigma}GT^{3/4} = \O(T^{3/4}).
\end{eqnarray*}
\end{theorem}
\begin{proof}
Let $\x_{\gamma}$ be the optimal solution to $\min_{g(\x) \leq - \gamma} \sum_{t=1}^T f_t(\x)$. Similar to the proof of Theorem \ref{thm:ogd} when applied to functions in (\ref{eqn:alg2}) we have
\begin{eqnarray*}
\lefteqn{\sum_{t=1}^T f_t(\x_t) - \sum_{t=1}^T f_t(\x) + \lambda \sum_{t=1}^T (g(\x_t)+\gamma)- \left(\sum_{t=1}^T\lambda_t\right)(g(\x) + \gamma) - \frac{\delta\eta T}{2}\lambda^2}  \\
& \leq &-\frac{\delta\eta}{2}\sum_{t=1}^T\lambda_t^2 + \frac{R^2 + \lambda^2}{2\eta} + \frac{\eta T}{2}\left(2G^2+3 (D^2+\gamma^2)\right) + \frac{\eta }{2} \left(2G^2+3\delta^2\eta^2\right)\sum_{t=1}^T \lambda_t^2.
\end{eqnarray*}
By setting  $\delta \geq 2G^2+3\delta^2\eta^2$ which is satisfied by $\delta = 4G^2$, we  cancel the terms including $\lambda_t$ from the right hand side of above inequality.  By maximizing for $\lambda$ over the range $(0, +\infty)$  and noting that $\gamma \leq b$, for the optimal solution $\x_{\gamma}$, we have

\begin{eqnarray*}
\begin{aligned}
& \sum_{t=1}^T \left[f_t(\x_t) - f_t(\x_\gamma)\right] + \frac{\Big{[} \sum_{t=1}^T g(\x_t) + \gamma T\Big{]}_{+}^2}{2(\delta \eta T + 1/\eta)}  \leq \frac{R^2}{2\eta} + \frac{\eta T}{2}\left(2G^2 + 3(D^2+b^2)\right),
\end{aligned}
\end{eqnarray*}

which, by optimizing  for $\eta$ and applying the lower bound for the regret as $\sum_{t=1}^T f_t(\x_t) - f_t(\x_\gamma) \geq - FT$, yields the following inequalities 
\begin{eqnarray}
\sum_{t=1}^T f_t(\x_t) - f_t(\x_{\gamma}) \leq a\sqrt{T}  \label{eqn:19} 
\end{eqnarray}
and
\begin{eqnarray}
\sum_{t=1}^T g(\x_t) \leq \sqrt{2\left( FT + a\sqrt{T} \right)\sqrt{T}\left(\frac{\delta R^2}{a} + \frac{ a}{R^2} \right)} - \gamma T, \label{eqn:20}
\end{eqnarray}
for the regret and the violation of the constraint, respectively. Combining (\ref{eqn:19}) with the result of  Theorem~\ref{thm:jin-1} results in
$    \sum_{t=1}^T f_t(\x_\gamma) \leq \sum_{t=1}^T f_t(\x_*) + a \sqrt{T} +(G/\sigma)\gamma T$.
By choosing $\gamma = b T^{-1/4}$ we attain the desired regret bound as  
\[
\sum_{t=1}^T f_t(\x_t) - f_t(\x_*) \leq a \sqrt{T} + \frac{bG}{\sigma}T^{3/4} = O(T^{3/4}).
\]
To obtain the bound on the violation of constraints,  we note that in (\ref{eqn:20}), when $T$ is sufficiently large, that is, $FT \geq a \sqrt{T}$, we have  $\sum_{t=1}^T g(\x_t) \leq 2 \sqrt{F(\delta R^2 a^{-1}+a R^{-2})} T^{3/4} - b T^{3/4}$. Choosing $b = 2 \sqrt{F(\delta R^2 a^{-1}+a R^{-2})} T^{3/4}$ guarantees the zero bound on the violation of constraints as claimed.
\end{proof}

%%%%%%%%%%%%%%%%%%%%%%%%%%%%%%%%%%%%%%%%%%%%%%%%%%%
\section{Proofs of Convergence Rates}\label{sec-proofs-convergence}

In this section we provide the proof of the main lemmas omitted from the analysis of convergence rate.

\subsection{Proof of Lemma~\ref{lem:1}}
Following the standard analysis of gradient descent methods, we have for any $\x\in\B$, 

\begin{eqnarray*}
\begin{aligned}
&\|\x_{t+1}-\x\|_2^2-\|\x_t-\x\|_2^2\leq \|\x'_{t+1}-\x\|_2^2-\|\x_t-\x\|_2^2 \\
&= \|\x_t -\eta_t(\mathbf{g}(\x_t,\xi_t) + \lambda_t \nabla g(\x_t))-\x\|_2^2-\|\x_t-\x\|_2^2 \\
&\leq \eta_t^2\|\mathbf{g}(\x_t,\xi_t) + \lambda_t \nabla g(\x_t)\|_2^2 - 2\eta_t\dd{\x_t-\x}{\mathbf{g}(\x_t,\xi_t) + \lambda_t \nabla g(\x_t)}\\
&\leq \eta_t^2\|\mathbf{g}(\x_t,\xi_t) + \lambda_t \nabla g(\x_t)\|_2^2  \\
&\hspace{0.5cm}- 2\eta_t(\x_t-\x)^{\top}\underbrace{(\nabla f(\x_t) + \lambda_t \nabla g(\x_t))}\limits_{\equiv \nabla_\x \L(\x_t, \lambda_t)} + 2\eta_t\underbrace{(\x-\x_t)^{\top}(\mathbf{g}(\x_t,\xi_t) -\nabla f(\x_t))}\limits_{\equiv \zeta_t(\x)},
\end{aligned}
\end{eqnarray*}

Then we have

\begin{eqnarray*}
\begin{aligned}
&\dd{\x_t-\x}{\nabla_\x \L(\x_t, \lambda_t)} \leq \frac{1}{2\eta_t}\left(\|\x_t-\x\|_2^2-\|\x_{t+1}-\x\|_2^2\right) + \frac{\eta_t}{2}\|\mathbf{g}(\x_t,\xi_t) + \lambda_t \nabla g(\x_t)\|_2^2 + \zeta_t(\x)\\
&\leq  \frac{1}{2\eta_t}\left(\|\x_t-\x\|_2^2-\|\x_{t+1}-\x\|_2^2\right) + \eta_t\|\mathbf{g}(\x_t,\xi_t)\|_2^2 + \eta_t\lambda_t^2 \|\nabla g(\x_t)\|_2^2 + \zeta_t(\x)\\
&\leq \frac{1}{2\eta_t}\left(\|\x_t-\x\|_2^2-\|\x_{t+1}-\x\|_2^2\right) \\
& \hspace{0.5cm}+ 2\eta_t\underbrace{\|\mathbf{g}(\x_t,\xi_t)-\nabla f(\x_t)\|_2^2}\limits_{\equiv \Delta_t} + 2\eta_t\|\nabla f(\x_t)\|_2^2 + \eta_t\lambda_t^2 \|\nabla g(\x_t)\|_2^2 + \zeta_t(\x)
\end{aligned}
\end{eqnarray*}

By using the bound on $\|\nabla f(\x_t)\|_2$ and $\|\nabla g(\x_t)\|_2$, we obtain the first inequality in Lemma~\ref{lem:1}.  To prove the second inequality, we follow the same analysis, i.e., 

\begin{eqnarray*}
\begin{aligned}
&|\lambda_{t+1}- \lambda|^2 - |\lambda_t- \lambda|^2 \leq |\lambda_t + \eta_t(g(\x_t) - \gamma\lambda_t)|^2 - |\lambda_t- \lambda|^2\\
&\leq \eta_t^2|g(\x_t) - \gamma \lambda_t|^2 + 2\eta_t(\lambda_t-\lambda)\underbrace{(g(\x_t) - \gamma\lambda_t)}\limits_{\equiv \nabla_\lambda \L(\x_t,\lambda_t)}.
\end{aligned}
\end{eqnarray*}

Then we have

\begin{eqnarray*}
\begin{aligned}
(\lambda-\lambda_t)\nabla_\lambda \L(\x_t,\lambda_t)\leq \frac{1}{2\eta_t}\left( |\lambda_t- \lambda|^2- |\lambda_{t+1}- \lambda|^2\right) + \frac{\eta_t}{2}|g(\x_t) - \gamma \lambda_t|^2. 
\end{aligned}
\end{eqnarray*}

By induction, it is straightforward to show that  $\lambda_t\leq C_2/\gamma$, which yields  the second inequality in Lemma~\ref{lem:1}, i.e., 

\begin{eqnarray*}
\begin{aligned}
(\lambda-\lambda_t)\nabla_\lambda \L(\x_t,\lambda_t)\leq \frac{1}{2\eta_t}\left( |\lambda_t- \lambda|^2- |\lambda_{t+1}- \lambda|^2\right) +2\eta_tC_2^2 .
\end{aligned}
\end{eqnarray*}

\subsection{Proof of Lemma~\ref{lem:2}}
%\begin{proof}
Since $\L(\x, \lambda)$ is convex in $\x$ and concave in $\lambda$,  we have the following inequalities

\begin{eqnarray*}
\begin{aligned}
\L(\x, \lambda_t) - \L(\x_t, \lambda_t)&\geq \dd{\x-\x_t}{\nabla_\x \L(\x_t, \lambda_t)}, \\
\L(\x_t, \lambda) - \L(\x_t, \lambda_t)&\leq  (\lambda-\lambda_t)\nabla_\lambda \L(\x_t, \lambda_t). 
\end{aligned}
\end{eqnarray*}

Using the inequalities in Lemma~\ref{lem:1},  we have

\begin{eqnarray*}
\begin{aligned}
\L(\x_t, \lambda_t)- \L(\x, \lambda_t) &\leq \frac{1}{2\eta_t}\left(\|\x-\x_t\|_2^2-\|\x-\x_{t+1}\|_2^2\right) +2\eta_tG^2_1+\eta_t G^2_2\lambda_t^2 + 2\eta_t\Delta_t + \zeta_t(\x),\\
\L(\x_t, \lambda) - \L(\x_t, \lambda_t)&\leq \frac{1}{2\eta_t}\left(|\lambda-\lambda_t|^2-|\lambda-\lambda_{t+1}|^2\right) + 2\eta_tC_2^2,
\end{aligned}
\end{eqnarray*}

where $\zeta_t(\x) = \dd{\x-\x_t}{\mathbf{g}(\x_t,\xi_t) - \nabla f(\x_t)}$ as abbreviated before. Since $\eta_1=\cdots=\eta_T$, denoted by $\eta$, by taking summation of above two inequalities  over $t=1,\cdots, T$, we get

\begin{eqnarray*}
\begin{aligned}
\sum_{t=1}^T \L(\x_t, \lambda) - \L(\x, \lambda_t)&\leq \frac{\|\x\|^2_2}{2\eta} +  \frac{\lambda^2}{2\eta}+2\eta T(G_1^2  + C_2^2) + \sum_t \eta G_2^2\lambda_t^2 +2\eta\sum_{t=1}^T\Delta_t+\sum_{t=1}^T\zeta_t(\x).
\end{aligned}
\end{eqnarray*}

By plugging the expression of $\L(\x,\lambda)$, and  due to $\|\x\|_2\leq 1$, we have

\begin{eqnarray*}
\begin{aligned}
&\sum_{t=1}^{T}{f(\x_t)-f(\x)}+ \lambda \sum_{t=1}^Tg(\x_t)  -\left(\frac{\gamma T}{2}+\frac{1}{2\eta}\right)\lambda^2\\
&\leq\frac{1}{2\eta} + 2\eta T(G_1^2+C_2^2)    + \sum_t(\eta G_2^2-\gamma/2)\lambda_t^2+ \sum_t\lambda_t g(\x)+2\eta\sum_{t=1}^T\Delta_t+ \sum_{t=1}^T\zeta_t(\x).
\end{aligned}
\end{eqnarray*}

Let $\x=\x_*=\arg\min_{\x\in\K}f(\x)$. By  taking minimization over $\lambda\geq0$  on left hand side and considering $\eta=\gamma/(2G_2^2)$, we have

\begin{eqnarray*}
\begin{aligned}
&\sum_{t=1}^{T}{f(\x_t)-f(\x_*)} + \frac{[\sum_{t=1}^Tg(\x_t)]^2_+}{2(\gamma T+2G_2^2/\gamma)}\leq \frac{G_2^2}{\gamma}+ \frac{(G_1^2+C_2^2)}{G_2^2}\gamma T  + \frac{\gamma}{G_2^2}\sum_{t=1}^T\Delta_t +\sum_{t=1}^T\zeta_t(\x_*)
\end{aligned}
\end{eqnarray*}

%\end{proof}

\subsection{Proof of Lemma~\ref{lem:3}}
%\begin{proof}
Since $\mathcal{F}(\x)$ is strongly convex in $\x$, we have

\begin{eqnarray*}
\begin{aligned}
\mathcal{F}(\x) - \mathcal{F}(\x_t)&\geq \dd{\x-\x_t}{\nabla \mathcal{F}(\x_t)}+ \frac{\beta}{2}\|\x-\x_t\|_2^2.
%L(\x_t, \lambda) - L(\x_t, \lambda_t)&\leq  (\lambda-\lambda_t)^{\top}\nabla_\lambda L(\x_t, \lambda_t) -\frac{\gamma}{2}\|\lambda-\lambda_t\|_2^2
\end{aligned}
\end{eqnarray*}
Following the same analysis as in Lemma~\ref{lem:1},  we have

\begin{eqnarray*}
\begin{aligned}
\dd{\x_t-\x}{\nabla \mathcal{F}(\x_t)}&\leq \frac{1}{2\eta_t}\left(\|\x-\x_t\|_2^2-\|\x-\x_{t+1}\|_2^2\right) +\frac{\eta_t}{2}\|\mathbf{g}(\x_t,\xi_t) + p(\x_t)\lambda_0\nabla g(\x_t)\|_2^2  \\
&\hspace*{0.2in}+ \zeta_t(\x)-\frac{\beta}{2}\|\x-\x_t\|_2^2\\
& \leq \frac{1}{2\eta_t}\left(\|\x-\x_t\|_2^2-\|\x-\x_{t+1}\|_2^2\right) +\eta_tG_1^2+ \eta_t\lambda_0^2G_2^2  + \zeta_t(\x)-\frac{\beta}{2}\|\x-\x_t\|_2^2,
%L(\x_t, \lambda) - L(\x_t, \lambda_t)&\leq \frac{1}{2\tau_t}\left(\|\lambda-\lambda_t\|_2^2-\|\lambda-\lambda_{t+1}\|_2^2\right) + 2\tau_tC_2^2
\end{aligned}
\end{eqnarray*}

where 
\[ p(\x)= \frac{\exp\left({\lambda_0g(\x)/\gamma}\right)}{1+\exp\left({\lambda_0g(\x)/\gamma}\right)}.\] 
Taking summation of above inequality over $t=1,\cdots, T$ gives

\begin{eqnarray*}
\begin{aligned}
\sum_{t=1}^T \mathcal{F}(\x_t) - \mathcal{F}(\x)&\leq \sum_{t=1}^T \frac{1}{2}\left(\frac{1}{\eta_t}-\frac{1}{\eta_{t-1}}- \frac{\beta}{2}\right)\|\x-\x_t\|_2^2\\
&+\sum_{t=1}^T\eta_t (G_1^2+\lambda_0^2G_2^2)  +\sum_{t=1}^T\zeta_t(\x) - \frac{\beta}{4}\sum_{t=1}^T\|\x-\x_t\|_2^2.
\end{aligned}
\end{eqnarray*}

Since $\eta_t=1/(2\beta t)$, we have
%\begin{align*}
%&\sum_{t}f(\x_t)+ \lambda g(\x_t)  -\frac{\gamma}{2}\lambda^2 -f(\x)- \lambda_t g(\x) + \frac{\gamma}{2}\lambda_t^2 \\
%&\leq \frac{1}{\eta}+\eta (G_1^2+\sigma^2)T+ 2\sum_t\tau_t C_2^2  + \sum_t\eta G_2^2\lambda_t^2+ \sum_{t=1}^T\zeta_t
%\end{align*}
%and,

\begin{eqnarray*}
\begin{aligned}
&\sum_{t=1}^{T}{\mathcal{F}(\x_t)-\mathcal{F}(\x)}\leq\frac{(G_1^2+\lambda_0^2G_2^2)(1+\ln T)}{2\beta}  + \sum_{t=1}^T\zeta_t(\x)-\frac{\beta}{4}\sum_{t=1}^T\|\x-\x_t\|_2^2
\end{aligned}
\end{eqnarray*}

We complete the proof by letting  $\x=\x_*=\arg\min_{\x\in\K}f(\x)$.
%Let $\x_*=\arg\min_{\x\in\K}f(\x)$  and taking minimization over $\lambda>0$ of left hand side, and since $\eta_t=1/(2\beta t)\leq 1/(2\beta)= \gamma/(2G_2^2)$, we have
%\begin{align*}
%\sum_{t=1}^T(f(\x_t)-f(\x_*))+ \frac{[\sum_{t=1}^Tg(\x_t)]^2_+}{2\gamma T}&\leq \frac{G_1^2(1+\log T)}{2\beta} +  \frac{2C_2^2\beta (1+\log T)}{G_2^2} \\
%&+\sum_{t=1}^T\zeta_t-\frac{\beta}{4}\sum_{t=1}^T\|\x-\x_t\|_2^2
%\end{align*}
%\end{proof}

\subsection{Proof of Lemma~\ref{lem:4}}

The proof is based on the Berstein's inequality for martingales (see Theorem~\ref{theorem:bernsteinB}). To do so, define the martingale difference $X_t = \dd{\x- \x_t}{\nabla f(\x_t) - \mathbf{g}(\x_t,\xi_t)}$ and martingale $\Lambda_T= \sum_{t=1}^{T} X_t$. Define the conditional variance $\Sigma_T^2$ as
\[
    \Sigma_T^2 = \sum_{t=1}^{T} \E_{\xi_t}\left[X_t^2 \right] \leq 4G_1^2\sum_{t=1}^{T} \|\x_t - \x\|_2^2 = 4G_1^2D_T.
\]
Define $K = 4G_1$. We have

\begin{eqnarray*}
\begin{aligned}
 &\Pr\left(\Lambda_T\geq 2\sqrt{4G_1^2D_T \tau} + \sqrt{2}K\tau/3\right) \\
& =  \Pr\left(\Lambda_T\geq 2\sqrt{4G_1^2D_T \tau} + \sqrt{2}K\tau/3, \Sigma_T^2 \leq 4G_1^2D_T\right) \\
& =  \Pr\left(\Lambda_T\geq 2\sqrt{4G_1^2D_T\tau} + \sqrt{2}K\tau/3, \Sigma_T^2 \leq 4G_1^2D_T, D_T \leq \frac{4}{T} \right) \\
&  \hspace{0.5cm}+ \sum_{i=1}^m \Pr\left(\Lambda_T \geq 2\sqrt{4G_1^2 D_T \tau} + \sqrt{2}K\tau/3, \Sigma_T^2 \leq 4G_1^2D_T, \frac{4}{T}2^{i-1} < D_T  \leq \frac{4}{T} 2^{i} \right) \\
& \leq  \Pr\left(D_T \leq \frac{4}{T}\right) + \sum_{i=1}^m \Pr\left(\Lambda_T \geq \sqrt{2\times 4G_1^2\frac{4}{T} 2^{i}\tau} + \sqrt{2}K\tau/3, \Sigma_T^2 \leq 4G_1^2\frac{4}{T}2^i\right) \\
& \leq  \Pr\left(D_T \leq \frac{4}{T}\right) + me^{-\tau}.
\end{aligned}
\end{eqnarray*}

where we use the fact $\|\x_t - \x\|_2^2 \leq 4$ for any $\x\in\B$, and the last step follows the Bernstein inequality for martingales. We complete the proof by setting $\tau= \ln(m/\delta)$.

%\end{proof}

%%%%%%%%%%%%%%%%%%%%%%%%%%%%%%%%%%%%%%%%%%%%%%%%%%
\section{Summary}
\label{sec:con}
In this chapter, we made a progress towards making the SGD method efficient by proposing a framework in which it is possible to exclude the projection steps from the SGD algorithm. We have proposed two novel algorithms to overcome the computational bottleneck of the projection step in applying SGD to optimization problems with complex domains. We showed using 
novel theoretical analysis that the proposed algorithms  can achieve an $O(1/\sqrt{T})$ convergence  rate for general convex functions and an $O(\ln T/T)$ rate for  strongly convex functions with a overwhelming probability which are known to be optimal (up to a logarithmic factor) for stochastic optimization. 

We have also  addressed the problem of online convex optimization with constraints, where we only need the constraints to be satisfied in the long run. In addition to the regret bound which is the main tool in analyzing the performance of  general online convex optimization algorithms, we defined the bound on the violation of constraints in the long term which measures the cumulative violation of the solutions from the constraints for all rounds. Our setting is applied to solving online convex optimization without projecting the solutions onto the complex convex domain at each iteration,  which may be computationally inefficient for complex domains.  Our strategy is to turn the problem into an online convex-concave optimization problem and apply online gradient descent algorithm to solve it. 

%%%%%%%%%%%%%%%%%%%%%%%%%%%%%%%%%%%%%%%%%%%%%%%%%%
\section{Bibliographic Notes}
Generally, the computational complexity of the projection step in SGD has seldom been taken into account in the literature. Here, we briefly review the previous works on projection free convex optimization, which is closely related to the theme of this study. For some specific domains, efficient algorithms have been developed to circumvent the high computational cost caused by projection step at each iteration of gradient descent methods.  The main idea is to select an appropriate direction to take from the current solution such that the next solution is guaranteed to stay within the domain.  Clarkson~\cite{DBLP:journals/talg/Clarkson10} proposed a sparse greedy approximation algorithm for convex optimization over a simplex domain, which is a generalization of  an  old algorithm by Frank and Wolfe~\cite{frank56} (a.k.a conditional gradient descent~\cite{nonlinear-book}). Zhang~\cite{citeulike:3130770} introduced a similar sequential greedy approximation algorithm for certain convex optimization problems over a  domain given by a convex hull. Hazan~\cite{Hazan:2008:SAS:1792918.1792945} devised an algorithm for approximately maximizing a concave function over a trace norm bounded PSD cone, which only needs to compute the maximum eigenvalue and the corresponding eigenvector of  a symmetric matrix.  Ying et al.~\cite{Ying:2012:DML:2188385.2188386} formulated the distance metric learning problems into eigenvalue maximization and proposed an algorithm similar to~\cite{Hazan:2008:SAS:1792918.1792945}. 

Recently, Jaggi~\cite{thesisMJ} put these ideas into a general framework for convex optimization with a general convex domain. Instead of projecting the intermediate solution into a complex convex domain, Jaggi's algorithm solves a linearized problem over the same domain. He showed that Clark's algorithm , Zhang's algorithm and Hazan's algorithm discussed above are special cases of his general algorithm for special domains. It is important to note that all these algorithms are designed for batch optimization, not for stochastic optimization, which is the focus of this chapter.

The proposed stochastic optimization methods with only one projection are closely related to the online Frank-Wolfe  algorithm proposed in~\cite{hazan-projection-free}. It is a projection free online learning algorithm, built on the the assumption that it is possible to efficiently minimize a linear function over the complex domain. Indeed, the FW algorithm replaces the projection into the convex domain with a linear programming over the same domain which makes sense if solving the linear programming is cheaper than the projection. One main shortcoming of the OFW algorithm is that its convergence rate for general stochastic optimization is $O(T^{-1/3})$, significantly slower than that of a standard stochastic gradient descent algorithm (i.e., $O(T^{-1/2})$). It achieves a convergence rate of $O(T^{-1/2})$ only when the objective function is smooth, which unfortunately does not hold for many machine learning problems where either a non-smooth regularizer or a non-smooth loss function is used. Another limitation of OFW is that it assumes a linear optimization problem over the domain $\K$ can be solved efficiently. Although this assumption holds for some specific domains as discussed in ~\cite{hazan-projection-free}, but  in many settings of practical interest, this may not be true. The  algorithms proposed in this chapter address the two limitations explicitly as we showed that how two seemingly different modifications of the SGD can be used to avoid  performing expensive projections with similar convergency rates as the original SGD method.

The proposed online optimization with soft constraints  setup is reminiscent of regret minimization with side constraints or constrained regret minimization addressed in \cite{DBLP:conf/colt/MannorT06}, motivated by applications in wireless communication. In regret minimization with side constraints, beyond minimizing regret, the learner has some side constraints that need to be satisfied on average for all rounds. Unlike our setting, in learning with side constraints, the set $\K$ is controlled by adversary and can vary arbitrarily from trial to trial. It has been shown that if the convex set is affected by both decisions and loss functions, the minimax optimal regret is generally unattainable online \cite{DBLP:journals/jmlr/MannorTY09}.

\part{Appendix}\label{part-appendix}

\appendix

\chapter{Convex Analysis}
\label{chap:appendix-convex}

In this appendix, we introduce the basic definitions and results of convex analysis~\cite{nesterov2004introductory,boyd-convex-opt,borwein2010convex}
needed for the analysis of the  algorithms  in this thesis. We begin by formalizing a few topological definitions we used throughout the thesis followed by few definitions and results from convexity.

\begin{definition}[Euclidean Ball] A Euclidean ball with radius $r$ centered at point $\w_0$ is $\mathbb{B}(\w_0, r) = \{\w \in \R^d;\; \|\w-\w_0\|\leq r\}$. An $\ell_p$ ball is also defined equivalently but now the distance is defined by the $\ell_p$ norm as $\|\w\|_p = (\sum_{i=1}^{d}{|w_i|^p})^{1/p}$.
\end{definition}

\begin{definition} [Boundary and Interior] For  a given set~$\mathcal{W}$, we say that $\w$ is on the boundary of $\mathcal{W}$, denoted by $\partial \mathcal{W}$, if for some small $\epsilon > 0$, the ball centered at $\w$ with radius $\epsilon$ covers both inside and outside of the set, i.e., $\mathbb{B}(\w_0, \epsilon) \cap \mathcal{W} \neq \emptyset$ and $\mathbb{B}(\w_0, \epsilon) \cap \bar{\mathcal{W}} \neq \emptyset$, where $\bar{\mathcal{W}}$ is the complement of set ${\mathcal{W}}$. We say that a point $\w$ is in the interior of set ${\mathcal{W}}$ if $\exists \epsilon > 0 \; {\rm{s.t.}}\; \mathbb{B}(\w_0, \epsilon) \subset \mathcal{W}$.
\end{definition}

\begin{definition}[Convex Set] A set $\W$ in a vector space is convex if for any two vectors $\w, \w' \in \W$, the line segment connecting two points is contained in $\W$ as well. In other fords, for any $\lambda \in [0,1]$, we have that $\lambda \w + (1-\lambda)\w' \in \W$. 
\end{definition}
Intuitively, a set is convex if its surface has no "dips". The intersection of an arbitrary family of convex sets is obviously convex.

\begin{definition}[Projection onto Convex Sets]
Let $\mathcal{W} \subseteq \mathbb{R}^d$ be a closed convex set and let $\w \in \mathbb{R}^d$ be a point. The projection of $\w$ onto $\W$ is denoted by $\Pi_{\W}(\w)$ and defined by 
$$\Pi_{\W}(\w) = \arg \min_{\w' \in \W} \|\w - \w'\|.$$
In particular, for every $\w \in \mathbb{R}^d$ there exists a point $\w_*$ which satisfies this and it is unique. 
\end{definition}
Here are few examples of projections.  For the positive semidefinite (PSD) cone, i.e., $\mathscr{S}_{+}^{d} = \{\mathbf{X} \in \mathbb{R}^{d\times d}: \mathbf{X} \succeq 0\}$, the projection requires a full eigendecomposition of the input matrix. Precisely, let $\mathbf{W} \in \mathscr{S}_{+}^{d}$ has the  eigendecomposition $\mathbf{W} = \mathbf{U}^{\top} \Sigma \mathbf{U}$ where $\Sigma = {\rm{diag}}(\lambda_1, \lambda_2, \cdots, \lambda_d)$ and $\mathbf{U}$ is an orthogonal matrix. Then, $\Pi_{\mathscr{S}_{+}^{d}}(\mathbf{W}) = \mathbf{U}^{\top} \Sigma_{+} \mathbf{U}$ where $\Sigma_{+} = {\rm{diag}}([\lambda_1]_{+}, [\lambda_2]_{+}, \cdots, [\lambda_d]_{+})$.  For Hyperplane $\mathscr{H}=\{\mathbf{x} \in \mathbb{R}^d: \dd{\mathbf{a}}{\mathbf{x}} = b\},   \mathbf{a} \neq \mathbf{0}$, the projection is equivalent to $\Pi_{\mathscr{H}}(\w) = \w - \frac{\dd{\mathbf{a}}{\w}-b}{\|\mathbf{a}\|^2}\mathbf{a}$. For Affine set $\mathscr{A}=\{\mathbf{x} \in \mathbb{R}^d: \mathbf{A} \mathbf{x} = \mathbf{b}\}$, the projection can be obtained by $\Pi_{\mathscr{A}} (\w) = \w - \mathbf{A}^{\top} \left(\mathbf{A}\mathbf{A}^{-1}\right) (\mathbf{A}\w-\mathbf{b})$. For Polyhedron set, $\mathscr{P}=\{\mathbf{x} \in \mathbb{R}^d: \mathbf{A} \mathbf{x} \leq \mathbf{b}\}$ there is no analytical solution and the projection $\Pi_{\mathscr{P}}(\w)$ is an offline optimization by itself where the corresponding optimization problem $\arg \min_{\w' \in \mathscr{P}} \|\w'-\w\|$ must be solved approximately.

\begin{lemma}[Non-expensiveness of Projection]\label{lemma-app-a-non-exp}
Let $\W \subseteq \mathbb{R}^d$ be a convex set and consider the projection operation defined as above. Then, 
$$\|\Pi_{\W}(\w) - \Pi_{\W}(\w') \| \leq \|\w-\w'\|, \forall \w, \w' \in \mathbb{R}^d$$

\end{lemma}

\begin{definition}[Convex Function] A function $f: \W \mapsto \R$ is said to be convex if $\W$ is convex and  for every $\w, \w' \in \W$ and $\alpha \in [0,1]$,
\[ f(\lambda \w + (1-\lambda)\w') \leq \lambda f(\w) + (1-\lambda)f(\w'). \]
A continuously differentiable function is convex if $f(\w) \geq f(\w') + \dd{\nabla f(\w')}{\w-\w'}$ for all $\w,\w' \in \W$. If $f$ is non-smooth then this inequality holds for any sub-gradient $\g \in \partial f(\w')$.
\end{definition}

\begin{definition}[Subgradient] A subgradient of a convex function $f:\R^d \mapsto \R$ at some point $\w$ is any vector $\mathbf{g} \in \R^d$ that achieves the same lower bound as the tangent line to function $f$ at point $\w$, i.e., 

$$ f(\w') \geq f(\w) + \dd{\mathbf{g}}{\w' - \w},\; \forall \w' \in \R^d$$

\end{definition}

The subgradient $\mathbf{g}$ always exists for convex functions on the relative interior of their domain. Furthermore, if $f$ is differentiable  at $\w$, then there is a unique subgradient $\mathbf{g} = \nabla f(\w)$. Note that subgradients need not exist for non-convex functions. 

\begin{definition} [Subdifferential] The subdifferential of a convex function $f:\R^d \mapsto \R$ at some point $\w$ is the set of all subgradients of $f$ at $\w$ , i.e.,
$$\partial f(\w) = \{\mathbf{g} \in \R^d: \forall \w',\; f(\w') \geq f(\w) + \dd{\mathbf{g}}{\w' - \w} \}$$
\end{definition}

An important property of subdifferential set of a function $\partial f$ is that it is closed and convex set, even for non-convex functions which is straightforward to verify by following the definition.  Moreover, the subdifferential is also always nonempty for convex functions. This is a consequence of the supporting hyperplane theorem, which states that at any point on the boundary of a convex set, there exists at least one supporting hyperplane. Since the epigraph of a convex function
is a convex set, we can apply the supporting hyperplane theorem to the set of points $(\w; f(\w))$, which are exactly the boundary points of the epigraph. The   subdifferential of a differentiable function contains only single element which is the gradient of function at that point, i.e., $\partial f(\w) = \{\nabla f(\w)\}$.

We are now prepared to introduce the concept of Lipschitz continuity, designed to measure change of function values versus change in the independent variable for a general function $f(\cdot)$.

\begin{definition}[Lipschitzness] A function $f: \W \mapsto \R$ is $\rho$-Lipschitz over the set $\W$ if for every $\w, \w' \in \W$ we have that $|f(\w) - f(\w')| \leq \rho ||\w - \w'||$. \\
\end{definition}

The following lemma  shows a result on  the Lipschitz continuity of a function which is defined as the maximum of other Lipschitz continuous functions.

 \begin{lemma}
 \label{prop:lip}
Assume that functions $f_i:\W\mapsto\R, i \in [m]$ are Lipschitz continuous with constant $\rho$. Then, function $f(\x) = \max_{i \in [m]}{ f_i(\x)}$ is Lipschitz continuous with constant $\rho$, that is,
\[ |f(\x) - f(\x')| \leq \rho \| \x - \x'\| \; \; \text{for any} \; \x,\x' \in \W.\]
\end{lemma}
\begin{proof} First, we  rewrite  $f(\x) = \max_{i \in [m]} f_i(\x)$ as  $f(\x) =  \max_{\alpha \in \Delta_m} \sum_{i=1}^{m}{\alpha_i f_i(\x)}$ where $\Delta_m$ is the $m$-simplex, that is, $\Delta_m= \{ \alpha \in\R_{+}^m; \sum_{i=1}^{m}{\alpha_i} = 1\}$. Then, we have
\begin{eqnarray*}
|f(\x)-f(\x')| &=&  \left | \max_{\alpha \in \Delta_m} \sum_{i=1}^{m}{\alpha_i f_i(\x)} - \max_{\alpha \in \Delta_m} \sum_{i=1}^{m}{\alpha_i f_i(\x')}\right | \\
& \leq& \max_{\alpha \in \Delta_m}  \left | \sum_{i=1}^{m}{\alpha_i f_i(\x)} -  \sum_{i=1}^{m}{\alpha_i f_i(\x')}\right | \\ &\leq& \max_{\alpha \in \Delta_m}  \sum_{i=1}^{m}{\alpha_i \left | f_i(\x) - f_i(\x')\right |}
 \leq \rho \|\x-\x'\|,
\end{eqnarray*}
where  the last inequality follows from the Lipschitz continuity of $f_i(\x), i \in [m]$.
\end{proof}

\begin{definition}[Smoothness] A differentiable   function $f: \W \mapsto \R$ is said to be  $\beta$-smooth  with respect to a norm $\|\cdot\|$, if it holds that
\begin{equation}
\label{eqn:smoth}
f(\w) \leq f(\w') +  \dd{\nabla f(\w')}{\w-\w'} + \frac{\beta}{2}\|\w-\w'\|^2, \;\forall \; \w, \w'\in\W.
\end{equation}
\end{definition}

%\begin{lemma}Let $f$ be a $\beta$-smooth function on $\mathbb{R}^n$. Then for any $\w, \w' \in \mathbb{R}^d$, one has
%
%    $$|f(\w) - f(\w') - \dd{\nabla f(\w')}{\w - \w'}| \leq \frac{\beta}{2} \|\w - \w'\|^2 .$$
%\end{lemma}
%
%\begin{proof}If the function $f$ was convex one could use the definition of subgradient, Cauchy-Schwarz inequality and $\beta$-smoothness to obtain
%
%$$0 \leq f(x) - f(y) - \nabla f(y)^{\top} (x - y) \leq (\nabla f(x) - \nabla f(y))^{\top} (x - y) \leq \beta \|x-y\|^2.$$
%
%A slightly more clever proof applies Cauchy-Schwarz 'continuously', which improves the constant and also get rid of the convexity condition. The key idea is to represent the quantity $f(x) - f(y)$ as an integral, this gives:
%
%\begin{equation*}
%\begin{aligned}
%& |f(x) - f(y) - \nabla f(y)^{\top} (x - y)| \\
%& = \left|\int_0^1 \nabla f(y + t(x-y))^{\top} (x-y) dt - \nabla f(y)^{\top} (x - y)\right| \\
%& \leq \int_0^1 \beta t \|x-y\|^2 dt = \frac{\beta}{2} \|x-y\|^2 .
%\end{aligned}
%\end{equation*}
%
%\end{proof}
%Note that in fact it can be showed that $f$ is convex and $\beta$-smooth if and only if for any $x, y \in \mathbb{R}^n$, one has
%
%$$0 \leq f(x) - f(y) - \nabla f(y)^{\top} (x - y) \leq \frac{\beta}{2} \|x - y\|^2 ,$$
%
%and this inequality is often used as a definition. (The fact that the above inequality implies $\beta$-smoothness follows from Lemma 2 below.)

The following result is an important property of smooth functions which  has been utilized in the proof of few results in the thesis. 
\begin{lemma} [Self-bounding Property]For any $\beta$-smooth non-negative function $f: \R\rightarrow \R$, we have $|f'(w)| \leq \sqrt{4\beta f(w)}$
 \label{lem:smooth}
\end{lemma}
As a simple proof,  first from the smoothness assumption,  by  setting $w_1 = w_2 - \frac{1}{\beta}f'(w_2)$ in (\ref{eqn:smoth})  and rearranging the terms we obtain $f(w_2) - f(w_1) \geq \frac{1}{2 \beta} | f'(w_2)|^2$.  On the other hand, from the convexity of loss function  we have $f(w_1) \geq f'(w_2) + \dd{f'(w_1)}{w_1 - w_2}$. Combining these inequalities and considering the fact that the function is non-negative gives the desired inequality. \\

\begin{definition}[Strong Convexity] A  function $f(\w)$ is said to be $\alpha$-strongly convex w.r.t a norm $\|\cdot\|$, if  there exists a constant $\alpha > 0$ (often called the modulus of strong convexity) such that, for any $\lambda\in[0, 1]$ and for all $\w,\w'\in \W$, it holds  that
\begin{equation*}
f(\lambda \w+ (1-\lambda)\w')\leq \alpha f(\w) + (1-\lambda) f(\w') - \frac{1}{2}\lambda(1-\lambda)\alpha\|\w-\w'\|^2.
\end{equation*}
\end{definition}
If $f(\w)$ is twice differentiable, then an equivalent  definition of strong convexity is $ \nabla^2 f(\w) \succeq \alpha \mathbf{I}$ which indicates that  the smallest eigenvalue of the Hessian of $f(\w)$ is uniformly
lower bounded by $\alpha$ everywhere. When $f(\w)$ is differentiable, the strong convexity is equivalent to
\[f(\w) \geq f(\w') + \langle \nabla f(\w'), \w-\w'\rangle + \frac{\alpha}{2}\|\w-\w'\|^2,\; \forall \;\w,\w'\in\W.\] 

An important property of strongly convex functions that we used in the proof of few results in the following: 
\begin{lemma} \label{lemma-app-A-sc}
Let $f(\w)$ be a $\alpha$-strongly convex function over the domain $\W$, and $\w^*= \arg \min_{\w \in \W} f(\w)$. Then, for any $\w \in \W$, we have
\begin{equation} \label{eqn:lem:1}
f(\w)-f(\w^*) \geq \frac{\alpha}{2} \|\w-\w^*\|^2.
\end{equation}
\end{lemma}

\begin{definition} [Dual Norm]
Let $\|\cdot\|$ be any norm on $\R^d$. Its dual norm denoted by $\|\cdot\|_*$ is defined by
$$\|\w'\|_* = \sup_{\w \in \R^d} \dd{\w'}{\w} - \|\w\|$$
\end{definition}
An equivalent definition is $\|\w'\|_* = \sup_{\w \in \R^d} \{\dd{\w}{\w'} | \|\w\| \leq 1\}$. If $p, q \in [1, \infty]$ satisfy $1/p+1/q=1$, then the $\ell_p$ and $\ell_q$ norms are dual to each other.

\begin{definition} [Fenchel Conjugate~\footnote{Also known as the convex conjugate or Legendre-Fenchel transformation}] The Fenchel conjugate of a function $f:\W\mapsto \R$ is defined as 
$$f^*(\mathbf{v}) = \sup_{\w \in \W} \dd{\w}{\mathbf{v}} - f(\w).$$
\end{definition}
An immediate result of above definition is the Fenchel-Young inequality stating that for any $\w$ and $\mathbf{v}$ we have that $\dd{\w}{\mathbf{v}} \leq f(\w) + f^*(\mathbf{v}) $.
%\begin{lemma}[Fenchel-Young Inequality]
%\end{lemma}

\begin{definition} [Bregman Divergence] \label{def-app-a-bregman}
Let $\Phi: \K \to \mathbb{R}$ be a continuously-differentiable real-valued and strictly convex function defined on a closed convex set $\K$. The Bregman divergence between $\w$ and $\w'$  is the difference at $\w$ between $f$ and a linear approximation around
$\w'$
\[ \mathsf{B}_{\Phi} (\w, \w') = \Phi(\w) - \Phi(\w') - \dd{\w - \w'}{ \nabla \Phi(\w')}.\]  
\end{definition}
For example the squared Euclidean distance $\mathsf{B}(\w,\w') = \|\w - \w'\|^2$ is the canonical example of a Bregman distance, generated by the convex function $\Phi(\w) = \|\w\|^2$.     The  entropy function $\Phi(\mathbf{p}) = \sum_i p_i\log p_i - \sum p_i$ gives rises to the  generalized Kullback-Leibler divergence as $\mathsf{B}_{\rm{KL}}(\mathbf{p}, \mathbf{q}) = \sum p_i \log \frac{p_i}{q_i} - \sum p_i + \sum q_i$.

%\begin{definition} [Karush-Kuhn-Tucker Optimlatiy]
%
%
%\begin{lemma}[Jensen's Inequality]
%\end{lemma}
%
%\begin{lemma}[H\"{o}lder's inequality]
%\end{lemma}
%
%
%\begin{lemma}[]
%\end{lemma}

%\end{definition}

\chapter{Concentration Inequalities}\label{chap:appendix-concen}

This appendix is a quick excursion into concentration of measure. We consider some important concentration results  used in the proofs given in this thesis. Concentration inequalities give probability bounds for a random variable to be concentrated around its mean (e.g., see~\cite{dubhashi2009concentration}~, \cite{boucheron2004concentration} and~\cite{boucheron2013concentration} for a through discussion and derivation  of these inequalities).

\begin{lemma} [Markov's Inequality] 
For a non-negative random variable $X$ and $t>0$, 
\[ \mathbb{P}[X \geq t] \leq \frac{\E [X]}{t} \]
\label{lemma:markov}
\end{lemma}
The upside of Markov's inequality  is that it does not need almost no assumptions about the random variable, but the downside is that it only gives very weak bounds.

\begin{lemma}[Hoeffding's Lemma] Let $X$ be any real-valued random variable with expected value $\mathbb{E}[X] = 0$ and such that $X \in [a, b]$ almost surely. Then, for all $\lambda \in \R$,
\[ \mathbb{E} \left[ e^{\lambda X} \right] \leq \exp \left( \frac{\lambda^2 (b - a)^2}{8} \right). \]
\end{lemma}

\begin{theorem} [Hoeffding's Inequality] Let $X_1, \cdots, X_n$ be a sequence of i.i.d random variables. Assume  $X_i \in [a_i,b_i]$ and let $S = X_1 + X_2 + \cdots + X_n$. Then, 

\[    \mathbb{P}(S - \mathbb{E}[S] \geq t) \leq \exp \left( - \frac{2t^2}{\sum_{i=1}^n (b_i - a_i)^2} \right),\!\]
\[    \mathbb{P}(|S - \mathbb{E}[S]| \geq t) \leq 2\exp \left( - \frac{2t^2}{\sum_{i=1}^n (b_i - a_i)^2} \right).\!\] 
\label{thm:hoeffding-azumeB}
\end{theorem}
The following result extends Hoeffding's inequality to more general functions $f(x_1,x_2, \cdots,x_n)$.
\begin{theorem}  [McDiarmid's Inequality] Let $X_1, \cdots,X_n$ be independent real-valued random variables. Suppose that 
\[ \sup_{x_1, x_2, \cdots, x_n, x'_i} \left | f(x_1,  \cdots, x_{i-1}, x_i, x_{i+1}, \cdots, x_n) - f(x_1,  \cdots, x_{i-1}, x'_i, x_{i+1}, \cdots, x_n)\right | \leq c_i,\]
for $i = 1, 2,\cdots, n$. Then,

\[ \mathbb{P} \left( \left| f(X_1, X_2, \cdots, X_n) - \mathbb{E} \left[ f(X_1, X_2, \cdots, X_n) \right] \right| \geq t \right) \leq 2 \exp \left( -\frac{2 t^2}{\sum_{i=1}^{n}{c_i^2}}\right)\]
\label{thm:McDiarmid}
\end{theorem}

Hoeffding's inequality does not use any information about the random variables except
the fact that they are bounded. If the variance of $X_i, i \in [n]$ is small, then we can get a sharper
inequality from Bennett's  inequality and its simplified version in Bernstein's Inequality.

\begin{theorem} [Bennett's Inequality] Let $X_1, \cdots, X_n$ be a sequence of i.i.d random variables. Assume $\E[X_i] = 0 $, $\E[X_i^2] = \sigma^2$, and $|X_i| \leq M, i \in [n]$. Then, 
\[ \mathbb{P} \left( \sum_{i = 1}^{n}{X_i} \geq t \right) \leq \exp \left ( \frac{-n \sigma^2}{M^2} \phi \left(\frac{t M}{n \sigma^2} \right)\right), \]
where $\phi(x) = (1+x) \log(1+x) - x$.
\end{theorem}

\begin{theorem} [Bernstein's Inequality]
Let $X_1, \ldots , X_n$ be independent zero-mean $\E[X_i] = 0, i \in [n]$ random variables. Suppose that $|X_i| \leq M$ almost surely, for all $i \in [n]$. Then, for all positive $t$,

$$    \mathbb{P} \left (\sum_{i=1}^n X_i > t \right ) \leq \exp \left ( -\frac{t^2}{2 \left(\sum \mathbb{E} \left[X_j^2 \right ]+\tfrac{1}{3} Mt\right)} \right ). $$
\end{theorem}

The next result is an extension of Hoeffding's inequality to martingales with zero-mean and bounded increments which is due to~\cite{azuma1967weighted}. Before stating the Hoeffding-Azuma's inequality, we need few definitions. A sequence of random variables $(X_i)_{i \in \mathbb{N}}$ on a probability space $(\Omega, \mathcal{A}, \mathbb{P})$ is a martingale difference sequence with respect to a filtration $(\mathscr{F}_i)_{i \in \mathbb{N}}$ if and only if, or all $i \geq 1$, each random variable $X_i$ is $\mathscr{F}_i$-measurable and satisfies,
$$\E[X_i|\mathscr{F}_{i-1}] = 0.$$ 

\begin{theorem} [Hoeffding-Azuma Inequality] \label{eqn:azumaB}
Let $X_1, X_2,  \ldots$  be a martingale difference sequence with respect to a filtration  $ \mathscr{F} = (\mathscr{F}_i)_{i \in \mathbb{N}}$. Assume that for all $i \geq 1$, there exists a $\mathscr{F}_{i-1}$-measurable random variable $A_i$ and a non-negative constant $c_i$ such that  $X_i \in [A_i , A_i + c_i]$.  Then, the martingale $S_n$ defined by  $S_n = \sum_{i=1}^n X_i$, satisfies  for any $t > 0$,
\[
\mathbb{P}[ S_n > t] \leq \exp \left( -\frac{2t^2}{\sum_{i=1}^n c_i^2} \right).
\]
\end{theorem}
The Hoeffding-Azuma's inequality indicates that $S_n$ is sub-gaussian with variance $\sum_{i=1}^{n}{c_i^2}/4$, i.e.,
$$\mathbb{P} \left[ \max_{1\leq i \leq n} S_i \geq t \right] \leq \exp\left(\frac{-2t^2}{\sum_{i=1}^{n}{c_i^2}}\right).$$

The following inequality which is due to Freedman~\cite{freedman1975tail} extends Bernstein's result to the case of discrete-time martingales with bounded jumps  (a.k.a Freedman's inequality). This result demonstrates that a martingale exhibits normal-type concentration near its mean value
on a scale determined by the predictable quadratic variation of the sequence. 

\begin{theorem} [Bernstein's Inequality for Martingales]
Let $X_1, \ldots , X_n$ be a bounded martingale difference sequence with respect to the filtration $\mathscr{F} = (\mathscr{F}_i)_{1\leq i\leq n}$ and with $|X_i| \leq M$. Let
\[
S_i = \sum_{j=1}^i X_j
\]
be the associated martingale. Denote the sum of the conditional variances by
\[
    \Sigma_n^2 = \sum_{i=1}^n \E\left[X_i^2|\mathscr{F}_{i-1}\right],
\]
Then for all constants $t$, $\nu > 0$,
\[
\mathbb{P}\left[ \max\limits_{i=1, \ldots, n} S_i > t \mbox{ and } \Sigma_n^2 \leq \nu \right] \leq \exp\left(-\frac{t^2}{2(\nu + Mt/3)} \right),
\]
and therefore,
\[
    \mathbb{P}\left[ \max\limits_{i=1,\ldots, n} S_i > \sqrt{2\nu t} + \frac{\sqrt{2}}{3}M t \mbox{ and } \Sigma_n^2 \leq \nu \right] \leq e^{-t}.
\]
\label{theorem:bernsteinB}
\end{theorem}

\chapter{Miscellanea}\label{chap:appendix-technical}
\def \X {\mathcal{W}}
\section*{Extra-gradient Lemma}

 \begin{lemma}\label{appendix-c-lem:6}
Let $\Phi(\z)$ be a  $\alpha$-strongly convex function with respect to the norm $\|\cdot\|$, whose dual norm is denoted by $\|\cdot\|_*$,  and $\mathsf{B}(\x,\z) = \Phi(\x)- (\Phi(\z) + (\x-\z)^{\top}\Phi'(\z))$ be the Bregman distance induced by function $\Phi(\x)$. Let $\mathscr{Z}$ be a convex compact set, and $\mathscr{U}\subseteq \mathscr{Z}$ be convex and closed.  Let $\z\in \mathscr{Z}$, $\gamma>0$, Consider the points,

\begin{equation}
\begin{aligned}
 \x &= \arg\min_{\u\in \mathscr{U}} \gamma\u^{\top}\xi + \mathsf{B}(\u, \z)\label{eqn:project1C},\\
\end{aligned}
\end{equation}
\begin{equation}
\begin{aligned}
 \z_+&=\arg\min_{\u\in \mathscr{U}} \gamma\u^{\top}\zeta + \mathsf{B}(\u,\z),\label{eqn:project2C}
\end{aligned}
\end{equation}

 then for any $\u\in \mathscr{U}$, we have
\begin{equation*}
 \begin{aligned}\label{eqn:ineqC}
 \gamma\zeta^{\top}(\x-\u)\leq  \mathsf{B}(\u,\z) - \mathsf{B}(\u, \z_+) + \frac{\gamma^2}{\alpha}\|\xi-\zeta\|_*^2 - \frac{\alpha}{2}\left[\|\x-\z\|^2 + \|\x-\z_+\|^2\right].
 \end{aligned}
 \end{equation*}
 \end{lemma}
\begin{proof}
By using the definition of Bregman distance $\mathsf{B}(\u, \z)$, we can write two updates in~(\ref{eqn:project1C})  as

\begin{equation}
\begin{aligned}
 \x &= \arg\min_{\u\in U} \u^{\top}(\gamma\xi - \Phi'(\z)) + \Phi(\u),\\
 \z_+&=\arg\min_{\u\in U} \u^{\top}(\gamma\zeta-\Phi'(\z)) + \Phi(\u),
 \end{aligned}
 \end{equation}

 by the first oder optimality  condition, we have
 \begin{equation}
 \begin{aligned}
& (\u-\x)^{\top} (\gamma\xi - \Phi'(\z) + \Phi'(\x))\geq 0, \forall \u\in U,\label{eqn:b1C}\\
&(\u-\z_+)^{\top} (\gamma\zeta- \Phi'(\z) + \Phi'(\z_+))\geq 0, \forall \u\in U%.\label{eqn:b2}
 \end{aligned}
 \end{equation}

 Applying~(\ref{eqn:b1C}) with $\u= \z_+$ and $\u= \x$, we get

\begin{equation}
 \begin{aligned}
 & \gamma(\x-\z_+)^{\top}\xi\leq  (\Phi'(\z) - \Phi'(\x))^{\top}(\x-\z_+),\\
&\gamma(\z_+-\x)^{\top} \zeta\leq (\Phi'(\z) - \Phi'(\z_+))^{\top}(\z_+-\x).
 \end{aligned}
 \end{equation}

 Summing up the two inequalities, we have

\begin{equation}
 \begin{aligned}
 \gamma (\x-\z_+)^{\top}(\xi-\zeta) \leq (\Phi'(\z_+) - \Phi'(\x))^{\top}(\x-\z_+).
 \end{aligned}
\end{equation}

 Then

\begin{equation}\label{eqn:b3C}
 \begin{aligned}
 \gamma \|\xi-\zeta\|_*\|\x-\z_+\|&\geq   -\gamma (\x-\z_+)^{\top}(\xi-\zeta)\geq (\Phi'(\z_+) - \Phi'(\x))^{\top}(\z_+-\x)\\
 &\geq\alpha\|\z_+-\x\|^2.
 \end{aligned}
 \end{equation}

 where in the last inequality, we use the strong convexity of $\Phi(\x)$.

\begin{equation}\label{eqn:b3C8}
 \begin{aligned}
\mathsf{B}(\u, \z)&- \mathsf{B}(\u, \z_+)=\Phi(\z_+)-\Phi(\z) +(\u-\z_+)^{\top}  \Phi'(\z_+) -(\u-\z)^{\top}\Phi'(\z) \\
= &\Phi(\z_+)-\Phi(\z) +(\u-\z_+)^{\top}\Phi'(\z_+) -( \u - \z_+)^{\top}\Phi'(\z) -  (\z_+ -\z)^{\top} \Phi'(\z) \\
= & \Phi(\z_+) -\Phi(\z)-(\z_+ -\z )^{\top} \Phi'(\z) +(\u-\z_+)^{\top} (\Phi'(\z_+)- \Phi'(\z))\\
=& \Phi(\z_+) -\Phi(\z)  -( \z_+ -\z )^{\top} \Phi'(\z) + (\u-\z_+)^{\top}(\gamma\zeta+   \Phi'(\z_+)- \Phi'(\z)) - (\u-\z_+)^{\top} \gamma\zeta\\
\geq & \Phi(\z_+) -\Phi(\z)  -  ( \z_+ -\z)^{\top}  \Phi'(\z)  - (\u-\z_+)^{\top}\gamma \zeta \\
=& \underbrace{\Phi(\z_+)  -\Phi(\z) - (\z_+ -\z)^{\top}  \Phi'(\z)  -  ( \x-\z_+)^{\top} \gamma\zeta}\limits_{\epsilon} + (\x -\u)^{\top}\gamma\zeta,
\end{aligned}
\end{equation}

where the inequality follows from~(\ref{eqn:b1C}). We proceed by bounding $\epsilon$ as:

\begin{equation}
 \begin{aligned}
 \epsilon =& \Phi(\z_+)   -\Phi(\z)- (\z_+ -\z ) ^{\top}\Phi'(\z) - (\x-\z_+)^{\top}\gamma\zeta\\
=& \Phi(\z_+)  -\Phi(\z) - ( \z_+ -\z )^{\top} \Phi'(\z) - (\x-\z_+)^{\top}\gamma ( \zeta - \xi)-  ( \x-\z_+)^{\top}\gamma \xi \\
=&\Phi(\z_+)  -\Phi(\z) - ( \z_+ -\z )^{\top} \Phi'(\z) - (\x-\z_+)^{\top}\gamma ( \zeta - \xi)\\
& + (\z_+ - \x)^{\top}(\gamma  \xi - \Phi'(\z) + \Phi'(\x) )- (\z_+- \x)^{\top}(\Phi'(\x)-\Phi'(\z) )\\
\geq & \Phi(\z_+)  -\Phi(\z) - ( \z_+ -\z )^{\top} \Phi'(\z) - (\x-\z_+)^{\top}\gamma ( \zeta - \xi)-( \z_+- \x)^{\top} (\Phi'(\x)- \Phi'(\z))\\
= & \Phi(\z_+) -\Phi(\z) - (\x -\z)^{\top} \Phi'(\z) -(\x-\z_+)^{\top}\gamma (\zeta - \xi) - ( \z_+- \x)^{\top} \Phi'(\x) \\
= &\left[ \Phi(\z_+)- \Phi(\x) - (\z_+- \x)^{\top}\Phi'(\x) \right]+\left[\Phi(\x) -\Phi(\z) -(\x -\z )^{\top} \Phi'(\z) \right]- (\x-\z_+)^{\top}\gamma ( \zeta- \xi)\\
\geq & \frac{\alpha}{2} \| \x - \z_+\|^2 + \frac{\alpha}{2}\|\x - \z \|^2 - \gamma \| \x-\z_+ \| \| \zeta - \xi \|_*\\
\geq & \frac{\alpha}{2}\{ \| \x-\z_+ \|^2 + \|\x - \z \|^2\} -\frac{\gamma^2}{\alpha} \|\zeta - \xi \|_*^2,
\end{aligned}
\end{equation}

where the first inequality follows from~(\ref{eqn:b1C}), the second inequality follows from the strong convexity of $\Phi(\x)$, and the last inequality follows from~(\ref{eqn:b3C8}). Combining the above results, we have

\begin{equation}
\begin{aligned}
\gamma(\x -\u)^{\top}\zeta \leq \mathsf{B}(\u, \z)&- \mathsf{B}(\u, \z_+) + \frac{\gamma^2}{\alpha} \|\zeta - \xi \|_*^2 - \frac{\alpha}{2} \{\| \x-\z_+  \|^2 + \|\x - \z \|^2\} .
\end{aligned}
\end{equation}
\end{proof}

\section*{Stochastic Mirror Descent  for Smooth Losses}
In this subsection we provide a tight analysis of stochastic mirror decent algorithm for smooth losses. We will show that the convergence rate of the mirror descent algorithm~\cite{nemircomp1983} for stochastic optimization of both non-smooth and smooth functions is dominated by an $O(1/\sqrt{T})$ convergence rate. We consider the nonlinear projected subgradient formulation of the mirror descent from~\cite{mirror-beck-2003} which iteratively updates the solution by:
 \begin{eqnarray}
\label{basic-md-updateC}
\x_{t+1} = \arg \min_{\x \in \X} \Big{\{} \langle \x, \g_t \rangle + \frac{1}{\eta_t} \mathsf{B}_{\Phi}(\x, \x_t) \Big{\}}, 
\end{eqnarray}
where $\g_t$ is a (sub)gradient of $f(\x)$ at point $\x_t$ and $\mathsf{B}_{\Phi}(\cdot, \cdot)$ is the Bregmen diveregnce defined with strongly convex function $\Phi(\cdot)$. Define the average solution as $\widehat{\x}_T = \sum_{t=1}^{T}{\x_t}/T$ and let $\g(\w, \xi)$ denote a stochastic gradient of function at point $\w$, i.e., $\E[\g(\w, \xi)] = \nabla f(\w)$. We state few standard assumptions we make in our analysis about the stochastic optimization problem. 

\textbf{A1. (Lipschitz Continuity)} $\E[\|\g(\x, \xi)\|_*^2] \leq G^2$ 

\textbf{A2. (Bounded Variance)} $\E[\| \nabla f(\x) - \g(\x, \xi)\|] \leq \sigma^2$.

\textbf{A3. (Smoothness)} The function $f$ has  $L$ Lipschitz continuous gradient.

\textbf{A4. (Compactness)} $\mathsf{B}_{\Phi}(\x_*, \x) \leq R^2$. \\

The  proof is extensively uses the following key inequalities:\\
\noindent (1) Optimality condition
\[ \langle \x - \x_{t+1}, \g_t + \frac{1}{\eta_t} \nabla\Phi(\x_{t+1}) - \frac{1}{\eta_t}\nabla\Phi(\x_{t}) \rangle \geq 0\]
(2) Generalized inequality
\[ \mathsf{B}_\Phi(\x_*, \x_t) - \mathsf{B}_\Phi(\x_*, \x_{t+1}) -\mathsf{B}_\Phi(\x_{t+1}, \x_t) = \langle \x_* - \x_{t+1}, \nabla \Phi(\x_{t+1})- \nabla \Phi(\x_{t})\rangle\]
(3) Fenchel-Yang inequality applied to the conjugate pair $\frac{1}{2} \|\cdot\|^2$ and $\frac{1}{2}\|\cdot\|_*^2$ yields
\[ \langle \g_t, \x_t - \x_{t+1} \rangle \leq \frac{\eta_t}{2} \|\g_t\|_*^2 + \frac{1}{2 \eta_t} \|\x_t - \x_{t+1}\|^2\]
(4) \[\mathsf{B}_{\Phi}(\x, \y) \geq \frac{1}{2}\| \x - \y\|^2\]

\noindent (5) From smoothness assumption we have
\[f(\w') - f(\x) - \langle \nabla f(\x), \w' - \x \rangle \leq \frac{L}{2} \| \w' -\x\|^2\]

The analysis  of the convergence rate of basic mirror descent algorithm follows standard techniques in convex optimization and  for completeness we provide a proof here following~\cite{mirror-beck-2003}.
\begin{lemma} We have
\[ f(\x_t) - f(\x_*) \leq \frac{\eta_t}{2} ||\g_t ||_*^2 + \frac{1}{\eta_t} \left[\mathsf{B}_\Phi(\x_*, \x_t) -  \mathsf{B}_\Phi(\x_*, \x_{t+1})\right]\]
\end{lemma}
\begin{proof} 
From the convexity of $f(\x)$  and by setting $\zeta_t = \nabla f(\x_t) - \g_t$ we have
\begin{equation*}
\begin{aligned}
& f(\x_t) - f(\x_*) \\ 
& \leq \langle \nabla f(\x_t), \x_t - \x_*\rangle  \\ 
& \leq \langle \g_t, \x_t - \x_*\rangle +  \langle \nabla f(\x_t) - \g_t, \x_t - \x_*\rangle\\ 
 & \overset{(1)}{\leq}  \langle \g_t, \x_t - \x_*\rangle + \langle \x_* - \x_{t+1}, \g_t + \frac{1}{\eta_t} \nabla\Phi(\x_{t+1}) - \frac{1}{\eta_t}\nabla\Phi(\x_{t}) \rangle + \langle \zeta_t, \x_t - \x_* \rangle\\
 & = \langle \g_t, \x_t - \x_{t+1} \rangle + \frac{1}{\eta_t}  \langle \x_* - \x_{t+1}, \nabla \Phi(\x_{t+1})- \nabla \Phi(\x_{t})\rangle + \langle \zeta_t, \x_t - \x_* \rangle \\
 &\overset{(2)}{\leq} \frac{\eta_t}{2} \| \g_t\|_*^2+\frac{1}{2 \eta_t} \| \x_t - \x_{t+1}\|^2 +  \frac{1}{\eta_t} [\mathsf{B}_\Phi(\x_*, \x_t) - \mathsf{B}_\Phi(\x_*, \x_{t+1}) -\mathsf{B}_\Phi(\x_{t+1}, \x_t)] + \langle \zeta_t, \x_t - \x_* \rangle \\
& \overset{(3)}{\leq} \frac{\eta_t}{2} \| \g_t\|_*^2 + \frac{1}{\eta_t} [\mathsf{B}_\Phi(\x_*, \x_t) - \mathsf{B}_\Phi(\x_*, \x_{t+1})]+\langle \zeta_t, \x_t - \x_* \rangle
\end{aligned}
\end{equation*}

\end{proof}

\begin{theorem} By setting leaning rate as $\eta_t = \frac{R}{G \sqrt{t}}$ we have the following rate for the convergence of stochastic mirror descent for non-smooth objectives:
\begin{equation*}
\begin{aligned}
 \E[f(\widehat{\x}_T)] - f(\x_*) & \leq O(\frac{RG}{\sqrt{T}})
\end{aligned}
\end{equation*}

\end{theorem}

\begin{proof} Summing up Lemma 1 for all iterations we get
\begin{equation*}
\begin{aligned}
 \sum_{t=1}^{T}{f({\x}_t) - f(\x_*)} & \leq \frac{1}{\eta_1} \mathsf{B}_\Phi(\x_*, \x_1)+ \sum_{t=2}^{T}{\mathsf{B}_{\Phi}(\x_*, \x_t)(\frac{1}{\eta_t} -\frac{1}{\eta_{t-1}} )}+\sum_{t=1}^{T}{\frac{\eta_t}{2} \| \g_t\|_*^2} + \sum_{t=1}^{T}{\langle \zeta_t, \x_t - \x_* \rangle} \\
 & \leq \frac{R^2}{\eta_T} + \frac{G^2}{2} \sum_{t=1}^{T}{\eta_t} \\
 & = O(RG\sqrt{T}),
\end{aligned}
\end{equation*}

where the second inequality follows since each $\x_t$ is a deterministic function of $\xi_1, \cdots, \xi_{t-1}$, so $\E [ \langle \nabla f(\x_t) - \g_t, \x_t - \x_*\rangle | \xi_1, \cdots, \xi_{t-1}] = 0$
\end{proof}
We generalize the proof for smooth functions.

\begin{theorem} By setting $\eta_t = \frac{1}{L + \beta_t}$ where $\beta_t = (\sigma /R) \sqrt{t}$, the following holds for the convergence  rate of stochastic mirror descent for smooth functions:
\begin{equation*}
\begin{aligned}
 \E[f(\widehat{\x}_T)] - f(\x_*) & \leq O(\frac{LR^2}{T}+ \frac{\sigma R}{\sqrt{T}})
\end{aligned}
\end{equation*}

\end{theorem}
\begin{proof}
\begin{equation*}
\begin{aligned}
f(\x_t) - f(\x_*) & \leq \langle \nabla f(\x_t), \x_t - \x_*\rangle  \\
&\leq \langle \nabla f(\x_t), \x_t - \x_{t+1}\rangle + \langle \nabla f(\x_t), \x_{t+1} - \x_*\rangle \\
& \overset{(5)}{\leq}  f(\x_t) - f(\x_{t+1}) + \frac{L}{2} \| \x_{t+1} -\x_t\|^2 + \langle \nabla f(\x_t), \x_{t+1} - \x_*\rangle \\
\end{aligned}
\end{equation*}

By rearranging the terms we proceed as follows:
\begin{equation*}
\begin{aligned}
& f(\x_{t+1}) - f(\x_*) \\
 & \leq \langle \nabla f(\x_t), \x_{t+1} - \x_*\rangle  + \frac{L}{2} \| \x_{t+1} -\x_t\|^2\\
& \leq  \langle \g_t, \x_{t+1}-\x_* \rangle + \frac{L}{2} \| \x_{t+1} -\x_t\|^2 + \langle \zeta_t, \x_{t+1} - \x_* \rangle\\
 & \overset{(1)}{\leq}  \frac{1}{\eta_t}  \langle \x_* - \x_{t+1}, \nabla \Phi(\x_{t+1})- \nabla \Phi(\x_{t})\rangle + \frac{L}{2} \| \x_{t+1} -\x_t\|^2+\langle \zeta_t, \x_{t+1} - \x_* \rangle \\
& \overset{(2)}{\leq} \frac{1}{\eta_t} [\mathsf{B}_{\Phi}(\x_*, \x_t) - \mathsf{B}_{\Phi}(\x_*, \x_{t+1}) - \mathsf{B}_{\Phi}(\x_t, \x_{t+1})]+\frac{L}{2} \| \x_{t+1} -\x_t\|^2+\langle \zeta_t, \x_{t+1} - \x_* \rangle \\
& \overset{(4)}{\leq} \frac{1}{\eta_t} [\mathsf{B}_{\Phi}(\x_*, \x_t) - \mathsf{B}_{\Phi}(\x_*, \x_{t+1})] - (L + \beta_t)\mathsf{B}_{\Phi}(\x_t, \x_{t+1}) + L \mathsf{B}_{\Phi}(\x_t, \x_{t+1}) + \langle \zeta_t, \x_{t+1} - \x_* \rangle\\
& = \frac{1}{\eta_t}[\mathsf{B}_{\Phi}(\x_*, \x_t) - \mathsf{B}_{\Phi}(\x_*, \x_{t+1})] - \beta_t \mathsf{B}_{\Phi}(\x_t, \x_{t+1}) + \langle \zeta_t, \x_{t+1} - \x_t \rangle + \langle \zeta_t, \x_{t} - \x_* \rangle \\
& \overset{(3)}{\leq} \frac{1}{\eta_t}[\mathsf{B}_{\Phi}(\x_*, \x_t) - \mathsf{B}_{\Phi}(\x_*, \x_{t+1})]  + \frac{1}{2\beta_t} \| \zeta_t\|_*^2 + \langle \zeta_t, \x_{t} - \x_* \rangle
\end{aligned}
\end{equation*}

Summing for all $t$ yields
\begin{equation*}
\begin{aligned}
\sum_{t=1}^{T}{f(\x_{t+1}) - f(\x_*)}  & \leq \frac{1}{\eta_1} \mathsf{B}_{\Phi}(\x_*, \x_1) + \sum_{t = 2}^{T}{\mathsf{B}_{\Phi}(\x_* - \x_t)(\frac{1}{\eta_t}-\frac{1}{\eta_{t-1}} )} + \sum_{t=1}^{T}{\frac{\| \zeta_t\|_*^2}{2\beta_t}} + \sum_{t=1}^{T}{\langle \zeta_t, \x_{t} - \x_* \rangle} \\
& \leq \frac{R^2}{\eta_1}+ R^2 \sum_{t=2}^{T}{(\frac{1}{\eta_t} - \frac{1}{\eta_{t-1}})} + \frac{\sigma^2}{2} \sum_{t=1}^{T}{ \frac{1}{\beta_t}} + 0\\
& \leq  \frac{R^2}{\eta_T} + \frac{\sigma R}{2} \sqrt{T}  = O(LR^2 + \sigma R \sqrt{T}), 
\end{aligned}
\end{equation*}

where in the second in equality we used the fact  $\E [ \langle\nabla f(\x_t) - \g_t, \x_t - \x_* \rangle| \xi_1, \cdots, \xi_{t-1}] = 0$. By averaging the solutions over all iterations and applying the Jensen's inequality  we obtain the convergence rate claimed in the theorem.

\end{proof}

\bibliographystyle{plain}
\bibliography{thesis_references}

\end{document}